%% file: main.tex
\pdfoutput=1
\documentclass[11pt]{article}

\usepackage[utf8]{inputenc} 
\usepackage[T1]{fontenc}    
\usepackage{booktabs}       
\usepackage{amsfonts}       
\usepackage{nicefrac}       
\usepackage{microtype}      
\usepackage[table]{xcolor}         

\usepackage{./paper/smile}

\input{./paper/macro.tex}

\usepackage{fancyhdr}
\makeatletter
\@ifundefined{url}{
  \newcommand{\url}[1]{\texttt{#1}}
}{
  \renewcommand{\url}[1]{\texttt{#1}}
}
\makeatother

\acrodef{MFO}{\underline{\textbf{M}}aximal \underline{\textbf{F}}eature \underline{\textbf{O}}ccurrence}
\acrodef{MFCO}{\underline{\textbf{M}}aximal \underline{\textbf{F}}eature \underline{\textbf{C}}o-\underline{\textbf{O}}ccurrence}

\title{
  \vskip-30pt
  Taming Polysemanticity in LLMs: Provable Feature Recovery via Sparse Autoencoders
  }

\author{Siyu Chen$^*$ \quad Heejune Sheen$^*$ \quad Xuyuan Xiong$^\dag$ \quad Tianhao Wang$^\S$ \quad Zhuoran Yang$^*$
\vspace{1pt}
\and
{\small\textit{$^*$Department of Statistics and Data Science, Yale University}} 
\and
{\small\textit{$^\dag$Antai College of Economics and Management, Shanghai Jiao Tong University}}
\and  
{\small\textit{$^\S$Toyota Technological Institute at Chicago}}
\vspace{1pt}
\and
{
    \small\texttt{\{siyu.chen.sc3226, heejune.sheen, zhuoran.yang\}@yale.edu}
}
\and 
{
  \small\texttt{xxy2021@sjtu.edu.cn} \qquad\texttt{tianhao.wang@ttic.edu}
}
}
\date{}
\begin{document}

\maketitle


\pagestyle{plain}

\newcommand{\greencomment}[1]{\textcolor{lightgreen}{\textit{$\triangleright$ #1}}}

\begin{abstract}

We study the challenge of achieving theoretically grounded feature recovery using Sparse Autoencoders (SAEs) for the interpretation of Large Language Models. 
Existing SAE training algorithms often lack rigorous mathematical guarantees and suffer from practical limitations such as hyperparameter sensitivity and instability. To address these issues, we first propose a novel statistical framework for the feature recovery problem, which includes a new notion of feature identifiability by modeling polysemantic features as sparse mixtures of underlying monosemantic concepts. Building on this framework, we introduce a new SAE training algorithm based on ``bias adaptation'',  a technique that adaptively adjusts neural network bias parameters to ensure appropriate activation sparsity. We theoretically \highlight{prove that this algorithm correctly recovers all monosemantic features} when input data is sampled from our proposed statistical model. Furthermore, we develop an improved empirical variant, Group Bias Adaptation (GBA), and \highlight{demonstrate its superior performance against benchmark methods when applied to LLMs with up to 1.5 billion parameters}. This work represents a foundational step in demystifying SAE training by providing the first SAE algorithm with theoretical recovery guarantees, thereby advancing the development of more transparent and trustworthy AI systems through enhanced mechanistic interpretability.

\end{abstract}


\begin{AIbox}{TL;DR}
Existing Sparse Autoencoder (SAE) training algorithms often lack rigorous mathematical guarantees for feature recovery.  Empirically, methods such as $\ell_1$ regularization and TopK activation are sensitive to hyperparameter tuning and can exhibit inconsistency. Our work addresses these theoretical and practical issues with the following contributions:
\begin{enumerate}[
    leftmargin=2em,
    itemsep=0.15em,
]
\item \highlight{A new statistical framework} that formalizes feature recovery by modeling polysemantic features as sparse mixtures of underlying monosemantic concepts and establishes a rigorous notion of feature identifiability.

\item A novel training algorithm, \highlight{Group Bias Adaptation (GBA)}, that uses ``bias adaptation" to directly control neuron sparsity, overcoming the limitations of conventional regularization and achieving better control over neuron activation frequency.

\item The \highlight{first provable recovery guarantee} for an SAE training algorithm, where GBA recovers all monosemantic features when data is sampled from the statistical model.

\item \highlight{Superior empirical performance} on LLMs up to 1.5B parameters, where GBA achieves the sparsity-loss frontier while learning more consistent features than benchmarks.

\end{enumerate}

\end{AIbox}

{\color{sectionblue} \tableofcontents}

\input{./paper/paper.tex}

\newpage
\bibliographystyle{./paper/ims}
\bibliography{./paper/reference}

\newpage 
\appendix

\input{./paper/appendix/supplement.tex}

\input{./paper/appendix/add_experiments}
\input{./paper/appendix/sae_Identifiability.tex}

\input{./paper/appendix/init_gauss_cond.tex}
\input{./paper/appendix/sae_dynamics.tex}

\input{./paper/appendix/auxiliary.tex}
\end{document}

%% file: paper/macro.tex
\usepackage{tikz}
\usetikzlibrary{matrix, positioning}
\usepackage{wrapfig}
\usepackage{xcolor}
\usepackage{cancel}
\usepackage{acronym}

\DeclareMathOperator*{\spn}{span}

\newcommand{\eqi}[1]{\overset{\text{(\romannumeral#1)}}{=}}
\newcommand{\brai}[1]{\text{(\romannumeral#1)}}
\newcommand{\braRNum}[1]{\text{(\RNum{#1})}}

\definecolor{mypink}{HTML}{FF1493} 
\definecolor{darkred}{HTML}{BB253B}
\definecolor{lightgreen}{HTML}{90C0BF}
\newcommand{\highlight}[1]{{\textcolor{darkcyan}{#1}}}

\definecolor{darkcyan}{HTML}{0081B9}

\newmdenv[
  backgroundcolor=gray!10,
  skipabove=1em,
  skipbelow=1em,
  leftline=false,
  topline=false,
  bottomline=false,
  rightline=false,
  linecolor=gray!88,
  linewidth=4pt
]{githubquote}

\usepackage[most,skins,theorems]{tcolorbox}
\tcbset{
  aibox/.style={
    width=\linewidth,
    top=10pt,
    bottom=4pt,
    colback=blue!6!white,
    colframe=black,
    colbacktitle=black,
    enhanced,
    center,
    attach boxed title to top left={yshift=-0.1in,xshift=0.15in},
    boxed title style={boxrule=0pt,colframe=white,},
  }
}

\definecolor{rliableolive}{HTML}{BBCC33}
\definecolor{rliableblue}{HTML}{77AADD}
\definecolor{rliablered}{HTML}{EE8866}
\definecolor{LightCyan}{rgb}{0.88,1,1}
\definecolor{darkblue}{HTML}{2878D9}
\definecolor{navyblue}{HTML}{0000FF}

\newtcolorbox{AIbox}[2][]{aibox,title=#2,colback=rliableblue!10!white,#1}

\newenvironment{olivebox}{%
    \begin{tcolorbox}[colback=rliableolive!10!white,colframe=black,boxsep=3pt,top=4pt,bottom=4pt,left=3pt,right=3pt]
}{%
    \end{tcolorbox}
}

\newenvironment{lightbluebox}{%
    \begin{tcolorbox}[colback=rliableblue!15!white,colframe=gray!10!white,boxsep=3pt,top=4pt,bottom=4pt,left=3pt,right=3pt]
}{%
    \end{tcolorbox}
}

\newenvironment{olivebox2}{%
    \begin{tcolorbox}[colback=rliableolive!10!white,colframe=gray!20!white,boxsep=3pt,top=4pt,bottom=4pt,left=3pt,right=3pt]
}{%
    \end{tcolorbox}
}

\usepackage{subcaption}  

\usepackage{titlesec}
\usepackage{xcolor}
\definecolor{sectionblue}{HTML}{00356B}  

\titleformat{\section}
  {\Large\bfseries\color{sectionblue}}{\thesection}{1em}{}
\titleformat{\subsection}
  {\large\bfseries\color{sectionblue}}{\thesubsection}{1em}{}
\titleformat{\subsubsection}
  {\normalsize\bfseries\color{sectionblue}}{\thesubsubsection}{1em}{}
\titleformat{\paragraph}[runin]
  {\normalfont\normalsize\bfseries\color{black}}{}{0pt}{}[.]

%% file: paper/paper.tex
\section{Introduction}
\label{sec:intro}

\begin{wrapfigure}{r}{0.3\textwidth}
  \centering
  \includegraphics[width=\linewidth]{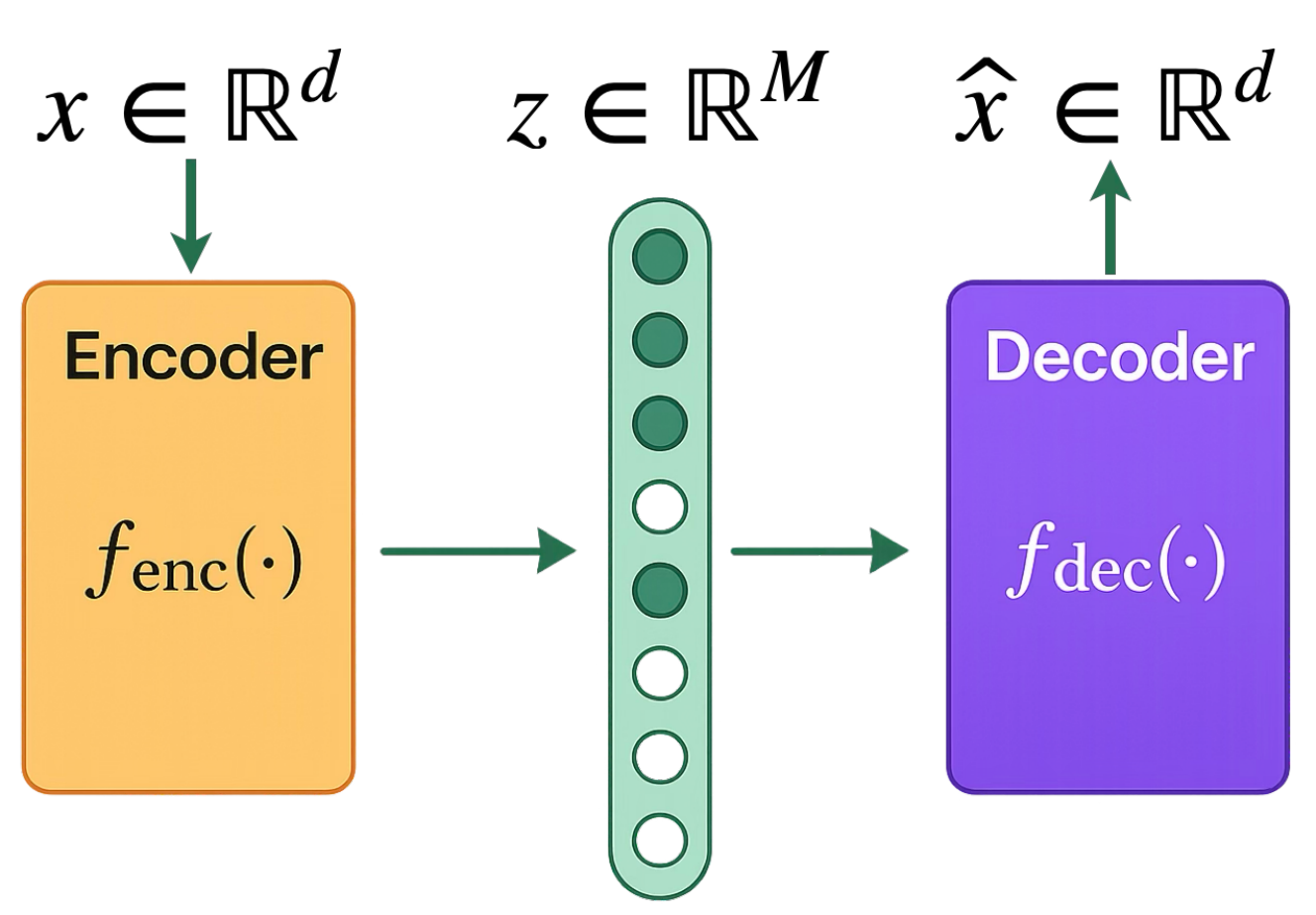}
  \caption{\small Illustration of SAE.}
  \label{fig:sae_illustration}
\end{wrapfigure}

Large Language Models (LLMs) have demonstrated remarkable capabilities across diverse tasks. 
It is found that LLMs encode vast amounts of information by \emph{superposition} \citep{lu2024large, xiong2024everything, elhage2022toy, bengio2013representation} ---  packing multiple concepts into the same weight or activation directions to maximize capacity. This efficiency comes at a cost: individual neurons (or activation vectors) become polysemantic \citep{scherlis2022polysemanticity}, meaning they respond to several monosemantic features at once. While superposition boosts representational power, it also makes interpreting what any single neuron “means” much harder.

Dictionary learning has recently been applied to disentangle polysemantic LLM representations, with Sparse Autoencoders (SAEs) emerging as a leading approach \citep{cunningham2023sparse,bricken2023monosemanticity,templeton2024scaling,gao2024scaling,rajamanoharan2024jumping}. An SAE encodes an LLM’s internal activation \(x\in\mathbb{R}^d\) into a high-dimensional, sparse code
\(
z = f_{\mathrm{enc}}(x)\in\mathbb{R}^M\) with \( M\gg d,
\) and 
then decodes
\(
\hat{x} = f_{\mathrm{dec}}(z)\approx x.
\)
By enforcing sparsity, such that only a few components of \(z\) are nonzero, each active neuron ideally reflects a single interpretable feature. Empirically, SAEs have revealed such monosemantic features in models like Pythia-70M \citep{cunningham2023sparse} and Claude 3.5 Sonnet \citep{templeton2024scaling}.

Despite these promising empirical advances, existing studies on SAEs often \emph{lack mathematically rigorous guarantees} regarding feature recovery. Furthermore, popular training algorithms frequently suffer from practical limitations, including instability and a strong sensitivity to hyperparameter tuning.
For instance, methods employing $\ell_p$ regularization require careful tuning of the regularization parameter. 
The widely used $\ell_1$ regularization often leads to ``activation shrinkage",
where the magnitudes of the learned features are systematically underestimated \citep{tibshirani1996regression}
The $\ell_0$ variant, while directly penalizing the number of non-zero activations, is computationally difficult to optimize. 
Other prominent approaches, like TopK activation \citep{makhzani2013k,gao2024scaling}, also depend on a manually set hyperparameter --- the threshold $K$. While enforcing a hard sparsity constraint, TopK can yield sets of learned features that are highly sensitive to the random seed used during initialization \citep{paulo2025sparse}, raising concerns about the reliability of the discovered features.


This landscape motivates us to address fundamental questions concerning the reliability and theoretical underpinnings of feature recovery with SAEs. We aim to tackle the following two primary questions:
\begin{enumerate}[
  leftmargin=2em,
  ]
    \item How can a \highlight{\emph{statistical model}} be formulated to formally study the feature recovery problem in the context of polysemantic representations in LLMs? Specifically, what constitutes an appropriate \highlight{\emph{generative model}} for such polysemantic representations? What are suitable mathematical notions of \highlight{\emph{feature identifiability}} under such a framework?
    \item Is it possible to design an SAE training algorithm that satisfies two crucial properties: \highlight{(a)} Under the proposed statistical framework, can the algorithm \highlight{\emph{provably and efficiently recover}} the underlying ground-truth monosemantic features? \highlight{(b)} Can this algorithm be effectively \highlight{\emph{applied to modern LLMs}}, overcoming the practical drawbacks associated with methods like TopK and $\ell_1$ regularization, such as intricate hyperparameter tuning and sensitivity to random initialization?
\end{enumerate}

\begin{table}[t]
  \centering
  \begin{tabular}{lccc}
    \toprule
    Method & Sparsity-Loss Frontier & Consistency & Tuning-free \\
    \midrule
    L1 \citep{templeton2024scaling}   & \(\times\)      & \(\checkmark\) & \(\times\)      \\
    TopK \citep{gao2024scaling}   & \(\checkmark\)  & \(\times\)     & \(\times\)      \\
    GBA (ours)  & \(\checkmark\)  & \(\checkmark\) & \(\checkmark\)  \\
    \bottomrule
  \end{tabular}
  \caption{\small Comparison of SAE Training Methods. The table summarizes the properties of different SAE training methods. The first column indicates whether the method achieves a sparsity-loss frontier. The sparsity-loss frontier is the curve for the best trade-off between $\ell_2$ reconstruction loss and neuron activation sparsity. 
  The second column indicates whether the method ensures feature consistency, meaning it learns consistent features across different runs. The third column indicates whether the method is tuning-free, meaning it does not require hyperparameter tuning for optimal performance.}
  \label{tab:method_comparison}
\end{table}

We provide affirmative answers to these two questions.
First, building on sparse dictionary learning and matrix factorization principles \citep{spielman2012exact,agarwal2016learning,arora2014new,barak2015dictionary,5484983}, \emph{we introduce a formal statistical model} for the feature recovery problem. We model a polysemantic representation vector $x$ as a random sparse linear combination $x \approx h^\top V$ of a set of unknown monosemantic features from a dictionary $V = [v_1, \dots, v_n]^\top \in\RR^{n\times d}$ where $h \in \mathbb{R}^n$ is a sparse coefficient vector. We define a precise notion of feature identifiability and cast the feature recovery problem as the reconstruction of the monosemantic features in $V$ from observation $x$.

Second, we \emph{introduce a novel SAE training algorithm} based on the principle of \textit{bias adaptation}: each neuron computes $z_m = \sigma(w_m^\top x + b_m)$ with $w_m\in\RR^d$ being the weight vector and $b_m\in\RR$ the bias, and we adaptively adjust the bias $b_m$ to enforce sparse activations while avoiding dead units.
\textbf{Theoretically}, we prove that our bias adaptation algorithm correctly \highlight{\emph{recovers all true monosemantic features}} when data are sampled according to our proposed statistical model under specific, well-defined conditions.
\textbf{Empirically}, we develop an improved variant of our algorithm --- Group Bias Adaptation (GBA) --- and \highlight{\emph{scale it to LLMs with up to 1.5B parameters}}. 
Our experiments demonstrate that GBA achieves superior performance in terms of reconstruction quality, activation sparsity, and feature consistency compared to benchmark methods.

The significance of our work is multifaceted. It serves as a foundational step in demystifying the training dynamics and empirical successes of SAEs by bridging the gap between their practical application and theoretical understanding. We propose what we believe to be the first SAE training algorithm accompanied by provable theoretical guarantees for feature recovery. By enhancing the reliability and robustness of SAEs, our contributions advance the broader objective of constructing transparent and trustworthy Artificial Intelligence systems through the continued development of mechanistic interpretability tools.

\subsection{Related Works}\label{sec:related_work}
\input{./paper/sections/related_work.tex}

\paragraph{Notation}
Let $\RR_+$ denote the set of non-negative real numbers. 
We use standard Big-$O$, Big-$\Theta$ and small-$o$ notation and use $a\gtrsim b$ to denote that $a\ge  b + O(\log\log(n)/\log n)$ for sufficiently large $n$. For two sets $A$ and $B$, we denote by $A\sqcup B$ the disjoint union of $A$ and $B$. We denote by $[n]$ the set $\{1,2,\ldots,n\}$ for positive integer $n$. 
We use $a\land b$ to denote the minimum of $a$ and $b$, and $a\lor b$ to denote the maximum of $a$ and $b$.
We denote by $\vone$ the all-ones vector, whose dimension will be clear from context.
For a finite set $A$, we denote by $|{A}|$ the cardinality of $A$. 
For two matrices $A$ and $B$ of the same shape $\RR^{n\times m}$, we denote by $\cos(A, B)$ the vector of cosine similarities between corresponding rows of $A$ and $B$, i.e.,
\begin{align} 
\label{eq:cosine_similarity}
\cos(A, B) = \left(\frac{A_{i, :}^\top B_{i, :}}{\norm{A_{i,:}}_2 \norm{B_{i, :}}_2}\right)_{i=1}^n.
\end{align}

\section{Preliminaries}\label{sec:preliminaries}
\input{./paper/sections/preliminaries.tex}
\section{Algorithm: Group Bias Adaptation}\label{sec:algorithm}
\input{./paper/sections/algorithm.tex}

\section{Experiments on Qwen$\mathbf{2.5}$-$\mathbf{1.5}$B LLM}\label{sec:experiments}
\input{./paper/sections/experiments.tex}

\section{Identifiability of  Features}\label{sec:identifiability}
\input{./paper/sections/identifiability.tex}
\section{Dynamics Analysis: SAE Provably Recovers True Features}\label{sec:main_theory_feature_recovery}
\input{./paper/sections/main_result_dynamics.tex}

\input{./paper/sections/proof_overview.tex}

\section{Conclusion and Future Work}\label{sec:conclusion}
\input{./paper/sections/conclusion.tex}

%% file: paper/sections/related_work.tex
\paragraph{SAE training methods}
Many methods have been proposed to train SAEs, addressing the trade-off between reconstruction fidelity and sparsity-induced interpretability from various perspectives. One canonical approach is imposing an $\ell_1$ penalty on the activations \cite{bricken2023monosemanticity}. Although $\ell_1$ penalty is a natural surrogate for enforcing $\ell_0$ sparsity, it typically suffers from activation shrinkage \cite{tibshirani1996regression}. Several works have attempted to overcome this drawback through alternative techniques \cite{wright2024addressing,taggart2024prolu,rajamanoharan2024improving,konda2014zero}. 
In particular, \cite{rajamanoharan2024jumping} proposed the JumpReLU activation, which achieves state-of-the-art performance. This method requires backpropagation with pseudo-derivatives due to the non-smooth nature of JumpReLU and the need for tuning the kernel density estimation bandwidth.
Another representative example is the use of TopK activation \cite{makhzani2013k}, which has proven effective when scaled to large models \cite{gao2024scaling}. However, it has been observed that features learned via TopK activation are quite sensitive to the random seed \cite{paulo2025sparse}, raising concerns about their reliability.

\paragraph{Sparse dictionary learning}
Beyond SAE training methods, there is a long history of research on sparse dictionary learning (SDL) dating back to \citet{olshausen1996emergence,kreutz2003dictionary}. Numerous techniques have been developed for applications in signal processing and computer vision \citep{bruckstein2009sparse, rubinstein2010dictionaries}. For example, \citet{spielman2012exact} proposed a polynomial-time algorithm that can accurately recover both the dictionary and its coefficient matrix, under the assumption of sparsity in the coefficients.

\paragraph{Using SAEs for model interpretation}
In recent years, SAEs have gained attention for model interpretation, particularly in the context of large language models (LLMs) \citep{bricken2023monosemanticity,paulo2025sparse}. Notably, \citet{bricken2023monosemanticity,dunefsky2024transcoders,ameisen2025circuit} have identified several interesting features and circuit patterns learned by SAEs or their variants. Beyond detecting monosemantic features, \citet{papadimitriou2025interpreting} found that groups of SAE-learned features remain remarkably stable across different training runs and encode cross-modal semantics. Additionally, the potential of SAE activations for steering model behavior has been explored \citep{ameisen2025circuit,shu2025beyond}.

%% file: paper/sections/preliminaries.tex
For a trained deep neural network, the learned features often represent a mixture of multiple underlying concepts, which can be viewed as a linear combination of monosemantic features \citep{bricken2023monosemanticity, cunningham2023sparse}.
As a motivating example, consider the following sentence: 
\begin{quote}
    \centering
    \textit{``The detective found a \colorbox{yellow!25}{\textbf{muddy footprint}} near the \colorbox{teal!10}{\textbf{broken window}}, \\
    leading him to suspect a \underline{ }\underline{?}\underline{ }''}
\end{quote}
When the model encounters this sentence, and in order to predict the missing word ``burglary'', the model needs to mix the two monosemantic features \textit{``footprint''} and \textit{``window''} at some layer to form a mixed representation $x$ as
        \begin{equation} \label{eq:motivating_example}
            x = h_1 \cdot \underbrace{\colorbox{yellow!25}{\textbf{\textit{``muddy footprint''}}}}_{\ds v_1} \:+\: h_2 \cdot \underbrace{\colorbox{teal!10}{\textbf{\textit{``broken window''}}}}_{\ds v_2} + \ldots, 
        \end{equation}
and then use the mixed representation $x$ to predict the missing word through the Feed Forward Neural Network (FFN) layer.
In this example, we have two monosemantic features $v_1$ and $v_2$ corresponding to the concepts \textit{``muddy footprint''} and \textit{``broken window''}, respectively, and the coefficients $h_1$ and $h_2$ represent the contribution of each feature to the mixed representation $x$.
In particular, the coefficients $h_1$ and $h_2$ are nonnegative for this mixing, as the opposite direction of a feature would often lead to a totally different or even contradictory concept.

\paragraph{A model for feature recovery}
In what follows, we formalize the motivating example into a statistical framework for feature recovery.
Assume that the model's hidden representation at a specific layer encodes $n$ distinct features in a $d$-dimensional space.
We aggregate these features in the \emph{\highlight{feature matrix}} $V\in\RR^{n\times d}$, where each row $v_i$ is a $d$-dimensional feature vector.
Suppose we have a training set with $N$ tokens. 
For each token, we extract the hidden representation as a $d$-dimensional vector, resulting in a \emph{\highlight{data matrix}} $X\in\RR^{N\times d}$.
Each row $x_\ell$ of $X$ corresponds to a mixture of the features in $V$ with nonnegative coefficients, i.e., $x_\ell=\sum_{i=1}^n h_{\ell,i} v_i$ for all $\ell \in [N]$.
We collect these coefficients into a \emph{\highlight{coefficient matrix}} $H\in\RR_+^{N\times n}$, where each row $H_{\ell,:}$ contains the coefficients for the corresponding data point $x_\ell$.
We then have the following data model:
\begin{align}\label{eq:data_model}
    X = H V \in \RR^{N\times d}.
\end{align}
In particular, we assume that every coefficient vector, i.e., each row of $H$, is $s$-sparse with all entries nonnegative.
The goal is to recover the feature matrix $V$ from the training data $X$, even though the coefficient matrix $H$ is not known.

In practice, both the embedding dimension $d$ and the number of features $n$ can be very large. In state-of-the-art LLMs, $d$ typically ranges from $2^{10}$ to $2^{12}$, while $n$—interpreted as the number of distinct “concepts” a model can represent at a token position—varies with the architecture and where the model is probed. 
For generality, we restrict $n$ to be polynomial in $d$, i.e., $\omega(1) < n < \mathrm{poly}(d)$, and we focus in particular on the \highlight{superposition} regime where $n > d$. 
In this regime, the feature vectors must be linearly dependent \citep{arora2018linear, olah2020zoom, elhage2022toy}. In the context of LLMs, $x$ is often a vector computed internally at a specific layer and token position. This vector encodes the semantic meaning of the sentence, which is typically a mixture of multiple underlying concepts. 
For example, consider passing the sentence in the motivating example in \S\ref{sec:preliminaries} through a transformer model. The hidden representation at a specific layer at token `\texttt{a}', used to predict the next word, will incorporate concepts from both \textit{``muddy footprint''} and \textit{``broken window''}, leading to \highlight{polysemantic} representation. 
The model in \eqref{eq:data_model} explicitly describes \highlight{superposition} as the phenomenon where \emph{a polysemantic vector $x$ is a sparse linear combination of multiple monosemantic features in $V$}. 
 
However, given a data matrix $X$, the factorization in \eqref{eq:data_model} is not unique.
That is,  a single data point (a row of $X$) may admit multiple valid representations through different feature combinations.
This potentially raises the issue of identifiability, which will be addressed in \Cref{sec:identifiability}.

\paragraph{SAE architecture}
At the core of our feature recovery method is the Sparse Autoencoder (SAE) \citep{vincent2010stacked}, which is a neural network architecture designed to learn sparse representations of input data through an unsupervised self-reconstruction process.
We follow \citet{gao2024scaling, cunningham2023sparse} and use a three-layer neural network for SAE with pre-bias and tied encoding and decoding weights.
Let $M$ be the width of the SAE, i.e., the number of neurons. Let $\phi : \RR \to \RR$ be a nonlinear activation function, such as ReLU or JumpReLU \citep{erichson2019jumprelu, rajamanoharan2024jumping}. 
The parameters of the SAE are denoted by \(
\Theta = \{(w_m, a_m, b_m)_{m=1}^M,\, b_{\mathrm{pre}}\},
\)
where $\{ w_m \}_{m\in [M]} \subseteq \RR^d$ are the shared weight vectors for both encoding and decoding, \(a_m \in \RR\) is the output scale for neuron \(m\), \(b_m \in \RR\) is the bias for neuron \(m\), and \(b_{\mathrm{pre}} \in \RR^d\) is the pre-bias vector that centers the input data.
For any input $x\in\RR^d$,  the output of the SAE with parameters $\Theta$ is defined as 
\begin{align}
    f(x; \Theta) = \sum\nolimits_{m=1}^M a_{m} \cdot w_{m} \cdot \phi\bigl( \underbrace{{w_{m}^\top (x-b_{\mathrm{pre}}) + b_{m}}}_{\ds \small\text{pre-activation}~y_m} \bigr) + b_{\mathrm{pre}}.
        \label{eq:sae}
\end{align}
In particular, in each neuron, we center the input $x$ by subtracting the pre-bias $b_{\mathrm{pre}}$, which is then added back to the output. 
In each neuron $m \in [M]$, we first compute the \emph{{pre-activation}} \(y_m = w_m^\top (x - b_{\mathrm{pre}}) + b_m\), and then apply the activation function \(\phi\) to obtain the neuron activation. 
The activation function $\phi$ brings nonlinearity to the model. 
We say a neuron \(m\) is \emph{{activated}} if \(y_m > 0\).
When neuron $m$ is activated, its contribution to the output is proportional to $w_m$, with a scaling factor $a_m$.


The symmetry introduced by tying the encoder and decoder weights endows the SAE with a more interpretable structure. In this configuration, the encoder weight \(w_m\) serves as a \emph{detector} that activates the neuron when its designated feature is present, while the same weight in the decoder acts as a \emph{reconstructor} that regenerates the input from the neuronal activations.
During training, the SAE is optimized to reconstruct the input \(x\) with reconstruction loss 
\begin{align}
    \cL_\mathrm{rec}(x; \Theta) = \frac{1}{2} \|f(x; \Theta) - x\|_2^2, \label{eq:def_rec_loss}
\end{align}

\paragraph{Reconstructing monosemantic features} 
Under the statistical model in \eqref{eq:data_model}, we can train the SAE to reconstruct the data matrix $X$ by minimizing the  loss function 
\begin{align} \label{eq:loss_sae_data}
    \cL(\Theta) = \frac{1}{N} \sum_{\ell = 1}^N  \cL_\mathrm{rec}(x_{\ell }; \Theta) = \frac{1}{2 N} \sum_{\ell = 1}^N  \|f(x_{\ell}; \Theta) - x\|_2^2. 
\end{align} 
By minimizing this loss function over $\Theta$ with a suitable optimization algorithm, in addition to being able to reconstruct the input data, the SAE can also learn to recover the underlying monosemantic features in $V$. 
Particularly, we aim to answer the following question:  
\begin{olivebox} 
     Can we design an SAE training algorithm that recovers the monosemantic features in $V$ from the training data $X$? That is,  the parameters $\{ w_m\}_{m\in [M]}$ recover all the monosemantic features in $V$. 
\end{olivebox}

\paragraph{Baseline methods}
In prior literature, a sparsity-inducing penalty,  typically the $\ell_1$ regularization, is applied to the neuron activations \citep{bricken2023decomposing,templeton2024scaling} to encourage the learning of sparse input representations. 
This leads to the loss function
\begin{align}
    \textbf{L1~method:} \quad \cL(x; \Theta) = \cL_\mathrm{rec}(x; \Theta) + \underbrace{\lambda \sum_{m=1}^M \|w_m\|_2\cdot \phi(w_m^\top (x - b_{\mathrm{pre}}) + b_m)}_{\ds \small\text{$\ell_1$ regularization}},
    \label{eq:L1 loss}
\end{align}
where \(\lambda\) is a regularization hyperparameter. 
Here, the $\norm{w_m}_2$ term gives higher weight to neurons whose weight vectors has larger magnitude. 
By applying the $\ell_1$ regularization, {only a subset of neurons are activated for each input, i.e., $\phi( w_m^\top (x - b_{\mathrm{pre}}) + b_m)$ is nonzero for only a small number of neurons.}
Alternatively, another line of work \citep{makhzani2013k,gao2024scaling} replaces the ReLU activation with a TopK activation function while still using the reconstruction loss \(\cL_\mathrm{rec}\) as the training objective.
Here, the TopK activation function retains only the top \(k\) largest neuron activations for each input in \eqref{eq:sae}. {That is, the top $K$ values of $\{  w_m^\top (x - b_{\mathrm{pre}}) + b_m \}_{m\in [M]}$ are kept and multiplied by the decoding weights, which effectively enforces sparsity in the neuron activations. In particular, exactly $K$ neurons are activated for each input.}

Nonetheless, the performance of these methods is often highly sensitive to hyperparameter tuning, such as the regularization parameter $\lambda$ in the $\ell_1$ loss or the threshold $K$ in TopK activations. 
Furthermore, these techniques have their own limitations. For instance, while $\ell_1$ loss encourages sparsity, it often triggers activation shrinkage during training by continuously diminishing neuronal pre-activations. This, in turn, can cause dead neurons and an underestimation of true feature magnitudes \citep{tibshirani1996regression}. In addition, TopK methods, although enforcing a strict sparsity constraint on the number of activated neurons, can produce feature sets that are overly sensitive to the randomness inherent in the initialization process \citep{paulo2025sparse}.

%% file: paper/sections/algorithm.tex


To address the limitations of existing Sparse Autoencoder (SAE) training methods, we propose a new algorithm called \emph{Grouped Bias Adaptation} (GBA).
Our algorithm has two main components: \highlight{\brai{1} a bias adaptation subroutine} that controls the activation frequency of each neuron, and \highlight{\brai{2} a neuron grouping strategy} that allows us to assign different target activation frequencies to different groups of neurons.
In the following, we first introduce the considerations behind these two components, and then give detailed illustrations of the algorithm.
The algorithm is shown in \Cref{alg:GBA} with a bias adaptation subroutine shown in \Cref{alg:bias_adapt}.

\begin{figure}[htbp]
  \centering
  \includegraphics[width=0.8\textwidth]{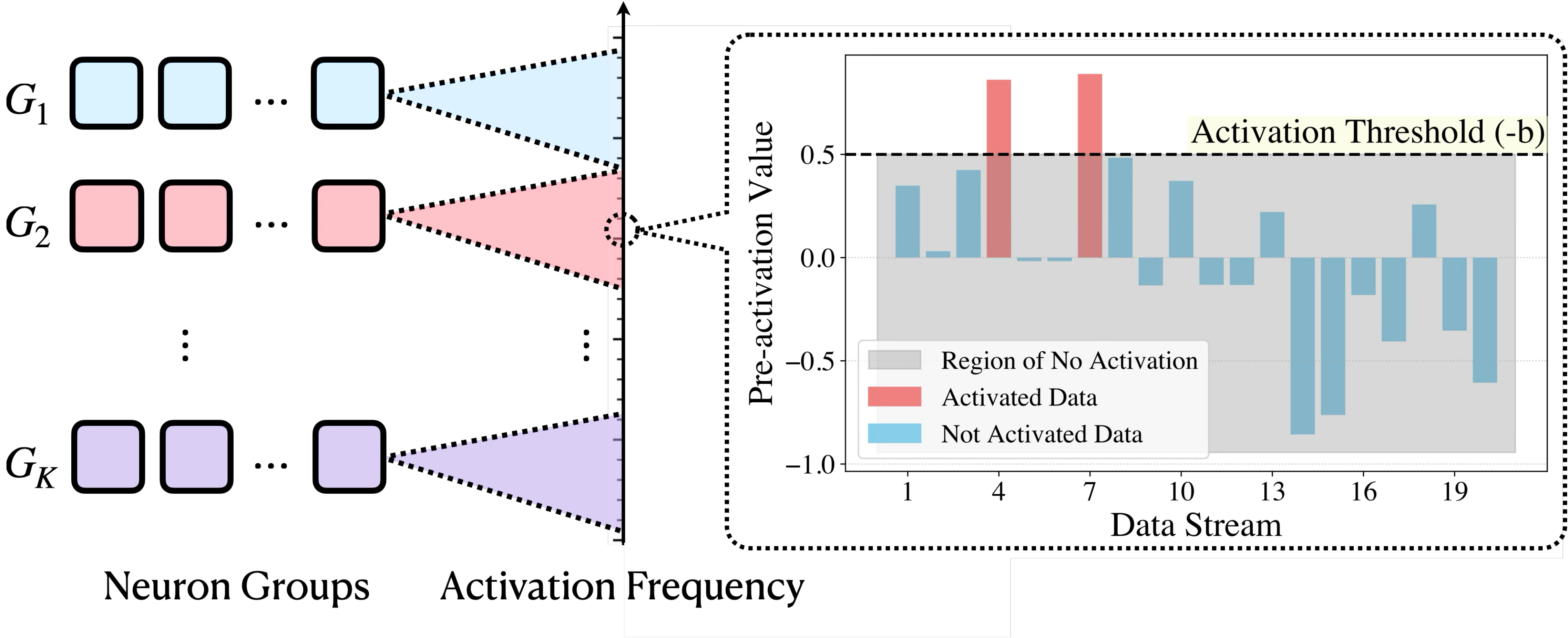}
  \caption{\small Illustration of neuron grouping and bias adaptation. Neurons are assigned to different groups $G_k$ with distinct Target Activation Frequencies (TAFs). (\textbf{Left}) Features with varying occurrence frequencies are captured by the group with designated TAFs.
  (\textbf{Right}) The bias adaptation mechanism shifts the neuron's bias $b$ to achieve TAF over a buffer of data, where each colored bar represents the neuron's pre-activation on a data point.
  }
  \label{fig:pre-activation_threshold}
\end{figure}
\subsection{Algorithm Motivations} \label{subsec:algorithm_motivations}
In what follows, we recall the SAE definition in \eqref{eq:sae} and rewrite the \emph{pre-activation} 
\begin{align} 
  y_m(x) =  {w_m^\top (x - b_{\mathrm{pre}})} + {b_m} \label{eq:pre_activation_def}
\end{align} 
for neuron $m$ as a function of input $x$. We say that a neuron is \emph{activated} on input $x$ if $y_m(x) > 0$.
Let $\cB$ denote a buffer of data, which stores a subset of the data points in $X$. 
We define the \emph{activation frequency} for each neuron $m$ over the buffer $\cB$ as the fraction of inputs for which neuron $m$ is activated:
\begin{align}
    \hat p_m = \frac{1}{|\cB|} \sum_{x\in\cB} \ind(y_m(x) > 0). 
    \label{eq:activation_frequency}
\end{align}
Ideally, if each neuron captures a monosemantic feature and data in $\cB$ is sufficiently diverse, 
each monosematic feature should not appear too often, and thus the activation frequency $\hat p_m$ should be small. Thus, we can use the activation frequency $\hat p_m$ to control the sparsity of the neuron activations. This consideration leads to the two main components of our algorithm.

\paragraph{Component 1: Bias adaptation}
An effective sparse autoencoder must strike a careful balance between activating neurons often enough to learn meaningful feature directions, but not so often that every neuron responds indiscriminately. In particular, we identify two key desiderata for neuron activation frequency:
\begin{enumerate}[
  label=\textbullet,
  leftmargin=2em, 
  ]
  \item \textbf{\textit{Ensure sufficient sparsity.}}  
    To discover a feature direction, the neuron targeting this feature must ``focus''. That is, activate preferentially on the subset of inputs that actually contain that feature. 
    { Suppose neuron $m$ targets a monosemantic feature $v$, then it should only activate on an input $x$ if $x$ has a linear component in the direction of $v$.}
  \item \textbf{\textit{Avoid over-sparsity.}}  
    If neurons activate too rarely, they receive almost no gradient signal and can become permanently inactive (``dead'' neurons), wasting representational capacity. { In other words, for each $m$, the activation frequency $\hat p_m$ over the whole dataset should not be too small.}
\end{enumerate}

Most existing SAE methods impose a sparsity penalty on the hidden activations.  While this encourages many activations to be zero, the exact activation frequency then depends critically on the balance between the reconstruction loss $\cL_{\mathrm{rec}}$ and the sparsity penalty strength (e.g., the $\ell_1$ regularization in \eqref{eq:L1 loss}).
Tuning that trade-off is cumbersome and data-dependent, and it becomes even more difficult when the relative scale of $\cL_{\mathrm{rec}}$ and the $\ell_1$ regularization  is dynamically changing during training due to both sampling random minibatches and the evolution of the neuron weights and biases. 

Instead, we seek \emph{direct} control over each neuron's long-term activation frequency.  
As the pre-activation $y_m(x)$ can be decomposed into a projection term $w_m^\top (x - b_{\mathrm{pre}})$ and a bias term $b_m$, we can control the activation frequency by simply adjusting the bias $b_m$ of each neuron to hit a \emph{\highlight{Target Activation Frequency (TAF)}}. 
In particular, we aim to decouple sparsity control from reconstruction loss minimization and at the same time avoid dead units --- motivating  the use of a \highlight{\emph{bias adaptation}} subroutine.

\paragraph{Implementing bias adaptation} 
Consider a simple example where we set the TAF to a number $p \in (0, 1)$ for all neurons.
Given a buffer of data $\cB$, we can compute the activation frequency $\hat p_m$ defined in \eqref{eq:activation_frequency} for each neuron $m$.
If $\hat p_m > p$,  neuron $m$ activates too frequently, and thus we reduce its bias $b_m$ to encourage it to activate less often.
This essentially ensures that the neurons are sparsely activated. 
Conversely, if $\hat p_m < \epsilon$ (where $\epsilon$ is a small positive number), it suggests that neuron $m$ is rarely activating, and we can increase its bias $b_m$ to encourage more frequent activation.
This avoids over-sparsity in the neuron activations.


\paragraph{Component 2: Neuron grouping with different TAFs}
While bias adaptation offers direct control over each neuron's activation frequency without complex tuning, using a uniform TAF ignores the diversity of features. A model might need to capture both common features (e.g., ``cat'') and rare ones (e.g., ``marmoset''), and a single TAF could either suppress rare features or overwhelm common ones. 
To overcome this limitation, we propose a \highlight{neuron grouping strategy} that assigns distinct TAFs to different groups of neurons, thereby accommodating the varied occurrence rates of the underlying features.

\paragraph{Implementing neuron grouping} 
We partition the neurons into $K$ groups, denoted by $\{G_k\}_{k=1}^K$, where each group $G_k$ contains $M/K$ neurons. Here $M$ is the number of total neurons. 
Each group $G_k$ is assigned a distinct TAF $p_k$, which is set to be exponentially decaying, e.g., $p_{k+1} = p_k / 2$.
The first group $G_1$ has the highest TAF $p_1$ set to some reasonably large value (e.g., $0.1$).
The exponential decay helps to cover a wide range of feature frequencies with a manageable number of groups. 
This grouping strategy allows us to capture features with varying occurrence frequencies.
Alternative grouping strategies with variable group sizes and custom TAFs are possible; however, we find that partitioning the neurons into equally sized groups with exponentially decaying TAFs not only yields robust performance in practice but also simplifies implementation.

\subsection{Implementation Details}

Combining the above two components, we arrive at the Grouped Bias Adaptation (GBA) algorithm.
We present the details of this algorithm and discuss the practical considerations of the algorithm as follows.

\paragraph{GBA algorithm overview}
The input of the GBA algorithm includes the training data $X$, the initial parameters $\Theta^{(0)}$ of the SAE, the neuron groups $\{G_k, p_k\}_{k=1}^K$ with their target activation frequencies, and a general first-order optimization method, denoted by $\mathtt{Opt}$.  
Recall that each group has $M/K$ neurons and $\{ p_k\}_{k\in[K]}$ is an exponentially decreasing sequence.
In addition, the hyperparameters include \brai{1} the total number of iterations $T$, \brai{2} the batch size $L$, \brai{3} designated buffer size $B$, and \brai{4} parameters $\{ \gamma_{+}, \gamma_{-}, \epsilon \}$ for bias adaptation. The output of the algorithm is the final parameters $\Theta^{(T)}$ of the SAE. 
The algorithm runs for $T$ iterations. Within each iteration, we sample a new mini-batch from the training data $X$ and use this mini-batch to update the weights of the SAE using a standard optimizer $\mathtt{Opt}$, such as AdamW \citep{loshchilov2017decoupled}, except that we do not update the biases by default. 
The biases $\{b^{(t)}\}_{t\in [T]}$ are only updated periodically via a bias adaptation subroutine $\cA_t$ (Algorithm \Cref{alg:bias_adapt}) that is triggered when the first buffer reaches its capacity $B$.
We present the details of this algorithm in \Cref{alg:GBA}.

\paragraph{Implementation of GBA} 
For any $t$, within the $t$-th iteration, we first sample a mini-batch $X_t \in \mathbb{R}^{L \times d}$ from the training data $X$, where $L$ is the batch size.
Then we normalize each row of $X_t$ to unit $\ell_2$ norm, i.e.,  for each row $x$ of $X_t$, we set $x \leftarrow x / \|x\|_2$ (Line \ref{line:data}).
This normalization is beneficial for training stability 
across different data points.
Then we compute the pre-activations $y^{(t)} \in \mathbb{R}^{M \times L}$ for all neurons across the batch (Line \ref{eq:compute_pre_activation}). Recall that the pre-activation for neuron $m$ is defined in \eqref{eq:pre_activation_def}. 
Using data $X_t$ and the current parameters $\Theta^{(t-1)}$, the pre-activations form a $M\times L$ matrix. 
In particular, $W^{(t-1)} \in \RR^{M\times d}$ is the weight matrix whose rows are $\{w_m^{(t-1)}\}_{m\times [M]}$, and $ \mathbf{1}_{L}^\top $ is an all-one vector in $\RR^{L}$. 
Based on these pre-activations, we can evaluate the outputs of the SAE according to \eqref{eq:sae} and the reconstruction loss in \eqref{eq:def_rec_loss}. 
In particular, the loss $\cL^{(t)}$ in Line \ref{line:minibatch_loss} is computed as the average reconstruction loss over the normalized mini-batch $X_t$:
\begin{align} \label{eq:minibatch_loss}
  \cL^{(t)} (\Theta) = \frac1 L \cdot \sum_{x \in X_t} \cL_\mathrm{rec}(x_{\ell }; \Theta).  
 \end{align}
Then, we compute the gradient of this loss function and pass the gradient to the optimizer $\mathtt{Opt}$ to update the parameters $\Theta^{(t-1) }\setminus \{b^{(t-1)}\}$ (Line \ref{line:opt_step}). 
In particular, we only update parameters $\{( w_{m}^{(t-1)}, a_m^{(t-1)})_{m=1}^M, b_{\mathrm{pre}}^{(t-1)}\}$, but not the biases $\{b_m^{(t-1)}\}_{m=1}^M$. 
Here $\mathtt{Opt}(\cdot , \cdot )$ is a general first-order optimization algorithm where the first argument is the parameters to be updated and the second argument is the gradient. It may contain additional hyperparameters such as the learning rate, which we omit for simplicity.
Furthermore, we use the buffer $\mathcal{B}_m$ to store the pre-activations of each neuron $m$. In the $t$-th iteration, we add the newly computed pre-activations $\{y_{m,1}^{(t)}, \ldots, y_{m,L}^{(t)}\}$ to  $\mathcal{B}_m$ for each neuron $m$ (Line \ref{line:buffer}).
Once the first buffer reaches its capacity $B$ (thus all buffers are of size $B$), the bias adaptation subroutine (\Cref{alg:bias_adapt}) is triggered to update the biases. We release all buffers when $b^{(t)}$ is updated (Line \ref{line:bias_adapt}).


\begin{algorithm}[htbp]
  \caption{Group Bias Adaptation (GBA)}
  \label{alg:GBA}
  \begin{algorithmic}[1]
  \STATE {\bf Input:}  data $X$, initialization $\Theta^{(0)}$, neuron groups and desired TAFs $\{G_k, p_k\}_{k=1}^K$, a first-order optimization algorithm $\mathtt{Opt}$. 

  \STATE {\bf Hyperparameters:} $T$, $L$, $B$, $\gamma_+$, $\gamma_-$, and $\epsilon$.

  \STATE For all $m \in [M]$, initialize buffer $\mathcal{B}_m \leftarrow \emptyset$.\\
  \STATE For $t = 1, \ldots, T$: 
  \vspace{3pt}
\STATE \qquad \greencomment{Update weights using standard optimization method, except for bias parameters.}
   \STATE    \qquad  Sample a mini-batch of training data $X_t \in \mathbb{R}^{L \times d}$ and normalize each row to unit $\ell_2$ norm. \label{line:data}

      \STATE \qquad Compute pre-activation:  $y^{(t)} \leftarrow W^{(t-1)}(X_t^\top -  b_{\mathrm{pre}}^{(t-1)} \mathbf{1}_{L}^\top  )  + b^{(t-1)} \mathbf{1}_{L}^\top \in \RR^{M\times L }$. \label{eq:compute_pre_activation}
      \STATE \qquad Compute loss using mini-batch: $\mathcal{L}^{(t)} \leftarrow \mathcal{L}(X_t, \Theta^{(t-1)})$ as defined in \eqref{eq:minibatch_loss}. \label{line:minibatch_loss}
      \STATE \qquad Update all parameters except for biases: $\Theta^{(t)} \leftarrow \mathtt{Opt}(\Theta^{(t-1)} \setminus \{b^{(t-1)}\}, \nabla \mathcal{L}^{(t)}) $. \label{line:opt_step}
      \STATE \qquad Update buffers:  $\mathcal{B}_m \gets \mathcal{B}_m \cup \{y_{m,1}^{(t)}, \ldots, y_{m,L}^{(t)}\}$ for all $m$. \label{line:buffer}

\vspace{3pt}
\STATE \qquad \greencomment{Update the bias parameters when the buffer sizes reach $B$.}

     \STATE \qquad If $|\mathcal{B}_1| \geq B$, update biases $b^{(t)} \gets \mathcal{A}_t(b^{(t-1)}, \mathcal{B})$ according to \Cref{alg:bias_adapt}  and empty all buffers by setting $\mathcal{B}_m \leftarrow \emptyset$ for all $m$. \label{line:bias_adapt}
         
     \vspace{3pt}
 \STATE {\bf Return} the final SAE parameters $\Theta^{(T)}$.
  \end{algorithmic}
  \end{algorithm}


\paragraph{Efficient bias adaptation}
The bias adaptation subroutine is outlined in \Cref{alg:bias_adapt}.
Note that the $M$ neurons are split into $K$ groups, denoted by $\{G_k\}_{k=1}^K$, and each group $G_k$ has a designated target activation frequency (TAF) $p_k$.
In this subroutine, based on a buffer $\cB_m$ of pre-activations for each neuron $m$, we periodically adapt the bias $b_m$ based on the information stored in $\cB_m$ to make sure that the activation frequency $\hat p_m$ of each neuron $m$ closely tracks the target TAF. 
As mentioned in \Cref{subsec:algorithm_motivations}, to achieve such a goal we need to simultaneously \emph{ensure sufficient sparsity} and \emph{avoid over-sparsity} of the neuron activations. 
In particular, in this subroutine, we propose to \highlight{\emph{decrease the bias for overly active neurons}} and \highlight{\emph{increase the bias for inactive neurons}}.

Specifically, given the buffer $\cB = \{ \cB_m \}_{m\in [M]}$ containing stored pre-activations for each neuron $m$, we first compute the activation frequency $\hat p_m$ as in \eqref{eq:activation_frequency}. 
Suppose neuron $m$ belongs to group $k$. 
Then, we adjust the bias $b_m$ to let $\hat p_m$ move closer to the target TAF $p_k$. 
We introduce two quantities: 
\begin{itemize}
    \item \textbf{Maximum pre-activation} $r_m = \max\{\max_{y\in\cB_m}y, 0\}$ for neuron $m$. This value represents the highest pre-activation that neuron $m$ attains over the buffer $\cB_m$, with the truncation at zero ensuring non-negativity. Intuitively, $r_m$ reflects the strongest activation of neuron $m$, and it is used to guide the bias adaptation: if we aim to reduce $b_m$ such that the neuron barely activates (i.e., nearly “dead”) for all inputs in $\cB_m$, then the ideal new bias would be $b_m - r_m$. In practice, we update the bias as $b_m - \gamma_- r_m$ (with $\gamma_- \in (0,1)$) to gradually decrease the activation frequency without completely deactivating the neuron.

    \item \textbf{Average maximum pre-activation} $\bar r_k = (\sum_{m\in G_k} \ind(r_m>0))^{-1} \sum_{m\in G_k} r_m$ for group $k$, which is the average of the maximum pre-activations of all neurons in group $k$. This quantity effectively captures the aggregated response of neuron group $k$ to the data points in the buffer $\cB$.
    The computation of $\bar r_k$ excludes dead neurons (i.e., those with $r_m=0$) and involves the information from all neurons in group $k$.
  
\end{itemize}

Based on these two quantities, we adapt the bias $b_m$ of each neuron $m$ in group $k$ as follows. Note that we initialize each bias $b_m$ to zero. We consider following two cases:
\begin{enumerate} 
\item [(i)]\textbf{\textit{Decrease bias for overly active neurons}}: 
If neuron $m$ activates too often ($\hat p_m > p_k$), we reduce $b_m$ by $\gamma_{-}\cdot r_m$, where $\gamma_{-} \in (0,  1)$ is a small hyperparameter. 
Here $r_m$ plays the same role as the gradient in standard optimization methods, providing a sufficient amount of adjustment that drives the activation frequency $\hat p_m$ toward the target TAF $p_k$. 
Multiplying by a small factor $\gamma_{-}$, which plays the role of a learning rate, ensures the bias is reduced gradually, preventing the neuron from becoming completely inactive. We further clamp $b_m$ at $-1$ to avoid overly large decreases in the bias. See Line \ref{line:decrease_bias} in \Cref{alg:bias_adapt}.


\item [(ii)] \textbf{\textit{Increase bias for inactive neurons}}:
If neuron $m$ rarely activates (i.e., if $\hat p_m < \epsilon$, with $\epsilon=10^{-6}$ in our experiments), we increase its bias $b_m$  by adding  $\gamma_+ \cdot \bar r_k$, where $\gamma_+\in (0,1)$ is the learning rate. This adjustment encourages the neuron to activate more frequently. 
In particular, when $\hat p_m$ is too small, we use the pre-activation information from the entire group $G_k$ to increase the bias, which helps to avoid dead neurons. Moreover, to make sure that the bias does not become too large, we clamp $b_m$ at $0$. See Line \ref{line:increase_bias} in \Cref{alg:bias_adapt}.

\end{enumerate}

\begin{algorithm}[htbp]
  \caption{GBA Subroutine $\mathcal{A}_t(b, \mathcal{B})$}
  \label{alg:bias_adapt}
  \begin{algorithmic}[1]
  \STATE {\bf Input:} current bias $b \in \mathbb{R}^M$, buffer $\mathcal{B} = \{ \cB_m \}_{m\in [M]}$ containing stored pre-activations for each neuron, neuron groups and desired TAFs $\{G_k, p_k\}_{k=1}^K$, and hypterparameters $\gamma_+$, $\gamma_-$, and $\epsilon$.
  
  \vspace{3pt}
  \STATE \greencomment{Compute activation frequency, maximum pre-actions and their group average.}
 
  \STATE Compute the activation frequency $\hat{p}_m$ and maximum pre-activation $r_m$  for each neuron $m \in [M]$:  
  $$
  {\textstyle \hat{p}_m \leftarrow |\mathcal{B}_m|^{-1} \sum_{y \in \mathcal{B}_m} \ind (y > 0), \qquad r_m \leftarrow \max\bigl \{\max_{y  \in \mathcal{B}_m} y , 0 \bigr\}. }$$

  \STATE Compute the average maximum pre-activation $\bar{r}_k$ for each group $k \in [K]$: 
$$
  {\textstyle \bar{r}_k \leftarrow \left(\sum_{m \in G_k} \mathbf{1}(r_m > 0)\right)^{-1} \sum_{m \in G_k} r_m.}
$$
  \vspace{3pt}
  \STATE \greencomment{Adapt the bias for each neuron based on its activation frequency.} 
  \STATE For each group $k = 1, \ldots, K$ and each neuron $m$ in group $G_k$:
 \STATE \qquad If  $\hat{p}_m > p_k$ then set $b_m \leftarrow \max \{ b_m - \gamma_- r_m,   -1\} $. \label{line:decrease_bias}

  \STATE \qquad If  $\hat{p}_m < \epsilon$ then set $b_m \leftarrow \min \{ b_m + \gamma_+ \bar{r}_k, 0\}$. \label{line:increase_bias}
  \vspace{3pt}
  
  \STATE Return updated bias vector $b$.
  \end{algorithmic}
  \end{algorithm}

\paragraph{Practical implementation of bias adaptation}
In practice, we perform bias adaptation every 50 gradient steps using the largest batch size permitted by the training hardware --- a trade-off that balances memory consumption against the accuracy of neuron activation frequency estimation. Furthermore, we clamp each bias to the interval $[-1, 0]$, where the upper bound ($0$) maintains sparsity and the lower bound ($-1$) prevents over-sparsification. 
While from the description of the algorithm, it seems that we need to store all the pre-activations in the buffer $\cB_m$ for each neuron $m$, this is not necessary in practice. 
In particular, we only need to track  $\hat p_m$  and  $r_m$  as the buffer $\cB_m$ increases, which can be \highlight{iteratively updated as new mini-batches are processed}. To see this, suppose that we have a buffer $\cB_m$ and then receive a mini-batch of size $L$ with pre-activations $\{y_{m, 1}^{(t)}, \ldots, y_{m, L}^{(t)}\}$ for neuron $m$. 
Then, we can update $\hat p_m$ and $r_m$ as follows:
\begin{align}
    \hat p_m \leftarrow \frac{1}{|\cB_m| + L} \cdot \Bigl( |\cB_m| \cdot \hat p_m + \sum_{i=1}^L \ind(y_{m, i}^{(t)}>0) \Bigr) , \quad r_m \leftarrow \max\bigl\{ y_{m, 1}^{(t)}, \ldots,  y_{m, L}^{(t)}, r_m, 0\bigr\}. 
\end{align}
This allows an efficient practical implementation of \Cref{alg:bias_adapt} without the need to store all pre-activations in the buffer $\cB_m$ for each neuron $m$.

Although our algorithm involves a few hyperparameters, they are easy to set and do not require much tuning as we have demonstrated above. 
The effectiveness of this grouping strategy is further demonstrated in \Cref{sec:experiments}. Our results indicate that it achieves comparable sparsity-loss trade-offs to state-of-the-art methods (e.g., TopK SAE) while yielding more consistent learned features across different random initializations---highlighting its robustness and reliability. We conclude this section by discussing two practical considerations regarding the grouping strategy.

\vspace{5pt}
\noindent{\textit{\textbf{What if a feature occurs more often than the largest TAF?}}} \quad
Although some features may occur even more frequently than $p_1$, we expect these cases to be rare, so a dedicated neuron group is unnecessary. Instead, we rely on a bias-clamping strategy that permits neurons to have $\hat p_m$ exceeding their TAFs when needed. For example, if a neuron $m$ in the first group aligns with a feature occurring more frequently than $p_1$, by construction, the bias adaptation subroutine will continue to reduce its bias because the $\hatp_m > p_1$ condition is triggered. 
This process will continue to decrease $b_m$ until the neuron becomes completely inactive, which is undesirable.
To avoid this case, we always clamp $b_m$ to be no less than $-1$.
This strategy is effective in practice, as we observe in the real-world experiments,  there are indeed several neurons that are enabled to learn features with frequencies exceeding the largest TAF. 
We defer the details to \Cref{sec:experiments}. 
\textit{\highlight{Compared to directly setting a group with a very large TAF, this strategy is more suitable for capturing a few extremely frequent features while ensuring better sparsity overall.}}

\vspace{5pt}
\noindent{\textit{\textbf{Could the early groups dominate training?}}} \quad 
A potential concern is that neurons in the early groups --- being activated more frequently --- might dominate training and reduce overall sparsity. Although this is a reasonable concern, it is mitigated by the inherent \highlight{\emph{selectivity}} of the neurons. In \Cref{sec:main_theory_feature_recovery}, we show via both theory and experiments that a feature with occurrence frequency \(f\) is primarily captured by neurons whose TAF \(p_k\) falls within a range close to \(f\). See also \Cref{fig:pre-activation_threshold} for an illustration.
\textit{\highlight{In other words, each neuron group is tuned to extract features whose frequencies lie within its designated exponential interval.}}

%% file: paper/sections/experiments.tex
\paragraph{Training data and architecture of SAE} To demonstrate the effectiveness of our proposed method, we conduct a series of experiments on 
the \texttt{Qwen2.5-1.5B} base model \citep{yang2024qwen2} using two datasets, \texttt{Pile Github} and \texttt{Pile Wikipedia} \citep{gao2020pile} with the first 100k tokens from each dataset.
We attach an SAE to the output of the LLM's MLP block at layers 2, 13, and 26 with \highlight{$M=66k$} hidden neurons.
That is, we train three separate SAEs for each of the datasets. 
The input and output dimensions of the SAEs are set to \highlight{$d = 1536$}.  
To achieve optimal performance, we adopt the JumpReLU activation \citep{erichson2019jumprelu, rajamanoharan2024jumping} for all training methods (In \Cref{app:exp_1B}, we also provide a detailed comparison between the use of ReLU and JumpReLU.). 
After preprocessing the data, we use \highlight{$N = 100m$} tokens to train the SAEs. In other words, the training data of the SAEs are obtained by feeding the $N$ tokens into the LLM and collecting the MLP outputs at the specified layers. 

\paragraph{Implementations of four methods} We compare the performance of our proposed Grouped Bias Adaptation (GBA) method with the following baselines:
the TopK, L1, and Bias Adaptation (BA) method. 
Here, the BA method is simply GBA using only one group with a hyperparameter $p$. 
In other words, our GBA can be viewed as a \textit{hyperparameter-free} version of the BA method without explicitly setting the value of $p$. 
All these algorithms are trained using AdamW \citep{loshchilov2017decoupled} with the same set of hyperparameters, including learning rate, weight decay, and batch size.
In the implementation of the GBA algorithm, by default, we set the number of groups to \highlight{$K = 10$} and the target frequencies to form a geometric sequence with the highest target frequency (HTF) set to \highlight{$p_1 = 0.1$} and the lowest target frequency (LTF) set to \highlight{$p_{10} = 0.001$}. 
The other hyperparameters of GBA, hyperparameters of the other methods,  and details of AdamW training can be found in \Cref{app:exp_1B}. 


\paragraph{Evaluation metrics} 
We evaluate each method based on two metrics: \highlight{(i) $\ell_2$ reconstruction loss}, and \highlight{(ii) the average fraction of activated neurons} (equivalent to the average $\ell_0$ loss divided by the total number of neurons) in the trained SAE.
 Ideally, a good SAE should probe the pareto frontier of these two metrics, achieving minimal reconstruction loss while maintaining a low fraction of activated neurons.

Through extensive experiments, we would like to answer the following questions:
\begin{olivebox} 
  \begin{itemize}
    \item [{\color{red!70!black}Q1}] How does the proposed GBA method compare with the TopK, L1, and BA methods in terms of  $\ell_2$ reconstruction loss and activation sparsity?
    \item [{\color{red!70!black}Q2}] How robust is the GBA method to the choice of hyperparameters, such as the number of groups and target frequencies?
    \item [{\color{red!70!black}Q3}] How consistent are the features learned by the GBA method across different runs with distinct random seeds?
  \end{itemize}
\end{olivebox}



\paragraph{Reconstruction loss and activation sparsity frontier} 
We first compare the normalized $\ell_2$ reconstruction loss and the average fraction of activated neurons across different methods
The experiment results are presented in \Cref{fig:loss_sparsity}, where we plot the average $\ell_2$ reconstruction loss against the average fraction of activated neurons for each method. 
We plot the results for the SAEs trained on data from \texttt{Github} dataset and the MLP outputs of layers 2, 13, and 26, and \texttt{Wikipedia} dataset with MLP layer 26. 
We note that there are two ways to measure the fraction of activated neurons for the TopK method: one based on the pre-activation values and another based on the post-activation values. See \Cref{app:other_training_methods} for details. For the other methods, measuring sparsity using pre- or post-activation values yields the same results.
Without specification, we use the post-activation for the TopK method by default in the sequel when we say activation frequency/fraction. 
We summarize the key findings as follows, which answer {\color{red!70!black}Q1}: 
\begin{enumerate}[
    leftmargin=2em, 
    ]
\item [(1)] \textbf{(GBA comparable to TopK)} Our method performs \highlight{comparably to the best-performing benchmark}, \highlight{TopK} with  post-activation sparsity. 
In addition, GBA outperforms TopK with pre-activation sparsity. 
Specifically, when these methods have the same average fraction of activated neurons, the reconstruction of GBA (green star) is comparable to that of TopK with post-activation sparsity and significantly better than that of TopK with pre-activation sparsity.

\item [(2)] \textbf{(GBA outperforms  L1)} GBA is \highlight{significantly better than the L1  method}. When they have the same average fraction of activated neurons, GBA achieves a lower reconstruction loss. 

\item [(3)] \textbf{(GBA outperforms BA)}
Moreover, we compare our method with its non-grouped counterpart BA and observe that \highlight{GBA consistently outperforms the non-grouped version across all experiments}. This provides strong evidence that the grouping mechanism enhances both sparsity and reconstruction performance.
\end{enumerate}

\begin{figure}[htbp]
  \centering
  \includegraphics[width=0.9\linewidth]{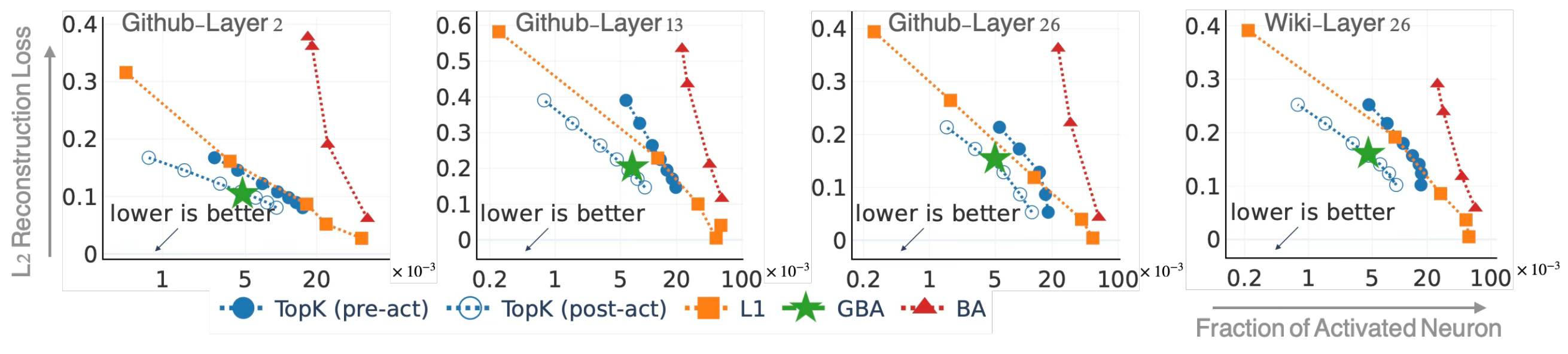}
  \caption{\small
      The reconstruction error with respect to the average fraction of neurons activated per data point. All experiments are conducted using an SAE with 66k neurons.
      For the TopK method, we vary $K$ within $\{50, 100, 200, 300, 400, 500, 600\}$.
      For the L1 method, we vary the penalty coefficient within $\{0.1, 0.03, 0.01, 0.003, 0.001\}$.
      A neuron is considered active in the pre-activation if 
      $\ind(y_m > 0),$ 
      and in the post-activation if 
      $\ind(S_K(\phi(y_m)) > 0)$, where $S_K$ is an operator that selects the largest $K$ values and sets the remaining to zero across all $\{\phi(y_m)\}_{m=1}^M$. 
      This distinction in sparsity between pre- and post-activation is relevant only for the TopK activation. See \Cref{app:other_training_methods} for details.
      An ablation study with varying configurations for the GBA method is presented in \Cref{fig:group_ablation}.
  }
  \label{fig:loss_sparsity}
\end{figure}

To see these findings, for any fixed value of the average fraction of activated neurons, TopK with post-activation sparsity achieves the lowest reconstruction loss among the benchmark methods. 
Note that \highlight{all these methods involve sparsity-related tuning parameters}, namely, $K$ in TopK, $\lambda$ in L1, and $p$ in BA. Varying these parameters, we obtain the curves in \Cref{fig:loss_sparsity}. 
It is clear that the curve for TopK with post-activation sparsity is the lowest, indicating that it achieves the best reconstruction loss for a given fraction of activated neurons.
In contrast, we do not explicitly tune the hyperparameters in our GBA method. Rather, we fix the number of groups $K$ and  a series of target frequencies $\{ p_k\}_{k\in[K]}$ for the different neuron groups. 
We only report the results for a default setting of these grouping parameters, and thus the results are shown as a single point in \Cref{fig:loss_sparsity}. 
Vertically comparing the results, we see that the GBA method achieves similar reconstruction losses as the TopK method with post-activation sparsity, best among the benchmark methods, and is strictly better than the rest of these methods. 

\paragraph{Robustness and nearly tuning-free}
Given the results in \Cref{fig:loss_sparsity}, it is unclear whether GBA is sensitive to the choices of $K$ and $\{ p_k \}_{k\in [K]}$, as raised by {\color{red!70!black}Q2}.
To address this question, we perform an ablation study on neuron grouping and target frequency (see \Cref{fig:group_ablation}). Specifically, we vary the number of groups $K$, the HTF $p_1$ and LTF $p_K$ for the GBA method. 
For the other parameters, we use the default configurations as in the loss-sparsity experiment. The detailed results are presented in \Cref{fig:group_ablation}.
In the left plot of \Cref{fig:group_ablation}, we plot the $\ell_2$ reconstruction losses against the average fraction of activated neurons for different choices of HTF, LTF, and the number of groups. 
We group these results according to the value of HTF,
and for each HTF, we plot the results for each value of $K$ with a different color. As a result, for each fixed HTF and $K$, there are three scatter points, each corresponding to a different LTF. 
We also plot the curve corresponding to the TopK method with post-activation sparsity for comparison, which is obtained by varying the value of the number of activated neurons in TopK.

\begin{figure}[htbp]
    \centering
    \includegraphics[width=0.9\linewidth]{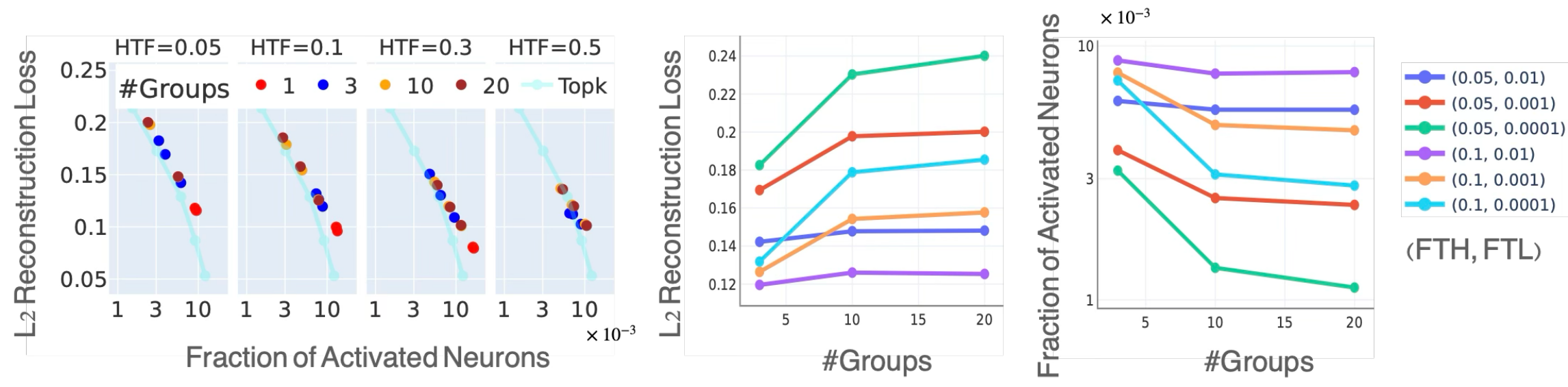}
    \caption{\small 
    Ablation study illustrating the impact of neuron grouping and the highest target frequency for \texttt{Github}-Layer 26. 
    For each run, we partition the neurons into \#Groups ($K$) groups and set the target frequencies as a geometric sequence between the Highest Target Frequency (HTF) and Lowest Target Frequency (LTF).  
    The LTFs are chosen from $\{1\times 10^{-3}, 5\times 10^{-3}, 1\times 10^{-4}\}$ and the HTFs are chosen from $\{0.05, 0.1, 0.3, 0.5\}$. 
    (\textbf{Left}) Illustrations for the $\ell_2$
    reconstruction loss versus the average fraction of activated neurons, grouped by HTF. For each HTF, different colors represent different values of $K$, while dots of the same color correspond to different LTFs.
    (\textbf{Middle \& Right}) Illustrations for $\ell_2$ reconstruction loss and the average fraction of activated neurons for different choices of $K$ with HTF and LTF fixed.
    Overall, these plots show that GBA is nearly tuning free, meaning that the performance is robust to the variations in the grouping parameters as long as HTF is sufficiently high (e.g., 0.5) and $K$ is sufficiently large (e.g., 10 or 20). 
    }
    \label{fig:group_ablation}
\end{figure}

From this plot, we observe a general pattern: as HTF increases, the scatter points with the same color converge together and they overall move downwards. 
This means that \highlight{increasing HTF stabilizes performance}, and that 
variations in the LTF have a relatively minor effect when HTF is sufficiently high, e.g., HTF = 0.5.
A low HTF may hinder the recovery of frequent features (e.g., HTF = 0.05 results in higher reconstruction loss). Thus, we recommend a higher HTF, ideally within the range of $0.1 \sim 0.5$. This is far less restrictive than tuning TopK’s $K \in [50,600]$ across 66k neurons. In fact, 0.5 represents the upper limit for HTF since, without sparsification (i.e., setting $b_m=0$), the average activation frequency would approximate a coin flip (yielding roughly equal proportions of negative and positive pre-activations).
Moreover, we observe that the scatter points roughly align with the curve of TopK, especially for $K = 10$ and $K = 20$. This validates that \highlight{with adequate grouping, the GBA method achieves performance comparable to TopK}. We also observe that scatter plots with different colors converge together as HTF increases. This indicates that the \highlight{performance of GBA is also insensitive to the choice of $K$}, which matches what we observe in \Cref{fig:group_ablation} (middle \& right) that both the reconstruction loss and the fraction of activated neurons stabilize when $K$ exceeds 10.

These ablation studies confirm that our GBA method is nearly tuning-free. We only need a sufficiently high HTF (e.g., 0.5) to capture the maximum feature activation frequency,  a modestly low LTF to establish a lower bound and a sufficiently large number of groups. 
We summarize the key findings as follows.

\begin{enumerate}[leftmargin=2em]
\item [(4)] 
When HTF and LTF are properly chosen (e.g., a high HTF and a modestly low LTF), with an adequate number of groups, the GBA method achieves performance comparable to TopK, and the performance becomes largely \highlight{insensitive to the specific choices of these parameters}. 
\end{enumerate}


\paragraph{Consistency of recovered features} 
Furthermore, we answer {\color{red!70!black}Q3} by assessing the consistency of the learned features across independent runs with different random seeds. 
Since the ground truth features are unavailable, consistency serves as a proxy for the reliability of the training method.
We measure the consistency of features using the Maximum Cosine Similarity (MCS), which is defined in \Cref{app:evaluation_metrics} mathematically. 
A neuron from one run is considered to have an MCS of at least $\tau$ if a similar neuron (cosine similarity $\ge \tau$) can be found in every other run.
If a large percentage of neurons exceeds this MCS threshold, the features learned by the SAE are consistent across different runs. Such a percentage is referred to as the \emph{neuron percentage}, which is defined in \eqref{eq:def_neuron_percentage}. 

To avoid the influence of rarely activated neurons, we compute the neuron percentage only on subsets of the total $M$ neurons. 
In particular, we restrict to the top-$\alpha$ proportion of neurons, based on the \emph{maximum activation} or \emph{neuron Z-score} computed across the validation set. 
These two metrics are defined in \eqref{eq:maximum_activation} and \eqref{eq:Z_max}, respectively. 
In particular, the maximum activation of a neuron $m$ is defined as the maximum value of the pre-activation of the neuron across the validation set. 
The Z-score of a neuron $m$ is defined as the largest Z-score of the post-activation values of neuron $m$ across the validation dataset. 
A higher Z-score indicates that the neuron has a higher maximum activation relative to its mean and variance over the whole set of validation tokens, suggesting that it is more specific to a particular subset of input tokens.

%
\begin{figure}[htbp]
    \centering
    \includegraphics[width=0.9\linewidth]{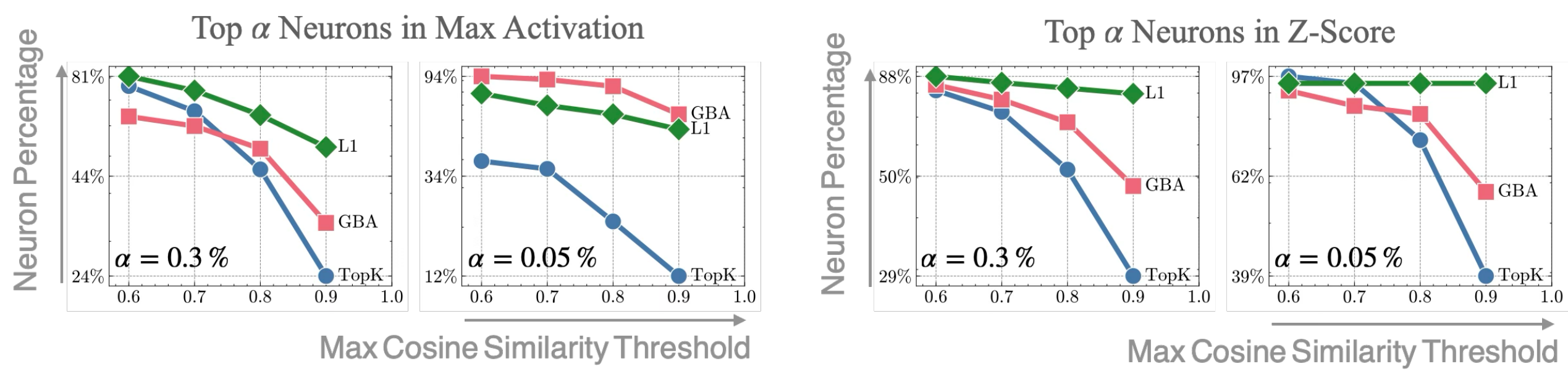}
    \caption{\small 
    Percentage of neurons attaining a Maximum Cosine Similarity (MCS) above a specified threshold for \texttt{Github}-Layer 26. The MCS is computed across three independent runs with distinct random seed initializations.
    A neuron is considered to have an MCS exceeding a threshold if its pairwise MCS with neurons from \emph{every} other run surpasses that threshold.
    A higher percentage of neurons exceeding a given MCS threshold indicates more consistent feature recovery across different runs.
    The percentage is computed based on top-$\alpha$ neurons selected by either maximum activation or Z-score, where $\alpha$ is chosen from \{$0.3\%, 0.05\%$\}. A higher curve indicates better consistency.}
    \label{fig:consistency}
\end{figure}

In \Cref{fig:consistency}, we plot the percentage of neurons with MCS above a threshold $\tau$, restricted to the top-$\alpha$ neurons in terms of either maximum activation or Z-score. 
Here we set $\alpha = 0.3\%$ and $0.05\%$, respectively. 
We plot the results for the SAEs trained with methods such as GBA, TopK, and L1 on the MLP output of layer 26 of the \texttt{Github} dataset.
The curves are generated by varying the threshold $\tau$ from $0.6$ to $0.9$. 

As shown in \Cref{fig:consistency}, we observe that the TopK method has the lowest MCS overall, and is  thus seed-dependent. 
GBA clearly outperforms TopK in terms of consistency, achieving a higher percentage of neurons with MCS above a given threshold. 
The L1 method is more consistent than TopK uniformly, and also more consistent than GBA in three of the four cases. However, when focusing on the top-$0.05\%$ neurons selected by the maximum activation criterion, GBA surpasses L1 in terms of consistency.
We summarize the key findings as follows, which answer {\color{red!70!black}Q3}:

\begin{enumerate}[leftmargin=2em,
    ]
\item [(5)]As the TopK method is shown to be seed-dependent \citep{paulo2025sparse}, it has the lowest MCS overall.
Our \highlight{GBA method outperforms TopK} in achieving a higher percentage of neurons with high MCS.
\item [(6)]The L1 method has been shown to be more consistent than TopK \citep{paulo2025sparse} uniformly and also more consistent than GBA in three of the four cases.
However, when focusing on neurons with the top-$0.05\%$ activations, our GBA method surpasses the L1 method.
\end{enumerate}


\paragraph{Additional results}
We further provide additional studies on the neurons learned by the GBA and TopK methods in terms of the three metrics used above: maximum activation, Z-score, and maximum cosine similarity across different runs with different random seeds. 
These metrics are computed based on the validation part of \texttt{Github} dataset, with the hook position at the MLP output of layer 26. 
For the Z-score, we compute the largest value among the tokens in the validation set, and for the maximum cosine similarity, we compute the smaller value among the two additional runs. See \Cref{app:evaluation_metrics}  for rigorous definitions of these metrics. 
In addition, for each neuron $m$, we also compute the \emph{activation fraction} (or activation rate), which is defined as the fraction of tokens where pre-activations of neuron $m$ are non-negative. 

Thus, for each neuron $m$, we have four metrics: maximum activation, Z-score, maximum cosine similarity, and activation fraction. We generate scatter plots by plotting the Z-score against the other three metrics. The results for GBA and TopK are presented in \Cref{fig:GBA_1B_metrics} and \Cref{fig:topk_GBA_1B_metrics}, respectively.


\paragraph{Z-score v.s. maximum activation}  In \Cref{fig:GBA_1B_metrics} (left), we present the scatter plot of the Z-score versus the maximum activation of neurons, which is shown in the logarithmic scale with base $10$. We observe an \highlight{almost linear relationship} between the two metrics, indicating that neurons with higher Z-scores also exhibit higher maximum activations. 
Notably, at the upper end of the distribution, a subset of neurons attains even higher Z-scores. 
This behavior suggests that these neurons capture a ``cleaner'' feature and fire exclusively when the feature is present.
By the definition of the Z-score, for neurons with the same maximum activation, a higher Z-score implies a lower variance. In other words, these neurons' activations tend to be bimodal---predominantly near a baseline when the feature is absent and significantly higher when the feature is present. This is consistent with the dashboard results for individual neurons as we shown in \Cref{fig:feature_dashboard_4688}.
\begin{figure}[h]
  \centering
  \begin{subfigure}[b]{0.325\textwidth}
    \centering
    \includegraphics[width=\textwidth]{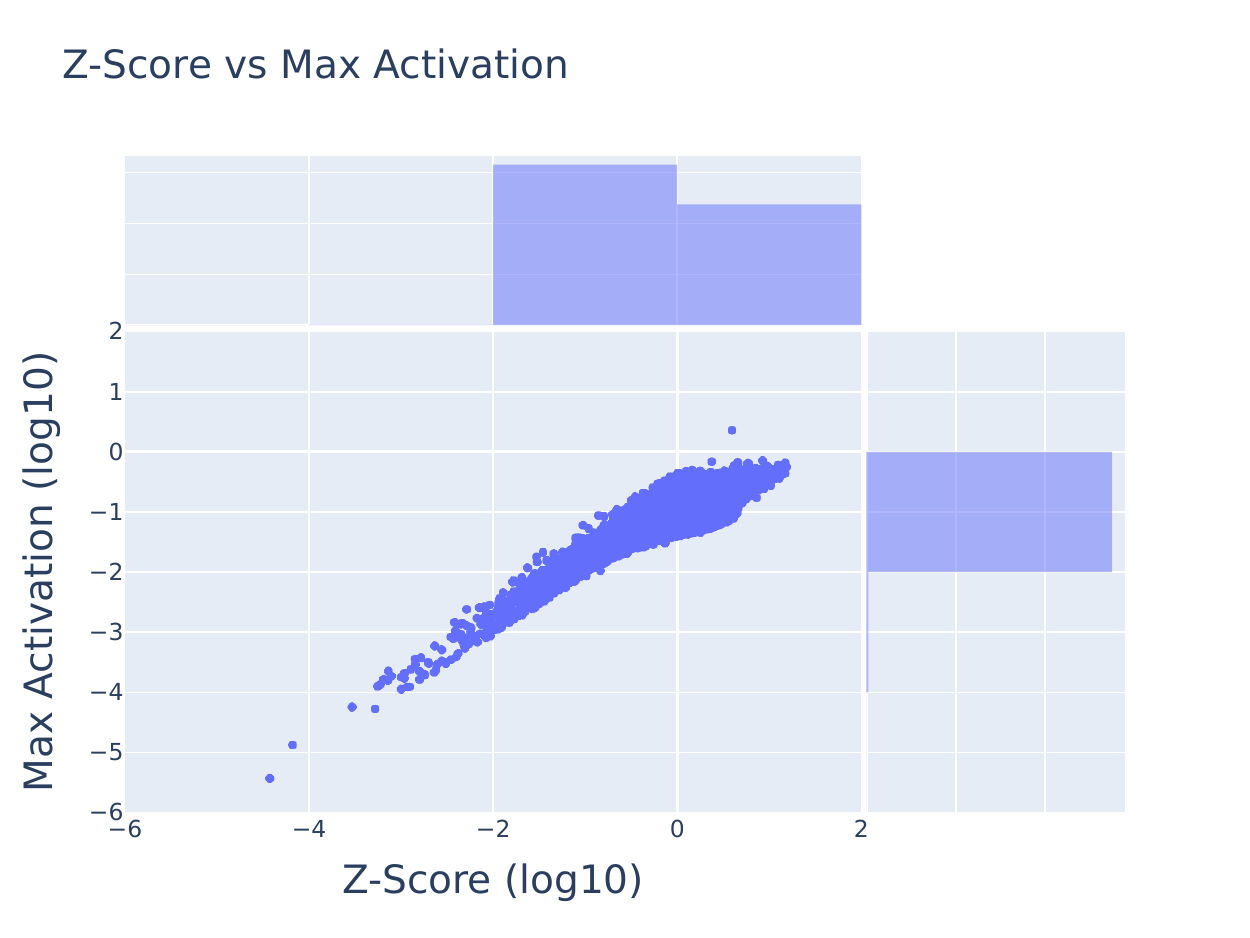}
  \end{subfigure}
  \begin{subfigure}[b]{0.325\textwidth}
    \centering
    \includegraphics[width=\textwidth]{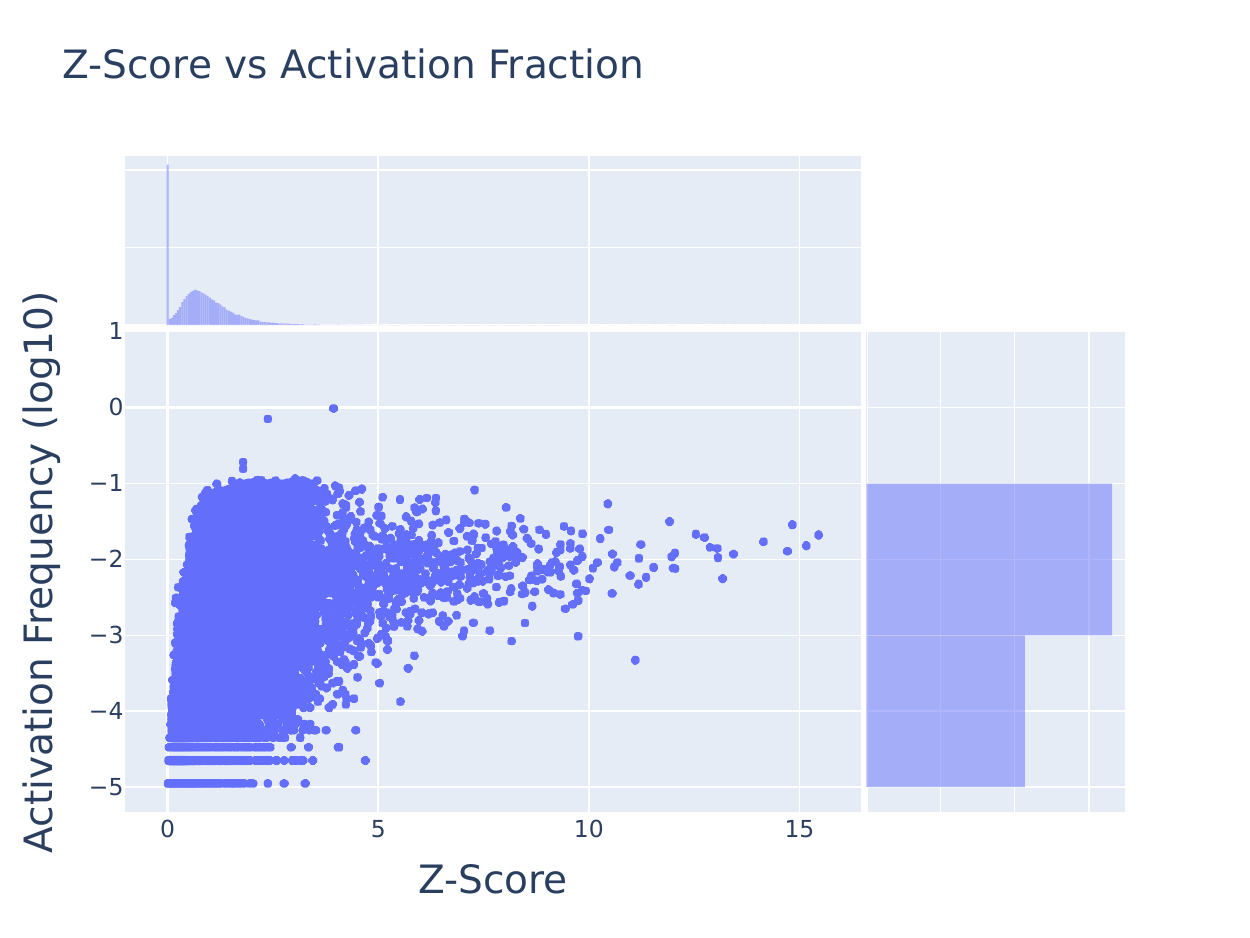}
  \end{subfigure}
  \begin{subfigure}[b]{0.325\textwidth}
    \centering
    \includegraphics[width=\textwidth]{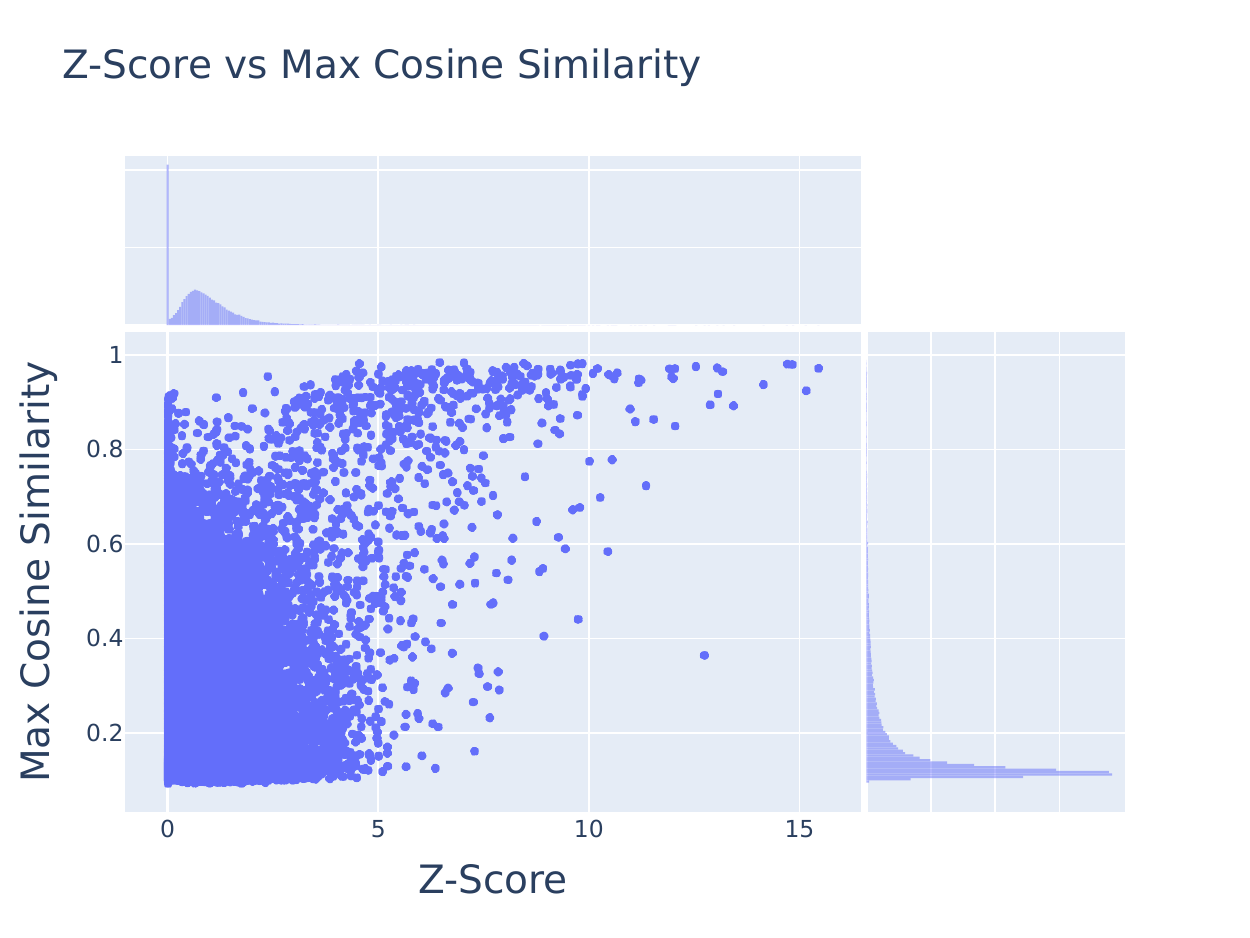}
  \end{subfigure}
  \caption{\small Scatter plots illustrating neuron properties for the GBA method: Z-score versus Maximum Activation, Fraction of Non-negative Pre-Activations (i.e., activation frequency), and Maximum Cosine Similarity across different runs with different random seed. The 66k-neuron SAE is trained on the \texttt{GitHub} dataset with a hook at the MLP output of layer 26. }
  \label{fig:GBA_1B_metrics}
\end{figure}

\paragraph{Z-score v.s. activation fraction} 
In \Cref{fig:GBA_1B_metrics} (middle), we present a scatter plot of the Z-score versus the activation fraction, which is shown in the logarithmic scale with base $10$. \highlight{Neurons with \emph{higher Z-scores} generally exhibit an activation fraction near 0.01} (around $-2$ in the figure). 
This suggests that they predominantly capture infrequent yet salient features. Moreover, the neuron grouping mechanism effectively adapts to diverse feature occurrence frequencies, underscoring the adaptivity of our approach.
Additionally, we observe several neurons with activation frequencies exceeding the HTF of 0.1 (by our default configurations). This behavior is facilitated by the bias-clamping mechanism, which prevents biases from becoming excessively negative, as discussed in \Cref{sec:algorithm}.

\paragraph{Z-score v.s. MCS}  In \Cref{fig:GBA_1B_metrics} (right), we present a scatter plot of the Z-score versus maximum cosine similarity across different runs with distinct random seeds. Recall that a higher maximum cosine similarity indicates more consistent feature recovery, and we observe that \highlight{neurons with higher Z-scores tend to exhibit higher levels of consistency}. 
This result supports the effectiveness of GBA in reliably extracting salient features.

\begin{figure}[h]
  \centering
  \begin{subfigure}[b]{0.325\textwidth}
    \centering
    \includegraphics[width=\textwidth]{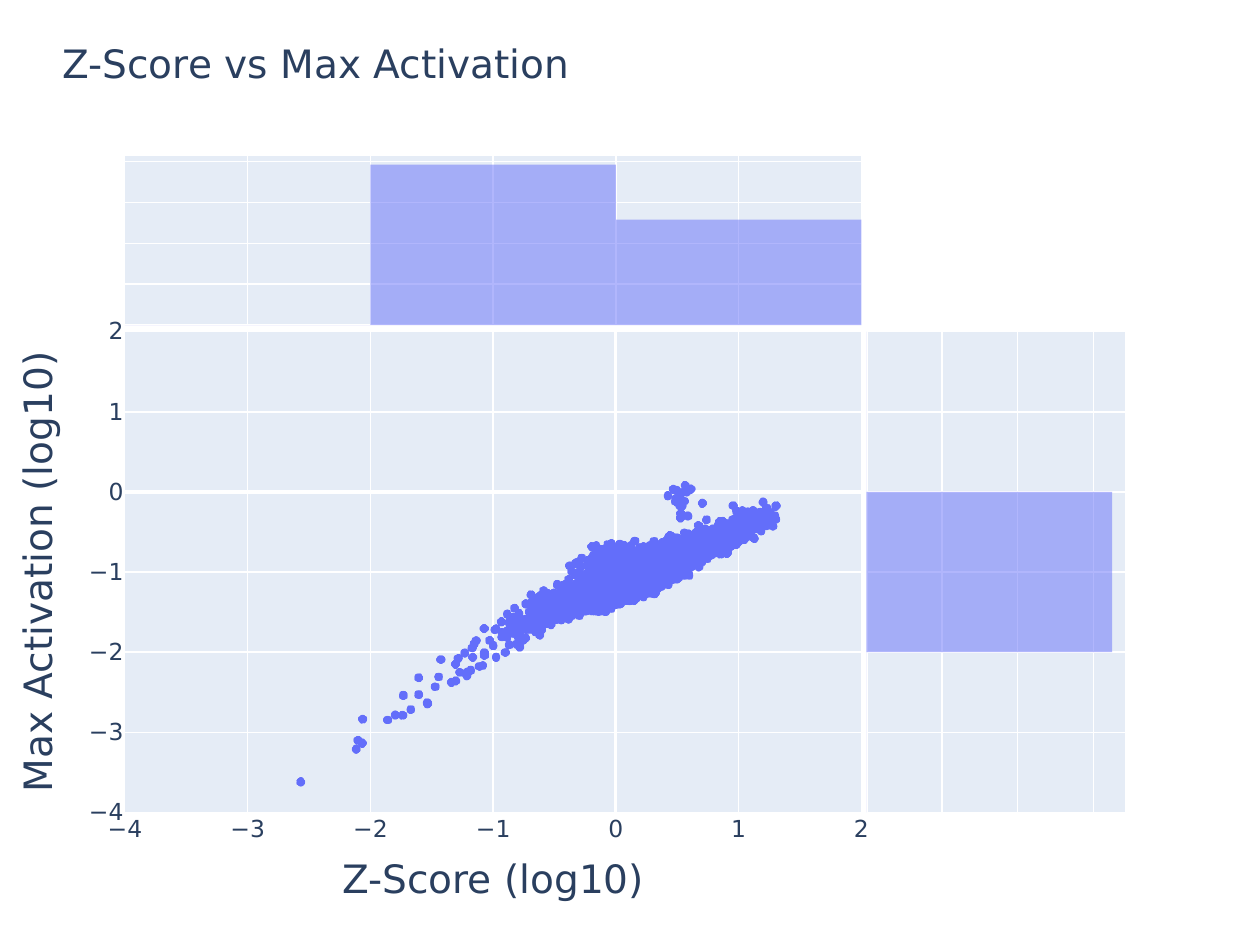}
  \end{subfigure}
  \begin{subfigure}[b]{0.325\textwidth}
    \centering
    \includegraphics[width=\textwidth]{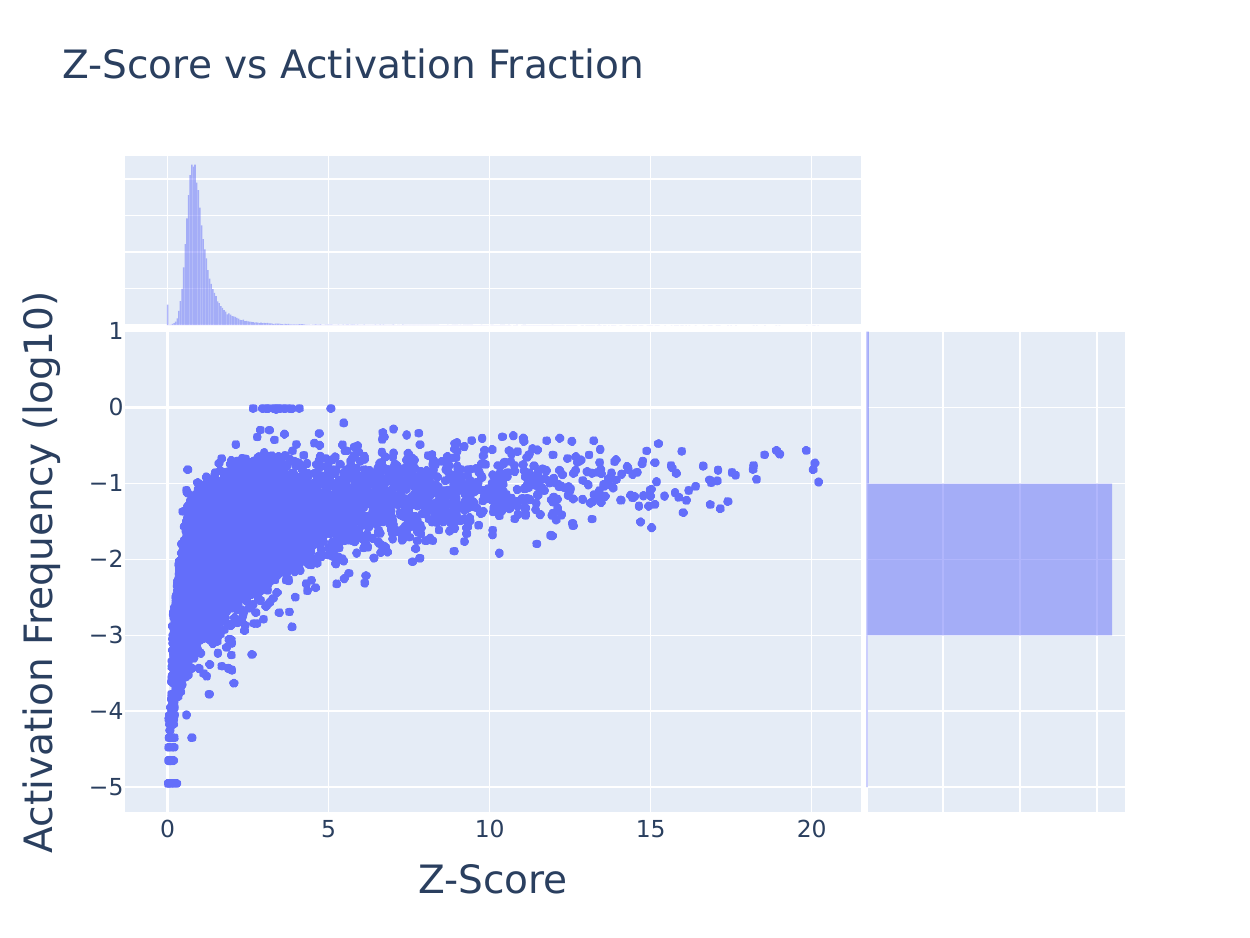}
  \end{subfigure}
  \begin{subfigure}[b]{0.325\textwidth}
    \centering
    \includegraphics[width=\textwidth]{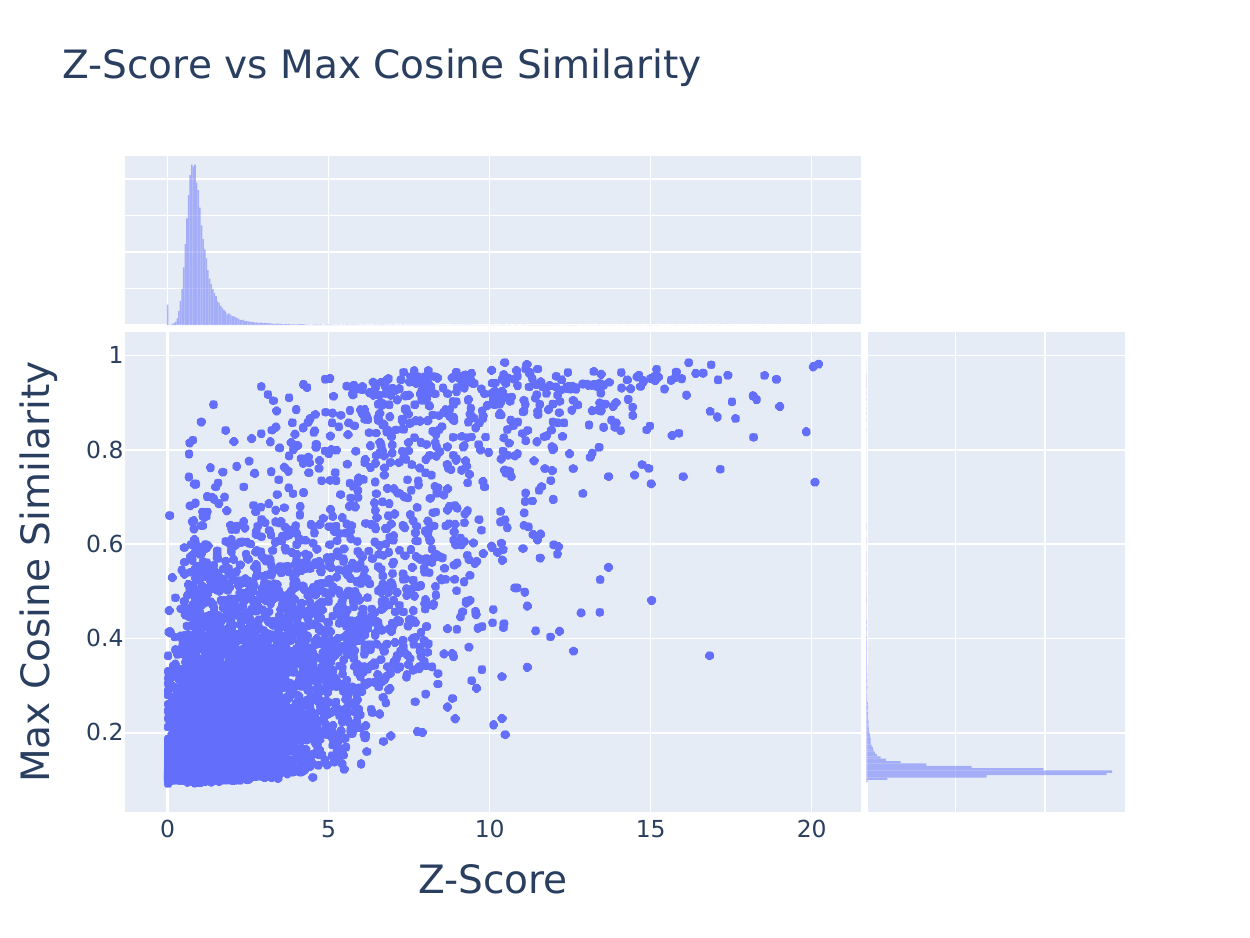}
  \end{subfigure}
  \caption{\small Scatter plots illustrating neuron properties trained using TopK with $K=300$. The other configurations are the same as \Cref{fig:GBA_1B_metrics}}
  \label{fig:topk_GBA_1B_metrics}
\end{figure}

We further analyze the neurons obtained using the TopK method with $K=300$, chosen because its reconstruction loss is comparable to that of the GBA method. In particular, as shown in \Cref{fig:topk_GBA_1B_metrics} (middle panel), neurons from TopK are generally less sparse than those from GBA. Moreover, the proportion of neurons achieving a maximum cosine similarity above 0.9 is lower for TopK, echoing the consistency results presented in \Cref{fig:consistency}.

\begin{figure}[htbp]
  \centering
  \includegraphics[width=0.9\textwidth]{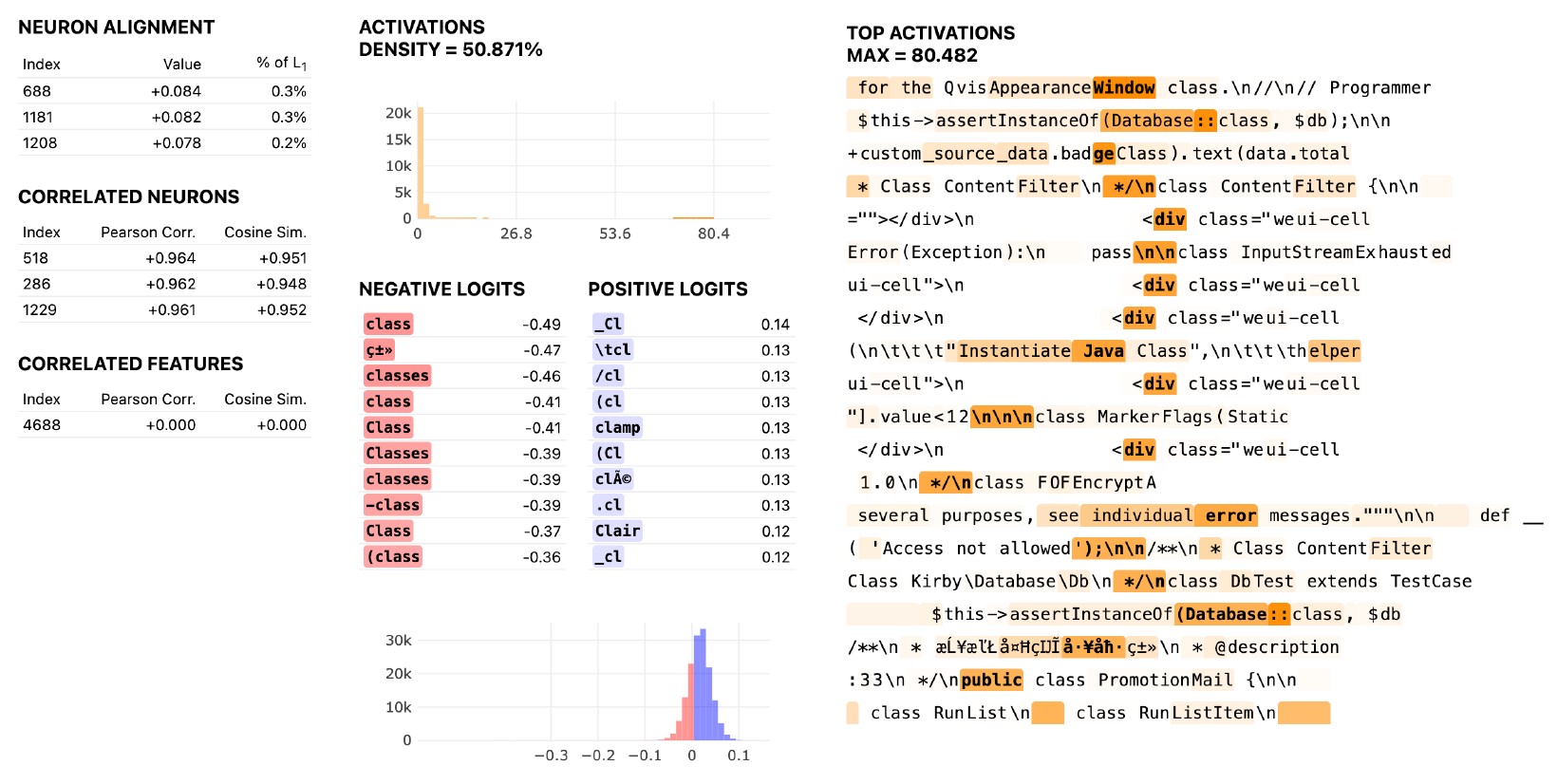}
  \caption{\small Feature dashboard for neuron 4688 in the GBA-SAE model trained on Pile Github at layer 26's MLP output position. 
  This neuron exhibits a clear bimodal activation pattern, and is activated before outputting the ``class'' token.}
  \label{fig:feature_dashboard_4688}
\end{figure}


    
    

%% file: paper/sections/identifiability.tex


In addition to empirical results, we establish theoretical guarantees for the proposed algorithm under the statistical model in \eqref{eq:data_model}. In particular, as we will show in the next section, the single-group variant of \Cref{alg:GBA}, i.e., the Bias Adaptation algorithm,  can successfully recover all monosemantic features in $V$, thus providing a positive answer to the question raised in \Cref{sec:preliminaries}.

Before establishing the feature recovery results, an immediate question is whether the true features are statistically identifiable under the statistical model in \eqref{eq:data_model}, especially under the superposition regime where the number of features \(n\) is larger than the embedding dimension \(d\).
Without answering this question, it would be vacuous to claim that the desired features can be learned from the data.
We examine the identifiability of the monosemantic features $V$ as follows. 

\subsection{Main Results on Identifiability of Features}
Below, we discuss three intrinsic ambiguities in the representation of features:
\begin{enumerate}[
    leftmargin=2em, 
    label=\textbullet,
]
    \item \textbf{Feature permutation:} The order in which features appear is arbitrary. For example, if we swap two rows of the feature matrix $V$ and adjust the corresponding entries in the coefficient matrix $H$ accordingly, the overall product $HV$ remains unchanged. This demonstrates that features can be permuted without affecting the data representation, which captures the \emph{homogeneity} of features.
    \item \textbf{Feature scaling:} Individual features can be scaled arbitrarily while compensating with inverse scaling in the coefficients. That is, if one multiplies a feature vector by a constant and divides the corresponding coefficient by the same constant, the resulting product still represents the same underlying concept. This rescaling ambiguity means that only the direction of the feature vector (and not its magnitude) is uniquely determined.
    \item \textbf{Linear combination ambiguity:} A given feature may be equivalently represented as a positive linear combination of other features. Under this scenario, different decompositions of the data are possible that yield the same aggregated representation. This ambiguity reflects the possibility that one feature could be fragmented into several components, each capturing part of the original concept, complicating the identification of a minimal and unique set of features.
\end{enumerate}
In light of these ambiguities, we introduce the notion of \emph{\highlight{$\varepsilon$-identifiability}} of features, which is a relaxed version of strict identifiability that allows for small perturbations in the feature directions. 
\begin{definition}[$\varepsilon$-identifiability]
    \label{def:identifiability}
    Let $\cG$ be a class of pairs of matrices $(H', V')$ where $H' \in \RR^{N \times n'}$ has nonnegative entries and $V' \in \RR^{n' \times d}$ where hidden dimension $n'$ is an arbitrary positive integer.
    A pair $(H, V) \in \cG$ with hidden dimension $n$ is said to be $\varepsilon$-identifiable within $\cG$ if for any other pair $(H^\prime, V^\prime) \in \cG$ with hidden dimension $n'$ (not necessarily the same as $n$) satisfying $H V =  H^\prime V^\prime$, there exists a permutation matrix $Q\in\RR^{n'\times n'}$ and a nonnegative block-diagonal matrix $\Omega=\diag(\omega_1, \ldots, \omega_n)\in \RR^{n \times n^\prime}$ such that $\norm{\vone - \cos(V,  \Omega  Q V^\prime)}_\infty \le \varepsilon$. Here, each $\omega_i$ is a nonnegative row vector and the cosine similarity is defined as in \eqref{eq:cosine_similarity}. 
\end{definition}

\begin{wrapfigure}{l}{0.15\linewidth}
    \centering
    \includegraphics[width=\linewidth]{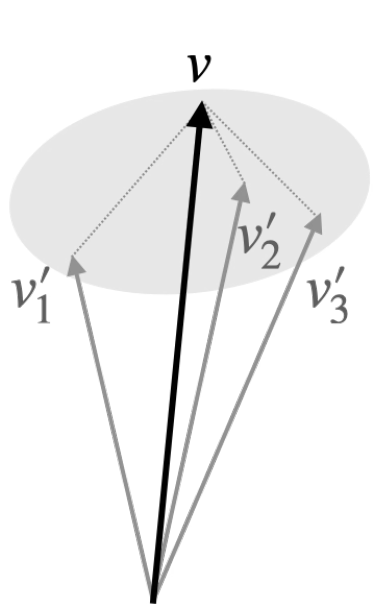}
    \caption{\small Illustration of the feature splitting phenomenon.}
    \label{fig:feature_splitting}
\end{wrapfigure}
The notion of $\varepsilon$-identifiability implies that if a feature matrix $V$ is $\varepsilon$-identifiable with respect to a coefficient matrix $H$, then for any alternative factorization $(H', V')$ of the dataset $X$, the matrix $V'$ must either be equivalent to $V$ (up to permutation and rescaling) or represent a refined splitting of each feature in $V$, with deviations bounded by $\varepsilon$.
To see this formally, we note that if the feature matrix $V$ is $\varepsilon$-identifiable, then for any other feature matrix $V'$ that factorizes the dataset $X$, there exists a permutation matrix $P$ and a block-diagonal matrix $\Omega$ such that $V\approx \Omega Q V'$ (up to scaling). The block-diagonal matrix $\Omega$ naturally partitions the \(n'\) rows of \(Q V'\) into \(n\) disjoint groups.  
Each group $i$ is then linearly combined by $\omega_i$ to form a single feature in $V$---this is the \emph{feature splitting} phenomenon, which was previously observed in \cite{bricken2023monosemanticity} and is particularly pronounced when the neuron size is large. 
In addition, in the special case where $n = n'$, $Q$ is a square permutation matrix and $\Omega$ is a diagonal scaling matrix, and thus $V$ and $V'$ are essentially identical up to permutation and rescaling.
Therefore,  our definition of feature identifiability captures all the three ambiguities discussed above.


Since $n'\ge n$ always holds (due to the definition of the block-diagonal matrix $\Omega$), the identifiable feature matrix $V$ is \highlight{\emph{minimal}} without any redundant features. 
Moreover, $V$ is also \highlight{\emph{unique}} in the sense that any alternative feature matrix $V'$ with the same minimal set of features is equivalent to $V$ up to permutation and rescaling.
Since there is inherent ambiguity in the scaling of features, our focus is solely on recovering the direction of these features.
Next, we detail the conditions on the data $X$ that guarantee the identifiability of features under our framework.

\begin{definition}[Decomposable Data] \label{asp:data_decomp_main} \label{assump:H_main}
    We say that the data matrix $X \in \mathbb{R}^{N \times d}$ is \highlight{\emph{decomposable}} if there exists a positive integer $n \in \mathbb{N}$, a \textbf{nonnegative} matrix $H \in \mathbb{R}_+^{N \times n}$ and a feature matrix $V \in \mathbb{R}^{n \times d}$ such that
    $X = H V.$ 
    Moreover, each row of $H$ has unit $\ell_2$ norm  and the $\ell_2$ norm of each row of $V$ is $\Theta(\sqrt{d})$. 
    Furthermore, 
    the coefficient matrix $H\in\RR^{N\times n}$ satisfies the following three conditions:
    \begin{enumerate}[
        leftmargin=3em,
        label=\textbf{(H\arabic*)}, 
        ref=(H\arabic*)
        ]
        \item \textbf{Row-wise sparsity:} 
        $\max_{\ell\in[N]} \|H_{\ell,:}\|_0 = s$ with $s=\Theta(1)$. \label{cond:H-sparsity}
        \item \textbf{Non-degeneracy:} For every $i \in [n]$, 
        $
        {\|H_{:,i}\|_1}/{\|H_{:,i}\|_0} = \Theta(1)$. \label{cond:H-nondeg}
        \item \textbf{Low co-occurrence:} 
        \( \rho_2\defeq \max_{i\neq j}{\langle \ind\{H_{:,i} \neq 0\},\, \ind\{H_{:,j}\neq 0\} \rangle}/{\|H_{:,i}\|_0} \ll n^{-1/2}\).\label{cond:H-cooccur}
    \end{enumerate}
    In addition, we further assume that the feature matrix $V \in \mathbb{R}^{n\times d}$ satisfies:
    \vspace{-5pt}
    \begin{enumerate}[
        label=\textbf{(V\arabic*)}, 
        leftmargin=3em,
        ref=(V\arabic*)
        ]
        \item \textbf{Incoherence:}
        For all $i\neq j$,
        \(
        {|\langle v_i, v_j \rangle|}/({\|v_i\|_2 \, \|v_j\|_2}) = o(1).
        \)\label{cond:V-incoherence}
    \end{enumerate}
    \end{definition}
    The \textit{nonnegativity} condition on $H$ is natural and removes the ambiguity regarding the sign of the features, as the opposite direction of a feature would often lead to a totally different or even contradictory concept.
    The \textit{Row-wise Sparsity} \labelcref{cond:H-sparsity} assumption ensures each data point cannot contain more than $s$ features, and it is essential for sparse recovery.
    The \textit{Non-degeneracy} \labelcref{cond:H-nondeg} condition ensures that when a feature is present in a data point, its average magnitude is sufficiently large. 
    When either of these conditions is violated, accurately isolating individual features from the data becomes significantly more challenging, and even impossible.
    Finally, the \textit{Low Co-occurrence} \labelcref{cond:H-cooccur} and \textit{Incoherence} \labelcref{cond:V-incoherence} assumptions ensure that two different features are distinct and well separated either in their occurrence or in the feature directions. 
    We provide more discussions on \labelcref{cond:H-cooccur} in \Cref{sec:discussion_cooccurence}.
    One thing to note is that the \textit{Incoherence} \labelcref{cond:V-incoherence} condition can be viewed as a generalization of the \textit{orthogonality} condition to almost orthogonal features, which is a common structural assumption in the sparse recovery literature \citep{marques2018review, candes2009near}.

With the necessary definitions and conditions established, we now present the main theorem that establishes the identifiability of the features.

\begin{theorem} \label{thm:identifiability}
    For a decomposable dataset 
        $X \in \RR^{N \times d}$,
    define 
        $\mathcal{G}$
    to be the class of pairs $(H',V')$ that satisfy conditions in \Cref{asp:data_decomp_main}
    with arbitrary hidden dimension $n'$. 
    Then there exists a pair $(H,V)$ such that $V$ is $\varepsilon$-identifiable within $\mathcal{G}$ with 
       $\varepsilon = o(1)$.

\end{theorem}

See \Cref{sec:proof_identifiability} for a proof. The key step is to show that under Conditions \labelcref{cond:H-sparsity,cond:H-nondeg,cond:H-cooccur}, $H$ is approximately pseudo-invertible. 
Hence, we can construct a suitable linear transformation $A$ such that $A X = (A H) \cdot V \approx V$ for identifying the features. 
In fact, to satisfy all the conditions in \Cref{def:identifiability}, the number of data $N$ must be no less than the number of features $n$ (i.e., $N \ge n$). Otherwise, the data matrix $X$ would not have enough information to recover the features.

\paragraph{Related work on identifiability}
We compare our identifiability results with those in the literature of sparse dictionary learning. 
We find the following salient differences:
\begin{enumerate}[
    leftmargin=2em, 
    label=\textbullet,
]
    \item \textbf{Model-free:} Our identifiability result does not require any probabilistic structure on either the coefficient matrix $H$ or the feature matrix $V$. In contrast, the work of \citet{spielman2012exact} assumes the coefficient matrix $H$ is drawn from a Bernoulli-Gaussian/Rademacher model, while \citet{schnass2014identifiability} assume $V$ to come from a unit norm tight frame. A variety of probabilistic assumptions can also be found in \citet{cohen2019identifiability, gribonval2015sparse}.
    \item \textbf{Unknown number of features:} Our identifiability result covers the case where the number of features \(n\) is unknown, and we are essentially identifying the \emph{minimal} set of features. All the works mentioned above assume the number of features \(n\) is known and fixed for identifying the set of features only up to permutation and scaling.  Therefore, the previous results do not capture the interesting phenomenon of \emph{feature splitting} in the notion of identifiability, as discussed before.
    \item \textbf{Full spectrum:} Our identifiability result holds for the full spectrum in terms of the relative scale of the number of features \(n\) versus the dimension \(d\), and also covers the interesting superposition regime where $n>d$. In contrast, \citet{cohen2019identifiability} focus on the case $n=d$ where the feature matrix $V$ is a square matrix, while \citet{5484983} show the uniqueness of the factorization only when the feature matrix is overcomplete $n\ge d$.
\end{enumerate}
Given these differences, we believe our identifiability result still holds a significant novelty for better understanding the identifiability of features in the SAE framework, particularly when the number of features is also unknown.





\subsection{Dicussion on Feature Co-occurrence}\label{sec:discussion_cooccurence}
Large feature co-occurrence often happens under a mixture of concepts or when one concept is highly correlated with another.  
From the perspective of a data graph where the nodes are data points and each edge represents that two data points share at least one common feature (as shown in \Cref{fig:cliques_sparse_connect}), the condition $\rho_2 \ll 1/\sqrt n$ implies that each clique corresponding to a feature has significantly sparser connections to nodes in other cliques. Hence, all cliques are well separated from each other.

\begin{wrapfigure}{r}{0.2\linewidth}
    \centering
    \includegraphics[width=\linewidth]{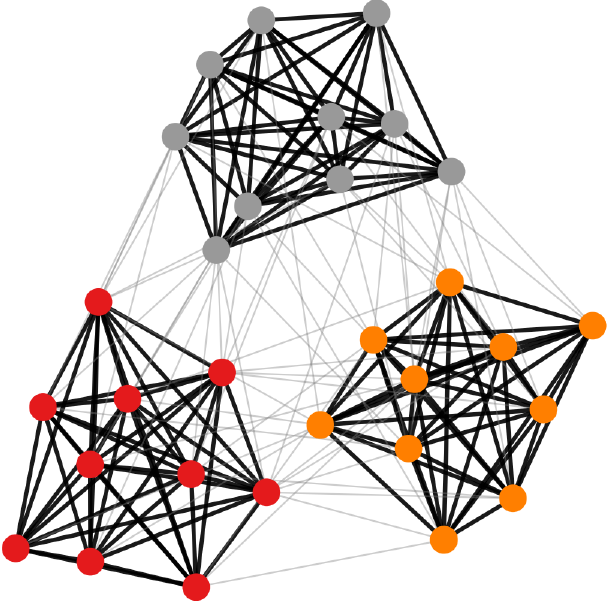}
    \caption{\small An illustration of data graph that contains cliques with sparse interconnections.}
    \label{fig:cliques_sparse_connect}
\end{wrapfigure}
Indeed, if one wants to run our SAE training algorithm to successfully recover all the features, $1/\sqrt{n}$ is also the critical threshold of co-occurrence, as we will show in \Cref{sec:main_theory_feature_recovery}.  
This fact demonstrates a fundamental connection between feature identifiability and SAE learnability. 
In the following, we conduct experiments using the proposed \Cref{alg:GBA} (with one group only) on datasets with Gaussian features and different levels of feature co-occurrence quantified by $\rho_2$. 
The algorithm for training the SAE is described in \Cref{sec:algorithm}.
\Cref{fig:learned_feats_vs_rho2} (left) shows that the Feature Recovery Rate (FRR) sharply declines when $\rho_2$ approaches $1/\sqrt{n}$, which is in line with our theoretical predictions. Furthermore, the histograms in \Cref{fig:learned_feats_vs_rho2} (middle \& right) demonstrate that a clear peak in the maximum cosine similarity between feature vectors and neuron weights in the SAE—indicative of successful feature learning—appears only when $\rho_2 < 1/\sqrt{n}$.

\begin{figure}[htbp]
    \centering
    \begin{subfigure}[t]{0.995\linewidth}
        \centering
        \includegraphics[width=\linewidth]{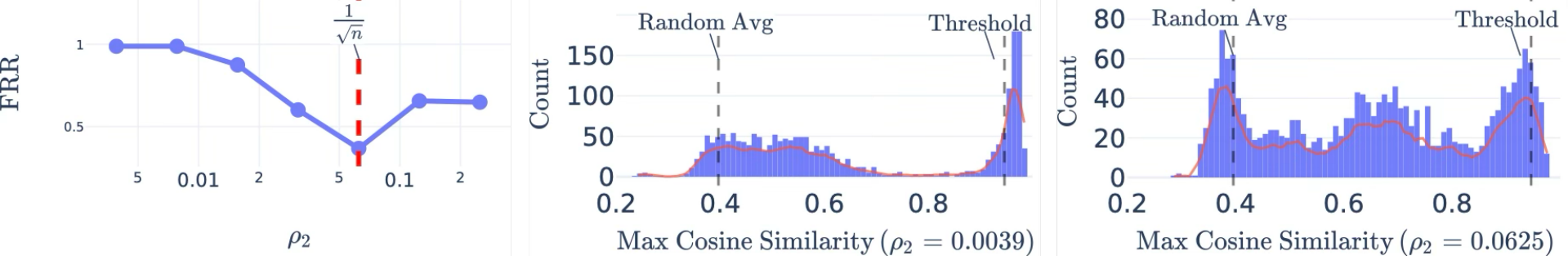}
    \end{subfigure}
    \hfill
    \caption{\small
        Illustration of the feature co-occurrence condition. (\textbf{Left}) Feature Recovery Rate (FRR) versus $\rho_2$. (\textbf{Middle \& Right}) Histograms of the maximum cosine similarity between neurons and features for $\rho_2 < 1/\sqrt n$ and $\rho_2 = 1/\sqrt n$. Here, we take $(n, d, M, s, p)=(256, 48, 2048, 3, 0.01)$ and feature learning threshold $0.946$. 
        Additional details on the experimental settings can be found in \Cref{sec:discuss-conditions}, and the feature learning threshold is detailed in \Cref{app:evaluation_metrics}.
        The Random Avg is the averaged max cosine similarity between features and randomly initialized neurons. 
    }
    \label{fig:learned_feats_vs_rho2}
\end{figure}

%% file: paper/sections/main_result_dynamics.tex
In this section, we aim to provide theoretical guarantees of feature recovery for the SAE trained with a simplified version of the GBA algorithm. 
We focus on the case where the data matrix $X$ admits a factorization $X=HV$ with identifiable monosemantic features $V$ (\Cref{assump:H_main}). 
On a high level, our goal is to prove that the SAE trained with our algorithm provably recovers all the monosemantic features in $V$ under proper conditions.
We will further examine the implications of these conditions from both theoretical and practical perspectives.


In the following, we first introduce a \emph{Modified Bias Adaptation (BA)} algorithm, which is a simplified version of the GBA algorithm with only one group of neurons and a fixed TAF. 
Then, we provide theoretical results on the training dynamics of Modified BA, which is accompanied by synthetic experiments to validate the theoretical findings.

\subsection{Simplification for Theoretical Analysis}\label{sec:simplified_algorithm}

We make several simplifications to the setup of SAE to facilitate theoretical analysis.

\paragraph{Decomposible data with Gaussian features} We assume that the data matrix $X\in\RR^{N\times n}$ is decomposable in the sense of \Cref{assump:H_main}. Moreover, we assume that the feature matrix $V\in\RR^{n\times d}$ has i.i.d. entries following $\cN(0,1)$. Such a choice of $V$ satisfies the 
incoherence condition \labelcref{cond:V-incoherence}.

\paragraph{Simplified SAE model} We consider a simplified version of the SAE model $f(x; \Theta)$ in \eqref{eq:sae}, where the only trainable parameters are the weights $\{ w_m\}_{m=1}^M$.  
\begin{enumerate}[leftmargin=2em, label=\textbullet]
    \item (\emph{Small output scale}) We assume that the output scale $a_m=a$ and $a$ is sufficiently small. When computing the gradient, we rescale the $\nabla \cL(\Theta)$ back to its original scale by  multiplying $a^{-1}$. 
    \item (\emph{Fixed pre-bias}) We fix the pre-bias $b_{\mathrm{pre}}=0$, as the data matrix $X$ is centered.
    \item (\emph{ReLU-like smooth activation}) We use a smooth, ReLU-like activation function $\phi$ (see \Cref{assump:activation} for details). One example is the softplus activation $\phi(x) = \log(1+\exp(x))$.  
    \item (\emph{Fixed bias}) 
    For each neuron $m \in [M]$, we fix the bias $b_m=b<0$ throughout training, where $b$ is a negative scalar whose value will be specified later.
\end{enumerate}

Besides, as shown in \Cref{assump:H_main}, each data point $x_{\ell}$ is a combination of at most $s$ monosemantic features, where $s = \Theta (1)$. As we will show below, we further assume that the occurrence of each monosemantic feature is relatively balanced. Thus, it is reasonable to try using a version of GBA algorithm with a single group to recover $V$. 
We introduce the algorithm as follows.

\paragraph{Modified Bias Adaptation (BA) algorithm} 
Recall that Bias Adaptation (BA) algorithm is a special case of GBA algorithm with only one group of neurons and a fixed TAF $p$. 
Here we determine the value of $p$ implicitly by choosing a fixed bias $b<0$, and they are related by $p = \Phi(-b)$, where $\Phi(\cdot)$ is the tail probability function for Gaussian distribution. That is, $\Phi(t) = \PP(Z\ge t)$ for $Z\sim\cN(0,1)$.

 Given the data matrix $X$ and the SAE model $f(x;\Theta)$, we can compute the loss function $\cL(\Theta)$ as in \eqref{eq:loss_sae_data} and its gradient with respect to the weights $\{w_m\}_{m=1}^M$. 
Since only the directions of the features $\{ v_i\}_{i=1}^n$ matter, we adopt spherical gradient descent to update the weights. That is, starting from the initial weights $\{w_m^{(0)}\}_{m\in[M]}$ uniformly sampled from the unit sphere $\mathbb{S}^{d-1}$, for any $t\geq 1$, in the $t$-th iteration, we update each $w_m^{(t-1)}$ by 
\begin{align}
    \textbf{Modified BA:}\quad w_m^{(t)} = \frac{w_m^{(t-1)}+\eta\,g_m^{(t)}}{\|w_m^{(t-1)}+\eta\,g_m^{(t)}\|_2}, \where g_m^{(t)} = \lim_{a\rightarrow 0} -a^{-1} \nabla_{w_m}\cL(\Theta^{(t-1)}).  \label{eq:gd_simplified}
\end{align}
Here, $g_m^{(t)}$ is the rescaled negative gradient of the loss function $\cL(\cdot)$ in \eqref{eq:loss_sae_data} with respect to the weight $w_m$ of neuron $m$ at iteration $t$. 
We will show that, under proper conditions, for any feature $v_i$, there exists at least one neuron $m_i \in [M]$ such that the alignment between $w_{m_i}^{(T)}$ and $v_i$ is arbitrarily close to one when $T$ is sufficiently large.

Before we proceed to the main theoretical results, we make several remarks on the above simplifications for theoretical analysis and their implications.




\paragraph{Fixed bias is without loss of generality} As we consider Gaussian features and always normalize $w_m^{(t)}$ to the unit sphere, it can be shown using the \highlight{Gaussian conditioning technique} that the pre-activations remain approximately Gaussian, i.e., 
$
y_m(x_l) = \langle w_m^{(t)},x_\ell\rangle+b\sim\mathcal{N}(b,1)
$
for a constant number of iterations $t$. See \Cref{sec:proof-sketch} for details. 
Therefore, to achieve the desired TAF $p$, it is without loss of generality to fix the bias $b<0$ such that $\Phi(-b)=p$, which means that the pre-activations  of each neuron will be non-negative for approximately $p$ fraction of the $N$ data points throughout the training.

\begin{wrapfigure}{r}{0.35\textwidth}
    \vspace{-10pt}
    \centering
    \includegraphics[width=\linewidth]{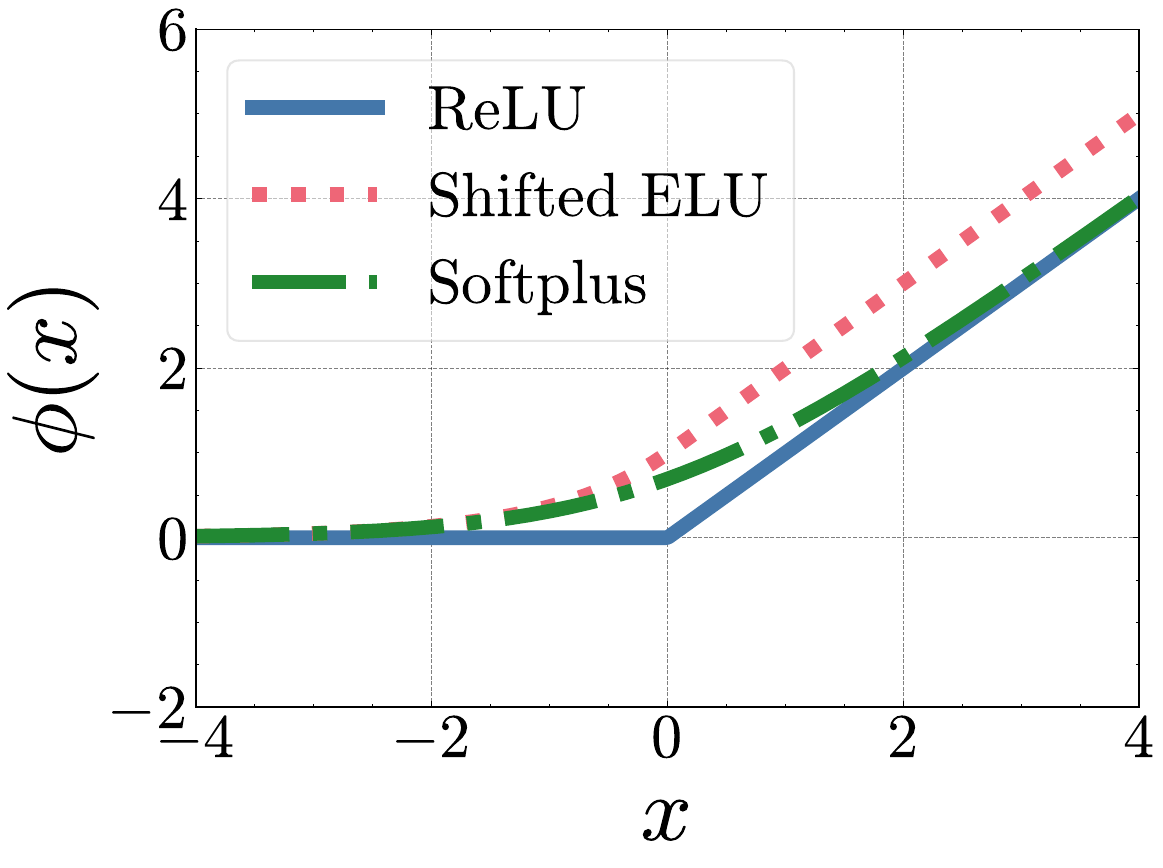}
    \caption{\small Smooth ReLU-like activations}
    \label{fig:activation_functions}
    \vspace{-10pt}
\end{wrapfigure}
\paragraph{Smooth ReLU-like activation approximates ReLU} 
We choose a smooth activation function for technical convenience.
These activations can be viewed as a smooth approximation to the ReLU function, as illustrated in \Cref{fig:activation_functions}.
This class of activations encompasses functions like Softplus and shifted ELU, and closely resembles the standard ReLU activation function. We believe that a more refined analysis can also be applied to the standard ReLU activation, but we leave this as future work.


\paragraph{Small output scale decouples neuron dynamics}
Following a common paradigm in the literature~(see e.g. \citet{lee2024neural,chen2025can}), we assume that the output scale of the SAE is sufficiently small. 
The benefit of this condition is that it \highlight{decouples the dynamics among the $M$ neurons}, making the analysis more tractable. 
Specifically, the rescaled negative gradient of the loss $\cL(\Theta)$ is given by
\begin{align}
    g_m = -a^{-1} \nabla_{w_m}\cL(\Theta)
    &= 
    \sum_{\ell=1}^N\big(\fcolorbox{white}{blue!5}{\(\varphi(w_m^\top x_\ell;b)x_\ell\)} - \fcolorbox{white}{red!5}{\(\psi_m(x_\ell;\Theta)\)}\big)
    \overset{a\to0}{=} 
    \sum_{\ell=1}^N \fcolorbox{white}{blue!5}{\(\varphi(w_m^\top x_\ell;b)x_\ell\)},
    \label{eq:gd_approx}
\end{align}
where we define $\varphi (\cdot, \cdot )$ and $\psi_m(\cdot; \Theta)$ as
\begin{gather}
\fcolorbox{white}{blue!5}{\(\varphi(u, v)=\phi(u+v)+\phi'(u+v)\cdot u\)} ,\\
\fcolorbox{white}{red!5}{\(\psi_m(x; \Theta)=\phi'(w_m^\top x+b)\cdot w_m^\top f(x;\Theta)\cdot x+\phi(w_m^\top x+b)\cdot f(x;\Theta)\)}.
\end{gather}
Here, $\varphi:\RR\mapsto\RR$ is a \emph{decoupled} term that depends only on each individual neuron's weight and bias, while $\psi_m: \RR^d\mapsto\RR^d$ is a \emph{coupling} term that captures the interaction between the neuron and the rest of the network. 
Since the scale of $f(x;\Theta)$ is proportional to $a$, this coupling term is negligible when $a$ is small.
As a result, when $a$ is infinitesimally small, each neuron $m$ evolves independently of the other neurons. 
Furthermore, thanks to the decoupled dynamics, the restriction to a single group with a fixed TAF $p$ does not result in any loss of generality, as the analysis of multiple groups is a straightforward extension.

\paragraph{Limitations: overlooking benefits of neuron correlation}
While having a small $a$ decouples the dynamics and thus simplifies the analysis, allowing neuron correlation can in fact be beneficial for feature learning.  
As we show in \Cref{fig:sigma_suite}, when setting $a = 1$ in simulation experiments, thus allowing neuron correlation, the SAE can successfully recover all features with a much smaller $M$ than the theoretical requirement. 
To see this benefit, we rewrite the gradient in \eqref{eq:gd_approx} as
\begin{align}
   g_m
    = \sum_{\ell=1}^N \Bigl( \phi(w_m^\top x_\ell + b_m) I_d + \phi'(w_m^\top x_\ell + b_m) \, x_\ell w_m^\top \Bigr)
    \cdot \bigl( \fcolorbox{white}{blue!5}{$x_\ell$} - \fcolorbox{white}{red!5}{$f(x_\ell;\Theta)$} \bigr). \label{eq:gradient_rewrite}
\end{align}
Once a feature $v$ is learned by the network, the term \fcolorbox{white}{red!5}{$f(x_\ell;\Theta)$}, containing the feature $v$
cancels out the contribution of  $v$ from \fcolorbox{white}{blue!5}{$x_\ell$}. 
That is, $x_{\ell} - f(x_\ell;\Theta)$ no longer contains the feature $v$.
As a result, under the incoherence condition \labelcref{cond:V-incoherence}, the correlation between $g_m $ and $v$ becomes negligible for each $m$, thus preventing any neuron from learning the same feature $v$ again and improving neuron utilization efficiency. We will revisit this point in \Cref{subsec:feature_balance} with simulation results. Our current theoretical analysis does not capture this benefit, which is left as an open question for future work to explore.

\subsection{Main Theorem on Training Dynamics}

\begin{wrapfigure}{r}{0.3\textwidth}
    \centering
    \includegraphics[width=0.95\linewidth]{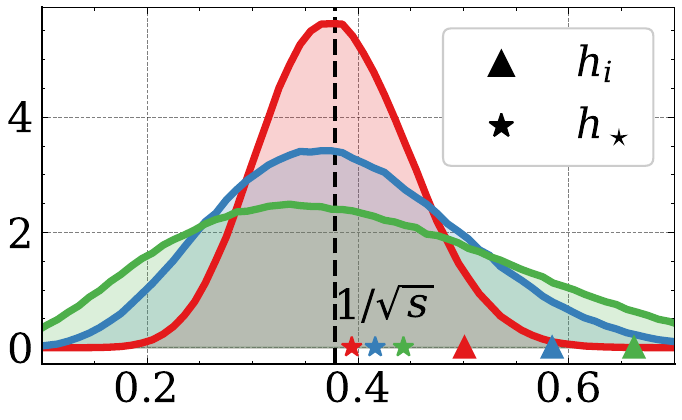}
    \caption{
        \small  Relationship between $s$, $h_\star$ and $h_i$ with different concentration level in $H$'s non-zero entries' empirical distribution (shadow). A less concentrated $H$ leads to larger $h_\star$ and $h_i$.}
    \label{fig:sigmoid_dist}
\end{wrapfigure}
Intuitively, to recover a feature $v$, it has to \highlight{appear in sufficiently many data points} with \highlight{sufficiently large coefficients}.  
To characterize this intuition, we introduce two key quantities based on the coefficient matrix $H$.
First,  for each feature index $i\in[n]$, let $\cD_i = \{l\in[N]: H_{l,i}\neq 0\}$ be the set of data indices that contain feature $v_i$. 
The occurrence of the feature $v_i$ is thus given by $|\cD_i|/N$. We define the \highlight{\emph{maximum feature occurrence}} as the largest occurrence among all features, i.e., 
\begin{align}
    \rho_1 = \max_{i\in[n]} \bigl\{ |\cD_i|/ N  \bigr\} = \max_{i\in[n]} \bigg\{  1/ N \cdot  \sum_{l \in [N]} \ind \{ H_{l, i} \neq 0 \}  \biggr\} . \label{eq:rho_1-def} 
\end{align} 
To ensure each feature $v_i$ appears in sufficiently many data points, we require that the occurrence of each feature is comparable to $\rho_1$, i.e., $|\cD_i|/(\rho_1 N)$ is not too small for each $i \in [n]$.

Second, to measure the magnitude of coefficients associated with each feature, we define the \highlight{\emph{cut-off}} level for the feature $i$ as
\begin{align}
    \!\!\! h_i := \max \Bigl\{h\le 1: 
    \frac{1}{|\cD_i|}\sum\nolimits_{l\in\cD_i} \ind\{ H_{l,i} \geq h\} \ge \polylog(n)^{-1} \Bigr\}. 
    \label{eq:s_i-def}
\end{align}
Intuitively, $h_i$ is a critical threshold such that, among all data points containing $v_i$, at least a $\polylog(n)^{-1}$ fraction of them have coefficients no smaller than $h_i$. 
In other words, $h_i$ reflects the magnitude of coefficients associated with feature $v_i$, within the subset of data points where $v_i$ is present.
Thus, $h_i$ can effectively be viewed as a notion of ``signal strength'' for feature $v_i$, and  we should require that $h_i$ is not too small for each $i \in [N]$.

Furthermore, we additionally introduce a global quantity called the \highlight{\emph{concentration coefficient}} $h_\star=h_\star(H)$, whose definition is technical and deferred to \eqref{eq:s_star-def} in the appendix. Intuitively, $h_\star$ characterizes the global concentration level of nonzero entries in $H$.
For now we can intuitively understand it as the variance of the nonzero entries in $H$, and thus $h_\star$ will increase when the nonzero entries in $H$ are less concentrated.


With these definitions, we are now ready to state the  main theorem on the training dynamics.

\begin{theorem}\label{thm:sae-dynamics}
Let $X=HV$ be decomposable in the sense of \Cref{asp:data_decomp_main} with $H\in\RR^{N\times n}$ satisfying all the conditions therein, 
and further assume that $V \in\RR^{n\times d}$ has i.i.d. entries following $\cN(0,1)$.
For this $X$, we train the SAE with \textbf{Modified BA} given  in \eqref{eq:gd_simplified}.
Let $\varsigma, \varepsilon\in(0, 1)$ be any small constants. 
We assume that the number of neurons $M$ is sufficiently large:
\begin{olivebox}
        \begin{equation}
        \begin{aligned}
            \textbf{Network Width:}
            \quad \frac{\log M}{\log n}\gtrsim \max_{i\in[n]} \bigg\{  \frac{b^2 } {2 (1-\varepsilon)^2 h_i^2\log n}  + 1 \biggr\}.  \label{eq:cond-width}
        \end{aligned}
    \end{equation}
    \end{olivebox}
\noindent Moreover,  we assume that the learning rate $\eta$ satisfies ${\log \eta} \gtrsim {(b^2/2 - \log N)}$ and that the bias $b<0$ is set to satisfy the following condition:
\begin{olivebox}
    \begin{equation}
    \begin{aligned}
        &\textbf{Bias Range:}\quad 1 \gtrsim \frac{b^2}{2\log n} \gtrsim \max\Bigl\{ \frac{1}{2} + \frac{h_\star^2}{2}, \:  2 (1+\varepsilon)^2 h_\star^2 , \: 1 - (1-\varsigma)\cdot \frac{\log d}{\log n} \Bigr\}. 
        \label{eq:cond-global}
    \end{aligned}
    \end{equation}
\end{olivebox}
\noindent Furthermore, we assume the coefficient matrix $H$ satisfies the following \emph{feature balance} condition: 
\begin{olivebox}
        \begin{equation}
        \begin{aligned}
            &\textbf{Feature Balance:}
            \quad \frac{|\cD_i|}{\rho_1 N} \ge \polylog(n)^{-1}, \quad h_i^2 \gg \frac{\log\log (n)}{\log(n)},  \qquad \forall i\in[n]. \label{eq:cond-individual}
        \end{aligned}
    \end{equation}
    \end{olivebox}
\noindent Then, it holds with probability at least $1-n^{-4\varepsilon}$ over the randomness of $V$ that 
for any feature $i\in[n]$, 
there exists at least one unique neuron $m_i$ such that after at most $T=\varsigma^{-1}$ iterations, the alignment between the weights of neuron $m_i$ and the feature vector $v_i$ satisfies
    \(
    {\langle w_{m_i}^{(T)}, v_i\rangle}/{\|v_i\|_2} \ge 1 - o(1).
    \)
    \end{theorem}
See \Cref{sec:sae-dynamics} for a detailed proof of this theorem. 
Theorem \ref{thm:sae-dynamics} shows that under appropriate conditions, \highlight{Modified BA provably recovers all monosemantic features within a constant number of iterations}. 
These conditions include that (i) the network is sufficiently wide compared to the number of features as specified in \eqref{eq:cond-width}, (ii) the bias $b$ is chosen within a certain range as specified in \eqref{eq:cond-global}, and (iii) the coefficient matrix $H$ satisfies the feature balance condition in \eqref{eq:cond-individual}, ensuring that each feature appears frequently enough with sufficiently large coefficients.
We revisit the theoretical and empirical implications of these conditions in \Cref{sec:discuss-conditions}.
To our best knowledge, this theorem is the first theoretical result that proves a SAE training algorithm can provably recover all monosemantic features.




\subsection{Key Conditions for Reliable Feature Recovery}\label{sec:discuss-conditions}

We examine the three key conditions in \Cref{thm:sae-dynamics} that ensure reliable feature recovery from both theoretical and empirical perspectives.

\paragraph{Simulation setup} 
We generate synthetic data  $X=HV$ satisfying the assumptions in \Cref{thm:sae-dynamics}. 
In the default setting, each row of $H$ contains exactly $s$ nonzero entries, each with value $1/\sqrt{s}$, and the support of each row is chosen independently at random.
We implement the BA algorithm with a fixed TAF $p$, where the SAE adopts the ReLU activation. 
We fix the output scale $a_m=1$ for all $m\in [M]$ and the pre-bias $b_{\mathrm{pre}}=0$, and initialize the weights $w_m^{(0)}$ uniformly on the unit sphere $\mathbb{S}^{d-1}$ with bias $b_m^{(0)}=0$. The details of the simulation setup are deferred to \Cref{sec:additional_experimental_details_feature_recovery}. 

To evaluate feature learning of neuron $m$, we use the Max Cosine Similarity (MCS) metric. For any neuron $m$, MCS is defined as $\max_{i\in[n]} |\langle w_m / \| w_m\|_2 ,v_i/\|v_i\|_2 \rangle| $.
Thus, MCS measures how well a neuron aligns with the most aligned feature in $V$. 
We say a neuron is \emph{aligned with some feature} if the MCS for that neuron exceeds a certain threshold. 
To evaluate overall feature recovery, we use the Feature Recovery Rate (FRR) metric, defined as the proportion of features that are aligned with at least one neuron. See \Cref{app:evaluation_metrics} for more details on these metrics and the choice of thresholds.


\subsubsection{Bias Range: Implications on Target Activation Frequency}\label{sec:cond-bias}
\paragraph{High and low superposition regimes} 
Recall from \Cref{sec:simplified_algorithm} that the pre-activation is approximately a  Gaussian, i.e., $\langle w_m^{(t)}, x_\ell \rangle + b \sim \cN(b, 1)$. 
As a result, the TAF $p$ should satisfy $p = \Phi(-b) \approx \exp(-b^2/2)$ by the Gaussian tail estimate. 
Combining this with the Bias Range condition  in \eqref{eq:cond-global}, we conclude  that this condition on $b$ effectively implies that 
 $p$ should satisfy
\begin{align}
    n^{-1} \ll p \ll \min \bigl\{ n^{-(1 + h_\star^2)/2}, d n^{-1}\bigr\}.
    \label{eq:cond-p-range}
\end{align}
Here the lower bound $n^{-1}$ captures the ideal activation frequency. To see this, note 
by the way $H$ is generated, each feature appears in roughly $s/n$ fraction of data points, and thus the optimal TAF should be $\Theta(n^{-1})$. 
This feasible range of TAF is visualized in \Cref{fig:freq_vs_d_theory} (Right) with different values of $h_\star$, where the yellow area represents the feasible region.
We discuss two regimes based on the relative scale between $d$ and $n$: 
\begin{enumerate}[
    leftmargin=2em, label=\textbullet
]
    \item (\textit{High superposition regime}). When $d$ is relatively small compared to $n$, the upper bound in \eqref{eq:cond-p-range} is attained at $d/n$, while the lower bound is $n^{-1}$. Thus, the feasible TAF is between $1/n$ and $d/n$, as shown in the left half of \Cref{fig:freq_vs_d_theory} (Right).
    We call this the \emph{high superposition regime} because a larger number of features are stored in a relatively low-dimensional space, leading to high superposition. Notably, in this regime, the upper bound for the feasible TAF grows linearly with $d$.
    \item (\textit{Low superposition regime}). When $d$ is relatively large compared to $n$, the upper bound in \eqref{eq:cond-p-range} is attained at $n^{-(1 + h_\star^2)/2}$, while the lower bound is still $n^{-1}$, as is shown in the right half of \Cref{fig:freq_vs_d_theory} (Right).
    Thus, the feasible TAF is between $1/n$ and $n^{-(1 + h_\star^2)/2}$. We call this the \emph{low superposition regime} because a smaller number of features are stored in a relatively high-dimensional space, leading to low or negligible superposition. Notably, in this regime, the feasible TAF is upper bounded by a flat line that does not depend on $d$.
\end{enumerate}
Moreover, in the extreme case where $h_\star$ is super small, the transition between the two regimes occurs at $d = \sqrt{n}$ by equating the two terms in the upper bound of \eqref{eq:cond-p-range}.
Therefore, we can roughly view $d = \sqrt{n}$ as the transition point between the two regimes. 
Also, the upper bound for the feasible TAF will approach $n^{-1/2}$ as $h_\star$ approaches zero in the low superposition regime.

\paragraph{Empirical results} Our empirical results corroborate the above theoretical insights. 
We plot the heatmap of FRR under various values of $p$ and $d$ in the left two plots of \Cref{fig:freq_vs_d_theory},
where we show the results for the low and high superposition regimes, respectively. 
Moreover, in the high superposition regime, \eqref{eq:cond-p-range} asserts that a valid $p$ is under a linear function of $d$. 
This is exactly the case in \Cref{fig:freq_vs_d_theory} (Middle), where the region with high values of FRR is under a line on the $(d, p)$ plane. 
We also notice some learnable regimes that are not captured by our theory in the low superposition regime $d>\sqrt n$. 
This could be due to two facts: (i) the algorithm used for the synthetic experiment is slightly different from what is proved by theory (ii) our theory is a positive result in nature, and our condition does not rule out the possibility that there are other learnable regimes. 
We leave the investigation to future work.

\begin{figure}[htbp]
    \centering
    \begin{subfigure}[b]{0.34\textwidth}
        \centering
        \includegraphics[width=\textwidth]{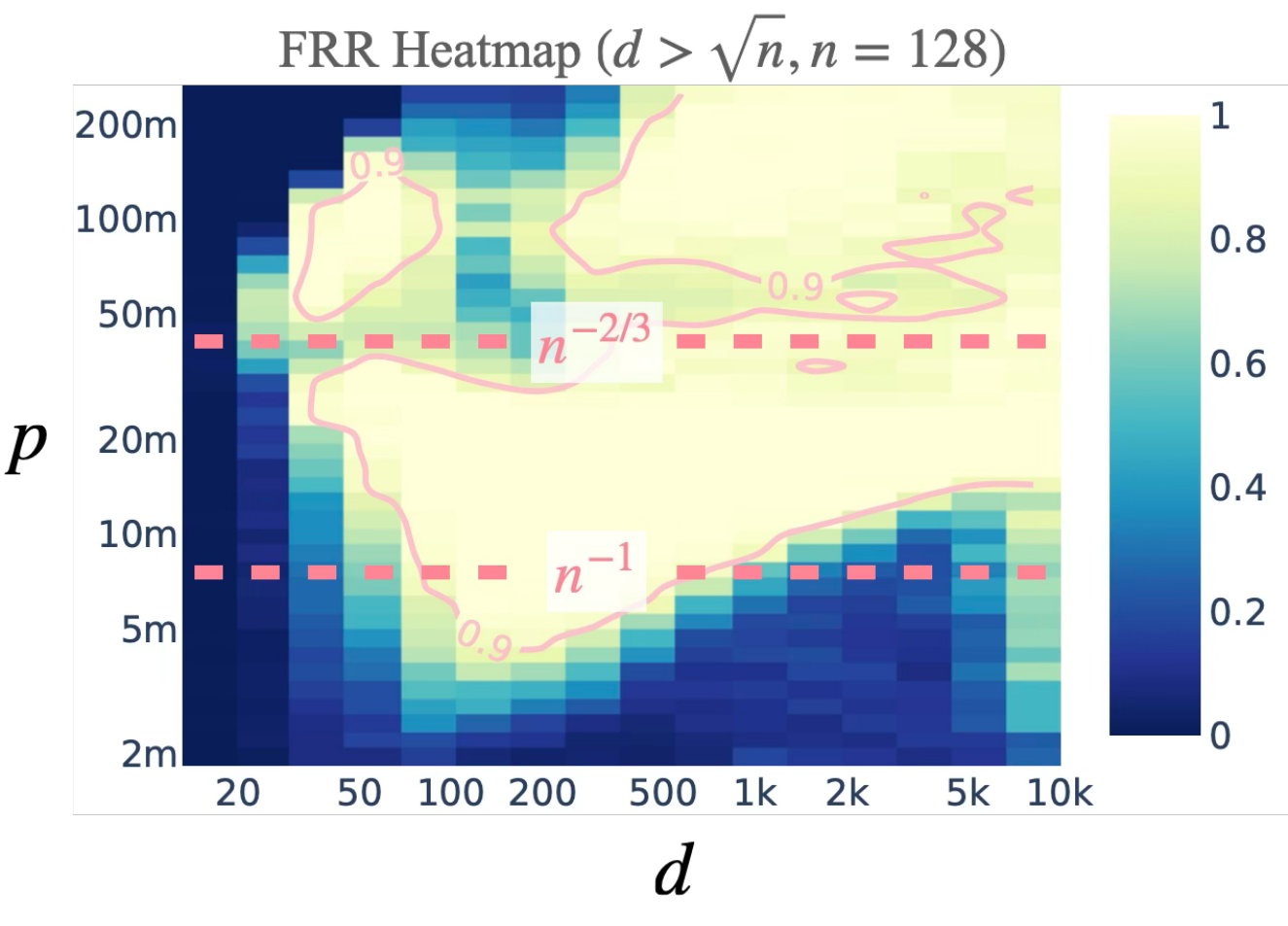}
    \end{subfigure}
    \hfill
    \begin{subfigure}[b]{0.34\textwidth}
        \centering
        \includegraphics[width=\textwidth]{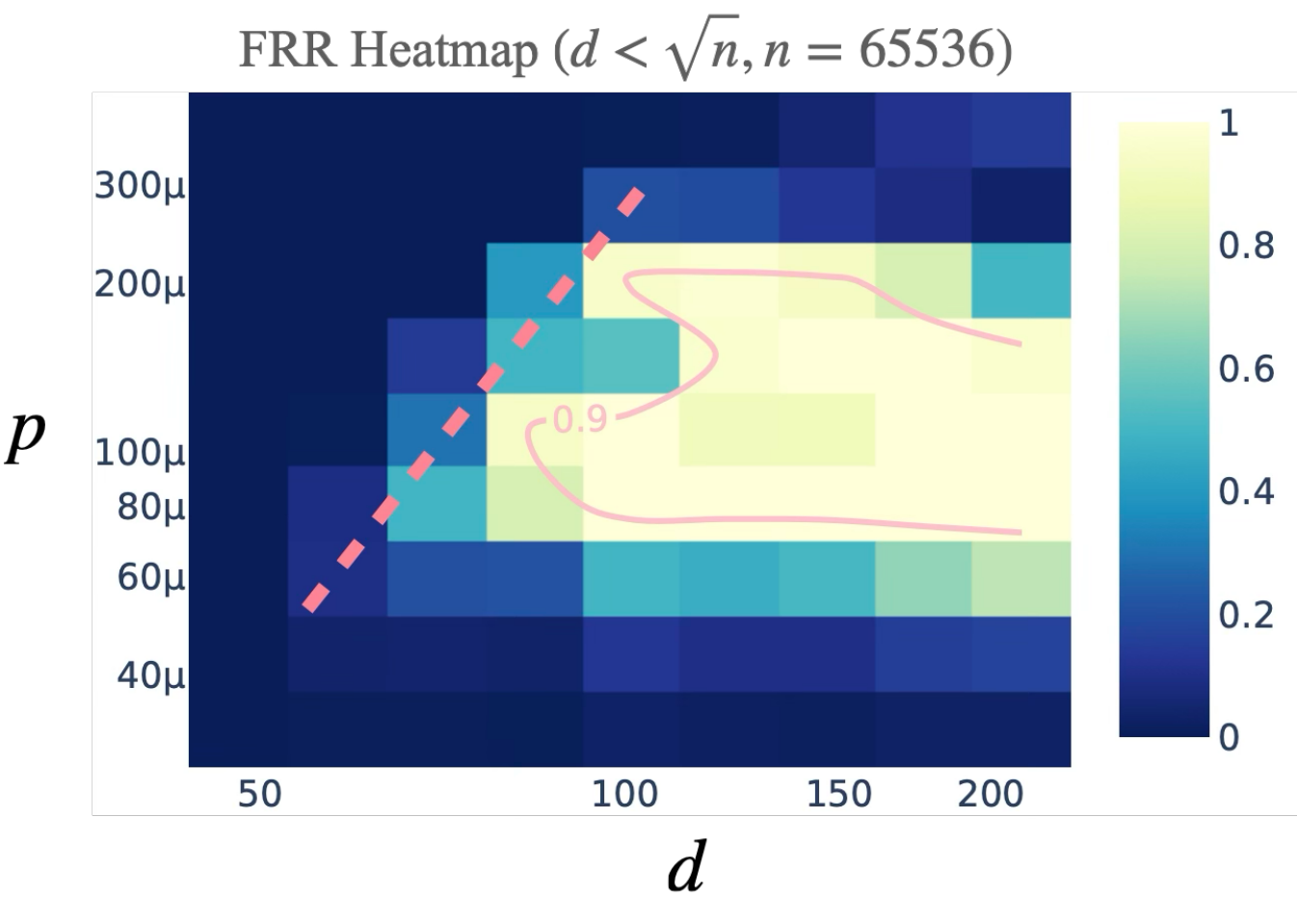}
    \end{subfigure}
    \hfill
    \begin{subfigure}[b]{0.30\textwidth}
        \centering
        \includegraphics[width=\textwidth]{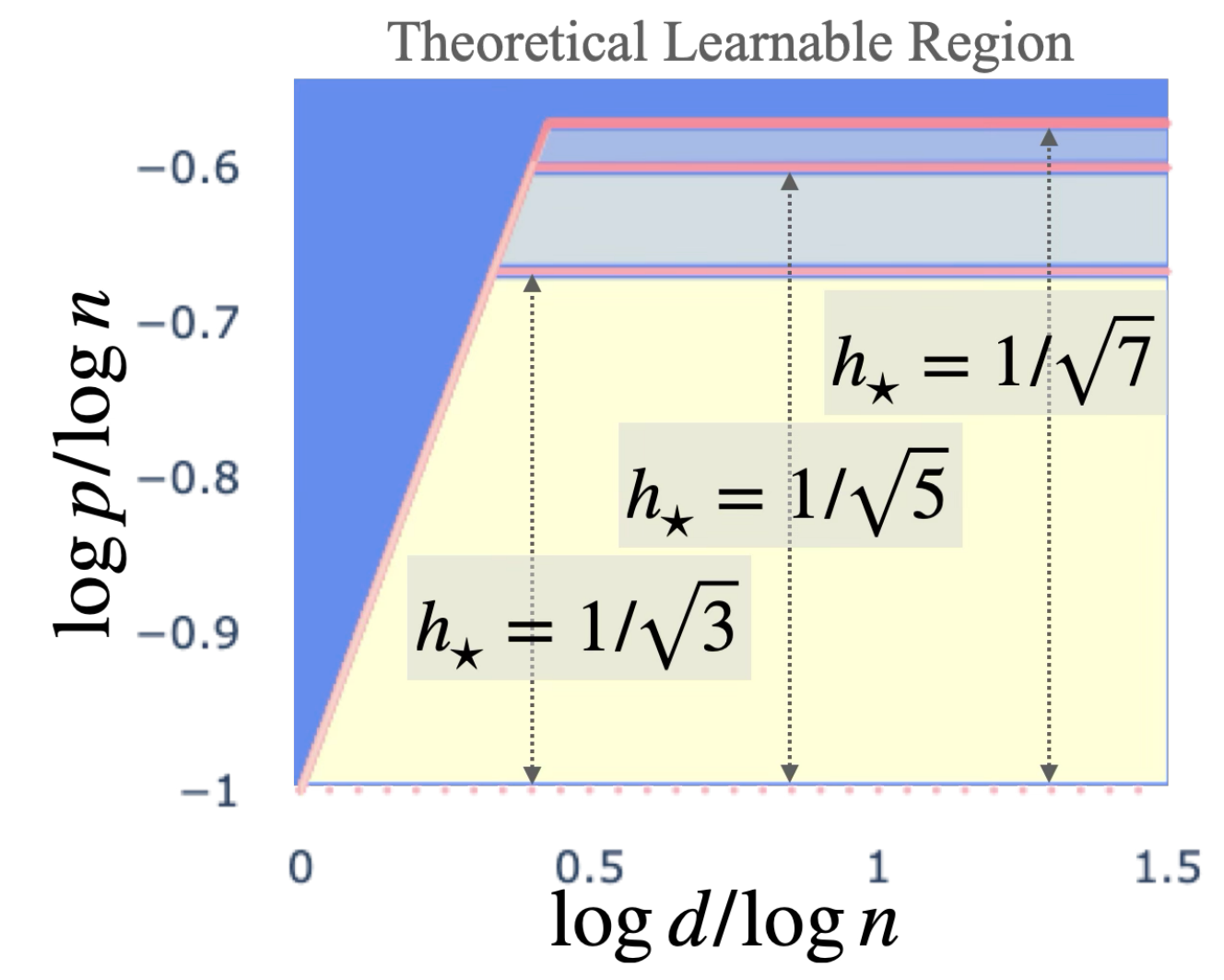}
    \end{subfigure}
        \caption{\small
            (\textbf{Left \& Middle}) FRR heatmaps for the BA algorithm under various TAFs $p$ and dimensions $d$ with both axes in log scale. Here, we set $(n, M, s, h_\star)$ as $(128, 512, 3, 1/\sqrt3)$ and $(65536, 2.62\times10^5, 3, 1/\sqrt3)$ respectively. 
            (\textbf{Right}) Theoretical learnable region of \eqref{eq:cond-p-range} (marked in yellow) for different $h_\star$. 
            (\textbf{Left}) Under \emph{low superposition},  the learnable region is $p \in (n^{-1},n^{-2/3})$.
            (\textbf{Middle}) Under \emph{high superposition}, the learnable region expands as $d$ grows.
            Both match our theory.}
    \label{fig:freq_vs_d_theory}
\end{figure}

\paragraph{{Neuron selectivity}}
In our synthetic experiments, as we regard $s$ as a constant, each feature appears in the data points at a $\Theta(1/n)$ frequency. 
Thus, we can rewrite \eqref{eq:cond-p-range} in terms of feature frequency $f = 1/n$. 
That is, when each feature appears in the $N$ data points with frequency $f$, for the BA algorithm to recover the features, the desired TAF $p$ should be chosen in the interval $(f,  \min\{f^{(1 + h_\star^2)/2}, d \cdot f\})$. 
The result implies that \highlight{feature learning is  selective}: a feature with frequency $f$ is more likely to be learned by a neuron with TAF $p$ that is higher than $f$ and roughly less than $\sqrt{f}$. This fact explains why GBA training is not dominated by the first group with the highest TAF---a feature tends to be learned by a neuron group with matching TAF.
Interestingly, this is analogous to the physical phenomenon of resonance. Each neuron group acts as a resonator tuned to a specific intrinsic frequency (TAF), while each feature acts like a sound with its own frequency (occupancy frequency). A feature is therefore preferentially learned by the group whose TAF \emph{resonates} with the feature's natural occurrence frequency.



\begin{AIbox}{Take-away from Bias Range Condition}
        \brai{1} In high superposition regime where $d$ is small, a lower TAF is required for feature recovery.\\
        \brai{2} Neurons are selective by aligning only with features whose occurrence frequencies fall within their designated target range.
\end{AIbox}

\subsubsection{Feature Balance and Network Width} \label{subsec:feature_balance}

In \Cref{thm:sae-dynamics} we present sufficient conditions for recovering all monosemantic features. If we only care about recovering a particular feature $v_i$, as we show in \Cref{thm:sae-dynamics-general} in appendix, it suffices to  replace \eqref{eq:cond-width} and choose a smaller $M$ such that 
\begin{align}\label{eq:linear_logM_h}
\log M \gtrsim   \frac{b^2 } {2 (1-\varepsilon)^2 h_i^2 }  + \log n ,
\end{align}
where $\varepsilon $ is a small constant. In addition, we also need the Feature Balance condition in \eqref{eq:cond-individual}. 
As a result, the learnability of each individual feature involves three key factors: (1)   the network width $M$, 
  (2) the feature   cut-off level $h_i$, defined in \eqref{eq:s_i-def}, and (3)
the relative occurrence of the feature defined in  \eqref{eq:rho_1-def}
We conduct controlled experiments to assess the impact of these factors.

\paragraph{Network width $M$ and feature cut-off  level $h_i$} 
We test how $M$ scales with $h_i$ in \Cref{fig:sigma_suite} (Left). We observe an exponential increase in the required network width $M$ as the $1/h_i^2$ increases, which again matches our theoretical prediction. 
In particular, we plot the heatmap of FRR with respect to $M$ and $1/ h_i^{2}$, where $M$ is shown in logarithmic scale. We observe that the region with high values of FRR is above a line in terms of $(1/h_i^2, \log M)$, which is consistent with \eqref{eq:linear_logM_h}. 

Moreover, 
in the special case where 
\(
H_{\ell,i}\in\{0,{1}/{\sqrt{s}}\},
\)
we have by definition that 
\(
{1}/{h_i^2}=s.
\)
This relationship also suggests that the required network width \(M\) might scale exponentially with the sparsity \(s\). 
However, if feature $v_i$ has a \emph{\highlight{Significant Strength}}---e.g., if for a proportion of data points \(H_{\ell,i} > 1/2\)---then the cutoff \(h_i\) will be much larger, thereby alleviating the exponential dependency on \(s\).
This message is intuitive: as ``signal strength'' of the feature becomes stronger, we need fewer parameters to learn it.

Observant readers may have noticed a discrepancy between the theoretical requirement for the network width \(M\) and the values used in our experiments.
The requirement in \eqref{eq:linear_logM_h} translates to 
\[
M \ge n\cdot \exp\Bigl(\frac{b^2}{2h_i^2}\Bigr) \approx n\cdot \Phi(-b)^{-1/h_i^2} = n\cdot \Bigl(\frac{1}{p}\Bigr)^{1/h_i^2},
\]
where the ``$\approx$'' holds
by using a Gaussian tail estimate. In contrast, our experiments use a much smaller \(M\le 3n\). This difference arises because the theory assumes independently evolving, randomly initialized neurons, while in practice the neurons are correlated during training. That is, the discrepancy between setting an infinitesimally small scale $a$ in theory versus setting $a = 1$ in experiments. 
When a neuron successfully learns a feature, the coupling term \(\psi_m\) in \eqref{eq:gd_approx} will prevent other neurons from going in the same directions, effectively improving the efficiency of the neuron utilization.

\begin{figure}[h]
    \centering
    \begin{subfigure}[t]{0.3\linewidth}
        \centering
        \includegraphics[width=\linewidth]{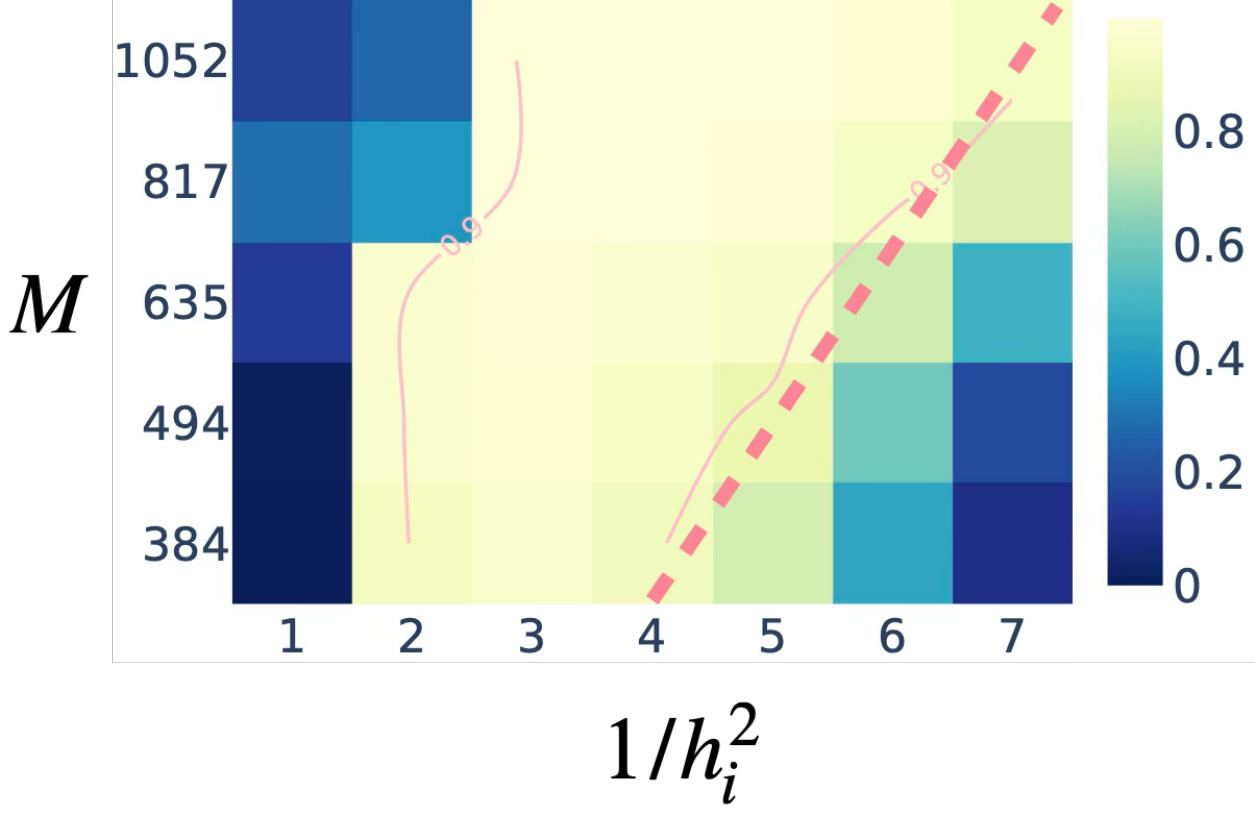}
    \end{subfigure}
    \hfill
    \begin{subfigure}[t]{0.3\linewidth}
        \centering
        \includegraphics[width=\linewidth]{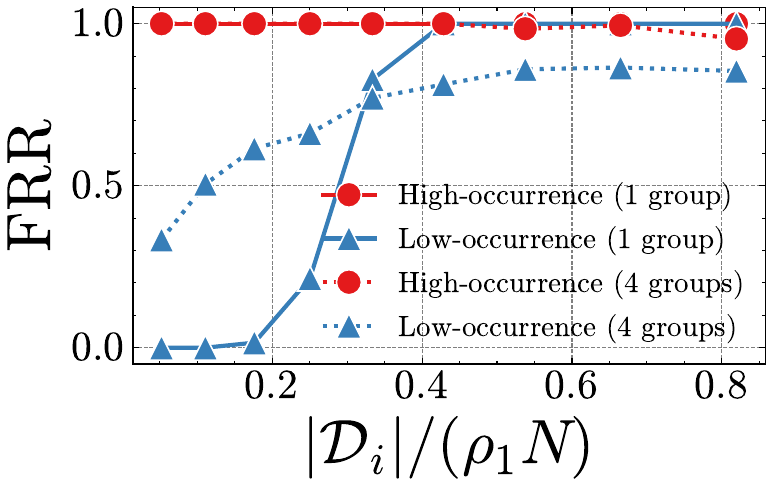}
    \end{subfigure}
    \hfill
    \begin{subfigure}[t]{0.3\linewidth}
        \centering
        \includegraphics[width=\linewidth]{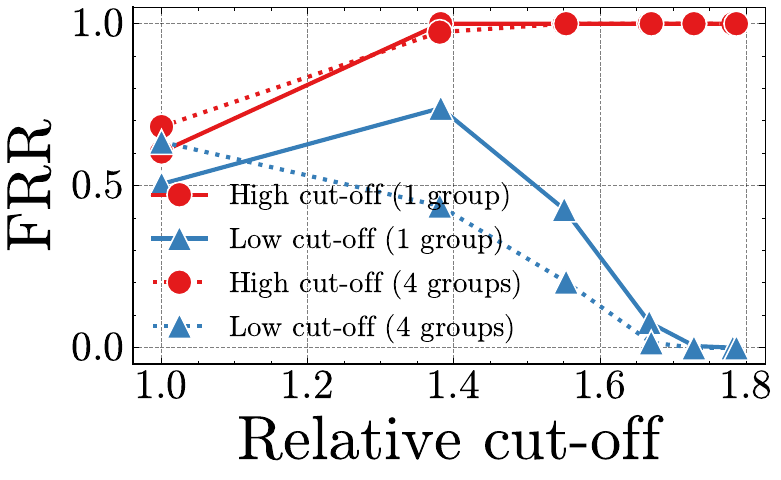}
    \end{subfigure} 
    \caption{\small
        (\textbf{Left}) Heatmap of FRR with respect to $(M, 1/h_i^2)$ for the GBA algorithm with $M$ axis in log scale.
        (\textbf{Middle}) FRR vs. relative occurrence $|\cD_i|/(N\rho_1)$ with $M=1024$, $s = 3$. 
        (\textbf{Right}) FRR vs. cut-off $h_i$ with $M=512$, $s=7$.
        All the experiments are for $(n, d)=(384, 100)$ and $p=0.01$.
    }
    \label{fig:sigma_suite}
\end{figure}

\paragraph{Strong vs. weak features}
We now investigate what happens if the feature balance condition is violated, that is, some features are much stronger than others.
A feature can be considered \emph{{strong}} if it either has a high occurrence frequency or a large cut-off level.
To this end, we conduct additional experiments where in data generation, we split the features into two halves, which correspond to the ``strong'' and ``weak'' features. Let $v_i$ and $v_j$ be representatives of the strong and weak features. 
We numerically test how the relative occurrence $|\cD_i| / |\cD_j| = |\cD_i|/(\rho_1 N)$ and relative cut-off $h_j/h_i$ affect the FRR, when we train the SAE using BA or GBA.

In \Cref{fig:sigma_suite} (Middle), we show how FRR is affected by the value of occurrence frequency. We see that both BA and GBA struggle with weak features, compared to results with strong features. In particular, FRR of BA drops significantly as the relative occurrence is less than 0.3. However, with grouping, GBA still learns quite a portion of the weak features even when the relative occurrence is less than 0.1. 

From \Cref{fig:sigma_suite} (Right), we show how FRR depends on the relative cut-off level. We observe that strong features with larger cut-off are learned more reliably than the weak features for both methods. More interestingly, in this case, grouping does not offer an evident advantage.

\begin{AIbox}{Take-away from Feature Balance and Network Width Conditions}
    \brai{1} A feature with significant strength, i.e., a large cut-off value as defined in  \eqref{eq:s_i-def}, can be learned with \emph{exponentially} less compute.\\
    \brai{2}~GBA can tolerate more feature imbalance in terms of the occurrence frequency.
\end{AIbox}

%% file: paper/sections/proof_overview.tex
\section{Proof Overview}\label{sec:proof-sketch}

In the following, we provide an overview of the key steps in the proof of \Cref{thm:sae-dynamics}. 
\subsection{Good Initialization with Wide Network}
By planting a large pool of i.i.d.\ random neurons at initialization, we can—\emph{with overwhelming probability}—(1) assign to each feature \(v_i\) one neuron \(m_i\) whose inner product with \(v_i\) is already very large, and (2) simultaneously ensure that this same neuron has only weak correlations with \emph{all} the other features.  Concretely, we prove that if \(M\) grows fast enough relative to \(n\), then there exists a choice of distinct neurons \(\{m_i\}_{i=1}^n\) such that
\begin{align}
  \textbf{InitCond-1:}\quad
    &\langle v_i, w_{m_i}^{(0)}\rangle 
      \ge (1 - \varepsilon)\,\sqrt{2\log (M/n)},\\
  \textbf{InitCond-2:}\quad
    &\max_{j\neq i}\bigl|\langle v_j, w_{m_i}^{(0)}\rangle\bigr|
      \le \sqrt{2}(1+\varepsilon)\,\sqrt{2\log n}.
      \label{eq:init_conds_sketch}
\end{align}
These two properties together ensure a \emph{\highlight{good initialization}} for the neuron \(m_i\) dedicated to feature \(v_i\). 
With $M\gg n^3$, we deduce that the weight vector $w_{m_i}$ aligns \emph{exclusively} with feature $v_i$.
In fact, as \(M\) increases the separation
between the two thresholds also increases, so \(w_{m_i}^{(0)}\) is ever more strongly aligned with its own feature \(v_i\) than with any other \(v_j\) at the start.  This widening margin precisely captures the \emph{benign over-parameterization} effect: having many neurons actually promotes clean, feature-specific initialization. See \Cref{lem:init} for more details.

\subsection{Pre-activations are Approximately Gaussian}
We give a brief overview of how we deal with the challenge of tracking the highly nonlinear dynamics in \eqref{eq:gd_simplified}.
With an abuse of notation,  let us denote by $w_t$ and $b_t$ one neuron's weight and bias after iteration $t$. 
For the first step, the pre-activations are Gaussian, i.e.,
\(
w_0^\top x_\ell + b \sim \mathcal{N}(b, 1). 
\)
For later steps, we expand the gradient descent update for the neuron weights $w_t$ at iteration $t$. 
Let us denote by $\varphi_t = (\varphi(w_{t-1}^\top x_\ell ; b))_{\ell\in [N]}$ 
 and $g_t=X^\top \varphi_t$ the gradient computed in \eqref{eq:gd_approx} 
 at iteration $t$. By the gradient formula in \eqref{eq:gd_approx}, we have
\begin{align}
    w_{t} = \sum_{\tau=1}^{t} \lambda_\tau \cdot X^\top \varphi_\tau + \lambda_0 \cdot w_0, \quad\text{and}\quad X w_t = \sum_{\tau=1}^t \lambda_\tau \cdot X g_\tau + \lambda_0 \cdot X w_0,
    \label{eq:gd_rollout}
\end{align}
for some coefficient $\lambda_\tau$. 
Let us recall the decomposition $X = H V$. 
The first equality in \eqref{eq:gd_rollout} indicates that $w_{t-1}$ only contains information of $V$ through the $(t-1)$-dimensional projection $\Phi = \mathrm{span}\{\varphi_\tau^\top H\}_{\tau=1}^{t-1}$. 
For the second equality, the most recent component $X g_t = H V g_t$ in the pre-activations contains a new gradient component that is not captured by the previous steps---projection of $g_t$ onto the orthogonal space of $G=\{ w_0, g_1, \ldots, g_{t-1}\}$, which we denote as $g_t^\perp$.
Data $X$'s projection onto this new direction can be decomposed as
\begin{align}
    X g_t = H\cdot (\Phi^\perp V g_t^\perp + \Phi V g_t^\perp), 
\end{align}
where $\Phi^\perp V g_t^\perp\in\RR^d$ only contains information from a subblock of $V$ that is orthogonal to both $\Phi$ in the row space and $G$ in the column space. 
Notably, all the previous updates only contain information of $V$ limited to the row space $\Phi$, i.e., $X^\top \varphi_\tau = V^\top H^\top \varphi_\tau$ for $\tau\le t-1$ or the column space $G$, i.e., $X g_\tau = H V g_\tau$ for $\tau\le t-1$.
Thus, the first term $\Phi^\perp V g_t^\perp$ is \emph{independent} of all the history updates, and $V g_t^\perp$ is a high-dimensional independent Gaussian vector plus a low-dimensional coupling term $\Phi V g_t^\perp$.
The argument holds true for all iteration steps, and if $t \ll d \land n$, we approximately have $x_\ell w_t + b\sim \cN(b, 1)$ thanks to the normalization of the weight $w_t$. 
This argument can be made rigorous by using the \emph{\highlight{Gaussian conditioning technique}} \citep{wu2023lower, bayati2011dynamics, montanari2023adversarial}  in the formal proof. 
See \Cref{sec:gaussian_conditioning_proof_sketch} for details.

\subsection{Weight Decomposition and Concentration under Sparsity} 
\begin{wrapfigure}{r}{0.35\textwidth}
    \centering
    \includegraphics[width=0.8\linewidth]{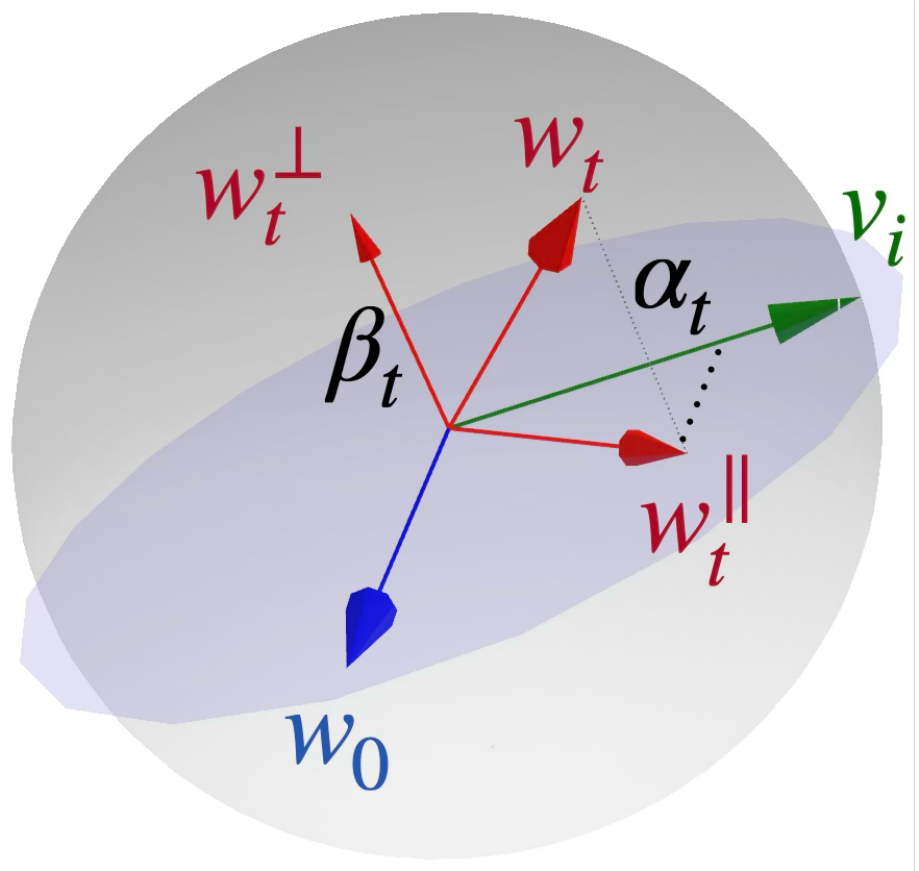}
    \caption{\small Visualization of the projection of $w_t$ onto 1) $w_t^\|$: the projection of $w_t$ onto the subspace spanned by $w_0$ and $v_i$ and 2) $w_t^\perp$: the projection of $w_t$ onto the subspace orthogonal to $v_i$ and $w_0$.}
    \label{fig:w_t_proj}
\end{wrapfigure}
For one neuron dedicated to the target feature $v_i$ and satisfying the initialization conditions in \eqref{eq:init_conds_sketch}, we decompose the weight $w_t$ into two directions: 
(1) the projection of $w_t$ onto the two-dimensional subspace spanned by $w_0$ and $v_i$; (2) the projection of $w_t$ onto the orthogonal space $w_t^\perp$. 
We define 
\begin{align}
    \alpha_t = \frac{\langle w_t, v_i\rangle}{\norm{v_i}_2}, \quad \beta_t = \norm{w_t^\perp}_2. 
\end{align}
The visualization of the projection is shown in \Cref{fig:w_t_proj}.
Using $\alpha_t$ and $\beta_t$, one can compute the first and second moments of the post-activation $\varphi_t$ under the decomposition of the pre-activations (into a high-dimensional Gaussian component and a low-dimensional coupling term) obtained by the Gaussian conditioning technique.
The post-activation $\varphi_t$ then gives rise to the next-step $w_{t+1}$, and we thus obtain an induced recursion over $\alpha_t$ and $\beta_t$. 
As a more concrete example, let us take learning rate $\eta=\infty$, and we can express $\alpha_t$ as 
\begin{align}
    \alpha_{ t} = \frac{\langle w_t, v_i\rangle}{\norm{v_i}_2} = 
    \frac{v_i^\top X^\top \varphi_t}{\norm{v_i}_2\cdot \norm{X^\top \varphi_t}_2} 
\end{align}
Recall that $X=HV$. By splitting $V$ into rows as $[V_{-i}; v_i^\top]$ and splitting $H$ column-wise into $H=[H_{-i}, H_i]$,  we have 
\begin{align}
     \alpha_{ t} = \frac{\norm{v_i}_2^2 \cdot  H_i^\top \varphi_t + v_i^\top V_{-i}^\top \cdot  H_{-i}^\top \varphi_t}{\norm{v_i}_2\cdot \norm{X^\top \varphi_t}_2} = \underbrace{\frac{\norm{v_i}_2^2 \cdot  H_i^\top \varphi_t}{\norm{X^\top \varphi_t}_2}}_{\dr Signal} \: +\: \underbrace{\frac{v_i^\top V_{-i}^\top \cdot  H_{-i}^\top \varphi_t}{\norm{v_i}_2\cdot \norm{X^\top \varphi_t}_2}}_{\dr Noise}. 
\end{align}
Here, we explicitly separate the signal from the noise. Our goal is to steer the neuron toward the direction of \(v_i\). Therefore, we treat the first term involving $H_i$ and $v_i$ as the signal while treating gradient contributions from other features as noise.

\paragraph{Controlling Moments of   Activations}
To proceed, we must tightly control both the signal and noise terms in the numerator and the denominator.  Concretely, this means bounding the first moment of the activation $\varphi_t$ (which enters the numerator) and its second moment (which controls the denominator), all while respecting the sparsity structure of $\varphi_t$.  A core difficulty stems from the pre-activation  
\[
   X w_t \;=\; H\,V\,w_t,
\]  
whose entries are not independent---even when $V$ has i.i.d.  Gaussian entries---because different data points may share the same features. To handle this challenge, we employ a \emph{\highlight{refined Efron-Stein inequality}} \citep{boucheron2003concentration} to tightly bound the moments related to the post-activation $\varphi_t$.

\subsection{State Recursion and Convergence}
We track at iteration \(t\) the \emph{alignment} \(\alpha_{-1,t}\) and the \emph{orthogonal component} \(\beta_t\) of the neuron weight $w_t$.  By exploiting the ``Gaussian-like'' concentration of the pre-activation \(Xw_t = H\,V\,w_t\) and applying the refined Efron-Stein inequality to handle both feature correlations and the nonlinearity, one obtains the coupled recurrences
\begin{align}
\frac{1}{\alpha_{-1,t}}
\le (1+o(1))
\;+\;
\lambda_t\Bigl(\frac{\Phi(-b)}{\rho_1d}\,\frac{1}{\alpha_{-1,t-1}} + \tilde\xi_t\Bigr), \quad 
\frac{\beta_t}{\alpha_{-1,t}}
\le \lambda_t\Bigl(\frac{\beta_{t-1}}{\alpha_{-1,t-1}} + \tilde\xi_t\Bigr).
\end{align}
Here, \(\lambda_t \propto \rho_1N/|\mathcal{D}_i|\), and \(\Phi(-b)\) denotes the Gaussian tail probability beyond the threshold \(-b\), which captures the activation sparsity. For clarity, we focus on the noiseless regime (i.e., assume \(\tilde\xi_t=0\)) so that all noise contributions are neglected.
We now elaborate on these recursions in detail:
\begin{enumerate}[
    leftmargin=2em,
]
    \item Recall that we require \(\beta_t \ll \alpha_{-1,t}\) since the neuron should eventually converge exclusively in the direction of the target feature. In our framework, the minimal growth rate of the ratio \(\beta_t/\alpha_{-1,t}\) is intrinsically controlled by \(\lambda_t=\tilde{O}(\rho_1N/|\mathcal{D}_i|)\), which characterizes how often the target feature \(v_i\) appears in the training data relative to the most frequent feature.
    By the definition of \(\rho_1\), this ratio is inherently larger than \(1\). Thus, to prevent an unbounded escalation of \(\beta_t/\alpha_{-1,t}\), we should restrict \(\lambda_t\) to, at most, a polylogarithmic scale, i.e., \(\lambda_t=\tilde{O}(1)\).
    \item If we additionally set \(\Phi(-b)/(\rho_1d)<d^{-\varsigma}\) for some \(\varsigma\in(0,1)\), then the map \(\alpha_{-1,t}^{-1}\mapsto\alpha_{-1,t+1}^{-1}\) is contractive.  Hence \(\alpha_{-1,t}\) grows from its initialization \(\tilde\Theta(d^{-1/2})\) to \(1 - o(1)\) in \(O(1)\) steps, and the growth rate is much faster than that of \(\beta_t/\alpha_{-1,t}\) thanks to the sparsity condition $\Phi(-b)/(\rho_1 d) \ll 1$. 
\end{enumerate}

From the above discussions, we have already justified the inclusion of the \textbf{Individual Feature Occurrence} condition $ \tfrac{|\cD_i|}{\rho_1 N} \ge \polylog(n)^{-1}$ in \eqref{eq:cond-individual} and part of the \textbf{Bias Range} condition $\tfrac{b^2}{2\log n} \gtrsim 1 - (1-\varsigma) \cdot \tfrac{\log d}{\log n} \Leftrightarrow \Phi(-b)\ll d^{1-\varsigma}/n = \tilde O(d^{-\varsigma}\cdot (\rho_1 d))$ in \eqref{eq:cond-global}.
The remaining conditions can be derived based on a more careful analysis, including the noise term \(\tilde\xi_t\) and the initialization conditions \eqref{eq:init_conds_sketch}.

%% file: paper/sections/conclusion.tex
\paragraph{Conclusion}
This research tackles the challenge of achieving theoretically grounded feature recovery using Sparse Autoencoders (SAEs) for Large Language Models (LLMs). Traditional SAE training algorithms often lack rigorous guarantees and suffer from issues such as hyperparameter sensitivity and instability.

To address these limitations, we introduce a novel statistical framework that redefines feature identifiability by modeling polysemantic representations as sparse mixtures of underlying monosemantic features. Within this framework, we propose an SAE training algorithm based on “bias adaptation,” which adaptively adjusts bias parameters to maintain appropriate activation sparsity while preventing neuron death. Rigorous theoretical analysis and experiments on synthetic data validate the method's ability to recover monosemantic features effectively, particularly when the frequency of feature occurrence aligns with the neuron's target range.

Building on this, we develop an enhanced empirical variant named Group Bias Adaptation (GBA). By partitioning neurons into groups and assigning distinct target frequencies, GBA offers a more flexible approach to feature recovery. Our analysis of group dynamics further refines the method for diverse network architectures. Experiments on a 1.5-B causal LLM show that GBA outperforms benchmark methods in reconstruction fidelity, activation sparsity, and feature consistency. Additionally, our ablation study demonstrates the robustness and largely tuning-free nature of GBA.

\paragraph{Limitations and future work}
While our work lays a solid foundation for understanding SAE training, several limitations remain. The theoretical guarantees on the dynamics are established under the assumption of Gaussian-distributed features, which may not hold universally. Moreover, simplifying assumptions related to model architecture and algorithmic complexity were made to facilitate analysis, though our synthetic experiments indicate that the key findings remain valid without these simplifications. 
Notably, in the identifiability results, the only assumption we need for the features is \textbf{incoherence}, which is a weaker and more natural assumption than Gaussianity. This means that the dynamics analysis---which relies heavily on the Gaussian conditioning technique---may be potentially extended to more general distributions, provided that the incoherence assumption holds.
Moreover, verifying the incoherence assumption in real-world datasets is challenging, and we believe that a closer examination and evaluation of the learned features is necessary for improving the robustness of current feature recovery methods.

Additionally, we plan to undertake intervention studies to validate the causal impact of the learned features on model behavior. 
Beyond feature recovery, the emerging field of circuit discovery offers promising opportunities for enhancing LLM interpretability. We anticipate that integrating our approach with existing circuit discovery techniques will further illuminate the inner workings of large-scale models.

%% file: paper/appendix/supplement.tex
\section{Supplementary Discussions}

\subsection{Details on Other Training Methods}
\label{app:other_training_methods}
We provide here more details on the training methods used in our experiments, including the Sparse Autoencoder (SAE) with TopK activation and SAE with $\ell_1$ regularization.

\paragraph{Sparse Autoencoder (SAE) with TopK activation}
In an SAE with TopK activation, sparsity is enforced by selecting only the $K$ neurons with the highest activation values in the hidden layer.
Let $y = W (x - b_{\mathrm{pre}}) + b$ be the pre-activation values of the hidden layer.  
Let $\phi(y)$ be the activations after applying a standard activation function.
The TopK selection mechanism, denoted as $S_K(\cdot)$, operates on $\phi(y)$. For a vector $v \in \mathbb{R}^M$, $S_K(v)$ produces a vector $v' \in \mathbb{R}^M$ such that:
$$
v'_j = \begin{cases} 
v_j & \text{if } v_j \text{ is among the $K$ largest values in } {v}, \\[1mm]
0 & \text{otherwise}
\end{cases}
$$
for $j \in [M]$.
The post-activation in a TopK SAE is:
$$
z = S_K\bigl(\phi(W (x - b_{\mathrm{pre}}) + b)\bigr),
$$
which by definition is $K$-sparse.
The reconstructed output is:
$$
\hat{x} =  \diag(a) \cdot W^\top z + b_{\mathrm{pre}}. 
$$
Let $\Theta = (W, b_{\mathrm{pre}}, b, a)$ be the parameters of the SAE.
The loss function for the TopK SAE is the reconstruction loss:
$$
\cL_{\mathrm{rec}}(x; \Theta) = ||x - \hat{x}||_2^2.$$

\paragraph{Sparse Autoencoder (SAE) with $\ell_1$ regularization}
In an SAE with $\ell_1$  regularization, sparsity is encouraged by adding a penalty term to the reconstruction loss, proportional to the sum of the absolute values of the hidden layer activations.
Let $y = W (x - b_{\mathrm{pre}}) + b$ be the pre-activation values of the hidden layer.
Let $z = \phi(y) = \phi(W (x - b_{\mathrm{pre}}) + b)$ be the activations after applying a standard activation function; these are the hidden layer representations that will be encouraged towards sparsity.
The reconstructed output is:
$$ \hat{x} = \diag(a)\cdot W^\top z + b_{\mathrm{pre}}. $$
The loss function for the L1 SAE, $\cL(x; \Theta)$, incorporates both the reconstruction error and the L1 penalty on the hidden activations $z$:
$$ 
\cL(x; \Theta) = \norm{x - \hat{x}}_2^2 + \lambda \cdot \sum_{j=1}^{m} |z_j| \cdot \norm{w_j}_2, 
$$
where $\lambda > 0$ is the sparsity penalty parameter that controls the strength of the regularization, $m$ is the number of neurons in the hidden layer, and $w_j$ is the $j$-th row of the weight matrix $W$.

\paragraph{JumpReLU}
In our real-data experiments, we also consider the \emph{JumpReLU} activation, a non-smooth, non-monotonic function. Conceptually, it behaves like ReLU for positive inputs but introduces a sharp jump for sufficiently large inputs. In our implementation, we adopt a simplified scalar form adapted to our neuron pre-activation $w_{m}^\top x + b_{m}$:
\begin{align}
    \mathrm{JumpReLU}(w_{m}^\top x;\, b_{m}) 
    = \begin{cases}
        0, & \text{if } w_{m}^\top x + b_{m} < 0, \\
        w_{m}^\top x, & \text{if } w_{m}^\top x + b_{m} \ge 0.
    \end{cases}
\end{align}
This activation acts as a hard thresholded identity: it passes the neuron's response only when the pre-activation crosses a bias-controlled threshold. Although JumpReLU does not satisfy the smoothness or Lipschitz conditions required in our theory (see \Cref{assump:activation}), it is empirically effective and included in our experimental comparisons \Cref{sec:experiments}.

\vspace{5pt}
\noindent\textbf{\emph{Pre- and post-activation sparsity.}}\quad For both the L1 and TopK SAE, we define the pre-activation sparsity as the number of non-zero entries in $y$, i.e., $\|y\|_0$, and the post-activation sparsity as the number of non-zero entries in $z$, i.e., $\|z\|_0$.
Note that the post-activation sparsity is always no larger than $K$ while the pre-activation sparsity can be larger than $K$.
This difference leads to two evaluation metrics for the TopK SAE: the \emph{pre-activation sparsity} and the \emph{post-activation sparsity} as used in \Cref{fig:loss_sparsity}.

\vspace{5pt}
\noindent\textbf{\emph{Minor notational discrepancy.}} In the main text and above definition we express the activation as $\phi(w_{m}^\top x + b_{m})$, whereas in the definition above the JumpReLU activation is indeed as a bivariate function of $w_{m}^\top x$ and $b_{m}$. This slight difference is purely notational and does not affect the underlying functionality or the definition of pre- and post-activation sparsity. For simplicity, we always stick to $\phi(w_{m}^\top x + b_{m})$.

\subsection{Evaluation Metrics} \label{app:evaluation_metrics}
We explain here the details of the evaluation metrics used in our experiments to assess how well the GBA algorithm recovers the underlying features. 

We first introduce the \emph{maximum activation} and \emph{neuron Z-score}, which are used to measure the quality of the learned neurons. 
Then, we introduce the notion of \emph{Max Cosine Similarity} (MCS) and \emph{Feature Recovery Rate} (FRR), which are used to measure the quality of the alignment between the learned neurons and the ground-truth features, or the consistency of the learned features across different runs. 
We also introduce the neuron percentage, constructed from the MCS, which is used to generate \Cref{fig:consistency}.


We introduce maximum activation and neuron Z-score of a neuron $m$ as follows. 

\paragraph{Maximum activation} Unless specified, we define the maximum activation of a neuron \(m\) as the maximum of its pre-activations over the validation set: 
\begin{align}
  \label{eq:maximum_activation}
    \textbf{Maximum~Activation}(m) = \max_{x\in \text{Validation Set}} y_m(x), \where y_m(x) = w_m^\top (x - b_{\mathrm{pre}}) + b_m.
\end{align}
Note that the maximum activation is computed based on the tokens in the validation set. It maps each neuron to a scalar, characterizing the maximum pre-activation of the neuron across all validation tokens.

\paragraph{Neuron Z-score}  
Let $\phi(\cdot)$ denote the neuron's activation function (e.g., ReLU, or JumpReLU).  
For each neuron $m$ and a minibatch $\{x_i\}_{i=1}^B$,  we define its post-activation responses as
\[
  \phi_{m,i}\;=\;\phi\bigl(w_m^\top (x_i  - b_{\mathrm{pre}}) + b_m\bigr)\,,\qquad i=1,\dots,B,
\]
where $w_m\in\RR^d$ is the neuron's weight vector and $b_m\in\RR$ is its bias.  
We can compute the  mean and standard deviation of these activations in the minibatch as
\[
  \mu_m
  \;=\;
  \frac{1}{B}\sum_{i=1}^B \phi_{m,i},
  \qquad
  s_m
  \;=\;
  \sqrt{\frac{1}{B}\sum_{i=1}^B\bigl(\phi_{m,i}-\mu_m\bigr)^2}\,.
\]
The Z-score of neuron $m$ on data point $x_i$ is defined as 
\[
  Z_{m,i}
  \;=\;
   (\phi_{m,i}-\mu_m) / s_m
  \;\in\;\R.
\]
We can also take the maximum of the Z-scores over the batch:
\begin{align} \label{eq:Z_max}
  Z_m^{\max}
  \;=\;
 (\phi_{m,\max}-\mu_m)/s_m \,, \where \phi_{m, \max} = \max_{1\le i\le B}  \phi_{m,i}.
\end{align}
A large value of $Z_{m,i}$ (or $Z_m^{\max}\gg0$) indicates that on some input $x_i$, the neuron's activation $\phi_{m,i}$ lies multiple standard deviations above its mean.  
Thus, when $Z_m^{\max}$ is large, neuron $m$ is  \emph{well-learned} to sensitively detect certain data points within the batch. 
More specifically, when $Z_m^{\max}$ is large, the two following conditions hold: 
\begin{itemize}[nosep,leftmargin=*]
  \item {\bf\emph{ Strong Selectivity:}} 
  There exists some $x_i$ within the batch such that $\phi_{m,i}\gg\mu_m$, i.e., 
the  neuron's activation $\phi_{m,i}$ “spikes'' for input $x_i$.  
  \item {\bf\emph{ Low Baseline Variability:}} Within the whole batch, the neuron's activation $\phi_{m,i}$ is relatively stable, i.e., the standard deviation $s_m$ is moderate. 
\end{itemize}
As a result, $Z_m^{\max}$  serves as a quantitative measure of the neuron's specificity on the batch of data. 
When generating \Cref{fig:consistency}, we use the maximum Z-score of each neuron across the whole validation set to select a subset of neurons. 

Next, we introduce the \emph{Max Cosine Similarity} (MCS) and \emph{Feature Recovery Rate} (FRR) metrics, which are used to measure the quality of the alignment between the learned neurons and the ground-truth features, or the consistency of the learned features across different runs.


\paragraph{Max Cosine Similarity (MCS) for synthetic data} For each neuron \(m\) with weight vector \(w_m\in\R^d\), we define
\[
    \mathrm{MCS}(m)
    \;=\;
    \max_{i\in[n]} \frac{\langle w_m,\,v_i\rangle}{\|w_m\|_2\,\|v_i\|_2}
    \;\in\; [-1,1].
\]
By definition, \(\mathrm{MCS}(m)=1\) if and only if \(w_m\) coincides with one of the true features \(v_i\).  

\paragraph{Max Cosine Similarity (MCS) for real data} For real data, as we do not have access to the ground-truth features, we define the MCS as the maximum cosine similarity between neurons across different runs. 
This definition is used in \Cref{fig:consistency}. 
Specifically, consider the trained neurons weights $W^{(j)}\in \RR^{M\times d}$ for $j=1,\ldots,J$ where $J$ is the number of runs with different random seeds. We fix the first run as the \emph{host} run and compute the MCS for the $m$-th neuron in the host run with respect to the $j$-th run with $j\ge 2$ as follows:
\begin{align}
    \mathrm{MCS}(m, j) = \max\bigl\{\cos(W^{(j)},  w_m^{(1)})\bigr\}. 
\end{align}
Here, the term inside the max is the cosine similarity between the $m$-th neuron in the host run and all neurons in the $j$-th run, which is an $M$-dimensional vector.
The maximum taken outside can be interpreted as finding the best match for the $m$-th neuron in the host run. 
Now, given a threshold $\tau$ for the MCS value, i.e., the x-axis in \Cref{fig:consistency}, we define neuron $m$ to \emph{have an MCS above the threshold if 
\( \mathrm{MCS}(m, j) \ge \tau\) for all $j\ge 2$.}
We require this condition to hold for all runs $j\ge 2$ because if the algorithm learns a consistent feature, it should be present no matter which random seed is used.
When this is the case, neuron $m$ in the host run can find a corresponding neuron in each of the other runs that has a cosine similarity above the threshold $\tau$.
Thus, by computing MCS for all the neurons in the host run, we evaluate the consistency of the learned features across different runs.

\paragraph{Neuron percentage in \Cref{fig:consistency}}  
Recall that we call the first run of the algorithm the \emph{host run}. 
Under the definition of MCS, in \Cref{fig:consistency} we plot the \textbf{neuron percentage} as a function of the MCS threshold $\tau$. 
In particular, for any threshold $\tau$ (x-axis in \Cref{fig:consistency}), we compute the fraction of neurons in the host run that have an MCS above the threshold across all runs.
That is, we define 
\begin{align} \label{eq:def_neuron_percentage}
  \textbf{Neuron Percentage}(\tau) = \frac{1}{M}\sum_{m=1}^M \ind\bigl(\mathrm{MCS}(m, j) \ge \tau, \forall j\ge 2\bigr). 
\end{align}
By definition, this quantity computes the fraction of neurons in the host run that have an MCS above the threshold $\tau$ across all runs $j\ge 2$. 
If this quantity is large, the algorithm is able to produce consistent results across different runs with different random seeds. 
Moreover, because a considerable portion of the neurons of SAE are rarely activated, instead of enumerating over all neurons as in \eqref{eq:def_neuron_percentage}, we can also consider the neuron percentage over a subset of neurons, denoted by $\cM \subseteq [M]$. 
Then, focusing on $\cM$, 
we define the neuron percentage as
\begin{align} \label{eq:def_neuron_percentage2}
  \textbf{Neuron Percentage}(\tau, \cM) = \frac{1}{|\cM | }\sum_{m \in \cM } \ind\bigl(\mathrm{MCS}(m, j) \ge \tau, \forall j\ge 2\bigr). 
\end{align}
In particular, in \Cref{fig:consistency}, we choose $\cM$ to be the top-$\alpha$ subset of neurons in terms of the maximum activations or neuron Z-score in the host run, which are defined in \eqref{eq:maximum_activation} and \eqref{eq:Z_max}, respectively.
Note that these two metrics are computed based on the validation dataset. 
The y-axis in \Cref{fig:consistency} is computed as in \eqref{eq:def_neuron_percentage2} with these two versions of $\cM$.


The notion of Feature Recovery Rate (FRR) is only used for synthetic data, where we have access to the ground-truth features.

\begin{wrapfigure}{r}{0.3\textwidth}
  \centering
  \includegraphics[width=0.95\linewidth]{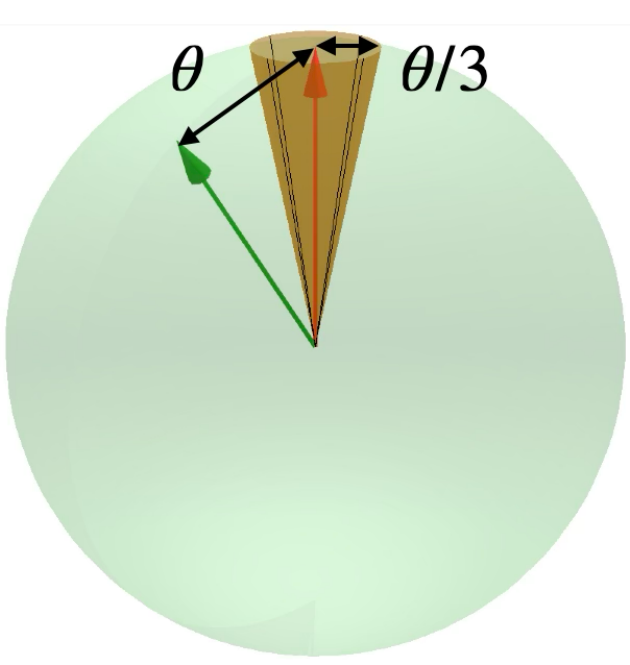}
  \caption{\small An illustration of the learnable region surrounding the feature. Any neuron weight within the cone has cosine similarity above the threshold with the feature.}
  \label{fig:cos_threshold}
\end{wrapfigure}
\paragraph{Feature Recovery Rate (FRR)}
For one monosemantic feature \(v_i\), we say it is \emph{recovered} if there exists a neuron \(m\in [M]\) such that the cosine similarity between the neuron and the feature is above a certain threshold \(\tau_{\mathrm{align}}\):
\[
  \ind_{i}
  =
  \begin{cases}
    1 & \text{if $\exists\,m\in[M]$ such that }
            \bigl|\langle \hat w_m,\,v_i\rangle\bigr| / \|v_i\|_2 \;\ge\;\tau_{\mathrm{align}},\\
    0 & \text{otherwise.}
  \end{cases}
\]
Then the \emph{Feature Recovery Rate} is
\[
  \mathrm{FRR}
  \;=\;
  \frac{1}{n}\sum_{i=1}^n \ind_{i}
  \;\in\; [0,1].
\]
In words, FRR is the fraction of ground-truth features \(v_i\) that have been recovered, i.e.,  aligned to at least one learned neuron.
Here, we find the following way to define the threshold \(\tau_{\mathrm{align}}\) useful:
\begin{align} \label{eq:define_tau_align}
    \tau_{\mathrm{align}} = \cos\Bigl(\frac{1}{3}\arccos\Bigl(\max_{i\neq j} \frac{\langle v_i,\,v_j\rangle}{\|v_i\|_2\,\|v_j\|_2}\Bigr)\Bigr).
\end{align}
Intuitively,  the angle given by  $\arccos$ in \eqref{eq:define_tau_align} is the smallest angle among all pairs of features \(v_i\) and \(v_j\) in $V$, which is denoted by $\theta $ in \Cref{fig:cos_threshold}.
Then, if a neuron exhibits a cosine similarity above the threshold \(\tau_{\mathrm{align}}\) with a feature \(v_i\), then it lies within the cone centered at $v_i$ with angle $\theta / 3$. See \Cref{fig:cos_threshold} for an illustration.
By our choice of \(\tau_{\mathrm{align}}\), these cones associated to all monosemantic features lie in the \(d-1\)-dimensional sphere without overlapping, ensuring that each neuron exceeding the threshold is \emph{uniquely} aligned with a single feature.

\subsection{Equivalance between Row Normalization of $X$ and $H$}
\label{app:discussion}
In the following, we argue that normalizing the data is essentially the same as normalizing each row of the coefficient matrix $H$ under the Gaussian feature setting. 
We first invoke the following proposition to show that normalizing the rows of $X$ can approximate the row-normalization of $H$. Let $h_i$ be the $i$-th row of $H$ and $x_i$ be the $i$-th row of $X$.
\begin{proposition}
\label{prop:normalization}
Suppose $\log n \ll d$ and $V_{ij}\iidfrom \cN(0, 1)$. 
Then for any universal constant $c>0$, there exists another universal constant $C>0$ such that with probability at least $1 - n^{-c}$, we have for all $i\in [N]$ that
\begin{align}
    \Bigl| \frac{\|x_i\|_2^2}{d \norm{h_i}_2^2} - 1 \Bigr| &\le C s \sqrt{\frac{\log n}{d}}.
\end{align}
\end{proposition}
\begin{proof}[Proof of \Cref{prop:normalization}]
    In the following proof, we denote by $C$ some universal constant that may vary from line to line.
    Note that the norm of $x_i$ is given by 
    \begin{align}
        \|x_i\|_2^2 = \biggl\|\sum_{j\in[n]: H_{ij} \neq 0} H_{ij} v_j\biggr\|_2^2 = \sum_{j\in[n]: H_{ij} \neq 0} H_{ij}^2 \|v_j\|_2^2 + \sum_{j\neq k: H_{ij}, H_{ik} \neq 0} H_{ij} H_{ik} v_j^\top v_k.
    \end{align}
    By \Cref{lem:chi-squared} and applying a union bound, we have with probability at least $1-n^{-c}$ that $|\norm{v_i}_2^2 - d| \le C(\sqrt{d \log (n )} + \log(n ))$ for all $i\in[n]$ and some universal constant $c, C$. For $Z=v_j^\top v_k$, we have the moment generating function of $Z$ as 
    \begin{align}
        \EE[e^{\lambda Z}] &= \prod_{l=1}^d \EE[e^{\lambda \cdot v_j[l] v_k[l]}] = \prod_{l=1}^d \EE[e^{\lambda^2 v_k[l]^2/2}] = \frac{1}{(1-\lambda^2)^{d/2}}, \where \lambda^2 < 1.
    \end{align}
    Using the Chernoff bound, we conclude that 
    \begin{align}
        \PP(\langle v_j, v_k\rangle \ge t) \le \inf_{0<\lambda <1} e^{-\lambda t} \cdot \EE[e^{\lambda Z}] = \inf_{0<\lambda <1} \frac{\theta^{-\lambda t}}{(1-\lambda^2)^{d/2}} \le \inf_{\lambda > 0} \exp\Bigl( - \lambda t + \frac{d}{2} \cdot \frac{\lambda^2}{1-\lambda^2} \Bigr).
    \end{align}
    By setting $\lambda = t/d \ll 1$, we get 
    \begin{align}
        \PP(\langle v_j, v_k\rangle \ge t) \le \exp\Bigl( -\frac{t^2}{d} + \frac{t^2}{2d} \cdot \frac{1}{1-t^2/d^2}\Bigr) \le \exp\Bigl( -\frac{t^2}{5d} \Bigr).
    \end{align}
    A similar bound holds for the negative case and we thus conclude that $\PP(|\langle v_j, v_k\rangle| \ge t) \le 2\exp( -\frac{t^2}{5d} )$ for $t /d \ll 1$. 
    Consequently, with probability at least $1 - n^{-c}$, it holds for all pair of $(i, j)\in[n]^2$ and $i\neq j$ that $|\langle v_i, v_j\rangle| \le C \sqrt{d \log n}$ for some universal constant $C$. Here, we are using the condition that $\log n \ll d$.
    Combining the above two results, we obtain that
    \begin{align}
        \bigl| \|x_i\|_2^2 - d\norm{h_i}_2^2  \bigr| &\le C (\sqrt{d \log n} + \log n)\cdot \norm{h_i}_2^2 + C \sqrt{d \log n} \cdot \sum_{j\neq k: H_{ij}\neq 0} H_{ij} H_{ik} \\
        &\le C \sqrt{d \log n} \cdot \norm{h_i}_1^2 \le C s \sqrt{d \log n} \cdot \norm{h_i}_2^2, 
    \end{align}
    where in the last step we use the h\"older's inequality and the fact that each row of $H$ has at most $s$ non-zero entries. 
    Hence, we conclude the proof. 
\end{proof}

\subsection{Omitted Details in \Cref{sec:main_theory_feature_recovery}}
\label{app:omitted_details_theory}


In this section, we provide the omitted details for \Cref{sec:main_theory_feature_recovery}.
We give a formal definition of ReLU-like activations. 

\begin{definition}[ReLU-like Activation]
    \label{assump:activation}
For the activation function $\phi:\RR\to\RR$, we define $\varphi$ as
\[
    \varphi(x) =\varphi(x;0) = \phi(x) + x\,\phi'(x).
\]
We say that $\phi$ is ReLU-like if it satisfies the following:  
\begin{enumerate}[leftmargin=0.2in,nolistsep]
        \item (Lipschitzness) The activation function $\phi$ is continuously differentiable, 1-Lipschitz, and $\gamma_1$-smooth with $\gamma_1=O(\polylog(n))$. 
        Furthermore, $\varphi(x)$ is $\gamma_2$-Lipschitz with $\gamma_2=O(\polylog(n))$. 
        \item (Monotonicity) The activation function $\phi$ is non-decreasing, and moreover, $\phi'(x) > C_0$ for some constant $C_0>0$ and all $x\ge 0$.
        \item (Diminishing Tail) There exists a threshold $\kappa_0=O((\log n)^{-1/2})$ and a sufficiently large constant $c_0>0$ such that for all $x<-\kappa_0$,
        \(
            \max\{|\phi(x)|,\, |\phi'(x)|,\, |x\,\phi'(x)|\} \le n^{-c_0}.
        \)
\end{enumerate}
\end{definition}

\noindent\textbf{\emph{Lipschitzness.}}
Under the above assumptions, we note that $\varphi(x;b)$
is $L$-Lipschitz in $x$ with $L=(\gamma_2 + |b| \gamma_1) = O(\polylog(n))>1$. 
The Lipschitz property of the function $\varphi$ is pivotal in our analysis since it enables control over error propagation across iterations. However, this property depends on the smoothness of the activation function $\phi$, a condition that the standard ReLU does not satisfy. Fortunately, many common activation functions—such as softplus, noisy ReLU, and shifted ELU (with the limit at $-\infty$ set to $0$)—do satisfy this smoothness requirement.
In particular, with a large smoothness parameter $\gamma_1=\polylog(n)$, we can use a smooth activation function to well approximate the ReLU function. For instance, we can take $\phi(x) = \gamma_1^{-1} \log(1 + e^{\gamma_1 x})$ for some $\gamma_1 = \polylog(n)$ as a smooth approximation of the ReLU activation function.

\noindent\textbf{\emph{Monotonicity.}}
The monotonicity property ensures that neurons with large pre-activations, which indicate a good alignment with the underlying features, will also have large post-activations.
This then guarantees a continuous growth of the corresponding neuron weights.

\noindent\textbf{\emph{Diminishing Tail.}}
The diminishing tail condition ensures that both the activation function $\phi$ and its derivative $\phi'$ are negligibly small when the input is below the threshold $-\kappa_0$. 
This property suppresses unwanted neuron activations, thereby promoting sparsity in the activations—a key factor in the successful training of the SAE.




%% file: paper/appendix/add_experiments.tex
\section{Additional Experiments Details}
\label{app:exp}
We provide additional experimental results and implementation details that complement the main findings presented in the paper.

\subsection{Additional Experimental Details for \Cref{sec:main_theory_feature_recovery}} \label{sec:additional_experimental_details_feature_recovery}
In \Cref{sec:main_theory_feature_recovery}, we present the theoretical results for the feature recovery problem along with the experimental results on synthetic data.

\subsubsection{Synthetic Data}
In synthetic experiments, we use Spherical Gaussian features.
For each sample \(x_j\) (\(j \in [N]\)), we randomly sample \(s\) indices (with replacement) from \([n]\) to form a multi-set \(S_j\). The corresponding features are then combined with a weight $1/\sqrt{s}$ to construct the reconstruction target:
\[
x_j = \sum_{i \in S_j} v_i /\sqrt{s}.
\]

\subsubsection{Comparing with TopK Activation}
We compare our algorithm with the TopK Algorithm on synthetic data.

\paragraph{Data Details}
We generate synthetic data comprising \(N = 10^5\) samples, each formed by directly summing \(s = 3\) independent spherical Gaussian features.  The dataset contains a total of \(n = 4096\) features, each of dimension \(d=128\).
Each sample is constructed by summing \(s = 3\) independent Gaussian features.

\paragraph{Comparison Details}
We evaluate the TopK algorithm for the following values of \(k\):
\[
k \in \{3, 6, 9, 14, 18, 22, 28, 34, 40, 50, 60, 65, 70, 75, 80, 85, 90, 95, 100, 105, 110, 115, 120, 125\}.
\]
For GBA, we set five groups and set the Highest Target Frequency (HFT) to be exponentially scattered within $[10^{-4}, 10^{-3}]$:
\[
\{1.00\times10^{-3},\;5.62\times10^{-4},\;3.16\times10^{-4},\;1.78\times10^{-4},\;1.00\times10^{-4}\},
\]
while the Lowest Target Frequency (LTF) is fixed at \(10^{-5}\).
In both cases, we train a SAE with \(M = 16384\) hidden neurons.

\paragraph{Training Details}
We use the AdamW optimizer with a learning rate of \( 10^{-3} \), weight decay $0.01$, a batch size of 128, and a maximum of 600 training epochs. The activation function is set to ReLU.

\paragraph{Result} The results are shown in \Cref{fig:GBA_vs_topK}. The GBA algorithm recovers nearly all features while maintaining an exceptionally low activation rate.

\begin{AIbox}{Take-away from \Cref{fig:GBA_vs_topK}}
\begin{enumerate}[
  label=\textbf{\arabic*.}, 
]
  \item GBA achieves a Feature Recovery Rate (FRR) close to 1, on par with the TopK algorithm, validating our theoretical results.
  \item Moreover, GBA \emph{automatically} identifies the correct number of active features per sample (\(s=3\)) through its activation percentage. Consequently, only the neurons corresponding to the ground truth features are frequently activated, while the others remain largely inactive. This results in a cleaner, more interpretable activation pattern.
\end{enumerate}
\end{AIbox}

\begin{figure}
    \centering
    \includegraphics[width=0.45\textwidth]{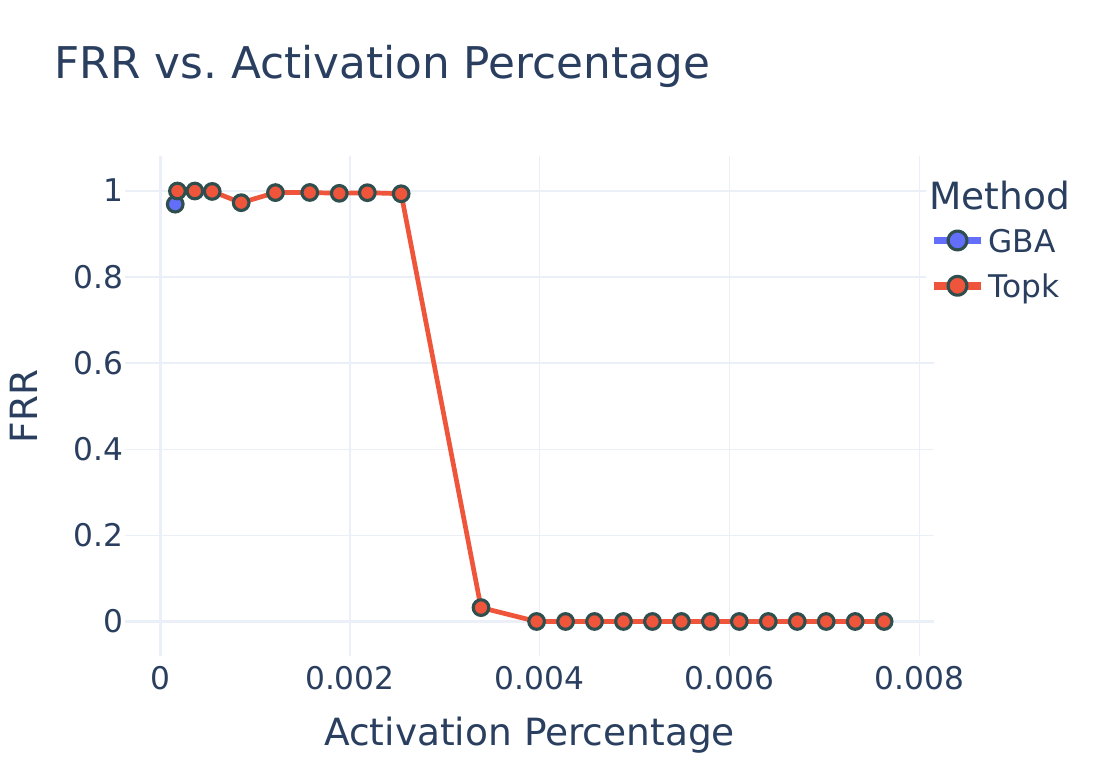}
    \caption{\small Comparison of GBA (blue) against the standard TopK algorithm (orange). Parameters: \( (n,d,M,s)=(4096,128,16384,3) \). We can see that GBA can automatically identify the ground truth number of active features per sample ($K=s=3$) in terms of activation percentage (the average percentage of neurons activated measured by post-activations for each data)
    }
    \label{fig:GBA_vs_topK}
\end{figure}
\vspace{10pt}

\subsubsection{Additional Details on \Cref{fig:sigma_suite}}  
\paragraph{Data details}  
We provide additional details on the experimental setup used in \Cref{fig:sigma_suite}. All experiments were conducted on synthetic data with spherical Gaussian features and \( N = 10^7 \) data points.
For \Cref{fig:sigma_suite} (Left), we used the setup described earlier to evaluate the learned feature percentage across different model widths \( M \) and sparsity levels \( s \).
Here, each data is obtained by a linear combination of $s$ random features with coefficients $1/\sqrt{s}$. 
By \Cref{prop:sparsity-concentrated}, we know that $1/h_i^2=s$. 
Therefore, it suffices to vary $s$ to achieve the same effect of varying $1/h_i^2$.
For \Cref{fig:sigma_suite} (Middle), we randomly designated half of the features as ``high-occurrence'' and the other half as ``low-occurrence.'' To control the imbalance in feature frequency, a fraction $\alpha$ of the data was generated using only high-occurrence features, while the rest was sampled uniformly from the full feature set.
By tuning the parameter $\alpha$, we achieve the effect of controlling the relative occurrence by identity $|\cD_i|/(N\rho_1) = (1-\alpha) / (1+\alpha)$.
For \Cref{fig:sigma_suite} (Right), we again split the features into two groups: ``high cut-off'' and ``low cut-off.'' For high cut-off feature $i$, we perturb the nonzero coefficients $H_{\ell, i}$ using the transformation \( H_{\ell, i} = |\tanh(\mathcal{N}(\mu, \sigma))| \), where \( \mu = \mathrm{arctanh}(1/\sqrt{7}) \) and \( \sigma \in \{0.0, 0.3, 0.6, 0.9, 1.2, 1.5\} \). Low cut-off features retain the constant weight \( 1/\sqrt{7} \). 
By the perturbation, we almost retain the mean while making the \emph{\highlight{cut-off}} larger. 
Each data point is constructed by randomly sampling 7 features ($s=7$) and assigning their corresponding weights.

\paragraph{Training details}  
We use the AdamW optimizer with a learning rate of \( 10^{-4} \), a batch size of 128, and a maximum of 30 training epochs. The activation function is set to ReLU. The Lowest Target Frequency (LTF) is fixed at \( 10^{-6} \).  
For the BA method (GBA with one group), the Highest Target Frequency (HTF) is set to \( 10^{-2} \). For the GBA method with four groups, we set the HTFs to \( 5 \times 10^{-3} \), \( 10^{-2} \), \( 2 \times 10^{-2} \), and \( 5 \times 10^{-2} \), respectively.

\subsection{Additional Details for \Cref{sec:experiments}}
\label{app:exp_1B}
\paragraph{Data and model details}
We choose the subsets of \texttt{Github} and \texttt{Wikipedia\_en} of \texttt{Pile} \citep{gao2020pile} without copyright as our datasets. 
The \texttt{Github} dataset is a collection of 1.2 billion tokens from public \texttt{GitHub} repositories, while the \texttt{Wikipedia\_en} subset contains 1.5 billion tokens from English \texttt{Wikipedia} articles. We use the first 99.8k rows from each dataset for training and the next 0.2k rows for validation. 
Each row in the dataset is truncated to the first 1024 tokens after tokenization.
Therefore, the total number of tokens is roughly \highlight{$N = 100m$}.
We use the \texttt{Qwen2.5-1.5B} base model \citep{yang2024qwen2} as our LLM, which has 1.5 billion parameters and MLP output dimension $1536$. 
We attach an SAE to the output of the LLM's MLP output at layer 2, 13, and 26 with \highlight{$M = 66k$} neurons, resulting in three different SAEs for each dataset. 
The dimension $d$ of the input data points is equal to \highlight{$d = 1536$}. 
We use the JumpReLU activation \citep{erichson2019jumprelu, rajamanoharan2024jumping} for all training methods.

\paragraph{Training details}
We train the SAEs using methods such as GBA, TopK, L1, and BA, where BA is simply GBA with one group.
For all these methods, we  use the AdamW optimizer with a learning rate of $10^{-4}$ and a weight decay of $10^{-2}$.
Since the sentences are truncated to 1024 with padding token removed, we set the batch size to \highlight{$ L= 8192$} tokens and a buffer size of \highlight{$B = 40k$} tokens. Each run can be completed using a single NVIDIA A100 GPU with 80GB memory, and we train 8 epochs for each method.
The hyperparameters of each method are set as follows:
\begin{itemize}
  \item For GBA, we set Highest Target Frequency (HTF) and the Lowest Target Frequency (LTF) are set to $0.1$ and $0.001$, respectively.
  We set a total number of \highlight{$K = 10$} groups with exponentially decaying target frequencies betweeen the HTF and LTF. We set $\gamma_-=0.01$ and $\gamma_+=0.01$.
  In other words, we have   \highlight{$p_1 = 0.1$}, \highlight{$p_{10} = 0.001$}, and $\{p_k\}_{k\in[10]}$ form a geometric sequence. 
  \item  For the BA method,  we set the HTF to be from $\{ 10^{-1}, 3\times 10^{-2}, 10^{-2}, 3\times 10^{-3}\} $ and vary the choice. The other parameters are the same as GBA.
  \item  For TopK method, we implement two versions --- the pre-activation TopK and the post-activation TopK. See \Cref{app:other_training_methods} for details. We vary the value of $K$ in $\{ 50, 100, 200, 300, 400, 500, 600\}$. 
  \item  For L1 method, we vary the penalty parameter $\lambda$ in $\{ 10^{-1}, 3\times 10^{-2}, 10^{-2}, 3\times 10^{-3}, 10^{-3}\}$.
\end{itemize}


\paragraph{Comparison between JumpReLU and ReLU activation}
For the SAE trained on the Github dataset at layer 26, we compare the performance between JumpReLU and the standard ReLU activations across all methods considered in this paper. As shown in \Cref{fig:activation_comparison_all}, the sparsity-loss frontiers for TopK and L1 methods are nearly identical under both activations. However, the GBA method demonstrates a marked improvement when using JumpReLU activation. With ReLU, decreasing the neuron bias also reduces the output magnitude. 
Thus more neurons are needed to compensate for the loss of output magnitude, which leads to a less sparse model, which degrades the sparsity-loss frontier. In contrast, JumpReLU decouples the neuron output magnitude from its bias—only the activation frequency is influenced—yielding a more robust sparsity-loss performance.

\begin{figure}[h]
  \centering
  \begin{subfigure}[b]{0.275\textwidth}
    \centering
    \includegraphics[width=\textwidth]{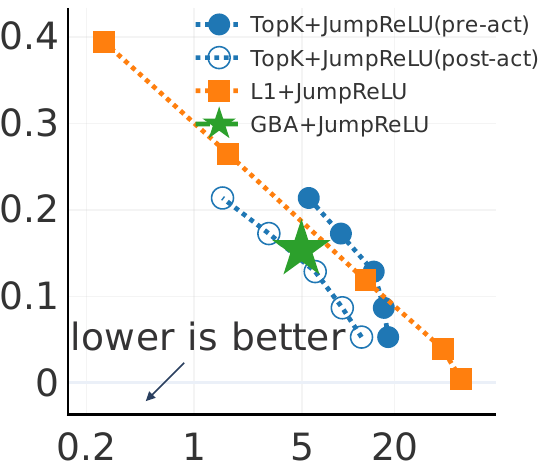}
  \end{subfigure}
  \begin{subfigure}[b]{0.275\textwidth}
    \centering
    \includegraphics[width=\textwidth]{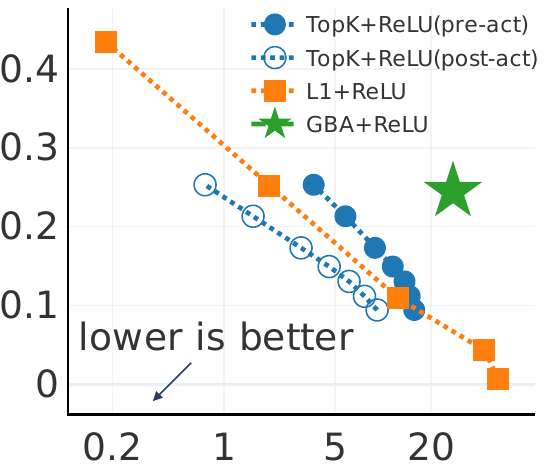}
  \end{subfigure}
  \begin{subfigure}[b]{0.43\textwidth}
    \centering
    \includegraphics[width=\textwidth]{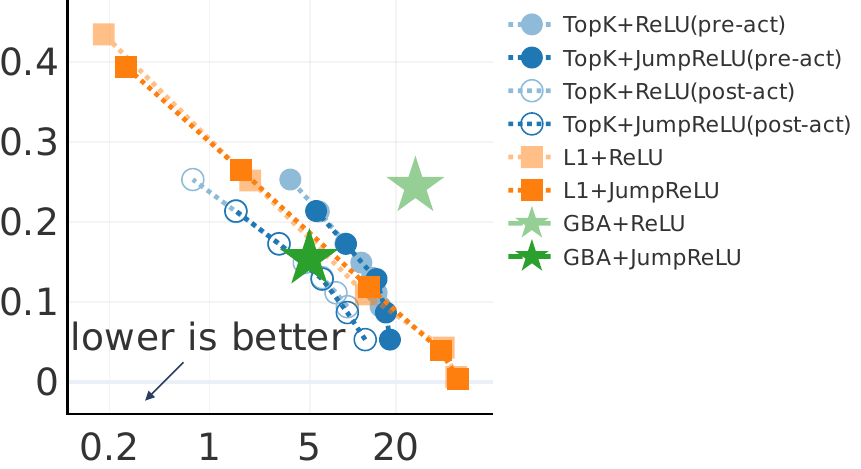}
  \end{subfigure}
  \caption{\small Comparison of sparsity-loss frontier between JumpReLU and ReLU activations. The \textbf{left} and \textbf{middle} plots show the sparsity-loss frontier with JumpReLU and ReLU activations, respectively. The \textbf{right} plot is a combination of the two, where the faded plots represent the sparsity-loss frontier of the ReLU activation. }
  \label{fig:activation_comparison_all}
\end{figure}

\paragraph{Improved explanation of bias clamping to prevent over-sparsification}
During the bias scheduling subroutine of the GBA algorithm (\Cref{alg:GBA}), we enforce a clamp on the bias values, restricting them to the range $[-1, 0]$. This constraint serves two primary purposes. The upper bound of $0$ ensures that a neuron is only activated when the input data exhibits a sufficient alignment with the neuron's weight vector. Consequently, allowing negative bias values ($b_m < 0$) effectively prevents excessive or premature activation of neurons.

The lower bound of $-1$ is implemented to avoid over-deactivation and the emergence of a reinforcing loop. We have observed experimentally that when the pre-bias ($b_{\mathrm{pre}}$) significantly deviates from zero, certain neurons may develop weights that are in opposition to the pre-bias to compensate for this drift. As these compensatory neurons are more likely to be activated by the initial pre-bias, the GBA algorithm might inadvertently continue to deactivate them by further reducing their bias ($b_m$). This deactivation would then necessitate an increase in the neuron's weight to maintain its influence, leading to a counterproductive cycle of deactivation and weight growth.

By limiting the bias to be no less than $-1$, we effectively interrupt this reinforcing loop and promote training stability. The rationale behind choosing $-1$ as the lower bound stems from the fact that our input data is normalized. This normalization typically results in pre-activation values that are significantly smaller than $1$, with values approaching $1$ only when the data strongly activates specific neurons. Therefore, a lower bias bound of $-1$ provides sufficient range for deactivation without causing the problematic feedback loop. This clamping strategy has been shown to significantly enhance the stability of the training process.

%% file: paper/appendix/sae_identifiability.tex
\section{Identifiability of Features: Proof of \Cref{thm:identifiability}}
\label{sec:proof_identifiability}


In this section, we prove \Cref{thm:identifiability} under the general assumptions. 
In particular, we consider the following relaxed version of \Cref{asp:data_decomp_main} that accounts for variability in the scaling of the data matrix $X$.
\begin{assumption}[Decomposable data] \label{asp:data_decomp_A} \label{assump:H}
We say that a dataset $X\in\RR^{N\times d}$ is decomposable if there exists a positive integer $n\in\NN$, a diagonal scaling matrix $D\in\RR_+^{N\times N}$ with positive diagonal entries, a nonnegative matrix $H\in\RR_+^{N\times n}$ and a feature matrix $V\in\RR^{n\times d}$ such that
\(
D X = H V.
\)
In addition,  each row of $H$ is normalized to one in $\ell_2$ norm and the $\ell_2$ norm each row of $V$ is normalized to $\Theta (\sqrt d)$. 
Moreover, for every $\ell\in[N]$ and all $i,j\in[n]$, the weight matrix $H\in\RR^{N\times n}$ obeys the following conditions:
\begin{enumerate}[
    label=\textbf{(H\arabic*)}, 
    leftmargin=*, 
    itemsep=0pt, 
    ref=(H\arabic*)
    ]
    \item \textbf{Row-wise sparsity:} Each row has at most $s$ nonzero entries, i.e., $\|H_{\ell,:}\|_0 \le s$ with $s=\Theta(1)$. 
    \item \textbf{Non-degeneracy:} For every $i \in [n]$, the average magnitude of the nonzero entries is constant, i.e.,
    \(
    {\|H_{:,i}\|_1}/{\|H_{:,i}\|_0} = \Theta(1).
    \)
    \item \textbf{Low co-occurrence:} For any two distinct columns, the frequency of simultaneous nonzeros is small, i.e., 
    \( \rho_2\defeq \max_{i\neq j}{\langle \ind \{H_{:,i} \neq 0\},\, \ind \{H_{:,j} \neq 0\} \rangle}/{\|H_{:,i}\|_0} \ll n^{-1/2}\).
\end{enumerate}
We further assume that the feature matrix $V \in \mathbb{R}^{n\times d}$ satisfies:
\begin{enumerate}[
    label=\textbf{(V\arabic*)}, 
    leftmargin=*, 
    itemsep=0pt, 
    ref=(V\arabic*)
    ]
    \item \textbf{Incoherence:} Features are nearly orthogonal, i.e.,
    \(
    {|\langle v_i, v_j \rangle|}/{\|v_i\|_2 \, \|v_j\|_2} = o(1) \text{ for all } i \neq j.
    \)
\end{enumerate}
\end{assumption}

Suppose that there exists another pair of $D' \in \RR^{N\times N}$, $H'\in \RR^{N\times n'}$ and $V'\in\RR^{n'\times d}$ such that all the conditions for $H$ and $V$ also hold for $H'$ and $V'$, and $X = D^{-1}H V = (D')^{-1}H' V'$.
Our goal is to show that under these conditions, the alternative factorization must essentially recover the same feature structure as the original: specifically, \highlight{each feature 
$v_i$ in $V$ can be expressed as a linear combination of a disjoint subset of the features in $V'$}, thereby establishing the identifiability of the feature recovery.

\paragraph{Reduction to the normalized case} In the main text (\Cref{asp:data_decomp_main}), we work with the \emph{normalized} factorization $X = HV,$ which corresponds to the special case $D = I_N$ in  \Cref{asp:data_decomp_A}.  Since all of our assumptions on $H$ and $V$ are identical in both settings, the arguments we present here for $DX = HV$
apply equally to the normalized case. 
Moreover, allowing an arbitrary $D$ lets us accommodate datasets whose rows of $X$ may be arbitrarily large or small in norm. 
The identifiability conclusion remains unchanged if we redefine \(\mathcal{G}\) as the class of normalized pairs \((\tilde H,\,V)\), with \(\tilde H = D^{-1}H\) and \((D,\,H,\,V)\) satisfying Assumption~\ref{asp:data_decomp_A}.
Consequently, \Cref{thm:identifiability} in the main text follows immediately from the more general result shown here.

\paragraph{Proof outline}
We proceed the proof in four main steps:
\begin{enumerate}[leftmargin=1.5em]
  \item \textbf{Row‐scale consistency.}  
    Given two decompositions $ D\,X=H\,V $ and $ D'\,X=H'\,V' $, we show
    that any two row-scalings $D$ and $D'$ differ only by constant factors (i.e., $D(D')^{-1} = \Theta(1)$).

  \item \textbf{Approximate pseudo‐inverse and weight vector construction.} 
    To identify the features from $HV = D(D')^{-1}H'V'$, we construct a matrix $ A $ such that $ A H \approx I_n $, and define weight vectors
    \begin{align}
    \omega_i = (A H)_{i,:}^\top, \quad \omega'_i = \bigl(A\,D(D')^{-1}H'\bigr)_{i,:}^\top.
    \end{align}
    From $AH\approx I_n$ and bounded scaling of $D(D')^{-1}$, we derive that each $\omega_i$ and $\omega'_i$ satisfies
    $\|\omega_i\|_2,\|\omega'_i\|_2 = \Theta(1)$
    and
    $\langle\omega_i,\omega_j\rangle,\;\langle\omega'_i,\omega'_j\rangle = o(1)$ for all $i\neq j$.

  \item \textbf{Disjoint support and feature reconstruction.}  
    Since weight vectors are nearly orthogonal, the large entries of each $\omega'_i$ cannot overlap. Thresholding each $\omega'_i$ at the level $\sqrt{2a}$ (with $a=\max_{i\neq j}\langle\omega'_i,\omega'_j\rangle$) produces disjoint support sets $\mathcal{K}_i^\star$.  Consequently, each original feature $v_i$ is associated with a unique subset of primed features in $V'$.  We then form
    \begin{align}
        u_i = \omega_{ii}\,v_i,
        \quad
        u'_i = \sum_{k\in\mathcal{K}_i^\star}\omega'_{i,k}\,v'_k,
    \end{align}
where $u_i$ is the true feature $v_i$ scaled by $\omega_{ii}$ and $u'_i$ is its reconstruction from the primed features, and show
    $ \|u_i - u'_i\|_2 = o(\sqrt d) $, so each original feature is well‐recovered.

  \item \textbf{From small residuals to cosine closeness.}  
    Since $ \|u_i\|_2,\|u'_i\|_2 = \Theta(\sqrt d) $ and $ \|u_i-u'_i\|_2 = o(\sqrt d) $,
    it follows that $ 1-\cos(v_i,u'_i)=o(1) $, and hence $ \varepsilon=o(1) $ identifiability.
\end{enumerate}

Now, we proceed with the proof of \Cref{thm:identifiability} by showing the above steps in detail.

\paragraph{Step 1: Row‐scale consistency}
We show that the elements of the diagonal matrix $D(D')^{-1}$ are of a constant order. 
By \Cref{asp:data_decomp_A}, we have $HV = D(D')^{-1}H'V'$ which implies that $ \langle H_{\ell,:}V, H_{\ell,:}V \rangle   = D_{\ell\ell}^2(D'_{\ell\ell})^{-2} \langle H'_{\ell,:}V',  H'_{\ell,:}V' \rangle$ for all $\ell\in[N]$.
Using the {incoherence} \labelcref{cond:V-incoherence} assumption, we can expand the inner product as follows:
\begin{align}
    \langle H_{\ell,:}V, H_{\ell,:}V \rangle & = H_{\ell,:}\diag(VV^\top) H_{\ell,:}^\top + H_{\ell,:}(VV^\top - \diag(VV^\top))H_{\ell,:}^\top \\
    & = \Theta(\|H_{\ell,:}\|_2^2 \cdot d) + o(\|H_{\ell,:}\|_1^2 \cdot d). 
\end{align}
Similarly, we have $\langle H'_{\ell,:}V', H'_{\ell,:}V' \rangle = \Theta(\|H'_{\ell,:}\|_2^2 \cdot d) + o(\|H'_{\ell,:}\|_1^2 \cdot d)$.
Since $\|H_{\ell,:}\|_2 = \|H'_{\ell,:}\|_2 = 1$ and $\|H_{\ell,:}\|_1, \|H'_{\ell,:}\|_1 = \Theta(1)$ by Assumptions \labelcref{cond:H-sparsity} and \labelcref{cond:H-nondeg}, 
by comparing $\langle H_{\ell,:}V, H_{\ell,:}V \rangle$ with $\langle H'_{\ell,:}V', H'_{\ell,:}V' \rangle$,
we conclude that 
$D_{\ell\ell}(D'_{\ell\ell})^{-1} = \Theta(1)$ for all $\ell\in[N]$.

\paragraph{Step 2: Approximate pseudo‐inverse and weight vector construction}
Next, we define the following nonnegative matrix $A \in \RR_+^{n\times N}$: 
\begin{align}
    A_{il} = \begin{cases}
        0 & \text{if } H_{li} = 0 \\ 
        \norm{H_{:, i}}_0^{-1} &  \text{if } H_{li} \neq 0
    \end{cases} \quad \text{for } i \in [n], \ell \in [N], 
\end{align}
Note that $AH\in\RR^{n\times n}$ is a square matrix.
By the non-degeneracy and nonnegativity conditions,  for the diagonal entries of $AH$, we have $(AH)_{ii} = \norm{H_{:, i}}_1/\norm{H_{:, i}}_0 = \Theta(1)$ for all $i \in [n]$.
By the low co-occurrence condition, we have for the off-diagonal entry 
$$(A H)_{ij} = \frac{\langle \ind(H_{:, i}\neq 0), H_{:, j}\rangle}{\norm{H_{:, i}}_0} \le \frac{\langle \ind(H_{:, i}\neq 0), \ind(H_{:, j}\neq 0) \rangle}{\norm{H_{:, i}}_0} \leq \rho_2$$ 
for all $i \neq j$.
The first inequality holds because all the elements in $H$ should be no larger than $1$ in their absolute value. 
By the construction of $A$, we have $AH \approx I_n$ and apply $A$ 
on both sides of the equation $HV = D(D')^{-1}H'V'$ to identify the feature correspondence between $V$ and $V'$.

Let us define the weight vectors
\begin{align}
    \omega_i = (A H)_{i,:}^\top \in \RR_+^n, \qquad
    \omega'_i = \bigl(A\,D(D')^{-1}H'\bigr)_{i,:}^\top \in \RR_+^{n'}.
\end{align}
To understand their behavior, we analyze their norms and pairwise inner products. For each $i \in [n]$, let $\mathcal{D}_i = \{\, \ell \in [N] : H_{\ell i} \neq 0 \,\}$ be the support of the $i$-th column of $H$. The $\ell_1$ norm of $\omega_i$ can be written as
\begin{align}
    \norm{\omega_i}_1 &= \sum_{\ell=1}^N \sum_{j=1}^n A_{i\ell} H_{\ell j}
    = \frac{1}{|\mathcal{D}_i|} \sum_{\ell \in \mathcal{D}_i} \norm{H_{\ell,:}}_1 \in [1, \sqrt{s}],
\end{align}
where the second equality uses the nonnegativity of $H$ and the row-wise averaging in the construction of $A$. Since $\norm{H_{\ell,:}}_2 = 1$ and each row is $s$-sparse, it follows that $\norm{H_{\ell,:}}_1 \in [1, \sqrt{s}]$, implying $\norm{\omega_i}_1 = \Theta(1)$ for all $i \in [n]$.
By a similar argument and the results from the first step that $D_{\ell\ell} (D'_{\ell\ell})^{-1} = \Theta(1)$ for all $\ell \in [N]$, we obtain $\norm{\omega'_i}_1 = \Theta(1)$ for all $i \in [n'].$

Next, for the $\ell_2$ norm of $\omega_i$, we have
\begin{align}
    \norm{\omega_i}_2 = \sqrt{(AH)_{ii}^2 + \sum_{j \neq i} (AH)_{ij}^2} 
    \in \left[ (AH)_{ii},\; \sqrt{(AH)_{ii}^2 + n\rho_2^2} \,\right].
\end{align}
Given that $(AH)_{ii} = \Theta(1)$ and $\rho_2 = o(n^{-1/2})$, it follows that $\norm{\omega_i}_2 = \Theta(1)$.
Also, since $\omega_{ii} = (AH)_{ii} = \Theta(1)$, $ \omega_{ij} = (AH)_{ij} \leq \rho_2 = o(n^{-1/2}),$
the inner product between $\omega_i$ and $\omega_j$ satisfies:
\begin{align}
    \langle \omega_i, \omega_j \rangle &=
    \begin{cases}
        \Theta(1) + n\rho_2^2 = \Theta(1) & \text{if } i = j, \\
        O(\rho_2 + n\rho_2^2) = o(1) & \text{if } i \neq j.
    \end{cases}
    \label{eq:omega_i-omega_j}
\end{align}

To analyze the pairwise inner product of $w_i'$, we examine the corresponding inner products after applying them to the feature matrices. Specifically, we observe that  $\langle (\omega'_i)^\top V, (\omega'_j)^\top V\rangle = \langle \omega_i^\top V, \omega_j^\top V\rangle$, which allows us to relate $\omega'_i$ and $\omega'_j$ through the known properties of $\omega_i$.
Since $\norm{v_i}_2 = \Theta(\sqrt{d})$ and $|\langle v_i, v_j \rangle| = o(d)$ for all $i \neq j$, we obtain:
\begin{align}
    \langle \omega_i^\top V,\; \omega_j^\top V \rangle
    &= \omega_i^\top \diag(VV^\top) \omega_j
    + \omega_i^\top (VV^\top - \diag(VV^\top)) \omega_j \\
    &= \Theta(\langle \omega_i, \omega_j \rangle \cdot d) \pm o(\norm{\omega_i}_1 \norm{\omega_j}_1 \cdot d) = 
    \begin{cases}
        \Theta(d) & \text{if } i = j, \\
        o(d) & \text{if } i \neq j.
    \end{cases}
    \label{eq:wV-innerprod}
\end{align}
By a similar calculation, we also have 
\[\langle (\omega'_i)^\top V, (\omega'_j)^\top V\rangle = \Theta(\langle \omega'_i, \omega'_j\rangle d) \pm  o(\norm{\omega'_i}_1 \norm{\omega'_j}_1 d), \]
where $\norm{\omega'_i}_1 = \Theta(1)$ for all $i \in [n']$.
To ensure the equality  $\langle (\omega'_i)^\top V, (\omega'_j)^\top V\rangle = \langle \omega_i^\top V, \omega_j^\top V\rangle$, the following properties must hold following the calculation in \eqref{eq:wV-innerprod}:
\begin{align}
    \langle \omega'_i, \omega'_j\rangle & = \begin{cases}
        \Theta(1) & \text{if } i = j,  \\ 
        o(1) & \text{if } i \neq j.
    \end{cases}
\end{align}

\paragraph{Step 3: Disjoint support and feature reconstruction}
We have established that the weight vectors $\omega_i$ and $\omega'_i$ are nearly orthogonal, i.e., $\langle \omega_i, \omega_j\rangle = o(1)$ for all $i \neq j$ and $\langle \omega'_i, \omega'_j\rangle = o(1)$ for all $i \neq j$. In this step, we define a set of coordinates for each $\omega'_i$ that captures the significant entries, and show that these sets are disjoint for different $i$ and $j$.
In particular, let $ a:= \max_{i \neq j }\langle \omega'_i, \omega'_j\rangle$, which follows that $a= o(1)$. We define the following set of coordinates for $\omega'_i$:
\begin{align}
    \cK_i^\star = \{ k \in [n'] : \omega'_{ik} \geq \sqrt{2 a} \} \subseteq [n'].
\end{align}
The set $\cK_i^\star$ can be viewed as the support of $\omega'_i$.
We claim that $\cK_i^\star \cup \cK_j^\star = \emptyset$ for all $i \neq j$, that is 
the set of sufficiently large coordinates of $\omega'_i$ and $\omega'_j$ are disjoint for $i\neq j$. 
\begin{lemma} \label{lem:disjoint_K}
    For all $i \neq j$, we have $\cK_i^\star \cap \cK_j^\star = \emptyset$.
\end{lemma}
\begin{proof}[Proof of \Cref{lem:disjoint_K}]
 Recall that we define $ a:= \max_{i \neq j }\langle \omega'_i, \omega'_j\rangle$ and it holds from the relationship $\langle (\omega'_i)^\top V, (\omega'_j)^\top V\rangle = \langle \omega_i^\top V, \omega_j^\top V\rangle$ that $a = o(1)$. 
    Suppose that there exists $k \in \cK_i^\star \cap \cK_j^\star$ for some $i\neq j$.
    By the definition of $\cK_i^\star$, we have by the nonnegativity of the weights that $\omega'_{ik} \ge \sqrt{2a}$ and $\omega'_{jk} \ge \sqrt{2a}$.
    Then, we have $\langle \omega'_i, \omega'_j\rangle  \geq \omega'_{ik} \omega'_{jk} \ge 2a.$
    This contradicts with the definition of $a$ where $\langle \omega'_i, \omega'_j\rangle \le a$ for all $i\neq j$.
\end{proof}
This disjointness plays a crucial role in showing that the feature $v_i$ is essentially a linear combination of the features $\{ v'_k :  k \in \cK_i^\star\}$.
With the definition of $\cK_i^\star$, we define a truncated version of $\omega_i'$ as
\begin{align}
    \omega'_i(\cK_i^\star)_k = \begin{cases}
        \omega'_{ik} & \text{if } k \in \cK_i^\star \\ 
        0 & \text{if } k \notin \cK_i^\star
    \end{cases}, \qquad \text{and}\qquad u'_i =  \omega'_i(\cK_i^\star)^\top V'.
\end{align}
Let us take $\overline\cK_i^\star = [n'] \setminus \cK_i^\star$. 
We define similarly $\omega_i'(\overline\cK_{i}^*)$ the vector which equals $w_i'$ on the set $\overline\cK_{i}^*$ and is zero elsewhere. 
In other words, $\omega_i'(\cK_i^\star)$ and $\omega_i'(\overline\cK_{i}^*)$ are the projection of $\omega_i'$ onto the top elements and the residual $w_i' - \omega_i'(\cK_i^\star)$, respectively.
Then, we can show that $u'_i$ closely approximates $(\omega'_i)^\top V'$.
We can write the residual as
\begin{align} 
    \norm{u'_i- (\omega'_i)^\top V'}_2^2 &= \norm{(\omega'_i(\overline\cK_i^\star))^\top V'}_2^2 \\
    & = (\omega'_i(\overline\cK_i^\star))^\top \diag(V'V'^\top) (\omega'_i(\overline\cK_i^\star)) 
    + \omega'_i(\overline\cK_i^\star)^\top (V'V'^\top - \diag(V'V'^\top)) \omega'_i(\overline\cK_i^\star) \\
    & = \Theta(\norm{\omega'_i(\overline\cK_i^\star)}_2^2 \cdot d) \pm o(\norm{\omega'_i(\overline\cK_i^\star)}_1^2 \cdot d). 
\end{align}
Here, note that the squared $\ell_2$ norm can be upper bounded by the $\ell_1$ norm times the $\ell_\infty$ norm, which is
\begin{align}
    \norm{u'_i- (\omega'_i)^\top V'}_2^2
    & \le \Theta(\norm{\omega'_i(\overline\cK_i^\star)}_1 \norm{\omega'_i(\overline\cK_i^\star)}_\infty  \cdot d) + o(\norm{\omega'_i(\overline\cK_i^\star)}_1^2 \cdot d) 
    = o(d), 
    \label{eq:upper_bound_u1}
\end{align}
where $\overline\cK_i^\star = [n'] \setminus \cK_i^\star$.
In the last equality, we use the property that 
$
    \norm{\omega_{i}'(\overline \cK_i^\star)}_1 \le \norm{\omega_{i}'}_1 = \Theta(1).
$
This implies that $u'_i \approx (\omega'_i)^\top V'$. Similarly, for $\omega_i$ we define
$u_i = \omega_{ii}v_i$ and the residual can be upper bounded as follows:
\begin{align} 
    \norm{u_i - (\omega_i)^\top V}_2^2 &\le \bigr\|\sum_{k \neq i }\omega_{ik}v_k \bigl\|_2^2 
    \le \omega_{i, -i}^\top \diag(VV^\top ) \omega_{i, -i} + \omega_{i, -i}^\top (VV^\top - \diag(VV^\top )) \omega_{i, -i}
    \\
    &\le \norm{\omega_{i,-i}}_2^2 \cdot \Theta(d) + o(\norm{\omega_{i,-i}}_1^2 \cdot d) 
     = o(d), 
     \label{eq:upper_bound_u2}
\end{align}
where $\omega_{i,-i}$ is the vector of the $i$-th row of $AH$ with the $i$-th entry removed. In the last inequality, we use the low co-occurrence condition such that 
\begin{align}
    \norm{\omega_{i, -i}}_2^2 = \sum_{j\neq i} \omega_{i, j}^2 = \sum_{j\neq i} (A H)_{ij}^2 \le n \rho_2^2 = o(1). 
\end{align}
Thus, we have $u_i \approx (\omega_i)^\top V$.



Finally, recall that \(\omega_i^\top V = (\omega'_i)^\top V'\). By the definitions of \(u_i = \omega_{ii} v_i\) and \(u'_i = \omega'_i(\mathcal{K}_i^\star)^\top V'\), we have
$u_i + \left( \omega_i^\top V - u_i \right) = u'_i + \left( (\omega'_i)^\top V' - u'_i \right).$
Applying the triangle inequality together with the residual bounds from \eqref{eq:upper_bound_u1} and \eqref{eq:upper_bound_u2}, we obtain
\[
\|u_i - u'_i\|_2 \le \|(\omega_i)^\top V - u_i\|_2 + \|(\omega'_i)^\top V' - u'_i\|_2 = o(\sqrt{d}),
\]
and since \(\|u_i\|_2 =  \omega_{ii}  \norm{v_i}_2 =\Theta(\sqrt{d})\), we easily conclude that \(\|u'_i\|_2 = \Theta(\sqrt{d})\). 

\paragraph{Step 4: From small residuals to cosine closeness}
Now, we can apply the cosine similarity bound to show that the directions of \(u_i\) and \(u'_i\) are close to each other: 
\[
1 - \cos(u_i, u'_i) \le \frac{\|u_i - u'_i\|_2^2}{2\|u_i\|_2\|u'_i\|_2}  = o(1).
\]
Next, recall that \(u_i = \omega_{ii} v_i\) and \(u'_i = \omega'_i(\mathcal{K}_i^\star)^\top V'\). Since cosine similarity is scale-invariant, we conclude
\[
1 - \cos(v_i, \omega'_i(\mathcal{K}_i^\star)^\top V') = o(1).
\]
Since the ordering of the supports in $\omega'_i(\mathcal{K}_i^\star)$ is irrelevant and the feature order in $V'$ is arbitrary, we can introduce a permutation matrix $Q$ that reorders the coordinates so that the indices in each $\{\mathcal{K}_i^\star\}_{i\in[n]}$ are arranged in increasing order (i.e., for $i< j$, the indices in $\mathcal{K}_i^\star$ come before those in $\mathcal{K}_j^\star$ after reordering).
In particular, let \(\Omega \in \mathbb{R}_+^{n \times n'}\) be a block-diagonal matrix with rows \(\Omega_i := \omega'_i(\mathcal{K}_i^\star)Q^{-1}\). Then:
\[
\|\mathbf{1} - \cos(V, \Omega Q V')\|_\infty = \max_{i \in [n]}\{1 - \cos(v_i, \omega_i'(\cK_i^\star)^\top V') \}= o(1).
\]

This verifies that \((H, V)\) is \(\varepsilon\)-identifiable within \(\mathcal{G}\) with \(\varepsilon = o(1)\), as defined in Definition~\ref{def:identifiability}. Furthermore, if \(n = n'\), then each \(\mathcal{K}_i^\star\) must be a singleton, and
\[
1 - \cos(v_i, v'_i) = o(1) \quad \Rightarrow \quad \frac{\|v_i - v'_i\|_2}{\|v_i\|_2} = o(1).
\]
 Thus, \(V\) and \(V'\) must match up to vanishing perturbation and permutation.
This completes the proof.

%% file: paper/appendix/init_gauss_cond.tex
\section{Good Initialization and Gaussian Conditioning}\label{app:init_gauss_cond}
In this section, we provide proofs for two important lemmas: \Cref{lem:init} on the initialization properties and \Cref{lem:gaussian conditioning_main} on the Gaussian conditioning. These lemmas provide the necessary foundation for analyzing the SAE training dynamics, enabling us to isolate and control the relevant sources of randomness throughout the analysis.

\subsection{Initialization Properties}
If we initialize the network with a sufficiently large number of neurons $M$, then for each neuron, there must exist a feature that aligns well with it.
However, the question is how many neurons we need to achieve a \emph{\highlight{sufficiently large alignment}} and with \emph{\highlight{all features}} of interest simultaneously. \Cref{lem:init} provides an answer to this question.
In particular, we prove that when $M$ is sufficiently large, for each feature $v_i$, we can find a neuron $m_i$ that aligns well with it (\textbf{\emph{InitCond-1}}) while maintaining a small alignment with all other features (\textbf{\emph{InitCond-2}}).

\begin{lemma}[Good initialization]
    \label{lem:init}
    Given $n$ i.i.d. features $\{v_i\}_{i=1}^n$ with $v_i \sim \mathcal{N}(0, I_d)$ and weights $\{w_m^{(0)}\}_{m=1}^M$ independently initialized from the uniform distribution on the unit sphere, then for any constants $\varepsilon \in (0,1)$ and $c > 0$ such that $n^{-c}$ upper bound $\exp(-n^{O(\varepsilon)})$, with probability at least $1 - n^{-c}$ over the randomness of both $\{v_i\}_{i=1}^n$ and $\{w_m^{(0)}\}_{m=1}^M$, one can select a sequence of neurons $\{m_i\}_{i=1}^n$ satisfying the following properties:
    \begin{enumerate}[
        ref=Property \arabic* of \Cref{lem:init},
    ]
        \item For any $i\in[n]$, we have 
        $$\textbf{InitCond-1}: \quad  \langle v_i, w_{m_i}^{(0)}\rangle \ge (1 - \varepsilon) \sqrt{2 \log (M/n)}. $$
        \item For any $i\in[n]$, when conditioned on the selection of neuron $m_i$, which \textbf{aligns well} with feature $v_i$ in the sense of \textbf{InitCond-1}, the distribution of the remaining features $\{v_j\}_{j\neq i}$ remains unchanged, i.e., they are independently drawn from $\mathcal{N}(0, I_d)$.
        \item \label{prop:init-cond-2}For any $i\in[n]$, when conditioned on selecting neuron $m_i$, with probability at least $1 - n^{-1-4\varepsilon}$ over the randomness of $\{v_{j}\}_{j\neq i}$, we have $$\textbf{InitCond-2}:  \quad \langle v_j, w_{m_i}^{(0)}\rangle \le \sqrt 2 (1 + \varepsilon) \cdot \sqrt{2\log n}, \quad \forall j\neq i$$
    \end{enumerate}
\end{lemma}

\begin{proof}[Proof of \Cref{lem:init}]
    \label{app:proof-init}

We present the proof by constructing such $m_1, m_2, \ldots, m_n$ explicitly. 
Suppose we are provided with $n$ features $v_1, v_2, \ldots, v_n$ and $M$ neurons with initial weights $w_1^{(0)}, w_2^{(0)}, \ldots, w_M^{(0)}$.
We first put all the pair-wise alignments $\langle v_i, w_m^{(0)}\rangle$ into a matrix $A \in \RR^{n\times M}$, where $A_{im} = \langle v_i, w_m^{(0)}\rangle$ for $i\in[n]$ and $m\in[M]$.
The algorithm execute as follows for $i$ going from $1$ to $n$:
\begin{enumerate}
    \item Randomly divide the $M$ neurons into $n$ disjoint groups $\cM_1, \cM_2, \ldots, \cM_n$ such that each group $\cM_i$ contains $M/n$ neurons.
    \item For each $\cM_i$, find the neuron $m_i$ as the one that maximizes the alignment with feature $v_i$, i.e.,
    \begin{align}
        m_i = \argmax_{m\in \cM_i} A_{i, m} = \argmax_{m\in \cM_i} \langle v_i, w_m^{(0)}\rangle.
    \end{align}
\end{enumerate}
By construction, we know that the selection of $m_i$ is independent of the selection of $m_j$ for $i\neq j$. 
It is not hard to see that the distribution of $\langle v_i, w_m^{(0)}\rangle$ is the same (up to scaling) as the distribution of the first coordinate of a random vector uniformly distributed on the unit sphere.
Therefore, for each $i\in[n]$, each group $\{A_{i, m}\}_{\cM_i}$ is iid sampled from the following distribution:
\begin{align}
    A_{i,m}\biggiven_{m\in\cM_i} \overset{d}{=} \frac{Z_1 \norm{v_i}_2}{\sqrt{Z_1^2 + \ldots + Z_d^2}}, \where Z_k \sim \cN(0, 1), \forall k\in[d].
\end{align}
By the concentration for Chi-square distribution, \Cref{lem:chi-squared}, we know that the denominator and also the norm of $\norm{v_i}_2$ satisfies  
\begin{align}
    \PP\biggl( \Bigl|\sum_{k=1}^d Z_k^2 - d \Bigr| \ge 2 \sqrt{d \log \delta^{-1}} + 2 \log \delta^{-1} \biggr) \le \delta, \\
    \PP\biggl( \Bigl| \norm{v_i}_2^2 - d \Bigr| \ge 2\sqrt{d \log \delta^{-1}} + 2\log \delta^{-1} \biggr) \le \delta.
\end{align}
To proceed, we label each $d$-dimensional random vector as 
\(
Z^{(i,m)} = (Z_1^{(i,m)}, \ldots, Z_d^{(i,m)}),
\)
where the superscript \((i,m)\) corresponds to feature \(i\) and neuron \(m\). Applying a union bound over all \(n\times M/n\) pairs of $(i, m)$ and choosing \(\delta = {n^{-c}}/{M}\) for some universal constant \(c\), we deduce that with probability at least \(1 - n^{-c}\), the following holds for all \(i\in[n]\) and \(m\in[M]\):
\begin{align}
    A_{i, m} \ge 
        \frac{Z_1^{(i, m)} \bigl(d - C\sqrt{d \log (nM)} - C \log(nM)\bigr)^{1/2}}{\bigl(d + C\sqrt{d \log (nM)} + C \log(nM)\bigr)^{1/2} }, 
\end{align}
where \(C\) is a universal constant.
Moreover, by property of the maximum of Gaussian random variables in \Cref{lem:gaussian_max_lower_tail}, it holds that
\begin{align}
    &\PP\biggl( \max_{m\in \cM_i} Z_1^{(i, m)} \ge (1-\varepsilon/2)\sqrt{2 \log (M/n)} \biggr)  \ge 1 - \exp\Bigl( - \frac{(M/n)^{\varepsilon -\varepsilon^2/4}}{3 \sqrt{\pi \log (M/n)}} \Bigr). 
    \label{eq:gaussian_max_lower_tail}
\end{align}
Here, we divide $\varepsilon$ by $2$ because
\begin{align}
    A_{i, m} \ge Z_1^{(i, m)}\cdot \frac{\bigl(d - C\sqrt{d \log (nM)} - C \log(nM)\bigr)^{1/2}}{\bigl(d + C\sqrt{d \log (nM)} + C \log(nM)\bigr)^{1/2} } \ge \frac{1-\varepsilon}{1-\varepsilon/2} Z_1^{(i, m)}
\end{align}
for small constant $\varepsilon$. Consequently, by multiplying both sides of the inequality inside $\PP(\cdot)$ in \eqref{eq:gaussian_max_lower_tail} by $\frac{1-\varepsilon}{1-\varepsilon/2}$, we can recast the probability statement so that the maximum of $A_{i,m}$ over all $m\in \cM_i$ exceeds $(1-\varepsilon)\sqrt{2\log (M/n)}$.
By taking a union bound for $i\in[n]$, the probability of successfully finding a sequence of neurons $m_1, m_2, \ldots, m_n$ satisfying $A_{i,m} > (1-\varepsilon)\sqrt{2\log (M/n)}$ for all $i\in[n]$ and $m\in[M]$ is at least
\begin{align}
     \PP\Bigl( \forall i\in[n]: \max_{m\in \cM_i} A_{i, m} > (1-\varepsilon)\sqrt{2\log (M/n)} \Bigr) \ge 1 - n \cdot \exp\Bigl( - \frac{(M/n)^{\varepsilon -\varepsilon^2/4}}{3 \sqrt{\pi \log (M/n)}} \Bigr) \ge 1 - n^{-c}.
\end{align}
where we can safely take $c$ to some constant as the failure probability is exponentially small in $n$ given that $M\ge n^2$. 
To this end, we conclude that with probability at least $1 - n^{-c}$, we can find a sequence of non-overlapping neurons $m_1, m_2, \ldots, m_n$ such that $A_{i, m_i} > (1-\varepsilon)\sqrt{2\log (M/n)}$ for all $i\in[n]$.

Observe that the selection of each neuron \(m_i\) is done independently for each feature. Consequently, when we condition on the selection of $m_i$, 
the distribution for the remaining features $\{v_j\}_{j\neq i}$ remains unchanged.
This proves the second statement.

It remains to analyze the probability that
\(
A_{j, m_i} < \sqrt 2(1+\varepsilon) \cdot \sqrt{2 \log n} \quad \text{for all } j\in[n] \text{ and } i\neq j.
\)
By the second statement, we know that when conditioned on neuron $m_i$, the collection \(\{A_{j, m_j}\}_{j\neq i}\) (for any fixed \(i\)) consists of $(n-1)$ independent and identically distributed random variables with distribution \(\mathcal{N}(0,1)\).
Thus, we can apply the tail probability for the maximum of Gaussian random variables in \Cref{lem:max gaussian_tail} to obtain
\begin{align}
    \PP\Bigl( \max_{j\in[n]: j\neq i} A_{j, m_i} > \sqrt 2(1+\varepsilon) \cdot \sqrt{2 \log n} \Bigr)
    \le n^{1 - 2(1+\varepsilon)^2} \le n^{-1- 4\varepsilon}. 
\end{align}
Thus, we prove the last argument for \Cref{lem:init}.
\end{proof}

A direct corollary of \Cref{lem:init} is that  \textbf{InitCond-1} and \textbf{InitCond-2} hold simultaneously for all $i\in[n]$ and $j\neq i$ with probability at least $1 - n^{-c} - n^{-4\varepsilon}\le 1 - n^{-\varepsilon}$ after taking a union bound over the success of \text{InitCond-2} for all $i\in[n]$. 
These two conditions together imply that the neuron $m_i$ exclusively focuses on feature $v_i$ at initialization, which is crucial for developing a $1-o(1)$ alignment with feature $v_i$ during training.

\subsection{Rewriting the Gradient Descent Iteration}
\paragraph{Single neuron analysis} 
In the previous \Cref{lem:init}, we have shown a correspondence between each feature $v_i$ and a neuron $m_i$ such that the initial weight of neuron $m_i$ aligns well with feature $v_i$ while maintaining small alignments with all other features. In other words, $m_i$ is the neuron that is most likely to learn feature $v_i$ during training. As the neuron dynamics are decoupled under the small output scale assumption, we only need to analyze the dynamics of neuron $m_i$ to understand how feature $v_i$ is learned.

\paragraph{Notation} In the following, we denote by $v$ the feature of interest and by $w_t$ the weight of the corresponding neuron at iteration $t$.
Let $T$ be the maximum number of steps considered and the time step $t$ ranges from $0$ to $T$.
For the sake of notational convenience, 
we also denote the feature of interest by $w_{-1}=v$ and the normalization $\barw_{-1} = v/\norm{v}_2$.
Meanwhile, $w_0=\barw_0$ is the initialization that is already normalized to unit length. 
Here, the bar notation indicates that the vector is normalized to unit length throughout the whole proof. 

\paragraph{Reformulating the iteration}
In this section, we reformulate the gradient descent update \eqref{eq:gd_simplified} to isolate the contribution of a specific feature  $v$ from the remaining features.  
Recall that the data matrix is given by  $X = H V$, where $ H \in \mathbb{R}^{N \times n} $ is the weight matrix and $ V \in \mathbb{R}^{n \times d}$ is the feature matrix.  
The gradient descent update \eqref{eq:gd_simplified} with gradient explicit in \eqref{eq:gd_approx} is 
\begin{align}
    \textbf{Modified BA:}\quad w_t = \frac{w_{t-1}+\eta\,g_t}{\|w_{t-1}+\eta\,g_t\|_2}, \quad \where\quad g_t = \sum_{\ell=1}^N  \varphi(w_{t-1}^\top x_\ell;b_t)x_\ell,
\end{align}
which can be written in terms of $ H $ and $V$ as:
\begin{equation}
\begin{gathered}
    y_t = V \barw_{t-1}, \quad 
    b_t = \cA_t(H y_t), \quad
    u_t = H^\top \varphi(H y_t; b_t), \\
    w_t = V^\top u_t + \eta^{-1} \barw_{t-1}, \quad 
    \barw_t = w_t/\|w_t\|_2.
    \label{eq:gd-iteration_0}
\end{gathered}
\end{equation}
Here, the meaning of these quantities are given as follows: 
\begin{enumerate}[
    label=\textbullet, 
    leftmargin=2em
]
    \item $y_t\in\RR^d$ is the projection of the normalized weight vector onto all the features, which we refer to as the \emph{feature pre-activation}.
    \item $b_t\in\RR$ is the bias term updated by a bias adaptation algorithm $\cA_t(\cdot)$ that depends on the feature preactivation and time $t$.
    \item $u_t\in\RR^n$ is the \emph{feature post-activation} that aggregates the post-activation information from all the data points back to the feature space.
    \item $w_t\in\RR^d$ is the unnormalized weight vector after one step of gradient descent update, and $\barw_t\in\RR^d$ is the normalized weight vector.
\end{enumerate}
In our analysis, as the bias is fixed, $\cA_t(\cdot)$ always returns the same bias value.
However, we keep this general form which can be useful for adapting the current proof framework to handle more complex bias adaptation algorithms.
Note that $\varphi(Hy_t; b_t)\in\RR^N$ obtained from the gradient calculation in \eqref{eq:gd_approx} is not exactly the post-activation (recall definition $\varphi(x; b) = \phi(x+b) + \phi'(x+b) x$, where $\phi$ is the actual activation function. )
However, in the following proof, we will abuse the notation and refer to $\varphi(Hy_t; b_t)$ as the post-activation for brevity.

\label{sec:rewriting-proof-sketch}
\begin{figure}[htbp]
    \centering
    \begin{subfigure}[b]{0.72\textwidth}
        \centering
        \includegraphics[width=\textwidth]{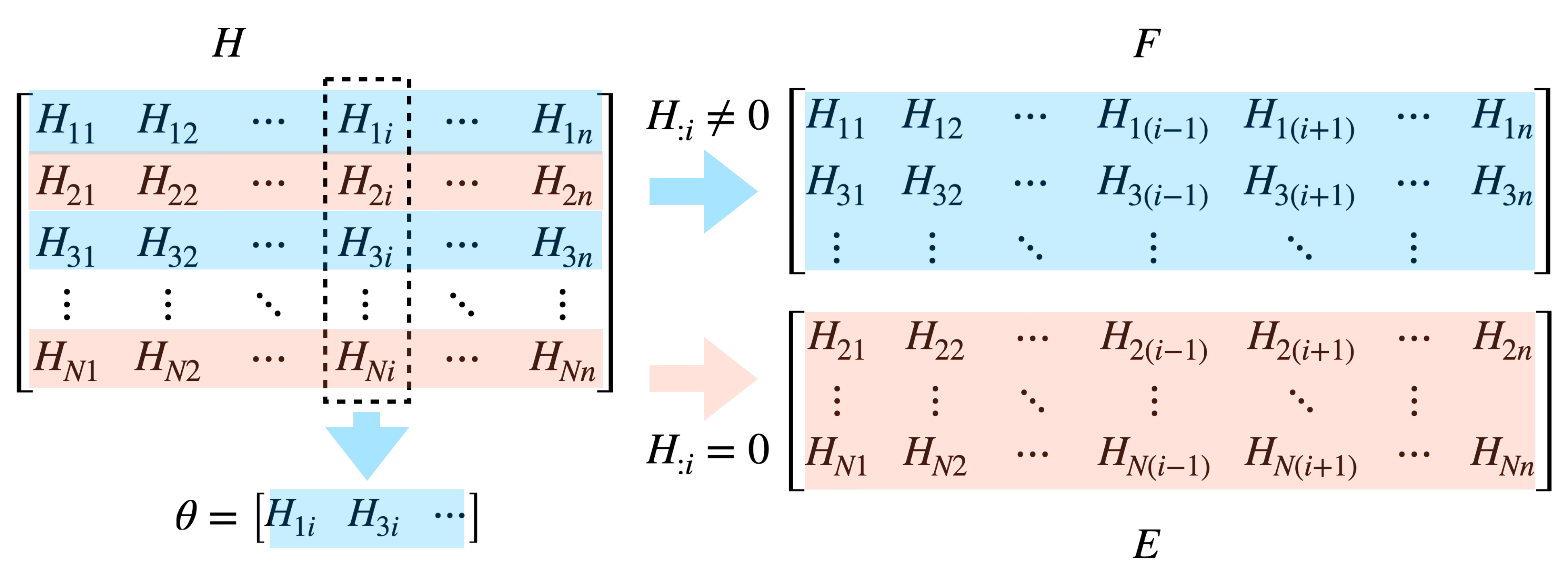}
        \caption{\small The weight matrix $H$ is splitted into matrices $E$ and $F$ by row according to whether the corresponding entries in the $i$-th column are zero or not. The nonzero entries in the $i$-th column of $H$ are collected as vector $\theta$.}
        \label{fig:H_split}
    \end{subfigure}
    \quad
    \begin{subfigure}[b]{0.22\textwidth}
        \centering
        \includegraphics[width=\textwidth]{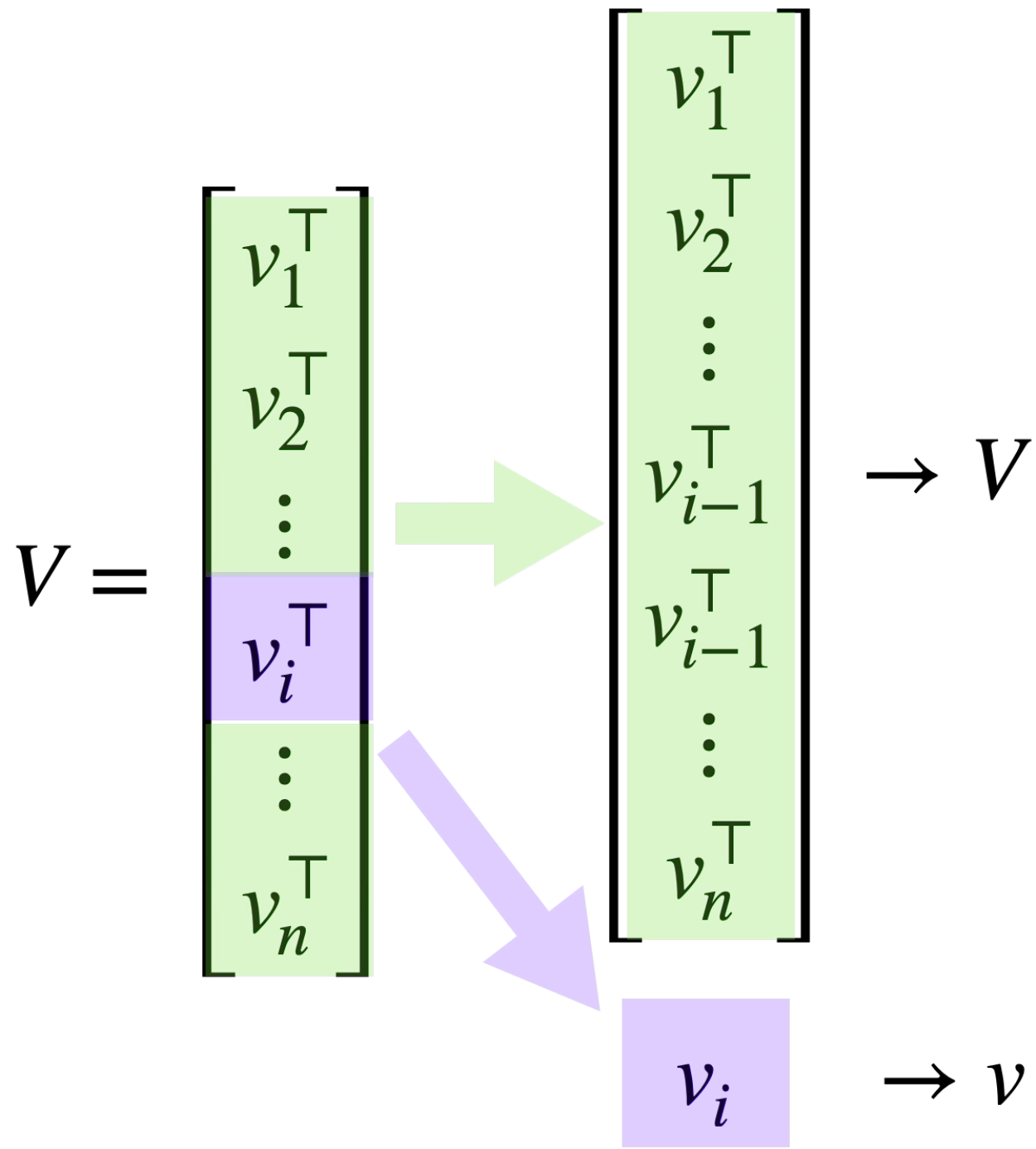}
        \caption{\small Isolating the $i$-th feature from feature matrix $V$.}
        \label{fig:V_split}
    \end{subfigure}
    \caption{\small Illustration of the split of matrices $H$ and $V$.}
    \label{fig:H_V_split}
\end{figure}


Without loss of generality, suppose $v$ is the $i$-th feature.
To \emph{\highlight{isolate the contribution from feature of interest $v$ from the remaining features}}, we decompose the weight matrix $H$ into three parts: \brai{1} $\theta$: the non-zero entries of the $i$-th column, \brai{2} $F$: the rows with non-zero entries in the $i$-th column, and \brai{3} $E$: the remaining rows with zero entries in the $i$-th column. 
Formally, suppose $v$ is the $i$-th feature, then we decompose $H$ as follows:
\begin{align}
    \theta &= \bigl(H_{ki}: H_{ki} \neq 0\bigr)_{k\in[N]}, \: F = \bigl(H_{kj}: H_{ki} \neq 0\bigr)_{k\in[N], j\in[n]\setminus\{i\}}, \: E = \bigl(H_{kj}: H_{ki} = 0\bigr)_{k\in[N], j\in[n]\setminus\{i\}}.
    \label{eq:H_decomposition}
\end{align}
Notably, the rows of $E$ and $F$ do not include the $i$-th column of $H$, as it is already isolated as vector $\theta$.
See \Cref{fig:H_split} for an illustration of this decomposition.  

Using the above decomposition, we can rewrite the actual projection of the weights $\barw_{t-1}$ on each data point as 
\begin{align}
    H V \barw_{t-1} &= \mathrm{Interleave}\bigl([F; E] \cdot V_{-i} \barw_{t-1} + [\theta; \vzero] \cdot v^\top \barw_{t-1} \bigr) \\
    &= \mathrm{Interleave}\bigl([F; E] \cdot y_{t, -i} + [\theta; \vzero] \cdot v^\top \barw_{t-1} \bigr), 
\end{align}
where $[E;F]$ is the vertical concatenation of $E$ and $F$, $V_{-i}$ is the feature matrix $V$ with the $i$-th row removed, and $y_{t, -i} = V_{-i} \barw_{t-1}$ is the vector $y_t$ with the $i$-th entry removed.
The interleave operation simply restores the original order of the rows in $H$.
Therefore, we can rewrite the original $u_t$ in \eqref{eq:gd-iteration_0} as 
\begin{align}
    u_t = H^\top \varphi(H y_t; b_t) = E^\top \varphi(E y_{t, -i}; b_t) + F^\top \varphi(F y_{t, -i} + \theta \cdot v^\top \barw_{t-1}; b_t).
    \label{eq:u_t_rewrite}
\end{align}
\highlight{In order to avoid overcomplicated subscripts, we let $V$ denote the feature matrix $V_{-i}$ with the $i$-th row removed, and let $v$ refer to the original $i$-th row of $V$.}
See \Cref{fig:V_split} for an illustration of this decomposition.
\highlight{We also rewrite $y_{t, -i}$ as $y_{t}$, and following the above notation, we still have $y_t = V \barw_{t-1}$.}
Now with \eqref{eq:u_t_rewrite}, we can explicitly separate the contribution of feature $v$ from the remaining features in the gradient descent iteration \eqref{eq:gd-iteration_0} and obtain the following equivalent iteration:
\begin{AIbox}{Gradient Descent Iteration}
\begin{equation}\label{eq:gd-iteration-loo-simplified}
    \begin{aligned}
    \textbf{feature pre-activation:}\quad & y_t = V \barw_{t-1}, \quad \barw_{t-1} = {w_{t-1} }/{\norm{w_{t-1}}_2},\\
    \textbf{bias scheduling:}\quad & b_t = \cA_t(b_{t-1}, Ey_t, Fy_t + \theta\cdot v^\top \barw_{t-1}), \\
    \textbf{feature post-activation:}\quad & u_t = E^\top \varphi(E y_t; b_t) + F^\top \varphi(F y_t + \theta \cdot v^\top \barw_{t-1}; b_t), \\
    \textbf{weight update:}\quad & w_t = V^\top u_t + v \theta^\top \varphi(F y_t + \theta \cdot v^\top \barw_{t-1}; b_t) + \eta^{-1} \barw_{t-1},\\
\end{aligned}
\end{equation}
\end{AIbox}
\noindent Note that the notation in \eqref{eq:gd-iteration-loo-simplified} is self-consistent with $E, F, \theta$ defined in \eqref{eq:H_decomposition} and $V, v$ defined below \eqref{eq:u_t_rewrite}. We will keep using this notation throughout the rest of the proof.


\subsection{Gaussian Conditioning}
\label{sec:gaussian_conditioning_proof_sketch}


Since both the feature of interest $v$ and each row of the feature matrix $V$ follow Gaussian distributions, we can leverage the properties of Gaussian distributions to simplify the dynamics.
However, the coupling between different iterations prohibits a direct application of Gaussian properties.
This challenge motivates us to explicitly split the intermediate variables in \eqref{eq:gd-iteration-loo-simplified} into two components: \brai{1} a \emph{coupling component} that lies in the subspace spanned by the previous intermediate variables, and \brai{2} an \emph{independent component} that is orthogonal to this subspace.
We can then apply some Gaussian concentration arguments to the orthogonal component to simplify the dynamics.

\paragraph{Additional notation}
To achieve this, we introduce some additional notations.
Let us define $P_{w_{-1:t-1}} x$ as the projection of $x$ onto the subspace spanned by $\{w_{-1}, \ldots, w_{t-1}\}$, and $P_{w_{-1:t-1}}^\perp x = x - P_{w_{-1:t-1}} x$ as the orthogonal projection. In the following, we use the notations $w_t^\perp = P_{w_{-1:t-1}}^\perp w_t$ to denote the new direction induced by $w_t$, 
and we define $u_t^\perp = P_{u_{1:t-1}}^\perp u_t$ in a similar manner (note that $u_t$ starts from $t=1$).
Note that when $t<2$, $u_{1:t-1}$ is empty and $P_{u_{1:t-1}}^\perp$ becomes the identity mapping. 
Also, we enforce $w_{-1} = w_{-1}^\perp =v$. 

In the following, we use the trick of Gaussian conditioning \citep{wu2023lower, bayati2011dynamics, montanari2023adversarial} to simplify the dynamics in \eqref{eq:gd-iteration-loo-simplified}. Specifically, we will define an alternative dynamics that is distributionally equivalent to the original one, where for each iteration, two new independent Gaussian vectors are introduced to replace the original Gaussian components coming from the $V$ matrix. 
To make the presentation clearer, we will denote the variables in the original dynamics in \eqref{eq:gd-iteration-loo-simplified} by $(y_t, w_t, u_t, b_t)$ and the variables in the alternative dynamics by $(\tilde y_t, \tilde w_t, \tilde u_t, \tilde b_t)$ in the following proofs. 
\begin{lemma}[Alternative dynamics]\label{lem:gaussian conditioning_main}
    For any $t\in \NN$, let $z_{-1}, z_0, \dots, z_{t}$ and $\tilde{z}_1, \dots, \tilde{z}_{t}$ be sequences of i.i.d. random vectors from $\mathcal{N}(0, I_{n-1})$ and $\mathcal{N}(0, I_{d-1})$, respectively, with mutual independence. 
    In addition $z_{-1:t}$ and $\tilde {z}_{1:t}$ are also independent of the initialization $\barw_0$ and the feature of interest $v$.
Consider the following alternative iteration for $(\tilde{y}_t, \tilde{w}_t)$:
    \begin{align}
        \tilde{y}_t &= \sum_{\tau=-1}^{t-1} \tilde{\alpha}_{\tau, t-1} \cdot P_{\tilde{u}_{1:\tau}}^\perp {z}_{\tau} + \sum_{\tau=1}^{t-1}\tilde{\alpha}_{\tau, t-1} \cdot \frac{\norm{\tilde{w}_{\tau}^\perp}_2}{\norm{\tilde{u}_{\tau}^\perp}_2} \cdot \frac{\tilde{u}_{\tau}^\perp}{\norm{\tilde{u}_{\tau}^\perp}_2}, \label{eq:gaussian-conditioning-1}\\[1mm]
        \tilde{w}_t &= \sum_{\tau=-1}^{t-1} \langle P_{\tilde{u}_{1:\tau}}^\perp z_\tau, \tilde{u}_t\rangle \cdot \frac{\tilde{w}_\tau^\perp}{\norm{\tilde{w}_\tau^\perp}_2} + \sum_{\tau=1}^{t-1} \frac{\langle  \tilde{u}_{\tau}^\perp, \tilde{u}_t\rangle}{\norm{\tilde{u}_{\tau}^\perp}_2} \cdot \frac{\norm{\tilde{w}_{\tau}^\perp}_2}{\norm{\tilde{u}_{\tau}^\perp}_2} \cdot \frac{\tilde{w}_\tau^\perp}{\norm{\tilde{w}_\tau^\perp}_2} \notag\\[1mm]
        &\qquad + P_{\tilde{w}_{-1: t-1}}^\perp \tilde{z}_{t} \cdot \norm{\tilde{u}_t^\perp}_2 + v\, \theta^\top \varphi(F\tilde{y}_t + \theta \cdot v^\top \tilde{\barw}_{t-1}; b_t) + \eta^{-1}\tilde{\barw}_{t-1},
    \end{align}
    where we define the alignment 
    \[
    \tilde{\alpha}_{\tau, t} = \frac{\langle \tilde{w}_\tau^\perp , \tilde{\barw}_t\rangle }{\norm{\tilde{w}_\tau^\perp}_2} \quad \text{with} \quad \tilde{\barw}_t = \frac{\tilde{w}_t}{\norm{\tilde{w}_t}_2}.
    \]
    In addition, $(b_t, \tilde{u}_t)$ in the alternative dynamics are updated by the same formula as in \eqref{eq:gd-iteration-loo-simplified}:
    \begin{equation}
    \begin{gathered}
        b_t = \cA_t(b_{t-1}, E\tilde{y}_t, F\tilde{y}_t + \theta \cdot v^\top \tilde{\barw}_{t-1}), \quad
        \tilde{u}_t = E^\top \varphi(E\tilde{y}_t; b_t) + F^\top \varphi(F\tilde{y}_t + \theta \cdot v^\top \tilde{\barw}_{t-1}; b_t). \label{eq:gaussian-conditioning-2}
    \end{gathered}
    \end{equation}
Then, conditioned on $\tilde w_{-1}=v$ (the same as our previous definition of $w_{-1}=v$) and $\tilde w_0=w_0$ being the initialization of the neuron weight, the alternative dynamics $(\tilde y_\tau, \tilde w_\tau, \tilde u_\tau, \tilde b_\tau)_{\tau=1}^{t}$ from \eqref{eq:gaussian-conditioning-1}  and \eqref{eq:gaussian-conditioning-2} and the original dynamics  $(y_\tau , w_\tau , u_\tau , b_\tau )_{\tau=1}^{t}$ from \eqref{eq:gd-iteration-loo-simplified} follow the same distribution.
\end{lemma}


\begin{proof}[Proof of \Cref{lem:gaussian conditioning_main}]\label{app:proof-gaussian conditioning}
    To show that the trajectory from \eqref{eq:gaussian-conditioning-1} and \eqref{eq:gaussian-conditioning-2} follow the same distribution as the trajectory from \eqref{eq:gd-iteration-loo-simplified}, we first decompose the iteration in \eqref{eq:gd-iteration-loo-simplified} in the following lemma.
\begin{lemma}[Decomposition]\label{lem:randomness decomposition}
    For the iteration in \eqref{eq:gd-iteration-loo-simplified}, define the alignment between the weight vector $\barw_t$ and the weight direction $w_t^\perp$ as
     $
          \alpha_{\tau, t} = {\langle \barw_t, w_\tau^\perp \rangle}/{\norm{w_\tau^\perp}_2},
     $
     Then, we have the following decomposition for the preactivation vector $y_t \in \RR^{n-1}$:
     \begin{align}
          y_t = \sum_{\tau=-1}^{t-1} \alpha_{\tau, t-1} \cdot P_{u_{1:\tau}}^\perp V \frac{w_\tau^\perp}{\norm{w_\tau^\perp}_2} + \sum_{\tau=1}^{t-1}\alpha_{\tau, t-1} \cdot \frac{\norm{w_{\tau}^\perp}_2}{\norm{u_{\tau}^\perp}_2} \cdot \frac{u_{\tau}^\perp}{\norm{u_{\tau}^\perp}_2},
     \end{align}
     and the following decomposition for the unnormalized weight vector $w_t \in \RR^{d}$:
     \begin{align}
          w_t &= \sum_{\tau=-1}^{t-1} \Bigl\langle P_{u_{1:\tau}}^\perp V \frac{w_\tau^\perp}{\norm{w_\tau^\perp}_2}, u_t\Bigr\rangle \cdot \frac{w_\tau^\perp}{\norm{w_\tau^\perp}_2} + \sum_{\tau=1}^{t-1} \frac{\langle  u_{\tau}^\perp, u_t\rangle}{\norm{u_\tau^\perp}_2} \cdot \frac{\norm{w_{\tau}^\perp}_2}{\norm{u_{\tau}^\perp}_2} \cdot  \frac{w_\tau^\perp}{\norm{w_\tau^\perp}_2} \\
          &\qquad + P_{w_{-1: t-1}}^\perp V^\top \frac{u_\tau^\perp}{\norm{u_\tau^\perp}_2}\cdot \norm{u_t^\perp}_2 + v \theta^\top \varphi(F y_t + \theta \cdot v^\top \barw_{t-1}; b_t) + \eta^{-1}\barw_{t-1}, 
     \end{align}
\end{lemma}
\begin{proof}
See \Cref{app:additional-proofs-gaussian-conditioning} for the proof of \Cref{lem:randomness decomposition}.
\end{proof}
With the above decomposition, if we do the following substitution for $y_t$ and $w_t$ in the above lemma: 
\begin{align} 
    P_{u_{1:t}}^\perp z_t \leftarrow P_{u_{1:t}}^\perp V \frac{w_{t}^\perp}{\norm{w_{t}^\perp}_2},  \qquad P_{w_{-1:t-1}}^\perp \tilde z_t \leftarrow P_{w_{-1:t-1}}^\perp V^\top \frac{u_{t}^\perp}{\norm{u_t^\perp}_2}, 
\end{align} 
the assertion in \Cref{lem:gaussian conditioning_main}  follows immediately. The following proof is devoted to showing that the substitution does not change the joint distribution of the whole dynamics. 
To show that, we just need to verify that for each iteration $t$, when \highlight{\emph{conditioned on all the history up to iteration $t-1$}}, the two newly introduced vectors $P_{u_{1:t}}^\perp V {w_{t}^\perp}/{\norm{w_{t}^\perp}_2}$ and $P_{w_{-1:t-1}}^\perp V^\top {u_{t}^\perp}/{\norm{u_t^\perp}_2}$ still follow a standard Gaussian distribution and are independent of all the history.

To proceed, we denote the original iteration in \eqref{eq:gd-iteration-loo-simplified} by $(y_t, w_t, u_t, b_t)$ and the alternative iteration in \eqref{eq:gaussian-conditioning-1} and \eqref{eq:gaussian-conditioning-2} by $(\tilde{y}_t, \tilde{w}_t, \tilde{u}_t, \tilde{b}_t)$.
Following explicitly from the decomposition in \Cref{lem:randomness decomposition} and the construction in \eqref{eq:gaussian-conditioning-1}, we can further derive the following dependency between the variables in both iterations.
\begin{lemma}
    \label{lem:dependency}
    For each iteration $(u_t,w_t)$ in \eqref{eq:gd-iteration-loo-simplified}, it holds for any $t\ge 1$ that
    \begin{align}
        u_t &\in \sigma\biggl( w_{-1:0}, \biggl\{P_{u_{1:\tau}}^\perp V \frac{w_\tau^\perp}{\norm{w_\tau^\perp}_2}\biggr\}_{\tau=-1}^{t-1}, 
        \biggl\{P_{w_{-1: \tau-1}}^\perp V^\top\frac{u_\tau^\perp}{\norm{u_\tau^\perp}_2}\biggr\}_{\tau = 1}^{t-1}\biggr), \\
        w_t &\in \sigma\biggl( w_{-1:0}, \biggl\{P_{u_{1:\tau}}^\perp V \frac{w_\tau^\perp}{\norm{w_\tau^\perp}_2}\biggr\}_{\tau=-1}^{t-1}, 
        \biggl\{P_{w_{-1: \tau-1}}^\perp V^\top\frac{u_\tau^\perp}{\norm{u_\tau^\perp}_2}\biggr\}_{\tau = 1}^{t}\biggr). 
    \end{align}
    where $\sigma(X)$ denotes the $\sigma$-algebra generated by the random variable $X$. For the Gaussian conditioning iteration $(\tilde u_t, \tilde w_t)$ in \eqref{eq:gaussian-conditioning-1} and \eqref{eq:gaussian-conditioning-2}, it holds for any $t\ge 1$ that    
    \begin{align}
        \tilde u_t \in \sigma(\tilde w_{-1:0}, \{P_{\tilde u_{1:\tau}}^\perp z_\tau\}_{\tau=-1}^{t-1}, \{\tilde z_\tau\}_{\tau=1}^{t-1}), \quad \tilde w_t \in \sigma(\tilde w_{-1:0}, \{P_{\tilde u_{1:\tau}}^\perp z_\tau\}_{\tau=-1}^{t-1}, \{\tilde z_\tau\}_{\tau=1}^{t}).
    \end{align}
\end{lemma}
\begin{proof}
See \Cref{app:additional-proofs-gaussian-conditioning} for a proof of \Cref{lem:dependency}.
\end{proof}
The message of the above lemma is intuitive: each iteration only inserts new randomness coming from 
\begin{align}
    P_{u_{1:{t-1}}}^\perp V \frac{w_{t-1}^\perp}{\norm{w_{t-1}^\perp}_2} \quad \text{and} \quad P_{w_{-1:t-1}}^\perp V^\top \frac{u_t^\perp}{\norm{u_t^\perp}_2}
\end{align}
for the original iteration, and from\begin{align}
    P_{\tilde u_{1:t-1}}^\perp z_{t-1} \quad \text{and} \quad P_{\tilde w_{-1:t-1}}^\perp \tilde z_t
\end{align}
for the alternative iteration.
Using the dependency results, we next prove the equivalence between the trajectory $\{\tilde w_{_{-1}}, \tilde w_{0},   (\tilde{y}_\tau, \allowbreak \tilde{w}_\tau, \tilde u_\tau, \tilde b_\tau)_{\tau=1}^{t}\}$ from the Gaussian conditioning and the trajectory $\{w_{_{-1}}, w_{0}, (y_\tau, w_\tau, u_\tau, b_\tau)_{\tau=1}^{t}\}$ from the original iteration by considering the conditional distribution of the newly introduced randomness at each iteration.
        Let us define $A_t$ as a realization of the random variables $(\tilde w_{-1: 0}, z_{-1:t-1}, \tilde z_{1:t})$ or
        \begin{align}
            \biggl(w_{-1: t}, \biggl\{P_{u_{1:\tau}}^\perp V \frac{w_\tau^\perp}{\norm{w_\tau^\perp}_2}\biggr\}_{\tau=-1}^{t-1}, \biggl\{P_{w_{-1: \tau-1}}^\perp V^\top \frac{u_\tau^\perp}{\norm{u_\tau^\perp}_2}\biggr\}_{\tau = 1}^{t}\biggr).  \end{align}
        By property of the Gaussian ensembles, it holds that
        \begin{align}
            &P_{u_{1:t}}^\perp V \frac{w_t^\perp}{\norm{w_t^\perp}_2} \Bigg|
            \Biggl\{\biggl(w_{-1: 0},  \biggl\{P_{u_{1:\tau}}^\perp V \frac{w_\tau^\perp}{\norm{w_\tau^\perp}_2}\biggr\}_{\tau=-1}^{t-1}, \biggl\{P_{w_{-1: \tau-1}}^\perp V^\top\frac{u_\tau^\perp}{\norm{u_\tau^\perp}_2}\biggr\}_{\tau = 1}^{t}\biggr) = A_t\Biggr\} \\       
            &\quad \overset{d}{=} P_{u_{1:t}}^\perp V_t \frac{w_t^\perp}{\norm{w_t^\perp}_2} \Bigg| 
                \Biggl\{\biggl(w_{-1: 0}, \biggl\{P_{u_{1:\tau}}^\perp V \frac{w_\tau^\perp}{\norm{w_\tau^\perp}_2}\biggr\}_{\tau=-1}^{t-1}, \biggl\{P_{w_{-1: \tau-1}}^\perp V^\top\frac{u_\tau^\perp}{\norm{u_\tau^\perp}_2}\biggr\}_{\tau = 1}^{t}\biggr) = A_t \Biggr\}\\
                &\quad \overset{d}{=} P_{\tilde u_{1:t}}^\perp {z}_t \given \bigl\{(\tilde w_{-1: t}, \tilde u_{1:t}, z_{-1:t-1}, \tilde z_{1:t}) = A_t\bigr\}.
                \label{eq:proof-gaussian-conditioning-1}
        \end{align}
        where $V_t\overset{d}{=} V$ is an independent copy of $V$ and is independent of all the histories.
        Here, the first equality holds because $P_{u_{1:t}}^\perp V w_t^\perp/\norm{w_t^\perp}_2$ is orthogonal to any of the previous row/column space that we have conditioned on. In particular, 
        \begin{enumerate}[
            leftmargin=2em, 
            label=\textbullet
        ]
            \item $P_{u_{1:t}}^\perp V \frac{w_t^\perp}{\norm{w_t^\perp}_2}$ is orthogonal to $\{P_{u_{1:\tau}}^\perp V \frac{w_\tau^\perp}{\norm{w_\tau^\perp}_2}\}_{\tau=-1}^{t-1}$ in the column space of $V$ since $w_t^\perp$ is orthogonal to $w_\tau^\perp$ for any $\tau<t$.
            \item $P_{u_{1:t}}^\perp V \frac{w_t^\perp}{\norm{w_t^\perp}_2}$ is orthogonal to $\{P_{w_{-1: \tau-1}}^\perp V^\top\frac{u_\tau^\perp}{\norm{u_\tau^\perp}_2}\}_{\tau = 1}^{t}$ in the row space of $V$ since $P_{u_{1:t}}^\perp$ is projecting to the row space orthogonal to $u_\tau^\perp$ for any $\tau<t$.
        \end{enumerate}
        Moreover, $V$ is also independent of $w_{-1}=v$ and the initialization $w_0$.
        See \Cref{fig:gaussian_cond} for a more intuitive explanation.
        Therefore, the conditional distribution of $P_{u_{1:t}}^\perp V {w_t^\perp}/{\norm{w_t^\perp}_2}$ is the same as that of an $(n-t)$-dimensional Gaussian vectors. 
        Hence, we are able to replace $V$ by an independent copy $V_t$.
        For the second equality, we can set $z_t = V_t {w_t^\perp}/{\norm{w_t^\perp}_2}$, which is again a Gaussian vector independent of all the histories. 
        Similarly, let $B_t$ be a realization of $(\tilde w_{-1: 0}, z_{-1:t-1}, \tilde z_{1:t-1})$ or
        \begin{align}
            \biggl(w_{-1: 0}, \biggl\{P_{u_{1:\tau}}^\perp V \frac{w_\tau^\perp}{\norm{w_\tau^\perp}_2}\biggr\}_{\tau=-1}^{t-1}, \biggl\{P_{w_{-1: \tau-1}}^\perp V^\top \frac{u_\tau^\perp}{\norm{u_\tau^\perp}_2}\biggr\}_{\tau = 1}^{t-1}\biggr)
        \end{align}
        we similarly have for $P_{w_{-1: t-1}}^\perp V^\top u_t^\perp/\norm{u_t^\perp}_2$ that
        \begin{align}
            &P_{w_{-1: t-1}}^\perp V^\top \frac{u_t^\perp}{\norm{u_t^\perp}_2} \Bigg| \Biggl\{\biggl( w_{-1: 0}, \biggl\{P_{u_{1:\tau}}^\perp V \frac{w_\tau^\perp}{\norm{w_\tau^\perp}_2}\biggr\}_{\tau=-1}^{t-1}, \biggl\{P_{w_{-1: \tau-1}}^\perp V^\top\frac{u_\tau^\perp}{\norm{u_\tau^\perp}_2}\biggr\}_{\tau = 1}^{t-1}\biggr) = B_t\Biggr\} \\
                &\quad \overset{d}{=} P_{\tilde w_{-1: t-1}}^\perp \tilde z_t \given \bigl\{(\tilde w_{-1: 0}, z_{-1:t-1}, \tilde z_{1:t-1}) = B_t \bigr\}.
                \label{eq:proof-gaussian-conditioning-2}
        \end{align}
        To this end, it can be concluded that
        \begin{enumerate}
            \item The initializations $(w_{-1}, w_0)$ and $(\tilde w_{-1}, \tilde w_0)$ are the same.
            \item By \eqref{eq:proof-gaussian-conditioning-1} and \eqref{eq:proof-gaussian-conditioning-2}, we have the same conditional distributions for the updates of $(P_{u_{1:t}}^\perp V w_t^\perp / \norm{w_t^\perp}_2, \allowbreak  P_{w_{-1: t-1}}^\perp V^\top u_t^\perp / \norm{u_t^\perp}_2)$ and those of $(P_{\tilde u_{1:t}}^\perp z_t, P_{\tilde w_{-1: t-1}}^\perp \tilde z_t)$, which means the conditional distributions of $(y_t, w_t)$ and $(\tilde y_t, \tilde w_t)$ given the past are the same.
            \item The updates of $(b_t, u_t)$ and those of $(\tilde b_t, \tilde u_t)$ are also the same. 
        \end{enumerate}
        We hence conclude that the joint distribution for the two iterations are the same for any time $t$. Consequently, we obtain that 
        $$
        \{\tilde w_{_{-1}}, \tilde w_{0},  (\tilde{y}_\tau, \tilde{w}_\tau, \tilde u_\tau, \tilde b_\tau)_{\tau=1}^{t}\} \overset{d}{=} \{w_{_{-1}}, w_{0}, (y_\tau, w_\tau, u_\tau, b_\tau)_{\tau=1}^{t}\}.
        $$
        This completes the proof.
\end{proof}

Since the alternative dynamics in \Cref{lem:gaussian conditioning_main} are distributionally equivalent to the original dynamics, we work exclusively with the alternative formulation below. 
We emphasize the following key point when running the alternative dynamics for $T$ steps:
\begin{olivebox}
    The randomness in the alternative dynamics comes from the initialization $\barw_0$, the feature of interest $v$, and the random vectors $z_{-1:T}$ and $\tilde z_{1:T}$. 
\end{olivebox}
Since the system is rotation-invariant, without loss of generality, we fix the direction of the initialization $\barw_0$ in the following analysis, and only consider the randomness over $v$, $z_{-1:T}$, and $\tilde z_{1:T}$.

\begin{wrapfigure}{r}{0.35\textwidth} 
    \vspace{-20pt} 
    \centering
    \includegraphics[width=0.38\textwidth]{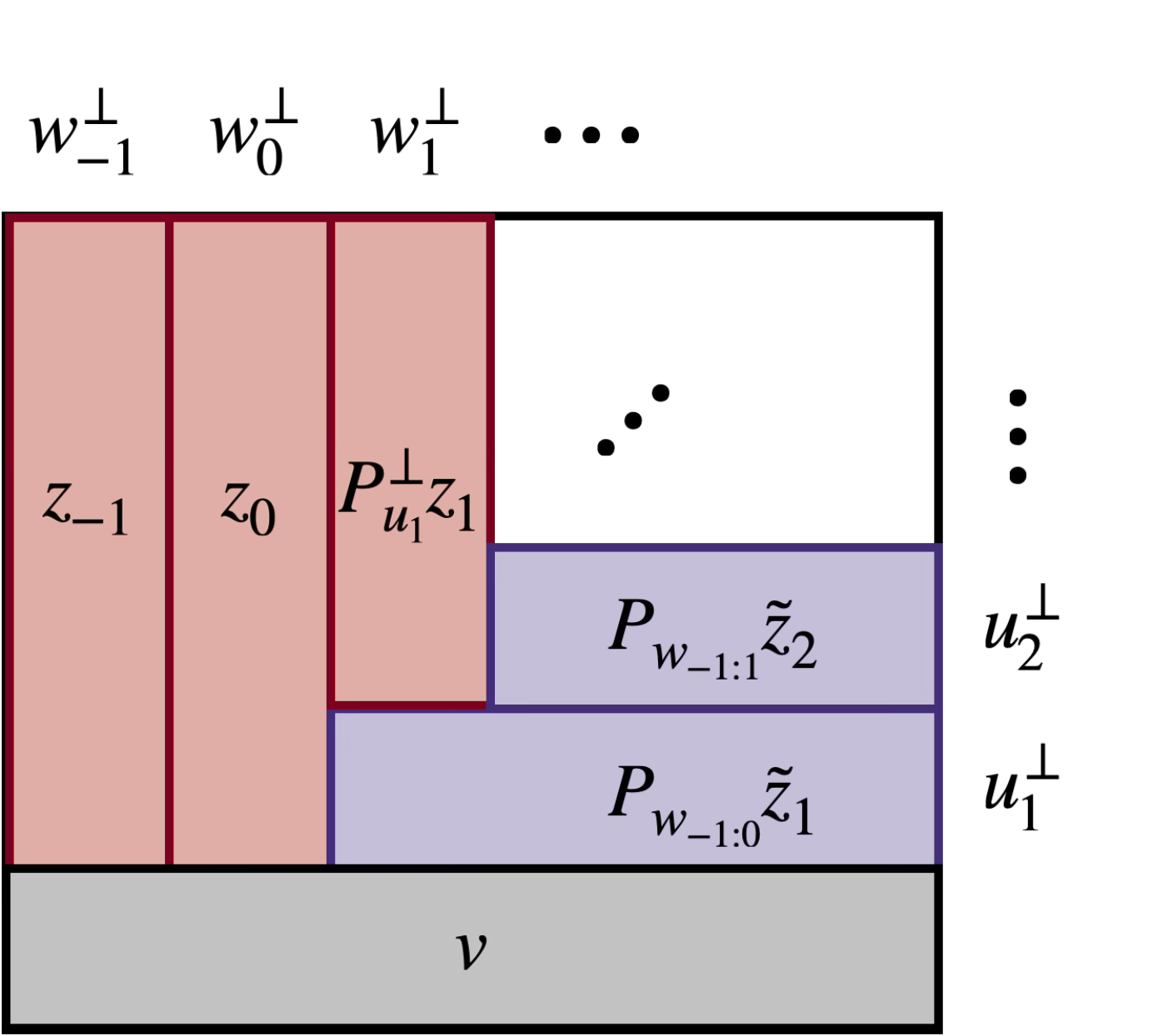}
    \caption{\small Illustration of the Gaussian conditioning. After removing the feature of interest $v$, the remaining  part of $V$ are sliced into $P_{u_{1:t}}^\perp z_t$ and $P_{w_{-1:t-1}}^\perp \tilde z_t$ that are orthogonal to each other.}
    \label{fig:gaussian_cond}
    \vspace{-30pt} 
\end{wrapfigure}
\paragraph{Remark} 
In fact, the iteration in \eqref{eq:gaussian-conditioning-1} is a reformulation of \eqref{eq:gd-iteration-loo-simplified} obtained by decomposing the random matrix $V$ into its projections along the row spaces $u_1^\perp, u_2^\perp, \ldots$ and column spaces $w_1^\perp, w_2^\perp, \ldots$, and then replacing the corresponding components by the following rules:
\begin{align} 
    P_{u_{1:t}}^\perp z_t &\leftarrow P_{\tilde u_{1:t}}^\perp V \frac{\tilde w_{t}^\perp}{\norm{\tilde w_{t}^\perp}_2},  \\ P_{w_{-1:t-1}}^\perp \tilde z_t &\leftarrow P_{\tilde w_{-1:t-1}}^\perp V^\top \frac{\tilde u_{t}^\perp}{\norm{\tilde u_t^\perp}_2}. 
\end{align} 
For a detailed explanation, we refer interested readers to \Cref{lem:randomness decomposition} and its following discussions. In essence, the terms on the right-hand side combine to reconstruct the matrix $V$, as illustrated in \Cref{fig:gaussian_cond}. A crucial property is that these terms are orthogonal in direction; within a Gaussian ensemble, such orthogonality implies their mutual independence. This decoupling of randomness across iterations considerably simplifies the subsequent analysis.

\paragraph{Rewriting the initial conditions under the alternative dynamics}
Let us now specify the randomness in equation \eqref{eq:gd-iteration-loo-simplified} by describing the distributions of the vector $v$ and the matrix $V$. In the absence of any conditioning on the initialization, $v$ and $V$ have i.i.d. standard normal entries. However, the neuron selected for analysis is not arbitrary; it must satisfy the initialization conditions detailed in \Cref{lem:init}. We first restate these conditions in the following more concise form:
\begin{align}
    \langle v, \tilde\barw_0\rangle \ge (1-\varepsilon) \sqrt{2\log(M/n)} \eqdef \zeta_0, \quad \tilde y_1 = V \tilde \barw_0 \preceq  \sqrt 2 (1 + \varepsilon) \cdot \sqrt{2\log n} \cdot \vone \eqdef \zeta_1 \cdot \vone, 
    \label{eq:init-cond-reformulate}
\end{align}
where $a \preceq b$ indicates that every element of $a$ is no greater than the corresponding element of $b$. In fact, these two conditions induce a correlation among $v$, $V$, and the initialization $\barw_0$. 
Under the alternative dynamics in \eqref{eq:gaussian-conditioning-1} and \eqref{eq:gaussian-conditioning-2}, we can reformulate these conditions without involving $V$ as follows:
\begin{align}
   & \textbf{InitCond-1:}\quad \alpha_{-1, 0} \, \|v\|_2 \ge \zeta_0, \quad \textbf{InitCond-2:}\quad  y_1 = \alpha_{-1, 0} z_{-1} + \alpha_{0, 0} z_{0} \preceq \zeta_1 \cdot \mathbf{1},
    \label{eq:init-cond-reformulate-2}
\end{align}
where 
\begin{lightbluebox}
    \begin{equation}
\zeta_0 \defeq (1-\varepsilon)\sqrt{2\log(M/n)}, \quad  \zeta_1 \defeq \sqrt{2}(1+\varepsilon)\sqrt{2\log n}\,.
\label{eq:def-zeta}
\end{equation}
\end{lightbluebox}
\noindent Here, we recall that $\alpha_{-1, 0} = {\langle v, \barw_0\rangle}/{\norm{v}_2}$ and $\alpha_{0, 0} = {\langle w_0^\perp, \barw_0\rangle}/{\norm{w_0^\perp}_2}$.


\paragraph{Decoupling the randomness}
In the following analysis, we can safely decouple the randomness in $v$ and $w_0$ from the randomness in $z_{-1:T}$ and $\tilde z_{1:T}$ by definition of the alternative dynamics.
Notably, the second initial condition in \eqref{eq:init-cond-reformulate-2} only couples $z_{-1}$ and $z_0$ if we treat $\alpha_{-1, 0}$ and $\alpha_{0, 0}$ as deterministic quantities when conditioning on $v$ and $w_0$.
In fact, if we condition on $v$ and $w_0$, the second condition can be satisfied with probability at least $1 -  n^{-\varepsilon}$ by \Cref{lem:init}.


\paragraph{Rewriting the alignment recurrence under the alternative dynamics}
Under the reformulation \eqref{eq:gaussian-conditioning-1}, the alignment we are interested in is $\alpha_{-1, t} = {\langle v, w_t\rangle}/{(\norm{v}_2\norm{w_t}_2)}$. 
Note that in the decomposition of $w_t$, only the terms in the direction of $w_{-1}^\perp = w_{-1} = v$ contribute to the inner product $\langle v, w_t\rangle$. Therefore, the alignment can be expressed as
\begin{align}
    \alpha_{-1, t} = \frac{\langle z_{-1}, u_t\rangle + \norm{v}_2 \cdot \theta^\top \varphi(F y_t + \theta \cdot v^\top \barw_{t-1}; b_t) + \eta^{-1}\alpha_{-1, t-1}}{\norm{w_t}_2}. 
    \label{eq:alignment-LOO-1}
\end{align}
This formula will be useful in the later proof.

\subsection{Additional Proofs}
\label{app:additional-proofs-gaussian-conditioning}
\begin{proof}[Proof of \Cref{lem:randomness decomposition}]
    The proof follows from a direct decomposition of the preactivation vector $y_t$ and the unnormalized weight vector $w_t$.
    By a direct decomposition of $V \barw_t^\perp$, we have
     \begin{align}
        V \barw_t^\perp 
        &= P_{u_{1:t}}^\perp V \barw_t^\perp + {u_t^\perp}\cdot \frac{\langle u_t^\perp,  V \barw_t^\perp\rangle}{\norm{u_t^\perp}_2^2} \cdot \ind(t\ge 1) + P_{u_{1:t-1}} V \barw_t^\perp \\
        &\eqi{1} P_{u_{1:t}}^\perp V \barw_t^\perp + {u_t^\perp}\cdot \frac{\langle V^\top u_t^\perp,  \barw_t^\perp\rangle}{\norm{u_t^\perp}_2^2} \cdot \ind(t\ge 1) \\
        &\eqi{2} P_{u_{1:t}}^\perp V \barw_t^\perp + \frac{u_t^\perp}{\norm{u_t^\perp}_2}\cdot \frac{\norm{w_t^\perp}_2}{\norm{u_t^\perp}_2} \cdot \norm{\barw_t^\perp}_2 \cdot \ind(t\ge 1). 
        \end{align}
        Here, \brai{1} follows from the fact that for any $\tau=1, \ldots, t-1$, 
        \begin{align}
        V^\top u_\tau = w_\tau - v \theta^\top \varphi(F y_\tau + \theta \cdot v^\top \barw_{\tau-1}; b_\tau) - \eta^{-1} \barw_{\tau-1} \in \spn(w_{-1: \tau}),
        \end{align}
    which is orthogonal to $\barw_t^\perp$. In \brai{2}, we use the fact that 
    \begin{align}
        V^\top u_t^\perp - w_t 
        &= V^\top u_t - w_t - V^\top P_{u_{1:t-1}} u_t \\
        &=  - v \theta^\top \varphi(F y_t + \theta \cdot v^\top \barw_{t-1}; b_t) - \eta^{-1} \barw_{t-1} - V^\top P_{u_{1:t-1}} u_t \in \spn(w_{-1: t-1}).
        \end{align}
        Therefore, $\langle V^\top u_t^\perp,  \barw_t^\perp\rangle = \langle w_t,  \barw_t^\perp\rangle = \langle w_t^\perp,  \barw_t^\perp\rangle = \norm{w_t^\perp}_2 \cdot \norm{\barw_t^\perp}_2$.

        Using the above result, we derive for the preactivation vector $y_t$ that 
        \begin{align}
        y_t 
        &= V \barw_{t-1} = \sum_{\tau=-1}^{t-1} \frac{\langle \barw_\tau^\perp, \barw_{t-1}\rangle}{\norm{\barw_{\tau}^\perp}_2^2} \cdot V \barw_\tau^\perp\\ 
        &= \sum_{\tau=-1}^{t-1} \frac{\alpha_{\tau, t-1}}{\norm{\barw_\tau^\perp}_2} \cdot \Bigl(P_{u_{1:\tau}}^\perp V \barw_\tau^\perp + \frac{u_\tau^\perp}{\norm{u_\tau^\perp}_2}\cdot \frac{\norm{w_\tau^\perp}_2}{\norm{u_\tau^\perp}_2} \cdot \norm{\barw_\tau^\perp}_2 \cdot \ind(\tau\ge 1)\Bigr)\\
        &= \sum_{\tau=-1}^{t-1} \alpha_{\tau, t-1} \cdot P_{u_{1:\tau}}^\perp V \frac{w_\tau^\perp}{\norm{w_\tau^\perp}} + \sum_{\tau=1}^{t-1}\alpha_{\tau, t-1} \cdot \frac{\norm{w_{\tau}^\perp}_2}{\norm{u_{\tau}^\perp}_2} \cdot \frac{u_{\tau}^\perp}{\norm{u_{\tau}^\perp}_2}.
        \end{align}
    And also for the unnormalized weight vector $w_t$, we have
    \begin{align}
        &w_t -  v \theta^\top \varphi(F y_t + \theta \cdot v^\top \barw_{t-1}; b_t) - \eta^{-1}\barw_{t-1} \\
        &\quad = P_{w_{-1: t-1}}^\perp V^\top u_t + \sum_{\tau=-1}^{t-1} \frac{\barw_\tau^\perp}{\norm{\barw_\tau^\perp}_2^2 } \cdot \langle V \barw_\tau^\perp, u_t\rangle \\
        &\quad = P_{w_{-1: t-1}}^\perp V^\top \frac{u_t^\perp}{\norm{u_t^\perp}} \cdot \norm{u_t^\perp}_2 + \sum_{\tau=-1}^{t-1} \langle P_{u_{1:\tau}}^\perp V \frac{w_\tau^\perp}{\norm{w_\tau^\perp}}, u_t\rangle \cdot \frac{\barw_\tau^\perp}{\norm{\barw_\tau^\perp}_2} + \sum_{\tau=1}^{t-1} \frac{\langle  u_{\tau}^\perp, u_t\rangle}{\norm{u_\tau^\perp}_2} \cdot \frac{\norm{w_{\tau}^\perp}_2}{\norm{u_{\tau}^\perp}_2} \cdot  \frac{\barw_\tau^\perp}{\norm{\barw_\tau^\perp}_2}. 
        \end{align}
    Therefore, we complete the proof of \Cref{lem:randomness decomposition}.
\end{proof}

\begin{proof}[Proof of \Cref{lem:dependency}]
    Recall that 
\begin{align}
    u_t = E^\top \varphi(E y_t; b_t) + F^\top \varphi(F y_t + \theta \cdot v^\top \barw_{t-1}; b_t).
\end{align}
This implies that $u_t$ can be expressed as a function of $y_t$ only.  This also holds for $\tilde u_t$.
For each iteration $(u_t,w_t)$ in \eqref{eq:gd-iteration-loo-simplified}, it holds by the explicit decomposition in \Cref{lem:randomness decomposition} that
    \begin{align}
        u_t &\in \sigma\biggl(w_{-1: t-1}, \, u_{1: t-1}, \, \biggl\{P_{u_{1:\tau}}^\perp V \frac{w_\tau^\perp}{\norm{w_\tau^\perp}_2}\biggr\}_{\tau=-1}^{t-1}\biggr), \\
        w_t &\in \sigma\biggl(w_{-1: t-1}, \, u_{1: t}, \, \biggl\{P_{u_{1:\tau}}^\perp V \frac{w_\tau^\perp}{\norm{w_\tau^\perp}_2}\biggr\}_{\tau=-1}^{t-1}, \, P_{w_{-1: t-1}}^\perp V^\top \frac{u_t^\perp}{\norm{u_t^\perp}_2}\biggr),
        \label{eq:corollary-sigma-algebra-1}
    \end{align}
    where $\sigma(X)$ denotes the $\sigma$-algebra generated by the random variable $X$. For the Gaussian conditioning iteration $(\tilde u_t, \tilde w_t)$ in \eqref{eq:gaussian-conditioning-1} and \eqref{eq:gaussian-conditioning-2}, it also holds that
    \begin{align}
        \tilde u_t \in \sigma\bigl(\tilde w_{-1: t-1}, \, \tilde u_{1: t-1}, \, \{P_{\tilde u_{1:\tau}}^\perp z_\tau\}_{\tau=-1}^{t-1}\bigr), \quad \tilde w_t \in \sigma\bigl(\tilde w_{-1: t-1}, \, \tilde u_{1: t}, \, \{P_{\tilde u_{1:\tau}}^\perp z_\tau\}_{\tau=-1}^{t-1}, \, P_{\tilde w_{-1:t-1}}^\perp\tilde z_t\bigr).
    \end{align}
Notably, for $u_1$ (only depending on $y_1$) we have
\begin{align}
    y_1 = \alpha_{-1, 0} \cdot V \frac{w_{-1}}{\norm{w_{-1}}_2} = \frac{\langle w_{-1}, \barw_0\rangle }{\norm{w_{-1}}_2} \cdot V \frac{w_{-1}}{\norm{w_{-1}}_2} \in \sigma\Bigl(w_{-1:0}, P_{u_{1:-1}}^\perp V \frac{w_{-1}^\perp}{\norm{w_{-1}^\perp}_2}\Bigr)
\end{align}
by the definition that $P_{u_{1:-1}}^\perp$ is the identity mapping and $w_{-1}^\perp = w_{-1}$.
Similarly, $w_1$ is also measurable by 
\begin{align}
    w_1 &\in \sigma\Bigl(w_{-1:0}, P_{u_{1:-1}}^\perp V \frac{w_{-1}^\perp}{\norm{w_{-1}^\perp}_2}, P_{w_{-1:0}}^\perp V^\top \frac{u_1^\perp}{\norm{u_1^\perp}_2}\Bigr).
\end{align}
This verifies the base case for $t=1$.
Now we can recursively apply the dependency results in \eqref{eq:corollary-sigma-algebra-1} for $t=2, 3, \ldots$ and obtain the desired conclusion. 
This completes the proof of \Cref{lem:dependency}.
\end{proof}

%% file: paper/appendix/sae_dynamics.tex
\section{Concentrations Results for the SAE Dynamics}\label{app:concentration}

\paragraph{Notation} 
In the following proofs, we use the {\color{navyblue!35!white} blue color box} to highlight the definitions that are used in the proofs for readers' convenience, and use the {\color{rliableolive!75!white} olive color box} to highlight different versions of the conditions in \eqref{eq:cond-global} and \eqref{eq:cond-individual} to inform the readers how the conditions evolve throughout the proof.
We use $N_1$ to denote the number of rows in matrix $E$ and $N_2$ to denote the number of rows in matrix $F$. 
In the statement of a lemma, we use $c>4, C>0$ to denote some universal constants that may change from line to line. 
We redefine
\begin{lightbluebox}
\begin{equation}
\begin{aligned} 
    \rho_1 &\defeq \max \Bigl\{ \max_{i\in[n]} \frac{\norm{H_{:, i}}_0}{N}\,,\: \max_{i\neq j}\frac{\sum_{l=1}^N \ind(H_{l, j}\neq 0) \ind(H_{l, i}=0)}{\sum_{l=1}^N \ind(H_{l, i}=0)} \Bigr\}, \\
    \rho_2 &\defeq \max_{i\neq j} \frac{\sum_{l=1}^N \ind(H_{l, i}\neq 0) \ind(H_{l, j}\neq 0)}{\sum_{l=1}^N \ind(H_{l, i}\neq 0)}.
    \label{eq:rho-def}
\end{aligned}
\end{equation}
\end{lightbluebox}
Compared to the original definition in the main text, we add an additional term in the definition of $\rho_1$. We remark that this is not an issue as 
\begin{align}
    \max_{i\neq j}\frac{\sum_{l=1}^N \ind(H_{l, j}\neq 0) \ind(H_{l, i}=0)}{\sum_{l=1}^N \ind(H_{l, i}=0)} \le \max_{i\neq j}\frac{\norm{H_{:,j}}_0}{N - \norm{H_{:, i}}_0} \le \frac{\max_{j\in[n]}\norm{H_{:, j}}_0/N}{1 - \max_{i\in[n]} \norm{H_{:, i}}_0/N}. 
\end{align}
Since we assume in the main theorem that $\max_{i\in[n]} \norm{H_{:, i}}_0/N \ll 1$, we have 
\begin{align}
     \max_{i\neq j}\frac{\sum_{l=1}^N \ind(H_{l, j}\neq 0) \ind(H_{l, i}=0)}{\sum_{l=1}^N \ind(H_{l, i}=0)} \le (1+o(1)) \cdot \max_{i\in[n]} \frac{\norm{H_{:, i}}_0}{N}. 
\end{align}
The two terms in the definition of $\rho_1$ are only different up to a factor of $1 + o(1)$, and hence we can safely stick to the new definition of $\rho_1$ in the proof.
Consequently, $\rho_1 \ge \max_{i\in [n-1]} \norm{E_{:, i}}_0/N_1$, $\rho_2 \ge \max_{i\in [n-1]} \norm{F_{:, i}}_0/N_2$.
In addition, $N_1 \ge (1-\rho_1) N$. By assuming $\rho_1\le 1/2$, we have $N_1 \ge N/2$. 
We use notation $z = x\pm y$ to indicate $z \in [x-y, x+y]$.

\paragraph{Initialization conditions}
In the following analysis, we focus on a single neuron whose initialization satisfies the conditions in \eqref{eq:init-cond-reformulate-2} for a given feature of interest, $v$. For clarity, we restate the initialization conditions:
\begin{align}
   & \textbf{InitCond-1:}\quad \alpha_{-1, 0} \, \|v\|_2 \ge \zeta_0, \quad \textbf{InitCond-2:}\quad  y_1 = \alpha_{-1, 0} z_{-1} + \alpha_{0, 0} z_{0} \preceq \zeta_1 \cdot \mathbf{1},
    \label{eq:init-cond-reformulate-3}
\end{align}
where 
\begin{lightbluebox}
    \[
\zeta_0 \defeq (1-\varepsilon)\sqrt{2\log(M/n)}, \quad  \zeta_1 \defeq \sqrt{2}(1+\varepsilon)\sqrt{2\log n}\,.
\]
\end{lightbluebox}
\noindent Once \textbf{InitCond-1} is satisfied for fixed $w_0$ and $v$, it remains to ensure that the Gaussian vectors $z_{-1}$ and $z_{0}$  satisfy \textbf{InitCond-2}. 
In the subsequent analysis, we sometimes relax \textbf{InitCond-2} so as to leverage the standard Gaussian properties of $z_{-1}$ and $z_{0}$. 
In fact, if an event $\cE$ holds with probabiity $1 - p$ without enforcing \textbf{InitCond-2}, then the joint event that both \textbf{InitCond-2} and $\cE$ hold occurs with probability at least $1 - p - n^{-\varepsilon}$ by a union bound. 
\emph{\highlight{For this reason, unless otherwise specified, we }}

\paragraph{Roadmap}
In \Cref{app:decompose-y}, we decompose the pre-activation $y_t$ into two parts: the Gaussian component $y_t^\star$, which aggregates independent Gaussian contributions and captures the nominal dynamics, and the non-Gaussian component $\Delta y_t$, which accounts for deviations induced by cross-iteration coupling that is typically non-Gaussian. Using this decomposition, in \Cref{app:sparse-activation} we demonstrate that only a small fraction of the training examples activate the neuron—a phenomenon we refer to as sparse activation.

\subsection{Isolation of Gaussian Component}
\label{app:decompose-y}
As is discussed in \Cref{sec:gaussian_conditioning_proof_sketch}, the key step in our analysis is to isolate the Gaussian component from the non-Gaussian component.
In the following, we decompose $y_t$, which is the alignments between the weight and all features, into the Gaussian component that contains weighted sum of i.i.d. Gaussian vectors, and a non-Gaussian part whose $\ell_2$-norm can be bounded by tracking the evolution of the dynamics. 
Recall the definition of $y_t$ in \eqref{eq:gaussian-conditioning-1}, we use the fact that $P_{u_{1:\tau}}^\perp z_\tau = z_\tau - P_{u_{1:\tau}} z_\tau$ to decompose $y_t$ as
\begin{align}
    y_t &= \sum_{\tau=-1}^{t-1} \alpha_{\tau, t-1} \cdot P_{u_{1:\tau}}^\perp z_\tau + \sum_{\tau=1}^{t-1}\alpha_{\tau, t-1} \cdot \frac{\norm{w_{\tau}^\perp}_2}{\norm{u_{\tau}^\perp}_2} \cdot \frac{u_{\tau}^\perp}{\norm{u_{\tau}^\perp}_2} \\
    &= \sum_{\tau=-1}^{t-1} \alpha_{\tau, t-1} \cdot z_\tau  + \Bigl(\sum_{\tau=1}^{t-1}\alpha_{\tau, t-1} \cdot \frac{\norm{w_{\tau}^\perp}_2}{\norm{u_{\tau}^\perp}_2} \cdot \frac{u_{\tau}^\perp}{\norm{u_{\tau}^\perp}_2} - \sum_{\tau=1}^{t-1} \alpha_{\tau, t-1} \cdot P_{u_{1:\tau}} z_\tau\Bigr).
\end{align}
We can thus define the Gaussian component $y_t^\star$ and the non-Gaussian component $\Delta y_t$ as
\begin{lightbluebox}
\begin{equation}
     y_t^\star \defeq \sum_{\tau=-1}^{t-1} \alpha_{\tau, t-1} \cdot z_\tau , \quad  \Delta y_t \defeq \sum_{\tau=1}^{t-1} \alpha_{\tau, t-1} \cdot \frac{\norm{w_{\tau}^\perp}_2}{\norm{u_{\tau}^\perp}_2} \cdot \frac{u_{\tau}^\perp}{\norm{u_{\tau}^\perp}_2} - \sum_{\tau=1}^{t-1} \alpha_{\tau, t-1} \cdot P_{u_{1:\tau}}z_\tau.
    \label{eq:y-decompose}
\end{equation}
\end{lightbluebox}
In the above, the Gaussian component 
\(
y_t^\star=\sum_{\tau=-1}^{t-1}\alpha_{\tau,t-1}z_\tau
\)
is obtained by summing independent Gaussian vectors \(z_{-1},z_0,\dots,z_{t-1}\) with weights \(\alpha_{\tau,t-1}\). Conditional on these coefficients, \(y_t^\star\) is simply a standard Gaussian vector independent of the learned directions \(w_{1:t-1}\) and \(u_{1:t-1}\). In contrast, the non-Gaussian component \(\Delta y_t\) quantifies the deviation of the true feature pre-activation \(y_t\) from \(y_t^\star\) due to cross-iteration coupling.

In the sequel, let us recall the form of \(\alpha_{\tau, t-1}\) in \eqref{eq:alignment-LOO-1} and define \(\beta_{t-1}\) as
\begin{lightbluebox}
\begin{equation}
\begin{aligned}
    \alpha_{-1, t} & = \frac{\langle z_{-1}, u_t\rangle + \norm{v}_2 \cdot \theta^\top \varphi(F y_t + \theta \cdot v^\top \barw_{t-1}; b_t) + \eta^{-1}\alpha_{-1, t-1}}{\norm{w_t}_2}, \\ \beta_{t-1} &\defeq \sqrt{\sum_{\tau=1}^{t-1} \alpha_{\tau, t-1}^2} = \norm{P_{w_{-1:0}}^\perp \barw_{t-1}}_2.
\end{aligned}
\label{eq:alpha-beta-def}
\end{equation}
\end{lightbluebox}
Here, $\alpha_{-1, t}$ is the alignment between $\barw_{t}$ and the feature of interest $v = w_{-1}$, and $\beta_{t}$ is the norm of the projection of $\barw_{t}$ onto the subspace orthogonal to both $\barw_{-1}$ and $\barw_0$.
Tracking $\alpha_{-1,t}$ quantifies how far the neuron has progressed from its initialization $\barw_0$ toward the feature direction $\barw_{-1}$. Ideally, we want $\alpha_{-1,t}\to 1$, indicating strong alignment with the feature while remaining confined to the plane spanned by $\barw_{-1}$ and $\barw_0$. In contrast, $\beta_{t}$ measures the extent to which the neuron drifts away from that plane due to the influence of irrelevant features.
We can build an interesting connection between the non-Gaussian component \(\Delta y_t\) and $\beta_{t-1}$ as stated in the following lemma.
\begin{lemma}[Upper bound the non-Gaussian component $\Delta y_t$]\label{lem:delta-y-l2}
    Suppose $T\le \sqrt d$ and $d \in (n^{1/c_1}, n^{c_1})$ for some universal constant $c_1>1$. For all $t=1,\ldots, T$, it holds with probability at least $1 - n^{-c}$ for some universal constants $c, C>0$ that
    \begin{align}
        \norm{\Delta y_t}_2^2 \le C d \cdot \beta_{t-1}^2.
    \end{align}
\end{lemma}
\begin{proof}
    See \Cref{app:proof-delta-y-l2} for a detailed proof.
\end{proof}

\subsection{Sparse Activation}\label{app:sparse-activation}
Before we move on to studying the evolution of $\alpha_{-1,t}$ and $\beta_{t}$ defined in \eqref{eq:alpha-beta-def}, we first present concentration results for the neuron's activation frequency. 
To leverage the benefits of sparse activation, we analyze how the scheduled bias $b_t$ induces sparsity in the neuron. 

\paragraph{Concentration for ideal activation}
We will first study the ideal case where $\Delta y_t = 0$, and then move on to the real case in \Cref{cor:E-activation-ideal} where we replace $y_t^\star$ with $y_t$ in \Cref{lem:E-activation-perturbed}.
For more generality, we present a full version in \Cref{lem:sparse-activation} and derive \Cref{cor:E-activation-ideal} as a direct corollary.
In the following, recall that $e_l$ is the $l$-th row of matrix $E$, which is a submatrix of $H$ defined in \eqref{eq:H_decomposition}.
We study the activation frequency of the neuron on the set of data that does not contain the feature $v$ (i.e., the rows contained in $E$).
\begin{corollary}[Concentration for ideal activation]
    \label{cor:E-activation-ideal}
    Let $e_l$ be the $l$-th row of matrix $E$.
    For $\kappa_0$ as the threshold defined in \Cref{assump:activation}, we denote by $\barb_t = b_t + \kappa_0$.
    Let 
    \(
        y_t^\star = \sum_{\tau=-1}^{t-1} \alpha_{\tau, t-1} z_\tau 
    \) with $z_\tau$ being the i.i.d. standard Gaussian vectors. 
    It holds for all $t\le T\le  n^c$, $\alpha_{t-1} = (\alpha_{-1, t-1}, \ldots, \alpha_{t-1, t-1})^\top \in \SSS^{t}$, $b_t\in\RR$ and any $\delta\in(\exp(-n/4), 1)$ that with probability at least $1-\delta$ over the randomness of $z_{-1:T}$, the following holds:
    \begin{align}
        \frac{1}{N_1}\sum_{l=1}^{N_1} \ind(e_l^\top y_t^\star + \barb_t > 0) \le C \cdot \bigl( \Phi(-\barb_t) + \rho_1 s t \log(n) + \rho_1 s \log(\delta^{-1})\bigr). 
        \label{eq:E-activation-ideal}
    \end{align}
\end{corollary}
\begin{proof}
    This is a direct corollary of \Cref{lem:sparse-activation}.
\end{proof}
Here, a neuron is considered active when its ideal pre-activation $e_l^\top y_t^\star + b_t$ exceeds the threshold $-\kappa_0$. In the idealized setting (i.e., as $N_1\rightarrow\infty$, and $y_t^\star\sim \cN(0, I_{n-1})$), the expected activation frequency is exactly $\Phi(-\barb_t)$, making the $\Phi(-\barb_t)$ term tight. The additional terms in the bound capture the empirical fluctuations in the activation frequency due to data coupling. In particular, the parameter $\rho_1$ quantifies the maximum fraction of data coupled through a single feature, thereby governing the fluctuation term.
A key point to note is that $\alpha_{t-1} \in \SSS^t$ also depends on the randomness of $z_{-1:T}$, hence how to approximate $y_t^\star$ with random Gaussian vector is not straightforward.
In the proof, we decouple the dependence of $y_t^\star$ on $\alpha_{t-1}$ by proving a concentration result for all $\alpha_{t-1}$ that form a covering net of $\SSS^t$, and then take a union bound over the covering net of size $n^{O(t)}$. This gives rise to the $t\log n$ factor in the bound when taking the logarithm of the covering number.

\paragraph{Efron-Stein inequality for handling data correlation}
In proving the lemma, we use a refined version of the Efron-Stein inequality \citep{boucheron2003concentration} to overcome challenges caused by data correlation. In our setting, two data points may be correlated if they share the same feature, which violates the independence assumption required by classical concentration results such as Bernstein's inequality.

Traditional techniques based on the bounded-differences property---for example, McDiarmid's inequality \citep{mcdiarmid1989method}---would treat the left-hand side (LHS) of \eqref{eq:E-activation-ideal} as a function 
\[
f\bigl(y_t^\star(1), \ldots, y_t^\star(n-1)\bigr)
\]
of $(n-1)$ variables, where $y_t^\star(i)$ is the $i$-th coordinate of $y_t^\star$.
Since altering a single coordinate of $y_t^\star$ has the same effect as modifying the projection of $\barw_t$ onto a single feature, and because each feature influences at most a $\rho_1 N_1$ fraction of the terms in the sum on the LHS, we obtain the bounded-differences property
\[|f(y_t^\star(1), \ldots, y_t^\star(i), \ldots, y_t^\star(n-1)) - f(y_t^\star(1), \ldots, y_t^\star(i)', \ldots, y_t^\star(n-1))| \le \rho_1.\]
Consequently, McDiarmid's inequality would yield a fluctuation bound of order
\[
\sqrt{\sum_{i=1}^{n-1}\rho_1^2}\approx \rho_1\sqrt{n},
\]
which is clearly suboptimal. 
Unlike McDiarmid's bounded-differences inequality, which requires each individual input change to have a uniformly small impact on $f$, Efron-Stein only demands a weaker bound on the variance incurred by altering one coordinate.  We defer interested readers to \Cref{app:proof-sparse-activation} for a detailed proof.

\paragraph{Concentration for original activation}
To fully characterize the behavior of the activation, we also need to take into account the non-Gaussian component $\Delta y$. This gives rise to the following lemma. 
\begin{lemma}[Activation with non-Gaussian component]
    \label{lem:E-activation-perturbed}
    Following the setup of \Cref{cor:E-activation-ideal}, suppose $\barb_t < -2$. 
    Then for all $t\le T \le n^c$, $\alpha_{t-1}\in\SSS^t$ and $b_t\in\RR$, it holds with probability at least $1-n^{-c}$ over the randomness of $z_{-1:T}$ that
    \begin{align}
        \frac{1}{N_1} \sum_{l=1}^{N_1} \ind(e_l^\top y_t + \barb_t>0) &\le C \cdot \bigl( \Phi(-\barb_t) + \rho_1 s t \log(n) + \rho_1 |\barb_t|^2 \norm{\Delta y_t}_2^2 \bigr).
    \end{align}
\end{lemma}
\begin{proof}
    See \Cref{app:proof-E-activation-perturbed} for a detailed proof.
\end{proof}
The fluctuation term in the upper bound now depends on both $\rho_1$ and the $\ell_2$ norm of the non-Gaussian $\Delta y_t$. This is because a larger $\norm{\Delta y_t}_2$ can shift the pre-activations further away from the ideal Gaussian case, thereby in the worst case, increasing the activation frequency.

\paragraph{Concentration for $\alpha_{-1,t}$ and $\beta_t$}
We next aim to characterize the evolution of the parameters $\alpha_{-1,t}$ and $\beta_t$ defined in \eqref{eq:alpha-beta-def}.
Note that in the formula of $\alpha_{-1, t}$
\begin{align}
    \alpha_{-1, t-1} &= \frac{\langle z_{-1}, u_t\rangle + \norm{v}_2 \cdot \theta^\top \varphi(F y_t + \theta \cdot v^\top \barw_{t-1}; b_t) + \eta^{-1}\alpha_{-1, t-1}}{\norm{w_t}_2},
\end{align}
we can decompose the first term in the numerator as follows:
\begin{align}
    \langle z_{-1}, u_t\rangle &= \langle z_{-1}, E^\top \varphi(E y_t; b_t)\rangle + \langle z_{-1}, F^\top \varphi(F y_t + \theta \cdot v^\top \barw_{t-1}; b_t)\rangle
\end{align}
according to the defintion of $u_t$ in \eqref{eq:gd-iteration-loo-simplified}.
Here, $E$ and $F$ are the submatrices of $H$ defined in \eqref{eq:H_decomposition}, where $E$ corresponds to the rows not containing the feature of interest $v$, and $F$ corresponds to the rows containing $v$.
To this end, we just need to control
\begin{gather}
    \langle z_\tau, E^\top \varphi(Ey_t; b_t)\rangle, \quad \langle z_\tau, F^\top \varphi(Fy_t + \theta \cdot v^\top \barw_{t-1}; b_t)\rangle, 
    \label{eq:1st-moment-terms-to-control}
\end{gather}
for general $\tau\in[-1:T]$ and then specialize to $\tau=-1$. 
Note that the above two terms for general $\tau$ will also be used in computing the norm of $\|w_t\|_2$ later. 
Let us just consider a simplfied case where $z_\tau$ is independent of $y_t$ (which does not hold in general). 
To control the fluctuation of the above terms, it is important to compute the second-order moments with respect to the randomness of $z_\tau$.
As a concrete example, for the first term, we have the second-order moment computed as
\begin{align}
    \EE_{z_\tau\sim\cN(0, I_{n-1})}\bigl[\langle z_\tau, E^\top \varphi(Ey_t; b_t)\rangle^2\bigr] &= \norm{E^\top \varphi(Ey_t; b_t)}_2^2. 
\end{align}
The second-order moment of the second term can be computed similarly. 
Therefore, as a first step, we will focus on the follwoing two terms:
\begin{gather}
    \norm{E^\top \varphi(Ey_t; b_t)}_2^2, \quad \norm{F^\top \varphi(Fy_t + \theta \cdot v^\top \barw_{t-1}; b_t)}_2^2.
    \label{eq:2nd-moment-terms-to-control}
\end{gather}
In \Cref{app:2nd-order-concentration}, we will first present concentration results for the second-order terms in \eqref{eq:2nd-moment-terms-to-control} and then use them to derive the concentration results for the two first-order terms in \eqref{eq:1st-moment-terms-to-control}. 
In addition, we will also derive the concentration result for the term $\theta^\top \varphi(F y_t + \theta \cdot v^\top \barw_{t-1}; b_t)$ as in the numerator of $\alpha_{-1, t-1}$. 

\subsection{Second Order Concentration}
\label{app:2nd-order-concentration}
In this subsection, we present concentration results for the second-order terms with respect to the Gaussian component $y_t^\star$ defined in \eqref{eq:y-decompose}: 
\begin{align}
\norm{E^\top \varphi(Ey_t^\star; b_t)}_2^2\quad\text{and}\quad \norm{F^\top \varphi(Fy_t^\star+\theta\cdot v^\top \barw_{t-1}; b_t)}_2^2. 
\label{eq:2nd-moment-terms-to-control-ideal}
\end{align}
We will bridge the gap between these two terms and the original terms in \eqref{eq:1st-moment-terms-to-control} by using the analysis of the non-Gaussian component $\Delta y_t$ in \Cref{app:drifting-error}. 
For now, let us focus on the two terms in \eqref{eq:2nd-moment-terms-to-control-ideal}.
We now present our concentration result formally in the following lemma.
\begin{lemma}[Second-order concentration for $E$-related term]
    \label{lem:E-2nd}
    Under \Cref{assump:activation}, let 
    \(
    \barb_t = b_t + \kappa_0 < 0,
    \)
    and assume further that 
    \(
    - \barb_t = \Theta\bigl(\sqrt{\log n}\bigr)
    \)
    and 
    \(
    - \barb_t < \zeta_1,
    \)
    with $\zeta_1$ defined in \eqref{eq:def-zeta} as required by \textbf{InitCond-2}.
    Suppose $\rho_1 < 1 - 1/C_1$ for some universal constant $C_1>0$. Then with probability at least $1-n^{-c}$ over the randomness of standard Gaussian vectors $z_{-1:T}$, it holds for all $t\le T$ with $T\le n^c$ that
    \begin{align}
        \frac{1}{N_1^2}\norm{E^\top \varphi(E y_t^\star; b_t)}_2^2\cdot \ind(\cE_0)
        & \le C L^2 \cdot \rho_1^2 s t^2  (\log n)^2 \cdot \cK_t^2\\
        &\qquad + C L^2 \cdot \Phi(|\barb_t|) \cdot \hat\EE_{l, l'}\left[\Phi\Bigl( |\barb_t|\sqrt{\frac{1-\langle h_l, h_{l'}\rangle}{1+\langle h_l, h_{l'}\rangle}}\Bigr) \langle h_l, h_{l'}\rangle\right].
        \label{eq:E-2nd-moment}
    \end{align}
    where $\hat\EE_{l, l'}$ denotes the empirical average over $l, l' \in [N]$, $h_l$ denotes the $l$-th row of $H$, $L = \gamma_2 + |b_t|\gamma_1$,  and $\cE_0$ is the event such that $z_{-1}$ and $z_0$ satisfy \textbf{InitCond-2}. Here we define $\cK_t$ as 
    \begin{lightbluebox}
    \begin{equation}
    \begin{aligned}
        \cK_t &\defeq
        \left(n\,|\barb_t|\,\Phi\!\biggl(\frac{-\barb_t}{\sqrt{\frac{3}{4}\,\hslash_{4,\star}^2+\frac{1}{4}}}\biggr) \right)^{1/4} 
        + \left(\rho_2 s n |\barb_t| \Phi\biggl(\frac{-\barb_t}{\sqrt{\frac{2}{3} \hslash_{3, \star}^2 + \frac{1}{3}}}\biggr) \right)^{1/4}  
        \\
        &\qquad + \biggl( \Phi\Bigl(-\frac{\barb_t + \hslash_{4, t} \zeta_t}{\sqrt{1-\hslash_{4, t}^2}}\Bigr) + \bigl(\rho_2 s\bigr)^{1/4} \biggr) \cdot \bigl(t\log(n)\bigr)^{1/4} + n^{1/4} \rho_2\,s\,t\log(n), 
        \label{eq:K_t-def}
\end{aligned}
\end{equation}
\end{lightbluebox}
In the above definition, we let $\hslash_{q, \star}$ and $\hslash_{q, t}$ for any positive $q>1$ and time $t\ge 1$ be the smallest real values in $[0, 1]$ such that the following inequalities hold: 
\begin{align}
    \max_{j\in[n]}\frac{1}{|\cD_j|}\sum_{l\in\cD_j} \Phi\biggl(\frac{-\barb_t}{\sqrt{\frac{q-1}{q}H_{l,j}^2 + \frac{1}{q}}} \biggr) &\le \Phi\biggl(\frac{-\barb_t}{\sqrt{\frac{q-1}{q} \hslash_{q, \star}^2 + \frac{1}{q}}} \biggr), 
    \label{eq:hslash-star-def}
    \\
     \max_{j\in[n]}\frac{1}{|\cD_j|}\sum_{l\in\cD_j} \Phi\Bigl(-\frac{\barb_t + H_{l,j} \zeta_t}{\sqrt{1-H_{l,j}^2}}\Bigr)^q &\le \Phi\Bigl(-\frac{\barb_t + \hslash_{q, t} \zeta_t}{\sqrt{1-\hslash_{q, t}^2}}\Bigr)^q.
     \label{eq:hslash-t-def}
\end{align}
Here $\cD_j = \{l\in[N]: h_{l, j} \neq 0\}$ is the set of row indices in matrix $H$ that has non-zero entries in the $j$-th column, and $\zeta_t = \zeta_1 + \ind(t\ge 2) \cdot C (\beta_{t-1} + |\alpha_{-1, t-1}| + |\alpha_{-1, 0}|) \sqrt{t\log(nt)} $ with the value $\zeta_1$ in \textbf{InitCond-2} and $\beta_{t-1} = \sqrt{\sum_{\tau=1}^{t-1} \alpha_{\tau, t-1}^2}$.
\end{lemma}
\begin{proof}
    See \Cref{app:proof-E-2nd} for a detailed proof. 
\end{proof}
\paragraph{Validity of the definition of $\hslash_{q, t}$ and $\hslash_{q, \star}$}
The definitions of $\hslash_{q, \star}$ and $\hslash_{q, 1}$ are valid as the right-hand sides (RHSs) of the above two inequalities are strictly increasing in terms of $\hslash_{q, \star}$ and $\hslash_{q, 1}$, respectively, under the condition $-\barb_t < \zeta_1$.
\begin{enumerate}[
    leftmargin=2em, 
    label=\textbullet
]
    \item To see this for $\hslash_{q, \star}$, we note that $\Phi(\cdot)$ is a strictly decreasing function, while $\frac{-\barb_t}{\sqrt{\frac{q-1}{q}H_{l,j}^2 + \frac{1}{q}}}$ is also strictly decreasing in terms of $H_{l,j}$. Therefore, the composition of the two functions is strictly increasing in terms of $\hslash_{q, \star}$.
    \item To see this for $\hslash_{q,t}$, observe that $\zeta_t \ge \zeta_1 > -c_1 = -\barb_t$, since the bias is fixed at $b_t=b$ in the current algorithm. Moreover, the derivative of the right-hand side of the inequality in \eqref{eq:hslash-t-def} with respect to $\hslash_{q,t}$ is 
    \begin{align}
        \frac{\rd }{\rd x} \Phi\Bigl(-\frac{\barb_t + x \zeta_1}{\sqrt{1-x^2}}\Bigr)^q = q \Phi\Bigl(-\frac{\barb_t + x \zeta_1}{\sqrt{1-x^2}}\Bigr)^{q-1} \cdot p\Bigl(-\frac{\barb_t + x \zeta_1}{\sqrt{1-x^2}}\Bigr) \cdot \frac{\zeta_1 - (-\barb_t) x}{(1-x^2)^{3/2}} > 0. 
        \label{eq:derivative-hslash-t}
    \end{align}
\end{enumerate}
Therefore, the definitions of $\hslash_{q, \star}$ and $\hslash_{q, t}$ as the smallest real values satisfying the inequalities in \eqref{eq:hslash-star-def} and \eqref{eq:hslash-t-def} are valid.

\paragraph{Heuristic derivation for $\norm{E^\top \varphi(Ey_t^\star; b_t)}_2^2$}
The first term involves the submatrix \(E\). 
Before we present the concentration result, let us derive heuristically what the concentration result should look like.
Let us denote by $e_l$ the $l$-th row of matrix $E$. We can compute the expectation of the squared norm as
\begin{align}
    \frac{1}{N_1^2}\cdot \EE\bigl[\|E^\top \varphi(E y_t^\star; b_t)\|_2^2\bigr] 
    &= \frac{1}{N_1^2}\sum_{l, l'=1}^{N_1} \EE\bigl[\bigl|\varphi(e_l^\top y_t^\star; b_t) \cdot \varphi(e_{l'}^\top y_t^\star; b_t)\bigr|\bigr] \cdot \langle e_l, e_{l'}\rangle.
\end{align}
If we assume $\alpha_{:, t-1}$ are fixed, then $y_t^\star$ is just a standard Gaussian vector, and $$(e_l^\top y_t^\star, e_{l'}^\top y_t^\star)
\sim \cN \left( \begin{bmatrix}
0 \\ 0 \end{bmatrix}, \begin{bmatrix} 1 & \langle e_l, e_{l'}\rangle \\ \langle e_l, e_{l'}\rangle & 1 \end{bmatrix} \right).$$
This fact enables a direct upper bound on the expectation, as detailed in \Cref{prop:2nd-moment-expectation}. 
\begin{lemma}\label{prop:2nd-moment-expectation}
    Let $\barb =b + \kappa_0 < 0$. Suppose $|\varphi(x; b)| \le (n\lor d)^{-c_0} + L (x + \barb) \cdot \ind(x> -\barb)$ for some $L>0$ and $c_0>0$ under \Cref{assump:activation}.
    For two independent $x, z\sim \cN(0, 1)$ and $\iota\in(0, 1)$, it holds that
    \begin{align}
        \EE[\varphi(x;b) \varphi(\iota x + \sqrt{1-\iota^2} \cdot z; b) ] \le C L (n\lor d)^{-c_0} + C(L^2+1) \cdot \Phi(|\barb|) \cdot \Phi\Bigl( |\barb|\sqrt{\frac{1-\iota}{1+\iota}}\Bigr).
    \end{align}
\end{lemma}
\begin{proof}
    See \Cref{sec:proof-2nd-moment-expectation} for a detailed proof.
\end{proof}
By relaxing the rows $e_l,e_{l'}$ of $E$ to the corresponding rows $h_l,h_{l'}$ of $H$, we derive the second term in the concentration result \eqref{eq:E-2nd-moment}. The first fluctuation term is obtained again via the Efron-Stein inequality, which needs a careful analysis up to the $4$-th moment. In particular, we also apply a uniform bound over the sphere $\SSS^t$ for $\alpha_{t-1}$, which gives rise to the dependency on $t$ in the definition of $\cK_t$ in \eqref{eq:K_t-def}.

We now turn to the second term in \eqref{eq:2nd-moment-terms-to-control-ideal}, which is $\|F^\top \varphi(Fy_t^\star + \theta \cdot v^\top \barw_{t-1}; b_t)\|_2^2$. 
\begin{lemma}[Second-order concentration for $F$-related term]
    \label{lem:F-2nd}
    Under \Cref{assump:activation}, suppose $b_t \leq -\kappa_0$ and let $L = \gamma_2 + |b_t|\gamma_1$.  For all $t\le T \le n^c$, it holds with probability at least $1-n^{-c}$ over the randomness of standard Gaussian vectors $z_{-1:T}$ that
    \begin{align}
        \frac{1}{N_2^2} \norm{F^\top \varphi(Fy_t^\star + \theta \cdot v^\top \barw_{t-1}; b_t)}_2^2\le C L^2  \rho_2 \cdot \bigl(\overline{\theta^2} \norm{v}_2^2 \alpha_{-1, t-1}^2 + \rho_2 n +\rho_2 t\log n \bigr), 
    \end{align}
    where $\overline{\theta^2} = \norm{\theta}_2^2 / N_2$. 
\end{lemma}
\begin{proof}
    See \Cref{app:proof-F-2nd} for a detailed proof.
\end{proof}

\subsection{First Order Concentration} 
\label{app:1st-order-concentration}
In this subsection, we continue to present the concentration results on the first order terms specified in \eqref{eq:1st-moment-terms-to-control}.
Let's first consider the concentration for $\langle z_\tau, E^\top \varphi(Ey_t^\star; b_t)\rangle$.

\paragraph{Heuristic derivation for $\langle z_\tau, E^\top \varphi(Ey_t^\star; b_t)\rangle$}
Let us recall that $y_t^\star = \sum_{\tau=-1}^{t-1} \alpha_{\tau, t-1} z_\tau$, and we can rewrite the term as 
\begin{align}
    \langle z_\tau, E^\top \varphi(Ey_t^\star; b_t)\rangle &= \sum_{l=1}^{N_1} e_l^\top z_\tau \cdot \varphi(e_l^\top y_t^\star; b_t) 
\end{align}
for $e_l$ being the $l$-th row of matrix $E$.
Moreover, we have for any fixed $\alpha_{t-1} = (\alpha_{-1, t-1}, \ldots, \alpha_{t-1, t-1})^\top \in \SSS^t$ and by the fact that $\norm{e_l}_2 = 1$ for all $l\in[N_1]$, we have
\begin{align}
    (e_l^\top z_\tau, e_l^\top y_t^\star) \sim \cN\left(
    \begin{bmatrix}
        0 \\ 0
    \end{bmatrix},
    \begin{bmatrix}
        1 & \alpha_{\tau, t-1} \\ \alpha_{\tau, t-1} & 1
    \end{bmatrix}
    \right) 
    \label{eq:approx_joint-gaussian}
\end{align}
where $j\in[n-1]$ is the entry index of the vectors.  
Hence, the term we are interested in should be close to 
\begin{align}
    \sum_{l=1}^{N_1} \EE_{\zeta, \xi\iidfrom\cN(0, 1)}\bigl[ \bigl(\alpha_{\tau, t-1} \zeta + \sqrt{1-\alpha_{\tau, t-1}^2} \cdot \xi \bigr) \cdot \varphi(\zeta; b_t) \bigr] = N_1 \cdot \alpha_{\tau, t-1} \cdot \hat\varphi_1(b_t), 
\end{align}
where we define 
\begin{lightbluebox}
$$\hat\varphi_1(b) = \EE_{u\sim\cN(0, 1)}[\varphi(u; b) u].$$
\end{lightbluebox}
Building on this intuition, the following lemma provides the concentration result in more detail.
\begin{lemma}[First-order concentration for $E$-related term]
    \label{lem:E-1st}
    Under the condition of \Cref{lem:E-2nd}, let $L = \gamma_2 + |b_t|\gamma_1$. For all $t\le T \le n^c$, it holds with probability at least $1-n^{-c}$ over the randomness of standard Gaussian vectors $z_{-1:T}$ that
    \begin{align}
        &\Bigl|\frac{1}{N_1} \langle z_\tau, E^\top \varphi(Ey_t^\star; b_t)\rangle - \alpha_{\tau, t-1} \cdot \hat\varphi_1(b_t)\Bigr| \\
        &\quad \le C L \alpha_{\tau, t-1} t \log(n)\cdot \bigl(\sqrt{s\rho_1 \Phi(|\barb_t|) t \log(n)} + s\rho_1 t \log(n)\bigr) \\
        &\qquad\quad+ \frac{C}{N_1} \sqrt{1-\alpha_{\tau, t-1}^2} \cdot \sqrt{\norm{E^\top \varphi(E y_t^\star; b_t)}_2^2 \cdot t \log(n)}. 
    \end{align}
\end{lemma}
\begin{proof}
    See \Cref{app:proof-E-1st} for a detailed proof.
\end{proof}
In the above lemma, we bound the deviation of the first-order term $\langle z_\tau, E^\top \varphi(Ey_t^\star; b_t)\rangle$ from its expectation $\alpha_{\tau, t-1} \cdot \hat\varphi_1(b_t)$ by some $\rho_1$ and $\Phi(|\barb_t|)$-dependent fluctuation terms. The dependence on $\Phi(|\barb_t)$ is consistent with the intuition that sparser activation which avoids unnecessary activations on other features except the one of interest, often leads to less fluctuation. 
The following lemma provides upper and lower bound for $\hat\varphi_1(b_t)$.
\begin{lemma}[Upper and lower bounds for $\hat\varphi_1(b_t)$]
    \label{prop:hatvarphi_1-bound}
    Suppose \Cref{assump:activation} holds and let $\barb_t = b_t + \kappa_0 < 0$, $L = \gamma_2 + |b_t|\gamma_1$.
    If $|\barb_t| = \omega(1)$, and $\kappa_0 |\barb_t| = O(1)$, then
    \begin{align}
        \frac{C_0}{4} \cdot \Phi(|\barb_t|) \le \hatvarphi_1(b_t) \le 2 \cdot C_0 L \Phi(|\barb_t|). \label{eq:varphi-1st}\label{eq:hat-varphi-lb}
    \end{align}
\end{lemma}
\begin{proof}
    See \Cref{sec:proof-hatvarphi_1-bound} for a detailed proof.
\end{proof}
The message from \Cref{prop:hatvarphi_1-bound} is quite straightforward: the expectation term $\hat\varphi_1(b_t)$ is on the same order as the activation sparsity level $\Phi(|\barb_t|)$. 

\paragraph{Heuristic derivation for $\langle z_\tau, F^\top \varphi(Fy_t^\star + \theta \cdot v^\top \barw_{t-1}; b_t)\rangle$}
Similar to the previous case, we still use the approximation in \eqref{eq:approx_joint-gaussian} except that this time each row $f_l$ of $F$ has norm $\sqrt{1 - \theta_l^2}$, and have 
\begin{align}
    (f_l^\top z_\tau, f_l^\top y_t^\star) \sim \cN\left(
    \begin{bmatrix}
        0 \\ 0
    \end{bmatrix}, (1-\theta_l^2) \cdot 
    \begin{bmatrix}
        1 & \alpha_{\tau, t-1} \\ \alpha_{\tau, t-1} & 1
    \end{bmatrix}
    \right) .
\end{align}
This leads to the following approximation:
\begin{align}
    \langle z_\tau, F^\top \varphi(Fy_t^\star + \theta \cdot v^\top \barw_{t-1}; b_t)\rangle \approx \sum_{l=1}^{N_2}\alpha_{\tau, t-1}\sqrt{1-\theta_l^2} \cdot \EE_{x\sim \cN(0, 1)} \Bigl[x \varphi\bigl(\sqrt{1-\theta_l^2} x + \theta_l v^\top \barw_{t-1}; b_t\bigr)\Bigr].
\end{align}
We now present the formal concentration result for $\langle z_\tau, F^\top \varphi(Fy_t^\star + \theta \cdot v^\top \barw_{t-1}; b_t)\rangle$ in the following lemma. 
\begin{lemma}[First-order concentration  for $F$-related term]
    \label{lem:F-1st}
    Under \Cref{assump:activation}, suppose $\barb_t = b_t + \kappa_0\le 0$ and let $L = \gamma_2 + |b_t|\gamma_1$.  For all $\tau < t\le T$ with $T \le n^c$, it holds with probability at least $1-n^{-c}$ over the randomness of standard Gaussian vectors $z_{-1:T}$ that
    \begin{align}
        &\frac{1}{N_2} \Bigl| \langle z_\tau, F^\top \varphi(F y_t^\star + \theta \cdot v^\top \barw_{t-1}; b_t)\rangle - \sum_{l=1}^{N_2}\alpha_{\tau, t-1}\sqrt{1-\theta_l^2} \cdot \EE_{x\sim \cN(0, 1)} \Bigl[x \varphi\bigl(\sqrt{1-\theta_l^2} x + \theta_l v^\top \barw_{t-1}; b_t\bigr)\Bigr]\Bigr| \\
        &\quad \le C L \alpha_{\tau, t-1} \cdot (\sqrt{t\log (n)} + \norm{v}_2 \alpha_{-1, t-1}) \cdot \sqrt{\rho_2 s} \cdot (t\log(n))^{3/2}\\
        &\hspace{2cm} + \frac{C}{N_2} \sqrt{1-\alpha_{\tau, t-1}^2} \cdot \sqrt{\norm{F^\top\varphi(F y + \theta \cdot v^\top \barw_{t-1}; b_t)}_2^2 \cdot t\log(n)} .
        \label{eq:F-1st-concentration}
    \end{align} 
\end{lemma}
\begin{proof}
    See \Cref{app:proof-F-1st} for a detailed proof.
\end{proof}
\paragraph{Heuristic derivation for $\theta^\top \varphi(F y_t + \theta \cdot v^\top \barw_{t-1}; b_t)$}
The last term we need to control is $\theta^\top \varphi(Fy_t^\star +\theta\cdot v^\top \barw_{t-1}; b_t)$. 
Using the Gaussian approximation $f_l^\top y_t^\star\sim \cN(0, 1-\theta_l^2)$ as in the previous case, we have 
\begin{align}
    \theta^\top \varphi(Fy_t^\star +\theta\cdot v^\top \barw_{t-1}; b_t) &\approx \sum_{l=1}^{N_2} \theta_l \cdot \EE_{x\sim\cN(0, 1)}\Bigl[\varphi(\sqrt{1-\theta_l^2} x + \theta_l v^\top \barw_{t-1}; b_t)\Bigr]. 
\end{align}
For our convenience, let us define 
\begin{lightbluebox}
    \begin{equation}
        \psi_t \defeq \frac{\sqrt d }{N} \sum_{l=1}^{N_2} \EE_{x\sim\cN(0,1)}\bigl[\theta_l \cdot \varphi(\sqrt{1-\theta_l^2} \cdot x + \theta_l \cdot v^\top \barw_{t-1}; b_t) \bigr], 
        \label{eq:psi_t-def}
    \end{equation}
    \end{lightbluebox}
\noindent and it follows that $\theta^\top \varphi(Fy_t^\star +\theta\cdot v^\top \barw_{t-1}; b_t) \approx N \cdot \psi_t /\sqrt d$.
Lastly, we present the concentration for $\theta^\top \varphi(Fy_t^\star +\theta\cdot v^\top \barw_{t-1}; b_t)$.
\begin{lemma}[First-order concentration for signal term]
    \label{lem:F-signal}
    Under \Cref{assump:activation}, suppose $\barb_t = b_t + \kappa_0\le 0$ and let $L = \gamma_2 + |b_t|\gamma_1$. For all $t\le T \le n^c$, it holds with probability at least $1-n^{-c}$ over the randomness of standard Gaussian vectors $z_{-1:T}$ that
    \begin{align}
        &\Bigl| \frac{1}{N_2} \theta^\top \varphi(Fy_t^\star +\theta\cdot v^\top \barw_{t-1}; b_t) - \frac{\psi_t N}{\sqrt d N_2} \Bigr| \le C L \bigl(\sqrt{t\log (n)} + \norm{v}_2 \alpha_{-1, t-1}\bigr) \cdot  \sqrt{\rho_2 s \overline{\theta^2}} \cdot t\log(n). 
    \end{align}
\end{lemma}
\begin{proof}
    See \Cref{app:proof-F-signal} for a detailed proof.
\end{proof}
Lastly, we provide a useful bound for the term $\psi_t$ defined in \eqref{eq:psi_t-def} in the following lemma, which is related to the \emph{strength} of the weight vector $\theta$ for the feature of interest. To quantify the strength, we make the following definition 
\begin{lightbluebox}
    \begin{equation}
        Q_t \defeq\frac{1}{N_2}\sum_{l=1}^{N_2} \ind\Bigl(\theta_l > \frac{-b_t}{\sqrt{d} \alpha_{-1, t-1}}\Bigr), \quad \overline{\theta^2} \defeq \frac{\norm{\theta}_2^2}{N_2}.
        \label{eq:Q_t-bartheta-def}
    \end{equation}
    \end{lightbluebox}
\begin{lemma}[Bounds for the signal term] \label{lem:signal-bounds}
    Under \Cref{assump:activation}, it holds for $\psi_t$ defined in \Cref{lem:F-signal} that
    \begin{align}
           C^{-1} \overline{\theta^2} Q_t \cdot N_2   d \alpha_{-1, t-1}  
           \le  N \psi_t 
          \le C L\overline{\theta^2} \cdot N_2  d \alpha_{-1, t-1}. 
    \end{align}
\end{lemma}
\begin{proof}
    See \Cref{sec:proof-signal-bounds} for a detailed proof.
\end{proof}

\subsection{Non-Gaussian Error Propogation}
\label{app:drifting-error}
In the following, let us define the following error terms
\begin{align}
    \Delta E_t &= E^\top \varphi(E y_t; b_t) - E^\top \varphi(Ey_t^\star; b_t), \\ 
    \Delta F_t &= F^\top \varphi(F y_t + \theta \cdot v^\top \barw_{t-1}; b_t) - F^\top \varphi(Fy_t^\star + \theta \cdot v^\top \barw_{t-1}; b_t), \\
    \Delta \varphi_{F, t} &= \varphi(F y_t + \theta \cdot v^\top \barw_{t-1}; b_t) - \varphi(Fy_t^\star + \theta \cdot v^\top \barw_{t-1}; b_t).
\end{align}
The last piece of the puzzle is to control the error propagation in the dynamics due to the non-Gaussian component $\Delta y_t$ in the pre-activation.
Let us recall the error terms
\begin{align}
    \Delta E_t &= E^\top \varphi(E y_t; b_t) - E^\top \varphi(Ey_t^\star; b_t) \\ 
    \Delta F_t &= F^\top \varphi(F y_t + \theta \cdot v^\top \barw_{t-1}; b_t) - F^\top \varphi(Fy_t^\star + \theta \cdot v^\top \barw_{t-1}; b_t).
\end{align}
We are interested in how the error $\Delta y_t$ propagates through the nonlinear function $\varphi$ in the update.

\begin{lemma}[Error propogation for $\Delta E_t$]\label{lem:E-error}
    Under \Cref{assump:activation} on the activation function, let $\barb_t = b_t + \kappa_0$, $L = \gamma_2 + |b_t|\gamma_1$ and suppose $\barb_t < -2$. For all $t\le T \le n^c$, it holds with probability at least $1-n^{-c}$ over the randomness of standard Gaussian vectors $z_{-1:T}$ that
    \begin{align}
        \norm{\Delta E_t}_1 &\le C L N_1 \cdot \Bigl(\bigl( \sqrt{s\rho_1 \Phi(-\barb_t)} + s\rho_1 \sqrt{t \log n} \bigr) \cdot \norm{\Delta y_t}_2 + \sqrt s \rho_1 |\barb_t| \cdot \norm{\Delta y_t}_2^2 \Bigr)  \\
        &\qquad + C N_1 \sqrt s (2+|b_t|) \cdot (n\lor d)^{-c_0},  
    \end{align}
    and the $\ell_2$ norm of $\Delta E_t$ are bounded as
    $
        \norm{\Delta E_t}_2 \le (\gamma_2 + |b_t| \gamma_1)\cdot \rho_1 N_1 \norm{\Delta y_t}_2.
        $
\end{lemma}
\begin{proof}
    See \Cref{app:proof-E-error} for a detailed proof.
\end{proof}
In the above lemma, we incorporate the sparsity in the activation to obtain a more refined bound for $\norm{\Delta E_t}_1$.
Next, we also present the error bound for $\Delta F_t$.
\begin{lemma}[Error propogation for $\Delta F_t$]\label{lem:F-error}
    Define $\Delta\varphi_{F, t} = \varphi(F y_t + \theta \cdot v^\top \barw_{t-1}; b_t) - \varphi(Fy_t^\star + \theta \cdot v^\top \barw_{t-1}; b_t)$. The following bounds hold:
    \begin{enumerate}
        \item $\norm{\Delta F_t}_1 \le \sqrt{s} N_2 L \cdot \norm{\Delta y_t}_2$.
        \item $\norm{\Delta F_t}_2 \le \rho_2 N_2 L \cdot \norm{\Delta y_t}_2$.
        \item $\norm{\Delta \varphi_{F, t}}_2 \le \sqrt{ \rho_2 N_2} L \cdot \norm{\Delta y_t}_2$.
    \end{enumerate}
\end{lemma}
\begin{proof}
    See \Cref{app:proof-F-error} for a detailed proof.
\end{proof}

\section{SAE Dynamics Analysis: Proof of \Cref{thm:sae-dynamics}}\label{sec:sae-dynamics}
In the sequel, we will first state a more general version of \Cref{thm:sae-dynamics}, accompanied by the full details on the related definitions and assumptions that are mentioned in the main text.
Then we will present the proof of the theorem. 

\subsection{A General Version of the Theorem}
In the follwoing, we first state the definition of the \emph{concentration coefficient} $h_\star$ and a general version of the main theorem. Then, we present the rigorous definition of the ReLU-like activation function.

\paragraph{Details on  
concentration parameters $h_\star$}
To measure the magnitude of coefficients associated with each feature, we recall in the definition of the \emph{cut-off} level for feature $i$ in \eqref{eq:s_i-def} as
\begin{align}
    \!\!\! h_i := \max \Bigl\{h\le 1: 
    \frac{1}{|\cD_i|}\sum\nolimits_{l\in\cD_i} \ind\{ H_{l,i} \geq h\} \ge \polylog(n)^{-1} \Bigr\}. 
\end{align}
To measure the concentration level of the global coefficients across all features, we define the \emph{concentration coefficient} $h_\star$ as follows.
We first recall the definitions of $\hslash_{q,\star}$ and $\hslash_{q,t}$ from \Cref{lem:E-2nd} (with $t=1$ for any $q> 1$). In particular, $\hslash_{q,\star}$ and $\hslash_{q,1}$ are defined as the \emph{\highlight{smallest}} numbers satisfying the following inequalities:
\begin{align}
    \max_{j\in[n]}\frac{1}{|\cD_j|}\sum_{l\in\cD_j} \Phi\biggl(\frac{-\barb_t}{\sqrt{\frac{q-1}{q}H_{l,j}^2 + \frac{1}{q}}} \biggr) &\le \Phi\biggl(\frac{-\barb_t}{\sqrt{\frac{q-1}{q} \hslash_{q, \star}^2 + \frac{1}{q}}} \biggr), \\
     \max_{j\in[n]}\frac{1}{|\cD_j|}\sum_{l\in\cD_j} \Phi\Bigl(-\frac{\barb_t + H_{l,j} \zeta_1}{\sqrt{1-H_{l,j}^2}}\Bigr)^q &\le \Phi\Bigl(-\frac{\barb_t + \hslash_{q, 1} \zeta_1}{\sqrt{1-\hslash_{q, 1}^2}}\Bigr)^q. 
\end{align}
Here, $\cD_j = \{l\in[N]: H_{l, j} \neq 0\}$ is the set of row indices in matrix $H$ that has non-zero entries in the $j$-th column, and $\zeta_1 = 2 (1+\varepsilon)\sqrt{\log n}$ that is formally defined in \eqref{eq:def-zeta}. 
Here, $\Phi(\cdot)$ is the tail probability function of the standard Gaussian distribution, i.e., $\Phi(x) = \int_x^\infty  e^{-u^2/2}/\sqrt{2\pi} \cdot du$.
The definitions of $\hslash_{q, \star}$ and $\hslash_{q, 1}$ are valid as the right-hand sides (RHSs) of the above two inequalities are strictly increasing in terms of $\hslash_{q, \star}$ and $\hslash_{q, 1}$, respectively. 
We defer readers to the discussion under \Cref{lem:E-2nd}. 
We define the \emph{concentration coefficient} for the weight matrix $H$, denoted by $h_\star$, as the \emph{smallest} number such that
\begin{align}
    \max\{\hslash_{4, \star}^2, \hslash_{3, \star}^2, \hslash_{4, 1}^2 \} \le h_\star, \quad \sum_{j=1}^{n}\frac{1}{|\cD_j|^2}\sum_{l, l'\in\cD_j} \Phi\biggl(|\barb| \sqrt{\frac{1 - H_{l, j}H_{l', j}}{1 + H_{l, j} H_{l', j}}}\biggr) \le n \Phi\biggl(|\barb| \sqrt{\frac{1-h_\star^2}{1+h_\star^2}}\biggr), 
    \label{eq:s_star-def}
\end{align}
In fact, the RHS of the last inequality in \eqref{eq:s_star-def} is also strictly increasing in terms of $h_\star$, and hence the definition is valid.
In the extreme case  where $H$ does not have any diversity in its nonzero entries, we have the following simple relationship between $s_\star, s_i$ and $s$:
\begin{proposition}[Concentrated coefficient $H$]
    \label{prop:sparsity-concentrated}
    If $H_{lj}\in\{0, 1/\sqrt s\}$ for all $l\in[N]$ and $j\in[n]$, then $h_\star = h_i = 1/\sqrt s$.
\end{proposition}
In this extreme case, every row of $H$ has exactly $s$ non-zero entries, and the non-zero entries are all equal to $1/\sqrt s$.
In the following, let us define $\overline{\theta_i^2} = \norm{\theta_i}_2^2 / N_2$ and $\hat\QQ_i(x) =|\cD_i|^{-1}\sum_{l\in\cD_i} \ind(H_{l,i} \geq x)$ for $x\in[0, 1]$. The following proposition relates $s_i$ and $s_\star$ to the sparsity $s$ through inequalities that must be satisfied. 
\begin{proposition}[General coefficient]
    \label{prop:sparsity}
Recall the definitions of $h_i$ in \eqref{eq:s_i-def} and  $h_\star$ in \eqref{eq:s_star-def}.
Suppose the bias $b< -\sqrt 3$, then for any feature $i\in[n]$ satisfying the conditions in \eqref{eq:cond-global} and \eqref{eq:cond-individual} and that $\overline{\theta_i^2} > \hat\QQ_i(h_i)$,
    we have the following inequalities:
    \begin{align}
        h_\star \ge 1/\sqrt s,
        \quad  h_i \ge \sqrt{\overline{\theta_i^2} - \hat\QQ_i(h_i)}.
    \end{align}
    \end{proposition}
    \begin{proof}
        See \Cref{app:proof-s_i-s_star-s} for a detailed proof.
    \end{proof}

\paragraph{General version of  \Cref{thm:sae-dynamics}}
In the following, we will let $s_\star = 1/h_\star^2$. 
To ensure consistency in the notation, we will also define $s_i = 1/h_i^2$ for $h_i$ defined in \eqref{eq:s_i-def}. We give a more general version of \Cref{thm:sae-dynamics} in the following theorem, which will be formally proved in the remaining part of this section.
\begin{theorem}
    \label{thm:sae-dynamics-general}
    For feature $i\in[n]$, let us take some small constant $\varepsilon\in(0, 1)$ and define $Q^{(i)}$ as 
    \begin{align}
        Q^{(i)} = \hat\QQ^{(i)}\biggl(\frac{-b/\sqrt{\log n}}{{(1-\varepsilon)\sqrt{2(\log_n M - 1)}}}\biggr). 
        \label{eq:Q_i-def}
    \end{align}
    Suppose 
    \begin{align}
        \frac{\log \eta}{\log n} \gtrsim \frac{b^2/2 - \log N}{\log n}. 
    \end{align}
    For any feature $i\in[i]$, consider the following joint conditions for $\rho_2$, $d$, $Q^{(i)}$ and bias $b<0$  with respect to constant parameter $\varsigma\in(0, 1)$: 
    \begin{olivebox2}
        \vspace{-10pt}
        \begin{align}
                &\textbf{Individual Feature Occurrence:}\quad \frac{\norm{H_{:, i}}_0}{\rho_1 N} \ge \polylog(n)^{-1}, \\
                &\textbf{Limited Feature Co-ocurrence:}\quad   \log_n (\rho_2^{-1}) \gtrsim \max\Bigl\{ -4\log_n Q^{(i)}, \: \frac{1}{2} - \log_n Q^{(i)}\Bigr\}, \label{eq:cond-general} \\ 
                &\textbf{Bias Range:}\: 1 \gtrsim \frac{b^2}{2\log n} \gtrsim \max\Bigl\{ \frac{1}{2} + \frac{h_\star^2}{2} - (1 + h_\star^2) \log_n Q^{(i)}, \: \frac{1}{4} + \frac{h_\star^2}{4} - (3 h_\star^2 + 1) \log_n Q^{(i)}, \\
                &\hspace{2cm} \bigl(\sqrt 2 h_\star (1+\varepsilon) + \sqrt{-(1 - h_\star^2)\log_n Q_1 }\bigr)^2, \: 1 - (1-\varsigma)\log_n d  - \log_n Q^{(i)} \Bigr\}. 
        \end{align}
    \end{olivebox2}
    \noindent
   Here $x\gtrsim y$ means $x\ge y + O(\log\log(n)/\log(n))$.
    Then with probability at least $1-n^{-4\varepsilon}$ over the randomness of the features $V$, for any feature $i$ such that there exists some constant $\varsigma_i$ satisfying the above conditions, there exists at least one unique neuron $m_i$ and after at most $T_i=\max\{(2\varsigma_i)^{-1},  1\}$ steps of training, we have $\langle w_{m_i}^{T_i}, v_i\rangle /\norm{v_i}_2 \ge 1 - o(1)$.
\end{theorem}

\noindent\textbf{{Relationship between \Cref{thm:sae-dynamics} and \Cref{thm:sae-dynamics-general}.}}
The main difference between \Cref{thm:sae-dynamics} and \Cref{thm:sae-dynamics-general} is that the latter allows $Q^{(i)}$ to have a larger range of values, while the former requires $Q^{(i)} = \hat\QQ^{(i)}(h_i)$ to be strictly larger than $\polylog(n)^{-1}$.
A direct consequence of this restriction in \Cref{thm:sae-dynamics} is that the range of $M$ is smaller compared to that in \Cref{thm:sae-dynamics-general}.
However, the conditions in \Cref{thm:sae-dynamics-general} have $Q^{(i)}$ and $\rho_2, b$ coupled together, which makes it difficult to gain a clear understanding, while in \Cref{thm:sae-dynamics}, we decouple the conditions by enforcing the range of $Q^{(i)}$. 
Specifically, 
\begin{enumerate}
    \item The condition $Q^{(i)} \ge \polylog(n)^{-1}$ is equivalent to 
    \begin{align}
        \frac{-b}{{(1-\varepsilon)\sqrt{2(\log_n M - 1)}}} \le h_i
    \end{align}
    by recalling the definition of $h_i$. This gives the range of $M$ as in \eqref{eq:cond-width} if we require all the features to be learned simultaneously. In fact, if the condition is satisfied for only a subset of features, our theorem still holds on that subset of features.
    \item The individual feature occurrence condition is the same in both theorems, and the limited feature co-occurrence condition in  \Cref{thm:sae-dynamics-general} will reduce to $\rho_2 \ll n^{-1/2 - o(1)}$, which is already implied by the data condition in \Cref{asp:data_decomp_main}. 
    \item The bias range condition in \Cref{thm:sae-dynamics-general} will reduce to the version in \Cref{thm:sae-dynamics} by removing the terms that involve $Q^{(i)}$ as $\log\log(n)/\log(n)$ gap is already enforced by the $\gtrsim$ notation.
\end{enumerate}
\noindent Moreover, we assume that $s\ge 3$ as mandated in \Cref{thm:sae-dynamics}. Since $s_\star\le s$ by \Cref{prop:sparsity}, if $s_\star\le s\le 2$ the following inequality
\begin{align}
    1 \gtrsim \frac{b^2}{2\log n} \gtrsim \frac{1}{s_\star}\Bigl(\sqrt{2}(1+\varepsilon) + \sqrt{-(s_\star-1)\log_n Q^{(i)}}\Bigr)^2
\end{align}
cannot hold, because the right-hand side would exceed~$1$.

\paragraph{Roadmap for the proof of \Cref{thm:sae-dynamics-general}}
The remaining part of this section is organized as follows:
\begin{enumerate}[
    leftmargin=2em, 
    label=\textbullet,
]
    \item \textbf{Concentration simplification:} In \Cref{sec:simplification-concentration}, we will combine the concentration results derived in \Cref{app:concentration} to derive explicitly the simplified concentration results for the atomic terms in \eqref{eq:alpha-beta-def} for the evolution of $\alpha_{-1, t}$ and $\beta_{t}$.
    \item \textbf{Conditions for strong alignment:} In \Cref{sec:two-state-recursion}, we formulate a set of conditions \labelcref{cond:basic,cond:Phi,cond:signal-t=1}, \labelcref{cond:alpha-beta-t,cond:alpha-beta-ratio} that will yield a simple two-state recursion. Building upon these conditions, we further identify \labelcref{cond:lambda_0,cond:Phi-lambda_0,cond:xi_1} that will guarantee a strong alignment $\alpha_{-1, T} = 1 - o(1)$ with only $T=O(1)$ steps of training.
    \item \textbf{Conditions simplification:} In \Cref{sec:final-stage-simplified}, we further simplified the series of conditions into a more concise form as in \eqref{eq:final-stage-cond-1}, which yields the full list of conditions in \Cref{thm:sae-dynamics-general}.
\end{enumerate}

\paragraph{Notation} Following the convention in \Cref{app:concentration}, we let $\barb_t= b_t + \kappa_0$ where $\kappa_0 = O((\log n)^{-1/2})$ is defined in \Cref{assump:activation}. Recall the definition $\zeta_1 = 2(1+\varepsilon)\sqrt{\log n}$ and $\zeta_0 = (1 - \varepsilon) \sqrt{2\log n}$ in \eqref{eq:def-zeta} for some small constant $\varepsilon\in(0, 1)$. We let $C$ be a universal constant that may vary from line to line.

\subsection{Concentration Results Combined}\label{sec:simplification-concentration}
We now combine the concentration results for the second-order terms in 
\Cref{lem:E-2nd} and \Cref{lem:F-2nd} under the assumption that 
$t\log n \ll n$. In particular, by taking the square root of the upper bounds in these lemmas and noting that 
$\|v\|_2^2=O(d)$ holds with probability at least $1-n^{-c}$ (see \Cref{lem:chi-squared}), 
we can express the combined square-root upper bound as
\begin{lightbluebox}
    \begin{equation}
        \begin{aligned}
            \xi_t &= \sqrt s\, t\log n\, \cK_t + \rho_1^{-1} \sqrt{ \Phi(|\barb_t|) \cdot \hat\EE_{l,l'}\!\Biggl[
            \Phi\Bigl( |\barb_t|\sqrt{\frac{1-\langle h_l, h_{l'}\rangle}{1+\langle h_l, h_{l'}\rangle}}
            \Bigr)
            \langle h_l, h_{l'}\rangle
        \Biggr]} \\
        &\qquad + \sqrt{\rho_2 d}\, |\alpha_{-1, t-1}| + \rho_2 \sqrt{n}\,.
        \end{aligned}
    \end{equation}
\end{lightbluebox}
\noindent We formally state the combination of the above two lemmas in the following corollary.
\begin{corollary}[Second-order concentration combined]
    \label{cor:norm-E-F}
    Then under the conditions $t\log n\ll n$, $-\barb_t = \Theta(\sqrt{\log n})<\zeta_1$, $\rho_1\ll 1$, it holds for all $t\le T \le n^c$ with probability at least $1-n^{-c}$ over the randomness of standard Gaussian vectors $z_{-1:T}$ and $v$ that
    \begin{align}
        \sqrt{\norm{E^\top \varphi(E y_t^\star; b_t)}_2^2  + \norm{F^\top \varphi(F y_t + \theta \cdot v^\top \barw_{t-1}; b_t)}_2^2} \le C L N\rho_1 \xi_t. 
    \end{align}
\end{corollary}
Here, the constant $C$ hides some factors from using the inequality $\sqrt{a} + \sqrt{b} \le \sqrt{2(a + b)}$. We refrain from a detailed proof here.
With the second order concentration results in \Cref{cor:norm-E-F}, we can now derive the first-order concentration results for the terms $\langle z_\tau, u_t\rangle$ based on \Cref{lem:E-1st} and \Cref{prop:hatvarphi_1-bound}. 
To further simplify the concentration bound, we impose the additional condition
\(
\Phi(|\barb_t|) \gg L\, s\,\rho_1\, (t\log(n))^3,
\)
which in particular holds if
\(
\Phi(|\barb_t|) \gg n^{-1}\,\mathrm{polylog}(n).
\)
This requirement is reasonable because it ensures that the neuron is not activated too rarely compared to the average occurrence frequency ($s/n$) of the features.
\begin{lemma}[First-order concentration combined]\label{lem:z_tau-ut}
    If $\Phi(|\barb_t|) \gg L s\rho_1 (t\log(n))^3$,
    $-\barb_t = \Theta(\sqrt{\log n})<\zeta_1$, $\kappa_0 |\barb_t| = O(1)$, for all $t\le T \le n^c$, it holds with probability at least 
    $1-n^{-c}$ over the randomness of standard Gaussian vectors $z_{-1:T}$ that
    \begin{align}
        \langle z_\tau, u_t\rangle 
        &= N \alpha_{\tau, t-1} \hat\varphi_1(b_t) \cdot (1\pm o(1))  \pm C N L \rho_1 \sqrt{\rho_2 s} (t\log n)^{3/2} \cdot d \, |\alpha_{\tau, t-1} \alpha_{-1, t-1}| \nonumber\\[1mm]
        &\quad \pm C N \rho_1 L \sqrt{t \log n} \cdot \xi_t \pm C L N \sqrt{\log n} \cdot \bigl( \sqrt{s \rho_1 d \Phi(|\barb_t|)}
        + \sqrt s \rho_1 |\barb_t| d\, \beta_{t-1} \bigr)\cdot \beta_{t-1} ,
    \end{align}
    where $\xi_t$ is defined in \Cref{cor:norm-E-F}.
\end{lemma}
\begin{proof}
    See \Cref{sec:proof-z_tau-u} for a detailed proof.
\end{proof}
In order to derive the recursion for $\alpha_{-1, t}$ in \eqref{eq:alpha-beta-def}, we need to control the numerator 
\begin{align}
    \alpha_{-1, t} \norm{w_t}_2 = \frac{\langle v, w_t \rangle}{\norm{v}_2} = \langle z_{-1}, u_t \rangle + \norm{v}_2 \cdot \theta^\top \varphi(F y_t + \theta \cdot v^\top \barw_{t-1}; b_t) + \eta^{-1} \alpha_{-1, t-1}.
\end{align}
Using \Cref{lem:z_tau-ut} and the concentration for the second term in \Cref{lem:F-signal}, we derive the following lemma for $\langle v, w_t\rangle/\norm{v}_2$. 
\begin{lemma}[Concentration for numerator in $\alpha$-recursion]\label{lem:v-w}
    Suppose $\rho_1 d (st\log n)^{-1} \gg \Phi(|\barb_t|) \gg L s\rho_1 (t\log(n))^3$,
    $-\barb_t = \Theta(\sqrt{\log n}) < \zeta_1$, $\kappa_0 |\barb_t| = O(1)$, $\sqrt{t s \log n} |\barb_t| \beta_{t-1}\ll 1$, and $\sqrt{d}\alpha_{-1, t-1} \gg 1$. Furthermore, assume that 
    \begin{align}
        \frac{N_2}{N} C_0 \overline{\theta^2} Q_t \gg \max\Bigl\{L \rho_1\sqrt{\rho_2 s} (t\log n)^{3/2},\:  L d^{-1} \Phi(|\barb_t|), \: L\sqrt{t\log n } \rho_1 \frac{\xi_t}{d\alpha_{-1, t-1}}, \: L \rho_1  \frac{\beta_{t-1}}{\alpha_{-1, t-1}} \Bigr\}.
    \end{align}
    If $\eta^{-1} \ll N_2 dC_0\overline{\theta^2} Q_t$
    Then it holds with probability at least $1-n^{-c}$ over the randomness of standard Gaussian vectors $z_{-1:T}$ and $v$ that
    \begin{equation}
        \frac{\langle v, w_t\rangle}{\norm{v}_2} = (1\pm o(1)) N \psi_t. 
    \end{equation}
\end{lemma}
\begin{proof}
    See \Cref{sec:proof-v-w} for a detailed proof.
\end{proof}
Now that we have characterized the \say{numerator} for $\alpha$-recursion. It remains to control the \say{denumerator} $\norm{w_t}_2$. 
In what follows, we will decompose the norm $\norm{w_t}_2$ into two parts: the projection onto the subspace spanned by $w_{-1:0}$ and the projection onto the orthogonal compliment of this subspace.
For $P_{w_{-1:0}}^\perp w_t$ being the projection onto the orthogonal complement of the subspace spanned by $w_{-1:0}$, we have the following bound.
\begin{lemma}\label{lem:perp-wt}
    Suppose $\rho_1 d (st\log n)^{-1} \gg \Phi(|\barb_t|) \gg L s\rho_1 (t\log(n))^3$,
    $-\barb_t = \Theta(\sqrt{\log n})<\zeta_1$, $\kappa_0 |\barb_t| = O(1)$, $\sqrt{t s \log n} |\barb_t| \beta_{t-1}\ll 1$, $\sqrt{d}\alpha_{-1, t-1} \gg 1$, $\sqrt{\rho_2 s} (t\log n)^{3/2} \ll 1$, and $\eta^{-1}\ll N \Phi(|\barb_t|)$. Then, for all $t\le T \le \sqrt{d}$, it holds with probability at least $1-n^{-c}$ over the randomness of standard Gaussian vectors $z_{-1:T}$ and $v$ that
    \begin{equation}
        \norm{P_{w_{-1:0}}^\perp w_t}_2 \le C N L \rho_1 \sqrt d \bigl(\xi_t + \sqrt{d} \beta_{t-1}\bigr).
    \end{equation}
\end{lemma}
\begin{proof}
    See \Cref{sec:proof-perp-wt} for a detailed proof.
\end{proof}
For $P_{w_{-1:0}} w_t$ being the projection onto the subspace spanned by $w_{-1:0}$, we have the following bound.
\begin{lemma}\label{lem:w-0-1-norm}
    Suppose $\rho_1 d (st\log n)^{-1} \gg \Phi(|\barb_t|) \gg L s\rho_1 (t\log(n))^3$,
    $-\barb_t = \Theta(\sqrt{\log n}) < \zeta_1$, $\kappa_0 |\barb_t| = O(1)$, $\sqrt{t s \log n} |\barb_t| \beta_{t-1}\ll 1$, and $\sqrt{d}\alpha_{-1, t-1} \gg 1$. Furthermore, assume for some constant $C_0>0$ that 
    \begin{align}
        \frac{N_2}{N} C_0 \overline{\theta^2} Q_t \gg \max\Bigl\{L \rho_1\sqrt{\rho_2 s} (t\log n)^{3/2},\:  L d^{-1} \Phi(|\barb_t|), \: L\sqrt{t\log n } \rho_1 \frac{\xi_t}{d\alpha_{-1, t-1}}, \: L \rho_1  \frac{\beta_{t-1}}{\alpha_{-1, t-1}} \Bigr\}.
    \end{align}
    If $\eta^{-1} \ll N_2 dC_0\overline{\theta^2} Q_t \land N \Phi(|\barb_t|)$, then it holds with probability at least $1-n^{-c}$ over the randomness of standard Gaussian vectors $z_{-1:T}$ and $v$ that
    \begin{equation}
        \bigl\|P_{w_{-1:0}} w_t\bigr\|_2 = (1\pm o(1)) \cdot \sqrt{(N\psi_t)^2 + \bigl(N \alpha_{0, t-1} \hatvarphi_1(b_t)\bigr)^2}.
    \end{equation}
\end{lemma}
\begin{proof}
    See \Cref{sec:proof-w-0-1-norm} for a detailed proof.
\end{proof}
Combining the results from \Cref{lem:perp-wt,lem:w-0-1-norm}, we obtain the upper bound for $\norm{w_t}_2$. 
\begin{lemma}\label{lem:w_t-2-norm}
    Suppose $\rho_1 d (st\log n)^{-1} \gg \Phi(|\barb_t|) \gg L s\rho_1 (t\log(n))^3$,
    $-\barb_t = \Theta(\sqrt{\log n})$, $\kappa_0 |\barb_t| = O(1)$, $\sqrt{t s \log n} |\barb_t| \beta_{t-1}\ll 1$, $\sqrt{d}\alpha_{-1, t-1} \gg 1$ and $\sqrt{\rho_2 s} (t\log n)^{3/2} \ll 1$. Furthermore, assume for some constant $C_0>0$ that 
    \begin{align}
        \frac{N_2}{N} C_0 \overline{\theta^2} Q_t \gg \max\Bigl\{L \rho_1\sqrt{\rho_2 s} (t\log n)^{3/2},\:  L d^{-1} \Phi(|\barb_t|), \: L\sqrt{t\log n } \rho_1 \frac{\xi_t}{d\alpha_{-1, t-1}}, \: L \rho_1  \frac{\beta_{t-1}}{\alpha_{-1, t-1}} \Bigr\}.
    \end{align}
    If $\eta^{-1} \ll N_2 dC_0\overline{\theta^2} Q_t \land N \Phi(|\barb_t|)$, then it holds for all $t\le T\le \sqrt d$ with probability at least $1-n^{-c}$ over the randomness of standard Gaussian vectors $z_{-1:T}$ and $v$ that
    \begin{equation}
        \norm{w_t}_2 \le  (1\pm o(1)) \cdot \sqrt{(N\psi_t)^2 + \bigl(N \hatvarphi_1(b_t)\bigr)^2} + CNL\rho_1 \sqrt d \xi_t. 
    \end{equation}
\end{lemma}
\begin{proof}[Proof of \Cref{lem:w_t-2-norm}]
    By the triangle inequality, it holds that 
    \begin{align}
        \norm{w_t}_2 &\le \norm{P_{w_{-1:0}} w_t}_2 + \norm{P_{w_{-1:0}}^\perp w_t}_2 \\
        &\le (1+o(1))\sqrt{(N\psi_t)^2 + \bigl(N \alpha_{0, t-1} \hatvarphi_1(b_t)\bigr)^2} + C N L \rho_1 \sqrt d \bigl(\xi_t + \sqrt{d} \beta_{t-1}\bigr). 
    \end{align}
    By condition
    \(
        \frac{N_2}{N} C_0 \overline{\theta^2} Q_t \gg L \rho_1  \frac{\beta_{t-1}}{\alpha_{-1, t-1}} 
    \)
    and the lower bound $N\psi_t \ge C \overline{\theta^2} Q_t N_2 d\alpha_{-1, t-1}$ shown in \Cref{lem:signal-bounds}, we have  $
        N\psi_t \gg  C N L \rho_1 d \beta_{t-1}$ satisfied and can be absorbed into the upper bound of $\norm{P_{w_{-1:0}} w_t}_2$.
    Hence, we conclude the proof.
\end{proof}
When we derive the above lemmas step by step, we collect all the conditions used in the final \Cref{lem:w_t-2-norm}. In the following proof, we will be focusing on the conditions listed in the statement of this lemma.

\subsection{A Two-State Alignment Recursion}
\label{sec:two-state-recursion}
From now on, we adhere to the fact that the bias remains fixed throughout the dynamics. Thus, we drop the time index in $b_t$ (writing it simply as $b$) and define 
\(
\barb = b + \kappa_0.
\)
To further simplify the conditions in the previous section, we have the following lemma. 
\begin{lemma}\label{lem:recursion}
Consider fixing the bias to be $b < 0$ and $\barb = b + \kappa_0 < 0$. Suppose \textbf{InitCond-1} and \textbf{InitCond-2} hold.
With the following conditions at initialization:
    \begin{enumerate}[
        ref={Cond.(\roman*)},
        label={(\roman*)}
    ]
        \item \label{cond:basic}$-\barb=\Theta(\sqrt{\log n}) < \zeta_1$, $\kappa_0|\barb| = O(1)$, $\sqrt{\rho_2 s} (T\log n)^{3/2} \ll 1$, $\eta^{-1} \ll N_2 dC_0\overline{\theta^2} Q_1 \land N \Phi(|\barb|)$.
        \item \label{cond:Phi}$\rho_1 d (sT\log n)^{-1} \gg \Phi(|\barb|) \gg L s\rho_1 (T\log(n))^3$.
        \item \label{cond:signal-t=1}$\frac{N_2}{N} C_0 \overline{\theta^2} Q_1 \gg \max\Bigl\{L \rho_1\sqrt{\rho_2 s} (T\log n)^{3/2},\:  L d^{-1} \Phi(|\barb|), \: L\sqrt{T\log n } \rho_1 \cdot \frac{\xi_1}{d\alpha_{-1, 0}} \Bigr\}$.
    \end{enumerate}
\noindent
If for some time step $t\le T\le \sqrt d$ we have
    \begin{enumerate}[
        ref={Cond.(\Roman*)}, 
        label={(\Roman*)}, 
        start=1]
        \item \label{cond:alpha-beta-t}$\alpha_{-1, t-1} \ge t^2  \alpha_{-1, 0}$, $\sqrt{Ts \log n} |\barb| \beta_{t-1} \ll 1$, 
        \item \label{cond:alpha-beta-ratio}$\frac{N_2}{N} C_0 \overline{\theta^2} Q_t \gg L \rho_1  \frac{\beta_{t-1}}{\alpha_{-1, t-1}}$. 
    \end{enumerate} 
\noindent
Let us define
\begin{lightbluebox}
\begin{equation}
    \lambda_0 = \frac{C L\rho_1}{C_0 \overline{\theta^2} \cdot N_2/N}, \quad \lambda_t = \frac{\lambda_0}{ Q_t}, \quad \tilde\xi_t = \frac{1}{\sqrt d} \Bigl( \frac{\xi_1}{\alpha_{-1, 0}} + C \sqrt{s} t^2 (\log n)^{3/2} \cdot \ind(t\ge 2)\Bigr)
\end{equation}
\end{lightbluebox}
\noindent
for some sufficiently large constant $C>0$. 
Under the above conditions, we have the following conclusions: 
\begin{enumerate}
    \item All the conditions in \Cref{lem:w_t-2-norm} hold for $t\le T$;
    \item Then with probability at least $1-n^{-c}$ over the randomness of standard Gaussian vectors $z_{-1:T}$ and $v$, we have the following two-state alignment recursion:
\begin{AIbox}{Two-State Alignment Recursion}
    \vspace{-5pt}
    \begin{equation}
    \begin{gathered}
    \frac{\beta_{t}}{\alpha_{-1, t}} \le \lambda_t \cdot \Bigl( \tilde\xi_t + \frac{\beta_{t-1}}{\alpha_{-1, t-1}}\Bigr), \quad 
    \frac{1}{\alpha_{-1, t}} \le  (1+o(1)) + \lambda_t \cdot \Bigl(\frac{\Phi(|\barb|)}{\rho_1 d} \cdot \frac{1}{\alpha_{-1, t-1}} + \tilde\xi_t\Bigr). 
\end{gathered}
\end{equation}
\end{AIbox}
\end{enumerate}
\end{lemma}
\begin{proof}
    See \Cref{sec:proof-recursion} for a detailed proof.
\end{proof}
From the above lemma, we can obtain the following observations: 
\begin{itemize}
    \item The ratio $\lambda_t \Phi(|\barb|) / \rho_1 d$ controls the growth of the alignment $\alpha_{-1, t}$. In order for the alignment to grow faster, we need a smaller activation frequency $\Phi(|\barb|)$, i.e., a larger bias $|\barb|$ in the absolute value. 
    \item The term $\lambda_t$ controls the growth of the ratio $\beta_{t}/\alpha_{-1, t}$. By definition, we know that $\lambda_t \ge 1$. 
    \item The maximum alignment achievable is $1 - o(1)$. 
\end{itemize}
Therefore, the best we can do is to set $\lambda_t$ as close to $1$ as possible while exploiting a small ratio $\Phi(|\barb|) / \rho_1 d$ to ensure that the alignment $\alpha_{-1, t}$ goes to $1$ before the ratio $\beta_{t}/\alpha_{-1, t}$ blows up. 
Since $\beta_0=0$, we have $\beta_{1}/\alpha_{-1, 1} = \lambda_1 \cdot \tilde\xi_1$.
This means we also need a small initial value $\tilde\xi_1$ to avoid a large ratio $\beta_{1}/\alpha_{-1, 1}$ at the beginning. In the sequel, we quantitatively analyze the evolution of the above recursions.
Before we proceed, by definition 
    \(
        Q_t =\frac{1}{N_2}\sum_{l=1}^{N_2} \ind\bigl(\theta_l > \frac{-b}{\sqrt{d} \alpha_{-1, t-1}}\bigr), 
    \)
    we note that $Q_t$ is nondecreasing in $\alpha_{-1, t-1}$. Therefore, we have the following fact:
    \begin{fact}\label{fact:Q-monotone-restate}
        If $\alpha_{-1, t-1} \ge \alpha_{-1, 1}$, then
        \(
            Q_{t} \ge  Q_{2}
        \) and 
        \(
            \lambda_t \le \lambda_2.
        \)
    \end{fact}

\paragraph{Expanding the recursions}
Let us define $T_0+1$ as the minimum of $t$ such that either of the following conditions fails:
\begin{enumerate}[
    leftmargin=6.5em, 
    label= \textbf{$T_0$-Cond.(\arabic{*}).},
    ref= $T_0$-Cond.(\arabic*)
]
    \item\labelcref{cond:alpha-beta-t} or \labelcref{cond:alpha-beta-ratio}; \label{cond:T0-1}
    \item $\alpha_{-1, t-1}\ge \alpha_{-1, 1}$; \label{cond:T0-2}
    \item $ t< \log(n)$. \label{cond:T0-3}
\end{enumerate}
In other word, $T_0$ is the \emph{stopping time} up to which all the conditions above hold.
We have $\lambda_t \le \lambda_2$ by \Cref{fact:Q-monotone-restate} and the definition $\lambda_t = \lambda_0 / Q_t$.
To obtain a simple recursion for $\alpha_{-1, t}$, we take 
\begin{align}
    C_1 = \Bigl(1 + o(1) + \frac{\lambda_2 \xi_1}{\sqrt d \alpha_{-1, 0}} + \frac{C\lambda_2 \sqrt s T_0^2 (\log n)^{3/2}}{\sqrt d} \Bigr) \cdot \frac{1}{1 - \lambda_2 \Phi(|\barb|) / \rho_1 d}.
    \label{eq:C1-def}
\end{align}
Here, we take the $o(1)$ term above to be the maximum of all the $o(1)$ terms in the recursion for $\alpha_{-1, t}$ for any $t\le T_0$.
For $2\le t\le T_0$, we have from substracting $C_1$ from both sides of the recursion for $\alpha_{-1, t}$ that 
\begin{align}
     \frac{1}{\alpha_{-1, t}} - C_1 &\le 
      (1+o(1)) + \lambda_t \cdot \Bigl(\frac{\Phi(|\barb|)}{\rho_1 d} \cdot \frac{1}{\alpha_{-1, t-1}} + \frac{1}{\sqrt d} \Bigl( \frac{\xi_1}{\alpha_{-1, 0}} + C \sqrt{s} t^2 (\log n)^{3/2} \cdot \ind(t\ge 2)\Bigr)\Bigr)
      \\
      &\qquad  \underbrace{- \Bigl(1 + o(1) + \frac{\lambda_2 \xi_1}{\sqrt d \alpha_{-1, 0}} + \frac{C\lambda_2 \sqrt s T_0^2 (\log n)^{3/2}}{\sqrt d} \Bigr) - \frac{C_1 \lambda_2 \Phi(|\barb|)}{\rho_1 d}}_{\ds \small{-C_1}}\\
     &\le \frac{\lambda_2 \Phi(|\barb|)}{\rho_1 d} \cdot \Bigl(\frac{1}{\alpha_{-1, t-1}} - C_1\Bigr), \quad \forall\: 2\le t\le T_0. 
    \label{eq:alpha-recursion-1}
\end{align}
Using the fact that 
\begin{align}
    \frac{1}{\alpha_{-1, 1}} - C_1 
    &\le 1 + o(1) + \Bigl(\frac{\lambda_1 \Phi(|\barb|)}{\rho_1 d} + \frac{\lambda_1 \xi_1}{\sqrt d } \Bigr) \cdot \frac{1}{\alpha_{-1, 0}} - C_1 \\
    &\le \Bigl(\frac{\lambda_1 \Phi(|\barb|)}{\rho_1 d} + \frac{\lambda_1 \xi_1}{\sqrt d } \Bigr) \cdot \frac{1}{\alpha_{-1, 0}}, 
    \label{eq:alpha-t=1}
\end{align}
we obtain that 
\begin{align}
    \frac{1}{\alpha_{-1, t}} \le \Bigl(\frac{\lambda_2 \Phi(|\barb|)}{\rho_1 d}\Bigr)^{t-1} \cdot \Bigl(\frac{\lambda_1 \Phi(|\barb|)}{\rho_1 d} + \frac{\lambda_1 \xi_1}{\sqrt d } \Bigr) \cdot \frac{1}{\alpha_{-1, 0}} + C_1, \quad \forall 1\le t\le T_0. 
    \label{eq:alpha-recursion}
\end{align}
In the above formula, we can extend $t$ to allow $t=1$ as 
\begin{align}
    \frac{1}{\alpha_{-1, 1}}
    \le 1 + o(1) + \Bigl(\frac{\lambda_1 \Phi(|\barb|)}{\rho_1 d} + \frac{\lambda_1 \xi_1}{\sqrt d } \Bigr) \cdot \frac{1}{\alpha_{-1, 0}} 
    \le C_1 + \Bigl(\frac{\lambda_1 \Phi(|\barb|)}{\rho_1 d} + \frac{\lambda_1 \xi_1}{\sqrt d } \Bigr) \cdot \frac{1}{\alpha_{-1, 0}}.
\end{align}
For the ratio $\beta_{t}/\alpha_{-1, t}$, we use the fact that $\lambda_t\le \lambda_2$ for $2\le t\le T_0$ and also that
\begin{align}
    \tilde\xi_t \le \frac{1}{\sqrt d} \Bigl( \frac{\xi_1}{\alpha_{-1, 0}} + C \sqrt{s} T_0^2 (\log n)^{3/2}\Bigr), \quad 2\le t\le T_0
\end{align}
to expand the recursion for $\beta_{t}/\alpha_{-1, t}$ as follows:
\begin{align}
    \frac{\beta_{t}}{\alpha_{-1, t}} 
    &\le \frac{1}{\sqrt d} \Bigl( \frac{\xi_1}{\alpha_{-1, 0}} + C \sqrt{s} T_0^2 (\log n)^{3/2}\Bigr) \cdot \sum_{\tau=2}^t \lambda_2^{t-\tau + 1} + \lambda_2^{t-1} \cdot \frac{\beta_1}{\alpha_{-1, 1}} \\
    &\le  \frac{T_0}{\sqrt d} \Bigl( \frac{\xi_1}{\alpha_{-1, 0}} + C \sqrt{s} T_0^2 \log(n)^{3/2}\Bigr) \cdot \lambda_2^{t-1} + \lambda_2^{t-1} \cdot \frac{\lambda_1 \xi_1}{\sqrt d \alpha_{-1, 0}} \\
    &= \frac{\lambda_2^{t-1}}{\sqrt d} \cdot \Bigl( (T_0 + \lambda_1) \cdot \frac{\xi_1}{\alpha_{-1, 0}} + C \sqrt{s} T_0^3 \log(n)^{3/2} \Bigr), \quad \forall\: 1\le t\le T_0, 
    \label{eq:beta-alpha-ratio-recursion}
\end{align}
where in the second inequality, we use the fact that $\lambda_2 \ge 1$ and the recursion for the ratio that
\begin{align}
    \frac{\beta_1}{\alpha_{-1, 1}} \le \lambda_1 \tilde\xi_1 = \frac{\lambda_1 \xi_1}{\sqrt d \alpha_{-1, 0}}. 
    \label{eq:beta/alpha-t=1}
\end{align}
Also in the last equality of \eqref{eq:beta-alpha-ratio-recursion}, we can relax the condition to allow $t=1$ as the right-hand side for $t=1$ clearly upper bounds the right-hand side of \eqref{eq:beta/alpha-t=1}. 
Using the results derived in \eqref{eq:alpha-recursion} and \eqref{eq:beta-alpha-ratio-recursion}, we now have the following statement. 
Now, building upon the results derived in \eqref{eq:alpha-recursion} and \eqref{eq:beta-alpha-ratio-recursion}, we have the following lemma, which summarizes the additional conditions needed to ensure that the alignment $\alpha_{-1, t}$ can be driven to $1-o(1)$.
\begin{lemma}\label{lem:final-stage}
    Let $\varsigma \in (0, 1)$ be a constant. 
    Take $\epsilon = C'\log \log n/(\varsigma \log d)$ for some sufficiently large constant $C'>0$.
    Suppose \textbf{InitCond-1} and \textbf{InitCond-2}, \labelcref{cond:basic,cond:Phi,cond:signal-t=1} hold.
    Under the following conditions
        \begin{enumerate}[
            ref={Cond.(\roman*)}, 
            label={(\roman*)}, 
            start=4]
            \item \label{cond:lambda_0} $ \lambda_0 = \Theta(\polylog(n))$. 
            \item \label{cond:Phi-lambda_0} $\lambda_0^{-1} Q_1 \cdot d^{-\varsigma} = \Phi(|\barb|)/{\rho_1 d} $.
            \item \label{cond:xi_1} $\xi_1 / Q_1 \ll d^{-\epsilon}/(\lambda_0 \sqrt s \log n)$. 
        \end{enumerate}
    \noindent
    there exists a time $t^\star\le ((2\varsigma)^{-1} \lor 1) \land T_0$ such that $\alpha_{-1, t} = 1 - o(1)$, where $T_0$ is the stopping time before and at which \labelcref{cond:T0-1,cond:T0-2,cond:T0-3} hold.
\end{lemma}
\begin{proof}
    See \Cref{sec:proof-final-stage} for a detailed proof of the lemma.
\end{proof}
Since $t^\star\le T_0$, \labelcref{cond:alpha-beta-t,cond:alpha-beta-ratio} hold for all $t\le t^\star$ \emph{automatically}.
In summary, in \Cref{lem:final-stage}, we have shown that under \labelcref{cond:basic,cond:Phi,cond:signal-t=1,cond:lambda_0,cond:Phi-lambda_0,cond:xi_1}, the alignment $\alpha_{-1, t}$ can be driven to $1-o(1)$ in constant time steps.

\subsection{Simplifying the Conditions of \Cref{lem:final-stage-simplified}}
\label{sec:final-stage-simplified}
To finish the proof of \Cref{thm:sae-dynamics}, it remains to simplify the conditions in \Cref{lem:final-stage}.
As a first step, we have the following lemma.
\begin{lemma}
    \label{lem:final-stage-simplified}
    Under \textbf{InitCond-1}, \textbf{InitCond-2}, and \Cref{assump:activation}, \labelcref{cond:basic,cond:Phi,cond:signal-t=1,cond:lambda_0,cond:Phi-lambda_0,cond:xi_1} hold upon the following conditions for some constant $\varsigma \in (0, 1)$ and $\epsilon = C'\log \log n/(\varsigma \log d)$ for some sufficiently large constant $C'>0$:
    \begin{olivebox2}
        \begin{equation}
            \begin{gathered}
            \frac{Q_1}{\lambda_0} \cdot d^{-\varsigma} = \frac{\Phi(|\barb|)}{\rho_1 d} \gg \max\Bigl\{ d^{\epsilon - \varsigma} \sqrt{s}\log n \cdot \xi_1, \: \frac{L s \log(n)^3}{d}\Bigr\},\\
            \lambda_0 = O(\polylog(n)), \quad \eta^{-1} \ll N \Phi(|\barb|). 
        \end{gathered}
        \label{eq:final-stage-cond-simplified}
        \end{equation}
        \end{olivebox2}
\end{lemma}
\begin{proof}
    See \Cref{sec:proof-final-stage-simplified} for a detailed proof.
\end{proof}
Next, we will plug in the definition of $\xi_1$ into the above condition to obtain the statement in \Cref{thm:sae-dynamics}.
In what follows, let us define $h_\star$ as the smallest number such that 
\begin{align}
    \max\{\hslash_{4, \star}^2, \hslash_{3, \star}^2, \hslash_{4, 1}^2 \} \le h_\star^2, \quad \sum_{j=1}^{n}\frac{1}{|\cD_j|^2}\sum_{l, l'\in\cD_j} \Phi\biggl(|\barb| \sqrt{\frac{1 - H_{l, j}H_{l', j}}{1 + H_{l, j} H_{l', j}}}\biggr) \le n \Phi\biggl(|\barb| \sqrt{\frac{1-h_\star^2}{1+h_\star^2}}\biggr), 
    \label{eq:s_star-def-restate}
\end{align}
where $\hslash_{4, \star}^2, \hslash_{3, \star}^2$ and $\hslash_{4, 1}^2$ are defined in \Cref{lem:E-2nd}.
The definition is valid as the right-hand sides of both inequalities are increasing in $h_\star$. 
In addition, we notice that $h_\star \le 1$ always holds, as $h_\star=1$ gives the trival upper bounds for all the inequalities in \eqref{eq:s_star-def-restate}.
In fact, the quantity $h_\star$ characterize the concentration level for the empirical distribution of $\{H_{l, j}\}_{l\in\cD_j}$.

\begin{lemma}
    \label{prop:2nd-moment-calc}
    If $(1-h_\star^2)/(1+h_\star^2)=\Theta(1)$ for $h_\star$ defined in \eqref{eq:s_star-def-restate}, it holds that 
    \begin{align}
        \hat\EE_{l, l'}\biggl[\Phi\Bigl( |\barb|\sqrt{\frac{1-\langle h_l, h_{l'}\rangle}{1+\langle h_l, h_{l'}\rangle}}\Bigr) \langle h_l, h_{l'}\rangle\biggr] \le C n\rho_1^2  \cdot \Phi(|\barb|)^{\frac{1-h_\star^2}{1+h_\star^2}} + \rho_1 \rho_2 s^2. 
        \label{eq:2nd-moment-calc}
    \end{align}
\end{lemma}
\begin{proof}
    See \Cref{sec:proof-prop-2nd-moment-calc} for a detailed proof.
\end{proof}
\noindent
Next, we also upper bound $\cK_1$ in terms of $h_\star$. 
\begin{lemma}
    \label{prop:K1-bound}
    Under the conditions that $\zeta_1 h_\star/|\barb| < 1 - \nu$ for some small constant $\nu\in(0, 1)$ and $\Phi(|\barb|) \ge \rho_1$, it holds for some sufficiently large constant $C>0$ that 
    \begin{align}
        C^{-1} \cK_1 
        & \le \bigl( n\,|\barb| \bigr)^{1/4}\,\Phi\!(|\barb|)^{\frac{1}{3h_\star^2 + 1}} + (\rho_2 s n |\barb|)^{1/4} \cdot \Phi(|\barb|)^{\frac{3}{8h_\star^2 + 4}} \\
        &\qquad + \bigl(\Phi(|\barb|)^{\frac{(1 - h_\star \zeta_1/|\barb|)^2}{1 - h_\star^2}} + (\rho_2 s)^{1/4} \bigr) \cdot (\log n)^{1/4} + n^{1/4}\rho_2 s \log n. 
    \end{align}
\end{lemma}
\begin{proof}
    See \Cref{sec:proof-K1-bound} for a detailed proof.
\end{proof}
\noindent
In the following, let us take $s = O(\polylog(n))$, $L=O(\polylog(n))$ and 
\begin{lightbluebox}
\begin{equation}
    d = n^{x_0}, \quad \rho_1 = n^{-x_1}, \quad \rho_2 = n^{-x_2}, \quad \Phi(|\barb|) = n^{-1 + x_3}. 
    \label{eq:reparameterization}
\end{equation}
\end{lightbluebox}
Using the above configurations, we have by the Mill's ratio that 
\begin{align}
     |\barb| = \sqrt{2(1-x_3)\log n \, \pm\, O(\log\log(n))}.  
     \label{eq:bias-approx}
\end{align}
In the following, we use the notation $x\lesssim y$ to denote that $x\le y + O(\log\log(n)/\log n)$, and $x\simeq y$ to denote that $x\lesssim y$ and $y\lesssim x$.
Consequently, we have $|\barb|/\sqrt{\log n} \simeq \sqrt{2(1-x_3)}$, and 
\begin{align}
    \frac{\zeta_1}{|\barb|} \simeq \frac{2(1+\varepsilon)\sqrt{\log n}}{\sqrt{2(1-x_3)\log n}} = (1+\varepsilon) \cdot \sqrt\frac{2}{1-x_3}
\end{align}
With \Cref{prop:2nd-moment-calc,prop:K1-bound}, we can now upper bound $\xi_1$ as 
\begin{align}
    \log_n \cK_1 
    &\lesssim \max\Bigl\{ \frac{1}{4} + \frac{x_3-1}{3 h_\star^2+1}, \: \frac{1-x_2}{4} + \frac{3(x_3-1)}{8 h_\star^2+4}, \: -\frac{(\sqrt{1-x_3}-\sqrt 2 h_\star (1+\varepsilon))^2}{1 - h_\star^2}, -\frac{x_2}{4}, \: \frac{1}{4} - x_2 \Bigr\}. 
\end{align}
In addition, using \Cref{prop:2nd-moment-calc}, the second term in the definition of $\xi_1$ is upper bounded as 
\begin{align}
    &\log_n \Biggl(\rho_1^{-1}\sqrt{\Phi(|\barb|)\hat\EE_{l, l'}\biggl[\Phi\Bigl( |\barb|\sqrt{\frac{1-\langle h_l, h_{l'}\rangle}{1+\langle h_l, h_{l'}\rangle}}\Bigr) \langle h_l, h_{l'}\rangle\biggr]} \Biggr)\lesssim \max\Bigl\{ \frac{x_3 + h_\star^2}{1+h_\star^2} - \frac{1}{2},\: \frac{x_3 - x_2 + x_1 -1}{2} \Bigr\}. 
\end{align}
Therefore, $\xi_1$ is upper bounded as 
\begin{align}
    \log_n \xi_1 
    &\lesssim \max\Bigl\{ \frac{1}{4} + \frac{ x_3-1 }{3 h_\star^2+1}, \: \frac{1-x_2}{4} + \frac{3(x_3-1)}{8 h_\star^2+4}, \: -\frac{(\sqrt{1-x_3}-\sqrt 2 h_\star (1+\varepsilon))^2}{1 - h_\star^2}, \\
    &\hspace{1.5cm} -\frac{x_2}{4}, \: \frac{1}{4} - x_2, \: \frac{x_3}{2} + \frac{(1-h_\star^2)(x_3 -1)}{2(1+h_\star^2)}, \: \frac{x_3 - x_2 + x_1 - 1}{2}, \:  -\frac{x_2}{2}, \: \frac{1}{2}-x_2  \Bigr\}. 
\end{align}
Plugging this bound into the first inequality in \eqref{eq:final-stage-cond-simplified}, we have the following reformulation:
\begin{align}
    \log_n Q_1 \simeq x_3-(1-\varsigma)x_0 \gtrsim \log_n \xi_1, \quad x_3 \gtrsim 0. 
\end{align}
where we note that $\epsilon = O(\log \log(n)/\log n)$ and can be ignored in the context of $\simeq$ notation.
Therefore, we just need to solve the following inequality system: 
    \begin{equation}
        \begin{aligned}
            &\log_n Q_1 \gtrsim \max\Bigl\{ \frac{1}{4} + \frac{ x_3-1 }{3 h_\star^2+1}, \: \frac{1-x_2}{4} + \frac{3(x_3-1)}{8 h_\star^2+4}, \: -\frac{(\sqrt{1-x_3}-\sqrt 2 h_\star (1+\varepsilon))^2}{1 - h_\star^2}, \\
            &\hspace{3cm} -\frac{x_2}{4}, \: \frac{1}{4} - x_2, \: \frac{x_3}{2} + \frac{(1-h_\star^2)(x_3 -1)}{2(1+h_\star^2)}, \: \frac{x_3 - x_2 + x_1 - 1}{2}, \:  -\frac{x_2}{2}, \: \frac{1}{2}-x_2  \Bigr\} \\
            & \hspace{1.5cm} 0\lesssim x_3 \simeq (1-\varsigma)x_0 + \log_n Q_1, \quad 0 \lesssim x_2 \lesssim 1, \quad x_0 \gtrsim 0, \quad 0\lesssim x_1\lesssim 1,   
        \end{aligned}
    \end{equation}
Solving this inequality system, we arrive at the following conditions that ensures \eqref{eq:final-stage-cond-simplified}:
\begin{olivebox2}
    \begin{equation}
        \begin{aligned}
            &1 \gtrsim x_2 \gtrsim \max\Bigl\{ -4\log_n Q_1, \: \frac{1}{2} - \log_n Q_1\Bigr\},  \quad \log_n\Bigl(\frac{N_2}{\rho_1 N}\Bigr) \gtrsim 0, \quad \eta^{-1} \ll N \Phi(|\barb|)\\
            &0 \lesssim x_3 \lesssim \min\Bigl\{ \frac{1}{2} - \frac{h_\star^2}{2} + (1+h_\star^2) \log_n Q_1, \: \frac{3}{4} - \frac{h_\star^2}{4} + (3 h_\star^2+1) \log_n Q_1, \\
            &\hspace{2cm} 1- \bigl(\sqrt 2 h_\star (1+\varepsilon) + \sqrt{-(1 - h_\star^2)\log_n Q_1 }\bigr)^2, \: (1-\varsigma) x_0 + \log_n Q_1 \Bigr\}, \\
        \end{aligned}
        \label{eq:final-stage-cond-1}
    \end{equation}
\end{olivebox2}
\noindent
Now, the first condition involving $x_2 = \log_n(\rho_2^{-1})$ can be transformed into the \emph{\textbf{Limited Feature Co-occurrence}} condition in \Cref{thm:sae-dynamics-general}. The second condition $\log_n (N_2/\rho_1 N) \gtrsim 0$ can be transformed into the \emph{\textbf{Individual Feature Occurrence}} condition in \Cref{thm:sae-dynamics-general} by noting that $N_2 = \norm{H_{:, i}}_0$ for feature $i$ of interest.  
The third condition $\eta^{-1} \ll N \Phi(|\barb|)$ can be transformed into 
\begin{align}
    \log \eta &\ge -\log N - \log \Phi(|\barb|) + O(\log \log n), 
\end{align}
where the second term on the right-hand side can be further upper bounded as 
\begin{align}
    - \log \Phi(|\barb|) \le \frac{\barb^2}{2} + O(\log \log n) \le  \frac{b^2}{2} + O(\log\log n + |\barb| \kappa_0 + \kappa_0^2) \simeq \frac{b^2}{2} + O(\log\log n), 
\end{align}
where we use the Mill's ration in the first inequality and the fact that $\kappa_0 = O((\log n)^{-1/2})$ and $|\barb|< \sqrt{2\log n}$ in the second inequality.
Therefore, a sufficient condition will be 
\begin{align}
    \frac{\log \eta}{\log n} \gtrsim \frac{b^2/2 - \log N}{\log n}. 
\end{align}
The last condition involving 
\begin{align}
    x_3 = 1 - \frac{{|b|^2 \pm O(\log\log(n))}}{2\log n} = 1 - \frac{|b|^2 \pm O\bigl(\log\log(n) + |b|\kappa_0 + \kappa_0^2\bigr)}{2\log n} \simeq 1 - \frac{b^2}{2\log n}
\end{align} 
can be transformed into the \emph{\textbf{Limited Feature Co-occurrence}} condition in \Cref{thm:sae-dynamics-general}.
Lastly, we remind the readers that $Q_1$ is also lower bounded as a function of $x_3$, which is shown in the following proposition.
\begin{proposition}
    \label{prop:Q_1-lb}
    Under \textbf{InitCond-1} and the reparameterization in \eqref{eq:reparameterization}, we have
    \begin{align}
        Q_1 \ge \hat\QQ\Biggl( \frac{-b/ \sqrt{\log n}}{(1-\varepsilon) \sqrt{\log_n M - 1}} \Biggr), \where \hat\QQ(x) \defeq \frac{1}{N_2} \sum_{l=1}^{N_2} \ind(\theta_l \ge x)
    \end{align}
    is the tail function for the empirical distribution of $\theta_l$.
\end{proposition}
\begin{proof}[Proof of \Cref{prop:Q_1-lb}]
    By \textbf{InitCond-1}, we have $\sqrt{d}\alpha_{-1, 0} \ge (1-\varepsilon)\sqrt{2\log(M/n)}$. Recall the definition of $Q_t$ in \eqref{eq:Q_t-bartheta-def}, we have by the non-increasing property of $\hat\QQ(\cdot)$ that
    \begin{align}
        Q_1&= \hat\QQ\Bigl( \frac{|b|}{\sqrt d \alpha_{-1, 0}}\Bigr)\ge \hat\QQ\Biggl( \frac{-b}{(1-\varepsilon)\sqrt{2\log(M/n)}} \Biggr).
    \end{align}
    This completes the proof of \Cref{prop:Q_1-lb}.
\end{proof}
Note that using the lower bound on $Q_1$ only strengthens the conditions in \eqref{eq:final-stage-cond-1}.
Hence, we can directly plug in the lower bound of $Q_1$ into all the conditions in \eqref{eq:final-stage-cond-1}, and this gives us the final statement of \Cref{thm:sae-dynamics-general}.

\section{Proofs for Concentration Results}
In this section, we provide proof for the concentration results presented in the previous section.
We first provide proofs for \Cref{lem:delta-y-l2} that controls the norm of the non-Gaussian component $\Delta y_t$. 
Then we give the proof for the concentrations of the second-order and first-order terms in the decomposition of the alignment recursion. 
Finally, we provide the proof for the error propagation in the dynamics due to the non-Gaussian component $\Delta y_t$.

\subsection{Proofs for Non-Gaussian Components}\label{app:proof-gaussian-non-gaussian}
In this subsection, we provide the proofs that are related to the Gaussian \& non-Gaussian components.
In particular, we provide the proof for \Cref{lem:delta-y-l2} that controls the norm of the non-Gaussian component $\Delta y_t$.

\subsubsection{Proof of \Cref
{lem:delta-y-l2}}\label{app:proof-delta-y-l2}
By definition of $\Delta y_t$ in \eqref{eq:y-decompose}, we further define
\begin{align}
    \Delta y_t^{(1)} = \sum_{\tau=1}^{t-1} \alpha_{\tau, t-1} \cdot \frac{\norm{w_{\tau}^\perp}_2}{\norm{u_{\tau}^\perp}_2} \cdot \frac{u_{\tau}^\perp}{\norm{u_{\tau}^\perp}_2}, \quad \Delta y_t^{(2)}= - \sum_{\tau=1}^{t-1} \alpha_{\tau, t-1} \cdot P_{u_{1:\tau}}z_\tau, 
\end{align}
and thus $\Delta y_t = \Delta y_t^{(1)} + \Delta y_t^{(2)}$.
The proof of \Cref{lem:delta-y-l2} is then based on the bounding the $\ell_2$ norm of $\Delta y_t^{(1)}$ and $\Delta y_t^{(2)}$ respectively.
To proceed with controlling the norm $\norm{\Delta y_t^{(1)}}_2$, we first control the ratio ${\norm{w_{\tau}^\perp}_2}/{\norm{u_{\tau}^\perp}_2}$ via the following lemma. 
\begin{lemma}[Ratio $\norm{w_{\tau}^\perp}_2/\norm{u_{\tau}^\perp}_2$]\label{lem:ratio-w-u}
    Take some total step $T\le \sqrt d$ and suppose $d \in (n^{1/c_1}, n^{c_1})$ for some universal constant $c_1\in(0, 1)$.
    For all $t = 1, \ldots, T$, it holds with probability at least $1 - n^{-c}$ for some universal constant $c, C>0$ that
    \begin{align}
        \Bigl|\frac{\norm{w_{t}^\perp}_2}{\norm{u_{t}^\perp}_2} - \sqrt d \Bigr| \le C (\log n)^{1/2},\quad \Bigl|\frac{\norm{w_{t}^\perp}_2^2}{\norm{u_{t}^\perp}_2^2} - d\Bigr| \le C \sqrt{d \log n}.
    \end{align}
\end{lemma}
\begin{proof}
    See \Cref{app:add-proofs-gaussian-non-gaussian} for a detailed proof.
\end{proof}

With \Cref{lem:ratio-w-u}, we can now control the $\ell_2$ norm of $\Delta y_t^{(1)}$ and $\Delta y_t^{(2)}$ respectively with the following two lemmas.
\begin{lemma}[$\ell_2$ norm of $\Delta y_t^{(1)}$]\label{lem:zeta-l2}
    Under the conditions in \Cref{lem:ratio-w-u}, for all $t=1,\ldots, T$, it holds with probability at least $1 - n^{-c}$ for some universal constant $c, C>0$ that
    \begin{align}
        (d - C \sqrt{d \log n}) \cdot \norm{P_{w_{-1:0}}^\perp \barw_{t-1}}_2^2\le \norm{\Delta y_t^{(1)}}_2^2  \le (d + C \sqrt{d \log n}) \cdot \norm{P_{w_{-1:0}}^\perp \barw_{t-1}}_2^2.
    \end{align}
\end{lemma}
\begin{proof}
    See \Cref{app:add-proofs-gaussian-non-gaussian} for a detailed proof.
\end{proof}

\begin{lemma}[$\ell_2$ norm of $\Delta y_t^{(2)}$]\label{lem:xi-l2}
    Under the conditions in \Cref{lem:ratio-w-u}, for all $t=1,\ldots, T$, it holds with probability at least $1 - n^{-c}$ for some universal constants $c, C>0$ that
    \begin{align}
        \norm{\Delta y_t^{(2)}}_2^2 \le C (t + \log n) \cdot \norm{P_{w_{-1:0}}^\perp \barw_{t-1}}_2^2.
    \end{align} 
\end{lemma}
\begin{proof}
    See \Cref{app:add-proofs-gaussian-non-gaussian} for a detailed proof.
\end{proof}
Combining \Cref{lem:zeta-l2} and \Cref{lem:xi-l2}, we complete the proof of \Cref{lem:delta-y-l2} by additionally noting that 
\begin{align}
    \norm{\Delta y_t}_2^2 \le 2\norm{\Delta y_t^{(1)}}_2^2 + 2\norm{\Delta y_t^{(2)}}_2^2 \le 2(d + C\sqrt{d\log n} + C(t+\log n)) \cdot \norm{P_{w_{-1:0}}^\perp \barw_{t-1}}_2^2.
\end{align}
As the first term $d \norm{P_{w_{-1:0}}^\perp \barw_{t-1}}_2^2 = d \beta_{t-1}^2$ is the leading term, we conclude the proof of \Cref{lem:delta-y-l2}.

\subsubsection{Additional Proofs for \Cref{lem:delta-y-l2}}
\label{app:add-proofs-gaussian-non-gaussian}
\begin{proof}[Proof of \Cref{lem:ratio-w-u}]
    Recall from \eqref{eq:gaussian-conditioning-1} that
    \begin{align}
        w_t &= \sum_{\tau=-1}^{t-1} \langle P_{u_{1:\tau}}^\perp z_\tau, u_t\rangle \cdot \frac{w_\tau^\perp}{\norm{w_\tau^\perp}_2} + \sum_{\tau=1}^{t-1} \langle  u_{\tau}^\perp, u_t\rangle \cdot \frac{\norm{w_{\tau}^\perp}_2}{\norm{u_{\tau}^\perp}_2} \cdot  \frac{w_\tau^\perp}{\norm{w_\tau^\perp}_2} + P_{w_{-1: t-1}}^\perp \tilde{z}_{t} \cdot \norm{u_t^\perp}_2 \\
        &\qquad + v \theta^\top \varphi(F y_t + \theta \cdot v^\top \barw_{t-1}; b_t) + \eta^{-1}\barw_{t-1}.
    \end{align}
    Applying projection $P_{w_{-1: t-1}}^\perp$ to both sides, we have
    \begin{align}
        w_t^\perp = P_{w_{-1: t-1}}^\perp w_t = P_{w_{-1: t-1}}^\perp \tilde{z}_{t} \cdot \norm{u_t^\perp}_2, 
    \end{align}
    which implies that $\norm{w_t^\perp}_2/\norm{u_t^\perp}_2 = \norm{P_{w_{-1:t-1}}^\perp\tilde{z}_{t}}_2$. Note that $\tilde z_t$ is independent of $\sigma(w_{-1:t-1})$ and follows standard Gaussian distribution. Therefore, $\norm{w_t^\perp}_2^2/\norm{u_t^\perp}_2^2 \sim \chi^2({d-t-1})$ and we have by the concentration in \Cref{lem:chi-squared} that with probability at least $1 - n^{-c}$ for all $t\in [T]$,
    \begin{align}
        \Bigl|\frac{\norm{w_{t}^\perp}_2^2}{\norm{u_{t}^\perp}_2^2} - d\Bigr| \le T + 2\sqrt{d \log(Tn^c)} + 2\log(Tn^c) \le C \sqrt{d \log(n)}, 
    \end{align}
    where the last inequality holds by conditions $T\le \sqrt d$ and $d \in (n^{1/c_1}, n^{c_1})$. Therefore, we conclude that with probability at least $1 - n^{-c}$ for all $t\in [T]$,
    \begin{align}
        \Bigl|\frac{\norm{w_{t}^\perp}_2}{\norm{u_{t}^\perp}_2} - \sqrt d \Bigr| \le C (\log n)^{1/2}. 
    \end{align}
    This completes the proof of \Cref{lem:ratio-w-u}.
\end{proof}

\begin{proof}[Proof of \Cref{lem:zeta-l2}]
    By definition of $\Delta y_t^{(1)}$, we have
    \begin{align}
        \bigl|\norm{\Delta y_t^{(1)}}_2^2 - d\norm{P_{w_{-1:0}}^\perp \barw_{t-1}}_2^2\bigr|
        &= \biggl|\sum_{\tau=1}^{t-1} \alpha_{\tau, t-1}^2 \cdot \Bigl(\frac{\norm{w_{\tau}^\perp}_2^2}{\norm{u_{\tau}^\perp}_2^2} -  d\Bigr) \biggr| \le \sup_{\tau=1, \ldots, t-1} \Bigl|\frac{\norm{w_{\tau}^\perp}_2^2}{\norm{u_{\tau}^\perp}_2^2} -  d\Bigr| \cdot \sum_{\tau=1}^{t-1} \alpha_{\tau, t-1}^2 \\
        &\le C \sqrt{d \log n} \cdot \norm{P_{w_{-1:0}}^\perp \barw_{t-1}}_2^2,
    \end{align}
    where the first equality holds by $\sum_{\tau=-1}^{t-1} \alpha_{\tau, t-1}^2 = 1$ according to the definition of $\alpha_{\tau, t}$, and the second inequality holds by \Cref{lem:ratio-w-u} with probability at least $1 - n^{-c}$.
\end{proof}

\begin{proof}[Proof of \Cref{lem:xi-l2}]
    By rewriting the definition of $\Delta y_t^{(2)}$, we have
    \begin{align}
        \Delta y_t^{(2)} = \sum_{\tau=1}^{t-1} \sum_{j=\tau}^{t-1}  \alpha_{j, t-1}\frac{u_\tau^\perp}{\norm{u_\tau^\perp}_2} z_j. 
    \end{align}
    We note that when conditioned on $\{\alpha_{\tau, T-1}\}_{\tau=-1}^{T-1}$ and $u_{1:T-1}$, the random variables $\{\frac{u_\tau^\perp}{\norm{u_\tau^\perp}_2} z_j\}_{j, \tau}$ for any $1\le \tau \le j\le t-1$ are i.i.d. standard Gaussian. 
    Let us denote the filtration $\cF = \sigma(\{\alpha_{\tau, T-1}\}_{\tau=-1}^{T-1}, u_{1:T-1})$.
    Therefore, we have 
    \begin{align}
        \Delta y_t^{(2)} \given \cF \overset{d}{=} \sum_{\tau=1}^{t-1} \sqrt{\sum_{j=\tau}^{t-1} \alpha_{j, t-1}^2} \cdot \frac{u_\tau^\perp}{\norm{u_\tau^\perp}_2} \cdot  z_\tau', 
    \end{align}
    where $\{z_\tau'\}_{\tau=1}^{t-1}$ are i.i.d. standard Gaussian independent of the filtration $\cF$. Hence, 
    \begin{align}
        \norm{\Delta y_t^{(2)}}_2^2 \given \cF \overset{d}{=} \sum_{\tau=1}^{t-1} \sum_{j=\tau}^{t-1} \alpha_{j, t-1}^2 \cdot (z_\tau')^2.
    \end{align}
    Using the concentration of $\chi^2$ distribution in \Cref{lem:chi-squared} gives us 
    \begin{align}
        &\PP\Biggl(
            \Bigl|\norm{\Delta y_t^{(2)}}_2^2 - \sum_{\tau=1}^{t-1} \sum_{j=\tau}^{t-1} \alpha_{j, t-1}^2 \Bigr| \ge C \sqrt{\sum_{\tau=1}^{t-1}\Bigl(\sum_{j=\tau}^{t-1} \alpha_{j, t-1}^2\Bigr)^2} \cdot \sqrt{\log (n)} + C \sum_{\tau=1}^t \alpha_{\tau, t-1}^2 \log (nT) \Biggiven \cF
        \Biggr) \le \frac{n^{-c}}{T}. 
    \end{align}
    Each term inside the probability can be upper bounded by 
    \begin{gather}
        \sqrt{\sum_{\tau=1}^{t-1}\Bigl(\sum_{j=\tau}^{t-1} \alpha_{j, t-1}^2\Bigr)^2} 
        \le \sqrt t \cdot (1 - \alpha_{-1, t-1}^2 - \alpha_{0, t-1}^2) = \sqrt t \cdot \norm{P_{w_{-1:0}}^\perp \barw_{t-1}}_2^2,\\
        \sum_{\tau=1}^t \alpha_{\tau, t-1}^2  = 1 - \alpha_{-1, t-1}^2 - \alpha_{0, t-1}^2 = \norm{P_{w_{-1:0}}^\perp \barw_{t-1}}_2^2, \qquad 
        \sum_{\tau=1}^{t-1} \sum_{j=\tau}^{t-1} \alpha_{j, t-1}^2 \le t \norm{P_{w_{-1:0}}^\perp \barw_{t-1}}_2^2.
    \end{gather}
    Therefore, we conclude that when conditioning on $\cF$, it holds with probability at least $1 - n^{-c}$ and for all $t=1, \ldots, T$ that
    \begin{align}
        \norm{\Delta y_t^{(2)}}_2^2 &\le C (t + \sqrt{t \log n} + \log n) \cdot \norm{P_{w_{-1:0}}^\perp \barw_{t-1}}_2^2 \le C (t + \log n) \cdot \norm{P_{w_{-1:0}}^\perp \barw_{t-1}}_2^2, 
    \end{align}
    where $C$ is a universal constant that changes from line to line. Here, we also use the condition that $T\le n$. Now, since for any event in the filtration $\cF$, the failure probability is at most $n^{-c}$, we can safely remove the conditioning and conclude the proof of \Cref{lem:xi-l2}.
\end{proof}

\subsection{Proofs for Concentration Lemmas}
In this subsection, we first provide a formal lemma that characterizes the sparsity of the activations when tuning the bias $b_t$ to be some negative value. 
Building upon this result, we then provide the proofs for the concentration results concerning the recursion of the alignment.

\subsubsection{Concentration for Ideal Activations}
\label{app:proof-sparse-activation}
The statement of the following lemma slightly generalize beyond the settings in \eqref{eq:y-decompose} for technical convenience.
Specifically, we want to understand how the neuron's activation frequency concentrates around $\Phi(-b_t)$. As we have the coefficient matrix $H$ decomposed into $E$ and $F$, we want to have a general result that can be applied to all of them.
Therefore, we consider a general sparse weight matrix $G$ in the following lemma.
\begin{lemma}[Concentration for Activations]\label{lem:sparse-activation}
    Let $G \in \RR_+^{L \times n}$ be a nonnegative weight matrix whose rows $(g_l)_{l\in[L]}$ satisfy $\|g_l\|_2=1$, and assume that $G$ is sparse in both rows and columns:
    \begin{itemize}
        \item For every coordinate $i\in[n]$, the $i$th column satisfies $\|G_{:,i}\|_0 \le \rho L$ for some $\rho\in [n^{-1},1]$.
        \item For every row $l\in[L]$, we have $\|g_l\|_0 \le s$.
    \end{itemize}
    For any integer $t\le n^c$ (with some fixed constant $c>0$), define
    \[
    y_t = \sum_{\tau=-1}^{t-1} \alpha_{\tau,t-1}\, z_\tau,
    \]
    where the vectors $z_\tau\in\mathbb{R}^n$ (for $\tau=-1,0,\dots,t-1$) are independent standard Gaussian random vectors, and the coefficients 
    \(
    \alpha_{t-1}=(\alpha_{\tau,t-1})_{\tau=-1}^{t-1}\in\mathbb{S}^{t}
    \)
    belong to the unit sphere in $\mathbb{R}^{t+1}$. Next, let $b_t\in\mathbb{R}$ be an arbitrary bias and let $\vartheta_t\in\mathbb{R}^{t+1}$ and $\varsigma = (\varsigma_l)_{l\in[L]}\in\mathbb{R}_+^L$ be fixed vectors. For each neuron $l\in[L]$, define its shifted bias by
    \[
    b_{t,l} = b_t - \varsigma_l\, \alpha_{t-1}^\top \vartheta_t.
    \]
    Then, for any failure probability $\delta\in\bigl(\exp(-n/4),1\bigr)$, there exists a universal constant $C>0$ such that with probability at least $1-\delta$ (over the randomness of the Gaussian vectors $\{z_{\tau}\}_{\tau=-1}^{t-1}$) the following holds simultaneously for all choices of $\alpha_{t-1}\in\mathbb{S}^{t}$ and $b_t\in\mathbb{R}$:
    \begin{align}
    \frac{1}{L}\sum_{l=1}^{L}\ind\bigl\{g_l^\top y_t > b_{t,l}\bigr\}
    \le C \Biggl(
    \frac{1}{L}\sum_{l=1}^{L}\Phi(b_{t,l})
    +\rho\, s\, t\, \log\bigl(n(1+\|\varsigma\|_\infty\|\vartheta_t\|_\infty)\bigr)
    +\rho\, s\, \log(\delta^{-1})
    \Biggr),
    \label{eq:sparse-activation-bound}
    \end{align}
    where $\Phi(\cdot)$ denotes the standard Gaussian tail probability.
    In particular, if $t$, $\alpha_{t-1}$ and $b_t$ are also fixed, then with probability at least $1-\delta$ it holds that
    \[
        \frac{1}{L}\sum_{l=1}^{L}\ind\bigl\{g_l^\top y_t > b_{t,l}\bigr\}
        \le C \left(
        \frac{1}{L}\sum_{l=1}^{L}\Phi(b_{t,l})
        +\rho\, s\, \log(\delta^{-1})
        \right).
    \]
\end{lemma}

\paragraph{Reduction to \Cref{cor:E-activation-ideal}}
We remark that when take $G$ to be the weight matrix $E$,  $L$ to be $N_1$, $n$ to be $n-1$, $\rho$ to be $\rho_1$, $b_t$ to be $\barb_t$, and letting $\vartheta_t=\vzero$, we directly obtain \Cref{cor:E-activation-ideal} as a special case. In the remaining of this subsection, we will present the proof of this lemma.

\begin{proof}[Proof of \Cref{lem:sparse-activation}]
In the following proof, we will use $C$ to denote universal constants that change from line to line.

\paragraph{Step I: Concentration for fixed $\alpha_{t-1}$, $b_t$ and $\vartheta_{t}$}
When fixing $\alpha_{t-1}$, $b_t$ and $\vartheta_t$, note that
$$
b_{t,l} = b_t - \varsigma_l \alpha_{t-1}^\top \vartheta_t
$$
is also fixed and the only randomness comes from the Gaussian vectors $z_{-1}, z_{0}, \ldots, z_{t-1}$. 
In particular, $y_t\sim \cN(0, I_n)$ since $\norm{\alpha_{t-1}}_2=1$ by assumption.
In the sequel, the discussion will be focused on one time step $t$ and we omit the subscript $t$ for simplicity.
The following is a table of the notations we will use in the proof:
\begin{table}[H]
    \centering
    \renewcommand{\arraystretch}{1.5}
    \setlength{\tabcolsep}{6pt} 
    \begin{tabular}{|c|c|}
        \hline
        \rowcolor{black!75} 
        \textcolor{white}{\bfseries Notation (Simplified)} & \textcolor{white}{\bfseries Definition} \\
        \hline
$y \gets y_t$ & $y_t = \sum_{\tau=-1}^{t-1} \alpha_{\tau,t-1} z_\tau$ \\\hline
$b_l \gets b_{t,l}$ & $b_{t,l} = b_t - \varsigma_l \alpha_{t-1}^\top \vartheta_t$ \\\hline
$\alpha \gets \alpha_{t-1}$ & $\alpha_{t-1} = (\alpha_{\tau,t-1})_{\tau=-1}^{t-1} \in \SSS^{t}$ \\\hline
$y^{(i)}$ & $y^{(i)}$ is the vector $y$ with the $i$-th coordinate $y_i$\\
& replaced by an independent copy $y_i'\sim \cN(0, 1)$ \\\hline
$Z$ & $Z = L^{-1} \sum_{l=1}^L \ind\bigl(g_l^\top y > b_{t,l}\bigr)$ \\\hline
$Z^{(i)}$ & $Z^{(i)} = L^{-1} \sum_{l=1}^L \ind\bigl(g_l^\top y^{(i)} > b_{t,l}\bigr)$ \\\hline
\end{tabular}
\caption{Summary of notations used in the proof of \Cref{lem:sparse-activation}.}
\label{tab:notation-sparse-activation}
\end{table}

Define $Z= L^{-1} \sum_{l=1}^L \ind\bigl(g_l^\top y  > b_{l}\bigr)$. 
To study the concentration of $Z$, we need to analyze the fluctuations when we change one coordinate of $y$.
This leads us to the definition of $y^{(i)}$ in \Cref{tab:notation-sparse-activation} with the corresponding $Z^{(i)} = L^{-1} \sum_{l=1}^L \ind\bigl(g_l^\top y^{(i)}  > b_{t, l}\bigr)$. Let us also define the Exceedance-Perturbed Variance (EPV) as follows:
\begin{align}
    V_+ = \EE\biggl[ \sum_{i=1}^n \bigl(Z^{(i)} - Z\bigr)^2 \ind(Z > Z^{(i)})\Biggiven y\biggr].
\end{align}
In the definition of EPV, we only count the contribution from the $i$-th coordinate of $y$ when $Z$ exceeds its perturbed counterpart $Z^{(i)}$.
Next, we show that $V_+$ is actually controlled by $Z$ itself up to a small factor. In particular, for the term inside the expectation in the definition of $V_+$, we have
\begin{align}
    \sum_{i=1}^n \bigl(Z^{(i)} - Z\bigr)^2 \ind(Z > Z^{(i)}) &=\sum_{i=1}^n \bigl(Z^{(i)} - Z\bigr)^2 \ind(y_i > y_i') \\ 
    &\le  \frac{1}{L^2} \sum_{i=1}^n \biggl(\sum_{l=1}^L \ind(g_{l, i}\neq 0) \cdot \ind\bigl(g_l^\top y > b_{t, l}\bigr)\biggr)^2 \\ 
    &\le \frac{\rho}{L} \cdot \sum_{i=1}^n \sum_{l=1}^L \ind(g_{l, i}\neq 0) \cdot\ind\bigl(g_l^\top y > b_{t, l}\bigr) = \rho s Z.
\end{align}
where 
\begin{itemize} 
    \item in the first identity, we use the fact that $Z$ is monotone in the $i$-th coordinate $y_i$ due to the \highlight{nonnegativity} of the weight matrix $G$. 
    \item In the first inequality, we use the fact that $0 \le \ind(g_l^\top y > b_{t, l}) - \ind(g_l^\top y^{(i)}> b_{t, l}) \le \ind(g_l^\top y > b_{t, l})$ thanks to the condition $y_i > y_i'$ which is guaranteed by the condition $Z > Z^{(i)}$.
    \item In the last line, we use the Cauchy-Schwarz inequality with the fact that $\sum_{l=1}^L \ind(g_{l, i}\neq 0) \ind(g_l^\top y > b_{t,l}) \le \sum_{l=1}^L \ind(g_{l, i}\neq 0) \,\highlight{\le \rho L}$. 
    Then, by also noting that each $g_l$ is also \highlight{$s$-sparse}, we obtain the last equality. 
\end{itemize}
Meanwhile, the mean of $Z$ is simply $\EE[Z] = L^{-1} \sum_{l=1}^L \Phi(b_{t, l})$, where we use the fact that \highlight{$\norm{g_l}_2 = 1$} by assumption and $g^\top y \sim \cN(0, 1)$. 
Invoking \Cref{lem:efron-stein-dominated-variance}, we conclude that for fixed $\alpha$ and $b_t$, we have with probability at least $1-\delta$,
\begin{align}
    Z \le \EE[Z] + C \sqrt{\rho s \EE[Z] \log \delta^{-1}} + C \rho s \log \delta^{-1} \le C\cdot \biggl(\frac{1}{L} \sum_{l=1}^L \Phi(b_{t, l}) + \rho s \log(\delta^{-1})\biggr)
    \label{eq:sparsity-0}
\end{align}
for some universal constant $C>0$. Here, we directly apply the inequality $\sqrt{a b} \le a + b$ for $a, b > 0$ in the last inequality.
In the following, we will apply a union bound on $\alpha_{t-1}$, $b_t$ to extend the above bound to arbitrary choices of $\alpha_{t-1}$ and $b_t$.

\paragraph{Step II: Union bound over $\alpha_{t-1}$ and $b_t$}
In the following argument, we will also drop the subscript $t$. 
Since $Z$ is a function of $\alpha$ and $b$, we use the following notation:
\begin{align}
    Z(\alpha, b) = \frac{1}{L} \sum_{l=1}^L \ind\biggl(\sum_{\tau=-1}^{t-1} \alpha_\tau g_l^\top z_\tau > b - \varsigma_l \alpha^\top \vartheta\biggr).
\end{align}
It is sufficient to construct a covering net for the pair $(\alpha, b)$.
Since the Gaussian vectors $z_\tau$ are unbounded, we first introduce a truncation step in our covering argument. By applying the Chernoff bound for Gaussian tails and then taking a union bound over all indices $\tau = -1, 0, \ldots, t-1$, we deduce that with probability at least 
$$
1 - (t+1)n\cdot \exp(-n/2) \ge 1 - \exp(-n/4)/2,
$$ 
we have 
$$
\max_{\tau = -1,0,\ldots,t-1}\|z_\tau\|_\infty \le \sqrt{n}\,.
$$
In what follows we condition on this high-probability event.

For $\alpha \in \SSS^{t}$, we take a uniform covering net on the sphere, denoted by $\cN_\alpha$, such that for any $\alpha$, there exists $\alpha' \in \cN_\alpha$ satisfying $\|\alpha - \alpha'\|_\infty \le \epsilon$. 
The covering number is upper bounded by $|\cN_\alpha| \le \epsilon^{-t}$. See for example Example 5.8 in \citet{wainwright2019high}. 
To proceed, let us define 
\[
\mu = (t+1) \cdot (\sqrt{s tn} + \|\varsigma\|_\infty \|\vartheta\|_\infty).
\]
The intuition for this definition is that $\mu$ represents the Lipschitz constant of $\sum_{\tau=-1}^{t-1}\alpha_\tau g_l^\top z_\tau + \varsigma_l \alpha^\top \vartheta$ with respect to any perturbation on $\alpha$ in the $\ell_\infty$-norm.
For $b$, leveraging the Gaussian tail property, we define the following covering net with size at most $4\mu \epsilon^{-1} + 4$:
\begin{align}
    \cN_{b} = \bigl\{ k\cdot \mu \cdot \epsilon \biggiven k\in \ZZ, k \in \bigl[-\lceil 2\epsilon^{-1} \rceil, \lfloor 2\epsilon^{-1}\rfloor \bigr] \bigr\} \cup \{-\infty\}.
\end{align}
There are three special points in $\cN_b$: $-\infty$, the minimal finite point $b_{\min}= - \lceil 2\epsilon^{-1} \rceil \cdot \mu \cdot \epsilon$, and the maximal point $b_{\max} = \lfloor 2\epsilon^{-1} \rfloor \cdot \mu \cdot \epsilon$.
For any $\alpha\in \SSS^{t}$ and $b \in \RR$, we pick $\hat\alpha=\argmin_{\alpha'\in \cN_\alpha} \|\alpha - \alpha'\|_\infty$ and $\hat b = \argmax\{b' \in \cN_b: b' < b - \mu \cdot \epsilon\}$. Therefore, we have by the monotonicity of the indicator function that 
\begin{align}
    Z(\alpha, b)  \le \frac{1}{L} \sum_{l=1}^L \ind\biggl(\sum_{\tau =-1}^{t-1} \hat\alpha_\tau\, g_l^\top z_t  > b - \varsigma_l \hat\alpha^\top \vartheta - \mu \cdot \epsilon\biggr) \le Z(\hat\alpha, \hat b).
    \label{eq:sparsity-2}
\end{align}
On the other hand, for $\hat b_{l} = \hat b - \varsigma_l \hat\alpha^\top \vartheta$, using the definition of $\hat \alpha$ and $\hat b$, it holds that 
\begin{align}
    \Phi(\hat b_{l}) &\le \Phi(b_{l} - 3\mu \cdot \epsilon) \cdot \ind(-b_{\min} \le \hat b_l < b_{\max}) + \ind(\hat b_{l} = -\infty) \\[1mm]
    &\quad + \Phi(2\mu - \varsigma_l \alpha^\top \vartheta) \cdot \ind(\hat b_{l} = b_{\max}).
\end{align}
The above inequality holds by considering three cases:
\begin{itemize}
    \item When $\hat b_l \in [-b_{\min}, b_{\max})$, we have $b_{l}$ close to $\hat b_{l}$ up to an approximation error of $3\mu \cdot \epsilon$, where one $\mu \cdot \epsilon$ comes from the approximation between $\alpha$ and $\hat\alpha$ and the other $2\mu \cdot \epsilon$ comes from the approximation between $b_l$ and $\hat b_l$.
    \item When $\hat b_l = -\infty$, we simply upper bound the tail probability by $1$.
    \item When $\hat b_l = b_{\max}$, we have $\Phi(\hat b_{l}) = \Phi(b_{\max}-\varsigma_l \alpha^\top \vartheta) \ge \Phi(2\mu-\varsigma_l \alpha^\top \vartheta)$.
\end{itemize}
Next, we characterize in each case the approximation error between $\Phi(b_{l})$ and the bound given above, which are $\Phi(b_{l} - 3\mu \cdot \epsilon)$, $1$, and $\Phi(2\mu - \varsigma_l \alpha^\top \vartheta)$ respectively.
In particular, 
\begin{itemize}
    \item For the first case $\hat b_t \in [-b_{\min}, b_{\max})$, we have the approximation error $\Phi(b_{l} - 3\mu \cdot \epsilon) - \Phi(b_{l})$ directly bounded by $3\mu\epsilon$ by Lipschitz continuity of the Gaussian tail function.
    \item For the second case $\hat b_t = -\infty$, it must hold that $b_t < b_{\min} + \mu \epsilon$, and the approximation error is thus upper bounded by $1 - \Phi(b_{t, l}) = 1 - \Phi(b_t - \varsigma_l \alpha^\top \vartheta) \le \exp(- (|b_{\min}| - (1+\epsilon)\mu)^2/2) \le \exp(-\mu^2/4)$.
    \item For the third case $\hat b_t = b_{\max}$, it must hold that $b_t > b_{\max}>2\mu$. 
    Hence, the approximation error is upper bounded by $\Phi(2\mu- \varsigma_l \alpha^\top \vartheta) - \Phi(b_{t, l}) \le \Phi(2\mu- \varsigma_l \alpha^\top \vartheta) \le \exp(-(2\mu - \mu)^2/2) \le \exp(-\mu^2/4)$.
\end{itemize}
Combining these three cases, we conclude that 
\begin{align}
    \Phi(\hat b_{l}) \le \Phi(b_{l}) + \exp(-\mu^2/4) + 3\mu \cdot \epsilon. 
    \label{eq:sparsity-3}
\end{align}
If we choose the covering net parameter $\epsilon = \rho \mu^{-1}$, then the upper bound can be simplified as $\Phi(\hat b_{l}) \le \Phi(b_{l}) + \exp(-n/4)+ 3\rho$.
Since $\rho$ is at least $1/n$, we can further conclude that $\Phi(\hat b_{l}) \le \Phi(b_{l}) + 4 \rho$ given that $\exp(-n/4)\ll 1/n \le \rho$. 
Lastly, note that the log cardinality of the joint covering net is upper bounded by 
\begin{align}
    \log(|\cN_\alpha|) + \log(|\cN_b|) \le t \log(\epsilon^{-1}) + \log(4\mu\epsilon^{-1}) \le C t \log(n(1 + \norm{\varsigma}_\infty\norm{\vartheta}_\infty))
    \label{eq:sparsity-4}
\end{align}
given that $\epsilon = \rho \mu^{-1} > (n\mu)^{-1}$. 
Here, for the last inequality, we use the fact that $\log(\mu) = \log( (t+1) (\sqrt{s tn} + \norm{\varsigma}_\infty \|\vartheta\|_\infty)) \le C\log(n(1 + \norm{\varsigma}_\infty \|\vartheta\|_\infty))$ since $t \le n^c$ for some constant $c>0$ and $s \le n$.
We can also apply a similar argument for every $t < n^c$. 
This only increases the size of the covering net by a factor $n^c$.
Combining \eqref{eq:sparsity-0}, \eqref{eq:sparsity-2} and \eqref{eq:sparsity-3} with the log cardinality \eqref{eq:sparsity-4}, we conclude that with probability at least $1 - \delta$ for all $\alpha, b_t$ and $\delta > \exp(- n/4)$ that 
\begin{align}
    Z(\alpha_{t-1}, b_t)  \le Z(\hat\alpha_{t-1}, \hat b_t) &\le C\cdot \biggl(\frac{1}{L} \sum_{l=1}^L \Phi(\hat b_{l, t}) + \rho s t\log(n(1 + \norm{\varsigma}_\infty\norm{\vartheta}_\infty)) + \rho s \log(\delta^{-1})\biggr) \\ 
    &\le C \cdot \biggl(\frac{1}{L} \sum_{l=1}^L \Phi(b_{l, t}) + \rho s t\log(n(1 + \norm{\varsigma}_\infty\norm{\vartheta}_\infty)) + \rho s \log(\delta^{-1})\biggr), 
\end{align}
where in the second inequality, we apply a union bound on the joint covering net for $\alpha$ and $b$ and also for all $t\le n^c$.
In the last inequality, we just need a change in the constant factor $C$ to absorb the approximation error $4\rho$ for the approximation error $\Phi( b_{t, l}) - \Phi(\hat b_{t, l})$.
Here, the lower bound $\delta > \exp(- n/4)$ is to ensure that the good event $\max_{\tau =-1, 0, \ldots, t-1} \norm{z_\tau}_\infty \le \sqrt{tn}$ holds true.
This concludes the proof of \Cref{lem:sparse-activation}.
\end{proof}

\subsubsection{Activations with Non-Gaussian Component: Proof of \Cref{lem:E-activation-perturbed}}\label{app:proof-E-activation-perturbed}
In the following proof, we will use $C$ to denote universal constants that change from line to line.
Let us denote by $\barb_t = b_t + \kappa_0$ as the shifted bias.
Let us pick $\varrho_t>0$ to be specified later.
For any $l\in[N_1]$, the neuron is activated only if either of the following two conditions hold:
\begin{enumerate}
    \item $e_l^\top y_t^\star + \barb_t > -\varrho_t$;
    \item $e_l^\top y_t^\star + \barb_t \le - \varrho_t$ and $e_l^\top \Delta y_t > \varrho_t$.
\end{enumerate}
For the first case, by \Cref{cor:E-activation-ideal}, we have with probability at least $1-\delta$ that
\begin{align}
    \frac{1}{N_1} \sum_{l=1}^{N_1} \ind(e_l^\top y_t^\star + \barb_t > -\varrho_t) \le C \cdot \bigl( \Phi(-\barb_t-\varrho_t) + \rho_1 s t \log(n) + \rho_1 s \log(\delta^{-1})\bigr).
    \label{eq:E-activation-1}
\end{align}
For the second case, we only need to control $N_1^{-1} \sum_{l=1}^{N_1} \ind(e_l^\top \Delta y_t > \varrho_t)$. We have the following upper bound
\begin{align}
    \frac{1}{N_1} \sum_{l=1}^{N_1} \ind(e_l^\top \Delta y_t > \varrho_t) 
    &\le \frac{1}{N_1} \sum_{l=1}^{N_1} \ind \Bigl(\norm{e_l}_2^2 \cdot \sum_{i=1}^{n-1} \Delta y_{t, i}^2 \ind(E_{l, i}\neq 0) > \varrho_t^2 \Bigr) \\
    &= \frac{1}{N_1} \sum_{l=1}^{N_1} \ind \Bigl(\sum_{i=1}^{n-1} \Delta y_{t, i}^2 \ind(E_{l, i}\neq 0) > \varrho_t^2 \Bigr) \\ 
    &\le \frac{1}{N_1 \varrho_t^2}\sum_{l=1}^{N_1} \sum_{i=1}^{n-1} \Delta y_{t, i}^2 \ind(E_{l, i}\neq 0) \le \frac{\rho_1}{\varrho_t^2}\cdot \norm{\Delta y_t}_2^2,
    \label{eq:E-activation-2}
\end{align}
where the first inequality holds by the Cauchy-Schwarz inequality and the following equality holds by the fact that $\norm{e_l}_2=1$. The second inequality follows from the fact that $\ind(x>a) \le x/a$ for any $a>0$ and $x>0$.
The last inequality holds by noting that $\norm{E_{:, i}}_0 \le \rho_1 N_1$.
Combining \eqref{eq:E-activation-1} and \eqref{eq:E-activation-2}, we conclude that with probability at least $1-n^{-c}$,
\begin{align}
    \frac{1}{N_1} \sum_{l=1}^{N_1} \ind(e_l^\top y_t > \barb_t) &\le C \cdot \bigl( \Phi(-\barb_t-\varrho_t) + \rho_1 s t \log(n) \bigr) + \frac{\rho_1}{\varrho_t^2}\cdot \norm{\Delta y_t}_2^2.
    \label{eq:E-activation-3}
\end{align}
Let us pick $\varrho_t = |\barb_t|^{-1}$. Note that by assumption \highlight{$\barb_t < -2$}, we have $-\barb_t -\varrho_t > 3/2$ and by the Mills ratio inequality $(x^{-1}-x^{-3}) < \Phi(x)/p(x) < x^{-1} - x^{-3} + 3x^{-5}$ for $x>0$, where $p(x) = \exp(-x^2/2)/\sqrt{2\pi}  $ is the density for standard Gaussian distribution, we have
\begin{align}\label{eq:E-activation-4}
    \Phi(-\barb_t - \varrho_t) 
    &\le \frac{1+ 3(|\barb_t| - \varrho_t)^{-4}}{\sqrt{2\pi} \cdot (|\barb_t| - \varrho_t)} \cdot \exp\Bigl( - \frac{(|\barb_t| - \varrho_t)^2}{2} \Bigr) \\ 
    &\le \frac{1 - |\barb_t|^{-2}}{\sqrt{2\pi} |\barb_t|} \cdot \exp\Bigl( - \frac{|\barb_t|^2}{2} \Bigr) \cdot 
    {\color{violet} \frac{(1+ 3(|\barb_t| - |\barb_t|^{-1})^{-4})|\barb_t|}{(|\barb_t| - |\barb_t|^{-1})(1 - |\barb_t|^{-2})} \cdot \exp\Bigl( \frac{2 -|\barb_t|^{-2}}{2} \Bigr)} \le C \Phi(-\barb_t), 
\end{align}
where in the last inequality, we note that the highlighted ratios are bounded by a universal constant. 
Combining \eqref{eq:E-activation-3} and \eqref{eq:E-activation-4}, we conclude the proof of \Cref{lem:E-activation-perturbed}.

\subsubsection{Concentration for $\norm{E^\top \varphi(E y_t^\star; b_t)}_2^2$: Proof of \Cref{lem:E-2nd}}
\label{app:proof-E-2nd}
When treating $\{\alpha_{\tau, t-1}\}_{\tau=-1}^{t-1}$ and $b_t$ to be deterministic, it follows that $y_t^\star \sim \cN(0, 1)$. 
When conditioned on the good event $\cE$, we always have  $\norm{y_t^\star}_\infty \le (1 + c)\sqrt{2(t+1)\log (nt)} $. 
In the following, we use $y$ to replace $y_t^\star$ for notation simplicity. 
We use $y_j$ to denote the $j$-th coordinate of $y$.
Let $\barb_t = b_t + \kappa_0$.

\paragraph{Good event on bounded Gaussian vectors}
Let $\cE_0$ denote the event that \textbf{InitCond-2} is satisfied by the vectors $z_{-1:0}$.
Throughout the proof, $C$ will denote a universal constant whose value may change from line to line.
Fix a time step $t \ge 1$ (we omit the subscript t for notational simplicity). Define the “good event”
\[
    \mathcal{E}_1 \;=\; \Bigl\{ \, \max_{\tau=-1,0,\dots,t-1} \,\|z_\tau\|_\infty \;\le\; (1+c)\sqrt{2\log (nt)} \Bigr\}.
\]
Then, by \Cref{lem:max gaussian_tail} (applied to the i.i.d. standard Gaussian vectors $z_{-1:t-1}$), we have
\[
    \mathbb{P}(\cE_1) \;\ge\; 1 - (nt)^{-c} \;\ge\; 1 - n^{-c}\,.
\] 

\paragraph{Good event on the activation sparsity}
Let us define $\cS_j = \{l\in [N_1]: E_{l, j}\neq 0\}$. It holds that $|\cS_j| \le N_1 \rho_1$.
In addition, we define event $\cE_2$ as 
\begin{align}
    \cE_2 = \left\{ 
        \begin{array}{l}
        \forall j\in[n-1] \\
        \forall \alpha_{t-1}\in\SSS^{t} \\
        \forall b_t\in\RR
        \end{array}, \quad  
    \sum_{l\in\cS_j} \ind(e_l^\top y + \barb_t > 0) \le C \cdot \biggl( \sum_{l\in\cS_j} \Phi\Bigl(-\frac{\barb_t + E_{l,j} y_j}{\sqrt{1-E_{l,j}^2}}\Bigr) + |\cS_j|\rho_2 s  t \log(n)\biggr)  
    \right\}. 
\end{align}
To show that $\cE_2$ holds with high probability, 
let us define $\tilde E$ as the submatrix of $E$ by keeping the rows indexed by $\cS_j$ while removing the $j$-th column.
We also normalize each row of $\tilde E$ to have $\ell_2$-norm equal to one. 
We then have
\begin{enumerate}
    \item $\norm{\tilde E_{l,:}}_2 = 1$, $\norm{\tilde E_{l, :}}_0 \le s$ and $\norm{\tilde E_{:, k}}_0 \le \sum_{l=1}^N \ind(H_{l, j}\neq 0) \ind(H_{l, k}\neq 0) \le |\cS_j| \rho_2$, where the last inequality holds by definition of $\rho_2$. 
    \item It holds that 
    $$|\cS_j|^{-1} \sum_{l\in\cS_j} \ind(e_l^\top y + \barb_t > 0) = |\cS_j|^{-1} \sum_{l\in\cS_j} \ind\Bigl(\tilde e_l^\top y_{-j} + \frac{\barb_t + E_{l,j} y_j}{\sqrt{1-E_{l,j}^2}} > 0 \Bigr),$$ where $\tilde e_l$ is the $l$-th row of $\tilde E$ and $y_{-j}$ is the vector $y$ with the $j$-th coordinate removed.
\end{enumerate}
In the following, we use $z_{\tau, j}$ to denote the $j$-th coordinate of $z_\tau$, and $y_j  = \sum_{\tau=-1}^{t-1} \alpha_{\tau, t-1} z_{\tau, j}$.
We denote by $z_{\tau, -j}$ the vector $z_\tau$ with the $j$-th coordinate removed.
Therefore, we can invoke \Cref{lem:sparse-activation} with the configurations \begin{center} 
    $G \leftarrow \tilde E$, $\rho \leftarrow \rho_2$, $\vartheta \leftarrow (z_{-1, j}, z_{0, j}, \ldots, \allowbreak z_{t-1, j})$, \\
$\varsigma_l \leftarrow E_{l, j}/\sqrt{1 - E_{l, j}^2}$, $b_t \leftarrow -\barb_t/\sqrt{1-E_{l,j}^2}$ and $z_{\tau} \leftarrow z_{\tau, -j}$ 
\end{center} 
to obtain that with probability at least $1 - \delta/n$ over the randomness of standard Gaussian vectors 
$z_{-1:t-1, -j}$, and 
for fixed $t$, $\alpha_{t-1}$, $b_t$ and $\vartheta = (z_{-1,j}, z_{0, j}, \ldots, z_{t-1, j})$,
\begin{align}
    \sum_{l\in\cS_j} \ind(e_l^\top y + \barb_t > 0) &\le C \cdot \biggl( \sum_{l\in\cS_j} \Phi\Bigl(-\frac{\barb_t + E_{l,j} y_j}{\sqrt{1-E_{l,j}^2}}\Bigr) + |\cS_j|\rho_2 s \log(n\delta^{-1})\biggr) \\
    &\le C \cdot \biggl( \sum_{l\in\cD_j} \Phi\Bigl(-\frac{\barb_t + H_{l,j} y_j}{\sqrt{1-H_{l,j}^2}}\Bigr) + |\cS_j|\rho_2 s \log(n\delta^{-1})\biggr), 
    \label{eq:E-2nd-sparse}
\end{align}
where $C$ is a universal constant independent of $t, \alpha_{t-1}, b_t$ and $\vartheta$.
Here, in the last inequality, we define $\cD_j = \{l\in [N]: H_{l,j} \neq 0\}$ as the set of rows in matrix $H$ that have nonzero $j$-th coordinate. 
Since $E$ is just a submatrix of $H$, adding more rows to the summation does not decrease the target value in the second inequality. 
Note that $z_{-1:t-1, -j}$ are independent of $z_{-1:t-1, j}$. 
We thus conclude that the above bound holds with probability at least $1 - \delta/n$ over the randomness of $z_{-1:t-1}$.
Further applying the union bound for all $j\in[n-1]$, we conclude that \eqref{eq:E-2nd-sparse} holds with probability at least $1 - \delta$ for all $j\in[n-1]$.

Note that the randomness discussed above is only over $z_{-1:t-1}$.
We invoke a covering argument over $\alpha_{t-1}\in\SSS^t$ and $b_t\in\RR$ similar to the proof of \Cref{lem:sparse-activation}. 
Since the argument is largely the same, we will not repeat it here.
The size of the covering net is ${n}^{O(t+1)}$, and we can pick $\delta=n^{-c - O(t+1)}$ in \eqref{eq:E-2nd-sparse}, which gives us the upper bound in the definition of $\cE_2$ with probability at least $1 - n^{-c}$.

\paragraph{Refined upper bound on $y$}
We work with a fixed time step $t$ and aim to bound every coordinate $y_j$ for $j \in [n-1]$.
Here, we recall definitions
\[
    y_j = \sum_{\tau=-1}^{t-1} \alpha_{\tau,t-1}\, z_{\tau,j},\quad \beta_{t-1} = \sqrt{\sum_{\tau=1}^{t-1} \alpha_{\tau,t-1}^2}
\]
where $\beta_{t-1}$ represents the \(\ell_2\)-norm of the component of \(\bar{w}_{t-1}\) in the subspace orthogonal to \(w_{-1:0}\). (Recall that the coefficients \(\{\alpha_{\tau,t-1}\}_{\tau=1}^{t-1}\) arise when projecting \(\bar{w}_{t-1}\) onto the orthonormal basis
\[
    \Bigl\{ \bar{w}_{-1},\, \frac{w_0^\perp}{\|w_0^\perp\|},\, \frac{w_1^\perp}{\|w_1^\perp\|},\, \dots,\, \frac{w_{t-1}^\perp}{\|w_{t-1}^\perp\|} \Bigr\}.
\]
To leverage \textbf{InitCond-2}, we make a change of basis for the first two directions, namely, we replace
\[
  \Bigl\{\bar{w}_{-1},\, \frac{w_0^\perp}{\|w_0^\perp\|}\Bigr\} \quad\text{with}\quad \bigl\{\bar{w}_0,\, \tilde{w}\bigr\},\where \tilde{w} = \alpha_{0,0}\,\bar{w}_{-1} - \alpha_{-1,0}\frac{w_0^\perp}{\|w_0^\perp\|} \,.
\]
Note that \(\tilde{w}\) is orthogonal to \(\bar{w}_0\).
The projection of \(\bar{w}_{t-1}\) onto the direction \(\tilde{w}\) satisfies
\[
  \bigl|\langle \bar{w}_{t-1},\, \tilde{w}\rangle\bigr| 
  = \bigl|\alpha_{0,0}\,\alpha_{-1,t-1} - \alpha_{-1,0}\,\alpha_{0,t-1}\bigr|
  \le |\alpha_{-1,t-1}| + |\alpha_{-1,0}|.
\]
Since \(\bar{w}_0\), \(\tilde{w}\), and \(\{w_\tau^\perp/\|w_\tau^\perp\|\}_{\tau=1}^{t-1}\) form an orthonormal basis, the component of \(\bar{w}_{t-1}\) orthogonal to \(\bar{w}_0\) is bounded by 
\(
    \beta_{t-1} + |\alpha_{-1,t-1}| + |\alpha_{-1,0}|.
\)
Moreover, we can also decompose \(y_t\) into the new basis as follows:
\begin{align}
    y_t &= \langle \bar{w}_0,\,\bar{w}_{t-1}\rangle\, \bigl(\alpha_{-1,0}\, z_{-1} + \alpha_{0,0}\, z_0\bigr) 
    + \langle \tilde{w},\,\bar{w}_{t-1}\rangle \bigl(\alpha_{0,0}\, z_{-1} -\alpha_{-1,0}\, z_0\bigr)
    + \sum_{\tau=1}^{t-1} \alpha_{\tau,t-1}\, z_{\tau} \\
    &= \langle \bar{w}_0,\,\bar{w}_{t-1}\rangle\, y_1 
    + \langle \tilde{w},\,\bar{w}_{t-1}\rangle \bigl(\alpha_{0,0}\, z_{-1} -\alpha_{-1,0}\, z_0\bigr)
    + \sum_{\tau=1}^{t-1} \alpha_{\tau,t-1}\, z_{\tau}
\end{align}
Under \textbf{InitCond-2} the first term, \(\langle \bar{w}_0,\,\bar{w}_{t-1}\rangle\,y_1\), is bounded by \(\zeta_1\). Moreover, since both 
\[
\alpha_{0,0}\,z_{-1} - \alpha_{-1,0}\,z_0 \quad \text{and} \quad \{z_\tau\}_{\tau=1}^{t-1}
\]
have their entries bounded by \(2(1+c)\sqrt{\log(nt)}\) on the good event \(\cE_1\), the contribution from the subspace orthogonal to \(\bar{w}_0\) is bounded by
\[
  C\,\bigl(\beta_{t-1}+|\alpha_{-1,t-1}| + |\alpha_{-1,0}|\bigr)\,\sqrt{t\log(nt)}.
\]
Thus, by the triangle inequality, for every coordinate \(j\) we have under event $\cE_0$ and $\cE_1$ that
\begin{align}
    y_j \le \zeta_1 + C\,\bigl(\beta_{t-1} + |\alpha_{-1,t-1}| + |\alpha_{-1,0}|\bigr)\,\sqrt{t\log(nt)}
    \eqqcolon \zeta_t.
    \label{eq:yt-bound}
\end{align}

\paragraph{Good event on the Bernstein concentration}
In the following, we will use another good event to control the upper bound in the definition of $\cE_2$. 
Consider the function $\Phi(-{(\barb_t + x y_j)}/{\sqrt{1-x^2}})^q$ for $q\ge 1$. 
We demonstrate that this function is Lipschitz continuous and monotonically increasing on the interval $x\in [0, 1]$ if $y_j > -\barb_t$ by taking the derivative with respect to $x$: 
\begin{align}
    \frac{\rd }{\rd x} \Phi\Bigl(-\frac{\barb_t + x y_j}{\sqrt{1-x^2}}\Bigr)^q = q \Phi\Bigl(-\frac{\barb_t + x y_j}{\sqrt{1-x^2}}\Bigr)^{q-1} \cdot p\Bigl(-\frac{\barb_t + x y_j}{\sqrt{1-x^2}}\Bigr) \cdot \frac{y_j - (-\barb_t) x}{(1-x^2)^{3/2}} > 0. 
\end{align}
Using the upper bound for $y$ specified in \eqref{eq:yt-bound}, 
we can define the \emph{critical value} $\hslash_{q, t}$ as the smallest real number such that the following inequality holds:
\begin{align}
    \frac{1}{|\cD_j|}\sum_{l\in\cD_j} \Phi\Bigl(-\frac{\barb_t + H_{l,j} y_j}{\sqrt{1-H_{l,j}^2}}\Bigr)^q \ind(\cE_0\cap \cE_1) \le \max_{j\in[n]}\frac{1}{|\cD_j|}\sum_{l\in\cD_j} \Phi\Bigl(-\frac{\barb_t + H_{l,j} \zeta_t}{\sqrt{1-H_{l,j}^2}}\Bigr)^q \le \Phi\Bigl(-\frac{\barb_t + \hslash_{q, t} \zeta_t}{\sqrt{1-\hslash_{q, t}^2}}\Bigr)^q. 
    \label{eq:critical-value-ub}
\end{align}
As we will only be using $q\in\{3, 4\}$ in the following proof, 
we define the event $\cE_3$ as the event such that for all $q\in\{3, 4\}$, $\alpha_{t-1}\in\SSS^t$,  $b_t\in\RR$ and $j\in[n-1]$, 
\begin{align}
   \cE_3:\quad &\sum_{j=1}^{n-1} \frac{1}{|\cD_j|}\sum_{l\in\cD_j} \Phi\Bigl(-\frac{\barb_t + H_{l,j} y_j}{\sqrt{1-H_{l,j}^2}}\Bigr)^q \ind(\cE_0)\ind(\cE_1)
    \\
    &\hspace{1cm}
    \le C\cdot\left(\sum_{j=1}^{n-1}  \frac{1}{|\cD_j|} \sum_{l\in\cD_j} \EE\biggl[\Phi\Bigl(-\frac{\barb_t + H_{l,j} y_j}{\sqrt{1-H_{l,j}^2}}\Bigr)^{q}\biggr]   +   \Phi\Bigl(-\frac{\barb_t + \hslash_{q, t} \zeta_t}{\sqrt{1-\hslash_{q, t}^2}}\Bigr)^q t \log (n)\right)\, ,
\end{align}
where $C$ is a universal constant independent of $t$, $\alpha_{t-1}$, $b_t$ and $\zeta_t$.
To show the event $\cE_3$ holds with high probability, we can apply the Bernstein concentration inequality in \Cref{lem:bernstein} for the bounded random variables 
\[ 
    \frac{1}{|\cD_j|}\sum_{l\in\cD_j} \Phi\Bigl(-\frac{\barb_t + H_{l,j} y_j}{\sqrt{1-H_{l,j}^2}}\Bigr)^q. 
\] That is, for fixed $\alpha_{t-1}$, $b_t$ and with probability at least $1-\delta$ over the randomness of $z_{-1:t-1}$, we have
\begin{align}
    &\sum_{j=1}^{n-1} \frac{1}{|\cD_j|}\sum_{l\in\cD_j} \Phi\Bigl(-\frac{\barb_t + H_{l,j} y_j}{\sqrt{1-H_{l,j}^2}}\Bigr)^q \ind(\cE_0)\ind(\cE_1)
     \\
    &\hspace{1cm} \le  \sqrt{2\log\delta^{-1} \cdot \sum_{j=1}^{n-1} \EE\biggl[\biggl(\frac{1}{|\cD_j|} \sum_{l\in\cD_j} \Phi\Bigl(-\frac{\barb_t + H_{l,j} y_j}{\sqrt{1-H_{l,j}^2}}\Bigr)^{q}\biggr)^2 \ind(\cE_0\cap\cE_1)\biggr]} \\
    &\hspace{2cm}+ \sum_{j=1}^{n-1}  \frac{1}{|\cD_j|} \sum_{l\in\cD_j} \EE\biggl[\Phi\Bigl(-\frac{\barb_t + H_{l,j} y_j}{\sqrt{1-H_{l,j}^2}}\Bigr)^{q}\biggr] +  \frac{1}{3}\Phi\Bigl(-\frac{\barb_t + \hslash_{q, t} \zeta_t}{\sqrt{1-\hslash_{q, t}^2}}\Bigr)^q\log (\delta^{-1}) .
\end{align}
Moreover, we have for the second moment term that 
\begin{align}
    &\EE\biggl[\biggl(\frac{1}{|\cD_j|} \sum_{l\in\cD_j} \Phi\Bigl(-\frac{\barb_t + H_{l,j} y_j}{\sqrt{1-H_{l,j}^2}}\Bigr)^{q}\biggr)^2 \ind(\cE_0\cap\cE_1)\biggr] \\
    &\qquad \le \sum_{j=1}^{n-1} \frac{1}{|\cD_j|^2} \cdot \sum_{l\in\cD_j} \EE\biggl[\Phi\Bigl(-\frac{\barb_t + H_{l,j} y_j}{\sqrt{1-H_{l,j}^2}}\Bigr)^{q} \cdot \ind(\cE_0\cap\cE_1)\biggr] \cdot \sum_{l'\in\cD_j} \Phi\Bigl(-\frac{\barb_t + H_{l',j} \zeta_t}{\sqrt{1-H_{l',j'}^2}}\Bigr)^{q}  \\
    &\qquad \le \sum_{j=1}^{n-1} \left(\frac{1}{|\cD_j|} \sum_{l\in\cD_j} \EE\biggl[\Phi\Bigl(-\frac{\barb_t + H_{l,j} y_j}{\sqrt{1-H_{l,j}^2}}\Bigr)^{q}\biggr] \right) \cdot \Phi\Bigl(-\frac{\barb_t + \hslash_{q, t} \zeta_t}{\sqrt{1-\hslash_{q, t}^2}}\Bigr)^q, 
\end{align}
where in the first inequality, we invoke the upper bound in \eqref{eq:critical-value-ub}. Using the fact that $\sqrt{a \cdot b} \le a + b$ for $a, b \ge 0$, we derive that 
\begin{align}
    &\sum_{j=1}^{n-1} \frac{1}{|\cD_j|}\sum_{l\in\cD_j} \Phi\Bigl(-\frac{\barb_t + H_{l,j} y_j}{\sqrt{1-H_{l,j}^2}}\Bigr)^q \ind(\cE_0)\ind(\cE_1)
     \\
    &\hspace{1cm} \le C\cdot \left(\sum_{j=1}^{n-1}  \frac{1}{|\cD_j|} \sum_{l\in\cD_j} \EE\biggl[\Phi\Bigl(-\frac{\barb_t + H_{l,j} y_j}{\sqrt{1-H_{l,j}^2}}\Bigr)^{q}\biggr] +  \frac{1}{3}\Phi\Bigl(-\frac{\barb_t + \hslash_{q, t} \zeta_t}{\sqrt{1-\hslash_{q, t}^2}}\Bigr)^q\log (\delta^{-1})\right)\,.
    \label{eq:E-2nd-bern}
\end{align}
Now, we apply the covering argument over $\alpha_{t-1}\in\SSS^t$ and $b_t\in\RR$ similar to the proof of \Cref{lem:sparse-activation}. 
The size of the covering net is ${n}^{O(t+1)}$, and we can pick $\delta=n^{-c - O(t+1)}$ in \eqref{eq:E-2nd-bern}, which gives us the upper bound in the definition of $\cE_3$ with $\PP(\cE_3)\ge 1 - n^{-c}$.

\paragraph{The Perturbed Variance}
Given the good events $\cE_0, \cE_1, \cE_2$, and $ \cE_3$, we define 
\begin{align}
    Z = \frac{1}{N_1^2} \sum_{l, l'=1}^{N_1} Z_{l, l'} , \where Z_{l, l'}= \varphi(e_l^\top y; b_t) \cdot \varphi(e_{l'}^\top y; b_t) \cdot \langle e_l, e_{l'}\rangle \cdot \ind(\cE_0\cap \cE_1\cap \cE_2\cap\cE_3).
    \label{eq:E-2nd-Z-def}
\end{align}
For concentration of $Z$, we consider the following Perturbed Variance (PV) defined as
\begin{align}
    V \defeq \EE \biggl[ \sum_{i=1}^{n-1} (Z - Z^{(i)})^2 \Biggiven y \biggr], 
\end{align}
where the perturbed term $Z^{(i)}$ is defined as follows:
\begin{align}
    Z^{(i)} = \frac{1}{N_1^2} \sum_{l, l'=1}^{N_1} Z_{l, l'}^{(i)}, \where Z_{l, l'}^{(i)} = \varphi(e_l^\top y^{(i)}; b_t) \cdot \varphi(e_{l'}^{\top} y^{(i)}; b_t) \cdot \langle e_l, e_{l'}\rangle \cdot \ind(\cap_{\iota=0}^3\cE_{\iota}^{(i)}).
\end{align}
Here, $y^{(i)} = \sum_{\tau=-1}^{t-1} \alpha_{\tau, t-1} z_\tau^{(i)}$ and $z_\tau^{(i)}$ is given by replacing the $i$-th coordinate of $z_\tau$ by an independent $\cN(0, 1)$ random variable.
In addition, the good events $\{\cE_\iota^{(i)}\}_{\iota=0}^3$ are defined similarly to $\cE_\iota$, but using $z_{-1:t-1}^{(i)}$ instead of $z_{-1:t-1}$.
We begin by noting the elementary inequality 
\(
(a - b)^2 \le 2a^2 + 2b^2.
\)
Thus, we obtain
\begin{align}
    V &\le \frac{2}{N_1^4}\,\underbrace{\EE\Biggl[ \sum_{i=1}^{n-1}\Biggl( \sum_{l,l'=1}^{N_1} Z_{l,l'}\,\ind\Bigl\{E_{l,i} \neq 0 \,\lor\, E_{l',i} \neq 0\Bigr\} \Biggr)^2 \,\Bigg|\, y \Biggr]}_{\ds \braRNum{1}} \nonumber\\[1mm]
    &\qquad + \frac{2}{N_1^4}\,\underbrace{\EE\Biggl[ \sum_{i=1}^{n-1}\Biggl( \sum_{l,l'=1}^{N_1} Z_{l,l'}^{(i)}\,\ind\Bigl\{E_{l,i} \neq 0 \,\lor\, E_{l',i} \neq 0\Bigr\} \Biggr)^2 \,\Bigg|\, y \Biggr]}_{\ds \braRNum{2}},
    \label{eq:V-bound}
\end{align}
where the upper bound is obtained by the following reasoning. For each perturbed quantity \(Z^{(i)}\), we have
\begin{align}
Z - Z^{(i)} &= \frac{1}{N_1^2} \sum_{l,l'=1}^{N_1} \bigl( Z_{l,l'} - Z_{l,l'}^{(i)}\bigr) \cdot 
\,\ind\bigl\{E_{l,i}\neq 0 \,\lor\, E_{l',i}\neq 0\bigr\}.
\end{align}
Note that the difference \(Z_{l,l'} - Z_{l,l'}^{(i)}\) is nonzero only when at least one of the vectors \(e_l\) or \(e_{l'}\) has a nonzero \(i\)th coordinate.
The two terms \(\braRNum{1}\) and \(\braRNum{2}\) correspond to the contributions from the original and the perturbed parts, respectively.
In what follows we focus on an upper bound for the term \(\braRNum{1}\); the term \(\braRNum{2}\) can be estimated by a completely analogous argument.

\paragraph{Controlling Term \(\braRNum{1}\)}
Due to the $L$-Lipschitz continuity of $\varphi$ with
\(
L = \gamma_2 + |b_t|\gamma_1,
\)
on the good event $\cE_1$, the absolute value of $\varphi(e_l^\top y; b_t)$ is bounded by
\(
|\varphi(e_l^\top y; b_t)| \;\le\; |\varphi(0;b_t)| + L\cdot|e_l^\top y|,
\)
which can be further bounded as
\begin{align}
    |\varphi(e_l^\top y; b_t)| 
    \;\le\; (d\lor n)^{-c_0} + L \sqrt{s}\cdot\norm{y}_\infty 
    \;\le\; C\,L \sqrt{t\,s\,\log(n)}
    \;\defeq\; B_t,
    \label{eq:E-2nd-Bt}
\end{align}
where we used that $t\le n^c$, $\norm{e_l}_1\le \sqrt{s}$, and that
\(
(d\lor n)^{-c_0}\le 1 \le L \sqrt{t\,s\,\log(n)}.
\)
Note that the same bound holds for $\varphi(e_l^\top y^{(i)}; b_t)$ on the corresponding good event $\cE^{(i)}_1$.
For $Z_{l, l'}$ defined in \eqref{eq:E-2nd-Z-def}, we first upper bound $\varphi(e_l^\top y; b_t) \cdot \varphi(e_{l'}^\top y; b_t)$ by 
$$\varphi(e_l^\top y; b_t) \cdot \varphi(e_{l'}^\top y; b_t) \le B_t^2 \ind(e_l^\top y + \barb_t > 0) \ind(e_{l'}^\top y + \barb_t > 0) + 2 B_t (d\lor n)^{-c_0} + (d\lor n)^{-2c_0},  $$ 
where we recall that if 
\(
e_l^\top y + \barb_t > 0,
\)
the neuron is deemed activated and its output is bounded above by \(B_t\). Otherwise, by \Cref{assump:activation}, the activation is bounded by \((d\lor n)^{-c_0}\).
Note that the term $(d\lor n)^{-c_0} B_t^{-1}$ can be made arbitrarily small as \highlight{$c_0$ is some large constant no less than $4$}.
Therefore, we just keep the first term above. 
Secondly, the inner product $\langle e_l, e_{l'}\rangle$ is upper bounded by $\sum_{j=1}^{n-1} \ind(E_{l,j}\neq 0) \cdot \ind(E_{l',j}\neq 0)$ as $\norm{E}_\infty \le 1$. Lastly, the indicator $\ind(E_{l, i}\neq 0 \lor E_{l', i}\neq 0)$ can be upper bounded by $\ind(E_{l, i}\neq 0) + \ind(E_{l', i}\neq 0)$.
For $\braRNum{1}$, we then have
\begin{align}
    \braRNum{1} 
    &\;\le\; \frac{C B_t^4}{N_1^4}\cdot \EE\Biggl[ \sum_{i=1}^{n-1} \Biggl( \sum_{j=1}^{n-1} \sum_{l,l'=1}^{N_1} \ind\bigl(e_l^\top y+\barb_t>0\bigr)\,\ind\bigl(e_{l'}^\top y+\barb_t>0\bigr) \nonumber\\[1mm]
    &\hspace{3cm}\cdot \Bigl(\ind\bigl\{E_{l,i}\neq 0\bigr\} + \ind\bigl\{E_{l',i}\neq 0\bigr\}\Bigr)
    \,\ind\bigl\{E_{l,j}\neq 0\bigr\}\,\ind\bigl\{E_{l',j}\neq 0\bigr\} \Biggr)^2 \cdot \ind\bigl(\cap_{\iota=0}^3 \cE_{\iota}\bigr) \,\Bigg|\, y \Biggr].
\end{align} 
Due to symmetry in the indices $l$ and $l'$, we can multiply the constant factor $C$ by $2$ and obtain
\begin{align}
    \braRNum{1} 
    &\le \frac{CB_t^4}{N_1^4} \cdot \EE\biggl[ \sum_{i=1}^{n-1} \Bigl( \sum_{j=1}^{n-1} \sum_{l=1}^{N_1} \ind(e_l^\top y + \barb_t > 0) \cdot \ind(E_{l, i}\neq 0)  \cdot \ind(E_{l,j}\neq 0)\\ 
    &\hspace{5cm}  \cdot \sum_{l'=1}^{N_1} \ind(E_{l',j}\neq 0) \cdot \ind(e_{l'}^\top y + \barb_t > 0)    \Bigr)^2 \cdot \ind(\cap_{\iota=0}^3\cE_{\iota}) \Biggiven y \biggr] \\
    &\le \frac{CB_t^4}{N_1^4} \cdot \EE\biggl[ \sum_{i=1}^{n-1} \sum_{j=1}^{n-1} \biggl( \sum_{l=1}^{N_1} \ind(e_l^\top y + \barb_t > 0) \cdot \ind(E_{l, i}\neq 0)  \cdot \ind(E_{l,j}\neq 0) \biggr)^2\\ 
    &\hspace{5cm}  \cdot \biggl(\sum_{l'=1}^{N_1} \ind(E_{l',j}\neq 0) \cdot \ind(e_{l'}^\top y + \barb_t > 0)    \biggr)^2 \cdot \ind(\cap_{\iota=0}^3\cE_{\iota})\Biggiven y \biggr],
\end{align}
where the last inequality holds by the Cauchy-Schwarz inequality.
Note that for $i\neq j$:
\begin{equation}
    \sum_{l=1}^{N_1} \ind(E_{l,i}\neq 0) \cdot \ind(E_{l,j}\neq 0) \le   \sum_{l=1}^{N} \ind(H_{l,i}\neq 0) \cdot \ind(H_{l,j}\neq 0) \le \rho_1 \rho_2 N. 
\end{equation}
Using $\rho_1 \rho_2 N$ to substitue one $\sum_{l=1}^{N_1} \ind(e_l^\top y + \barb_t > 0) \cdot \ind(E_{l, i}\neq 0)  \cdot \ind(E_{l,j}\neq 0)$ for $i\neq j$, we obtain 
\begin{align}
    \braRNum{1}
    &\le \frac{B_t^4 N\rho_1\rho_2}{N_1^4} \cdot \EE\biggl[ \sum_{j=1}^{n-1}\sum_{i\neq j}  \sum_{l=1}^{N_1} \ind(e_l^\top y + \barb_t > 0) \cdot \ind(E_{l, i}\neq 0)  \cdot \ind(E_{l,j}\neq 0)\\ 
    &\hspace{5cm}  \cdot \Bigl(\sum_{l'=1}^{N_1} \ind(E_{l',j}\neq 0) \cdot \ind(e_{l'}^\top y + \barb_t > 0)    \Bigr)^2 \cdot  \ind(\cap_{\iota=0}^3\cE_{\iota}) \Biggiven y \biggr] \\
    &\hspace{1cm} + \frac{B_t^4}{N_1^4} \cdot \EE\biggl[ \sum_{j=1}^{n-1}  \Bigl(\sum_{l=1}^{N_1} \ind(e_l^\top y + \barb_t > 0)   \cdot \ind(E_{l,j}\neq 0) \Bigr)^2\\ 
    &\hspace{5cm}  \cdot \Bigl(\sum_{l'=1}^{N_1} \ind(E_{l',j}\neq 0) \cdot \ind(e_{l'}^\top y + \barb_t > 0)    \Bigr)^2 \cdot  \ind(\cap_{\iota=0}^3\cE_{\iota}) \Biggiven y \biggr].
\end{align}
Rearranging the order of summation and using the fact that $\sum_{i\neq j} \ind(E_{l, i}\neq 0) \le s$ for any fixed $j$, we can further simplify the terms as
\begin{align}
    \braRNum{1} &\le \frac{2B_t^4 \rho_1\rho_2 s}{N_1^3} \cdot \EE\biggl[ \sum_{j=1}^{n-1} \Bigl(\sum_{l=1}^{N_1} \ind(e_l^\top y + \barb_t > 0) \cdot \ind(E_{l,j}\neq 0)\Bigr)^3 \cdot  \ind(\cap_{\iota=0}^3\cE_{\iota})\Biggiven y \biggr] \\
    &\hspace{1cm} + \frac{2B_t^4 }{N_1^4} \cdot \EE\biggl[ \sum_{j=1}^{n-1} \Bigl(\sum_{l=1}^{N_1} \ind(e_l^\top y + \barb_t > 0) \cdot \ind(E_{l,j}\neq 0)\Bigr)^4 \cdot  \ind(\cap_{\iota=0}^3\cE_{\iota})\Biggiven y \biggr]. 
    \label{eq:E-2nd-term-1}
\end{align}
Observe that the above two terms share a common structure. We define the common structure as 
\begin{align}
    \braRNum{3} \coloneqq \frac{1}{N_1^q} \cdot \EE\biggl[ \sum_{j=1}^{n-1} \Bigl(\sum_{l=1}^{N_1} \ind(e_l^\top y + \barb_t > 0) \cdot \ind(E_{l,j}\neq 0)\Bigr)^q \cdot  \ind(\cap_{\iota=0}^3\cE_{\iota})\Biggiven y \biggr],
\end{align}
where $q\in \{3, 4\}$.
Recall the definition $\cS_j = \{l\in [N_1]: E_{l, j}\neq 0\}$. It holds that $|\cS_j| \le N_1 \rho_1$.
We aim to control 
$$\sum_{l=1}^{N_1} \ind(e_l^\top y + \barb_t > 0) \cdot \ind(E_{l,j}\neq 0) = |\cS_j|^{-1}\sum_{l\in\cS_j} \ind(e_l^\top y + \barb_t > 0)$$
 in the following. 
By the definition of the good event $\cE_2$, we have
\begin{align}
    \braRNum{3} &\le \frac{C}{N_1^q} \cdot \sum_{j=1}^{n-1} \Biggl( \biggl( \sum_{l\in\cD_j} \Phi\Bigl(-\frac{\barb_t + H_{l,j} y_j}{\sqrt{1-H_{l,j}^2}}\Bigr) + |\cS_j|\rho_2 s t\log(n)\biggr)^q  \ind(\cap_{\iota=0}^3\cE_{\iota}) \Biggr) \\ 
        &\le \frac{2^{q-1}C}{N_1^q} \cdot \sum_{j=1}^{n-1} |\cD_j|^q \cdot \Biggl( \frac{1}{|\cD_j|}\sum_{l\in\cD_j} \Phi\Bigl(-\frac{\barb_t + H_{l,j} y_j}{\sqrt{1-H_{l,j}^2}}\Bigr)^q  \ind(\cap_{\iota=0}^3\cE_{\iota}) + \bigl(\rho_2 s t\log(n)\bigr)^{q} \Biggr) \\ 
        &\le  C \rho_1^q \cdot \Biggl( \sum_{j=1}^{n-1}  \frac{1}{|\cD_j|}\sum_{l\in\cD_j} \Phi\Bigl(-\frac{\barb_t + H_{l,j} y_j}{\sqrt{1-H_{l,j}^2}}\Bigr)^q  \ind(\cap_{\iota=0}^3\cE_{\iota})  + n \bigl(\rho_2 s t\log(n)\bigr)^{q} \Biggr).
    \label{eq:E-2nd-1}
\end{align}
where we use the H\"older's inequality for the second line, and in the last line, we absorb the constant factor $2^{q-1}$ into the universal constant $C$ and use the fact that $|\cS_j|\le |\cD_j| \le N \rho_1 \le N_1 \rho_1 / (1 - \rho_1) \highlight{\le C_1 N_1 \rho_1}$ for all $j\in[n-1]$, where we also absorb the factor $C_1^q$ into the universal constant $C$.
By the definition of the good event $\cE_3$, it holds that
\begin{align}
    &\sum_{j=1}^{n-1} \frac{1}{|\cD_j|}\sum_{l\in\cD_j} \Phi\Bigl(-\frac{\barb_t + H_{l,j} y_j}{\sqrt{1-H_{l,j}^2}}\Bigr)^q  \ind(\cap_{\iota=0}^3\cE_{\iota})
    \\
    &\hspace{1cm} \le  C \cdot \Biggl( \sum_{j=1}^{n-1}  \frac{1}{|\cD_j|} \sum_{l\in\cD_j} \EE\biggl[\Phi\Bigl(-\frac{\barb_t + H_{l,j} y_j}{\sqrt{1-H_{l,j}^2}}\Bigr)^{q}\biggr]   + \Phi\Bigl(-\frac{\barb_t + \hslash_{q, t} \zeta_t}{\sqrt{1-\hslash_{q, t}^2}}\Bigr)^q t\log (n) \Biggr) .
    \label{eq:E-2nd-bern-1}
\end{align}
To evaluate the expectation term, we use the Mills ratio $\Phi(x) \le C p(x)$ for some universal constant $C>0$, $x > 0$ and $p(x) = \exp(-x^2/2)/\sqrt{2\pi}$ to obtain
\begin{align}
    \EE\biggl[\Phi\Bigl(-\frac{\barb_t + H_{l,j} y_j}{\sqrt{1-H_{l,j}^2}}\Bigr)^q\biggr]
    &\le C \cdot \EE\biggl[ \exp\Bigl(-\frac{q(\barb_t + H_{l,j} y_j)^2}{2(1-H_{l,j}^2)}\Bigr) \ind(\barb_t + H_{l, j}y_j \le 0)\biggr] + \PP(\barb_t + H_{l,j} y_j > 0) \\
    &\le C \cdot \EE\biggl[ \exp\Bigl(-\frac{q(\barb_t + H_{l,j} y_j)^2}{2(1-H_{l,j}^2)}\Bigr)\biggr] + \Phi\Bigl(-\frac{\barb_t}{H_{l, j}}\Bigr) \\
    &= C \sqrt{\frac{1-H_{l, j}^2}{1+ (q-1)H_{l, j}^2}} \cdot \exp\Bigl(-\frac{\barb_t^2}{2(\frac{q-1}{q}H_{l,j}^2 + \frac{1}{q})}\Bigr) + \Phi\Bigl(-\frac{\barb_t}{H_{l, j}}\Bigr) , 
    \label{eq:E-2nd-2}
\end{align}
where the third equality holds by direct algebraic calculation for Gaussian integral. 
By the Mills ratio $\Phi(x) / p(x) \ge x^{-1} - x^{-3} = C x^{-1}$ for $x\gg 1$, and also the fact that $H_{l, j}\in [0, 1]$, we conclude that the right-hand side of \eqref{eq:E-2nd-2} is bounded by 
\begin{align} 
    \EE\biggl[\Phi\Bigl(-\frac{\barb_t + H_{l,j} y_j}{\sqrt{1-H_{l,j}^2}}\Bigr)^q\biggr] \le C |\barb_t| \Phi\biggl(\frac{-\barb_t}{\sqrt{\frac{q-1}{q}H_{l,j}^2 + \frac{1}{q}}} \biggr).
    \label{eq:E-2nd-2-bound}
\end{align}
Similar to the previous argument, we also have $\Phi\Bigl(-\frac{\barb_t}{\sqrt{\frac{q-1}{q} x^2 + \frac{1}{q}}}\Bigr)$ as a non-decreasing function of $x$ for $x\in [0, 1]$ by checking the derivative.
We define $\hslash_{q, \star}$ as the smallest real number such that the following inequality holds:
\begin{align}
    \sum_{j=1}^n\frac{1}{|\cD_j|}\sum_{l\in\cD_j} \Phi\biggl(\frac{-\barb_t}{\sqrt{\frac{q-1}{q}H_{l,j}^2 + \frac{1}{q}}} \biggr) \le n \cdot \Phi\biggl(\frac{-\barb_t}{\sqrt{\frac{q-1}{q} \hslash_{q, \star}^2 + \frac{1}{q}}} \biggr).
    \label{eq:E-2nd-2-hslash-2}
\end{align}
Plugging \eqref{eq:E-2nd-2-bound} and \eqref{eq:E-2nd-2-hslash-2} into \eqref{eq:E-2nd-bern-1}, we have that
\begin{align}
    &\sum_{j=1}^{n-1} \frac{1}{|\cD_j|}\sum_{l\in\cD_j} \Phi\Bigl(-\frac{\barb_t + H_{l,j} y_j}{\sqrt{1-H_{l,j}^2}}\Bigr)^q  \ind(\cap_{\iota=0}^3\cE_{\iota}) \\
    &\hspace{1cm} \le   C \cdot \Biggl( n |\barb_t|\, \Phi\Biggl(\frac{-\barb_t}{\sqrt{\frac{q-1}{q}\hslash_{q, \star}^2 + \frac{1}{q}}}\Biggr)  + \Phi\Bigl(-\frac{\barb_t + \hslash_{q, t} \zeta_t}{\sqrt{1-\hslash_{q, t}^2}}\Bigr)^q t\log (n) \Biggr) .
    \label{eq:E-2nd-bern-final}
\end{align}
Combining \eqref{eq:E-2nd-1} and \eqref{eq:E-2nd-bern-final}, we obtain 
\begin{align}
    \braRNum{3} \le C \rho_1^q \cdot \Biggl(
    n\,|\barb_t|\,\Phi\!\Biggl(\frac{-\barb_t}{\sqrt{\frac{q-1}{q}\,\hslash_{q,\star}^2+\frac{1}{q}}}\Biggr)
    +\Phi\Bigl(-\frac{\barb_t + \hslash_{q, t} \zeta_t}{\sqrt{1-\hslash_{q, t}^2}}\Bigr)^q t\log (n)
    + n\bigl(\rho_2\,s\,t\log(n)\bigr)^q
    \Biggr)\,.
\end{align}
Note that we always have $\hslash_{q, \star} \le 1$ and $\hslash_{q, t} \le 1$ for $t\ge 1$.
As both $\Phi\Bigl(\frac{-\barb_t}{\sqrt{\frac{q-1}{q} x^2+\frac{1}{q}}}\Bigr)$ and $\Phi\Bigl(-\frac{\barb_t + x \zeta_t}{\sqrt{1-x^2}}\Bigr)^q$ (when $\zeta_t > -\barb_t$) are non-decreasing functions with respect to $x$, for the first term in the right-hand side of \eqref{eq:E-2nd-term-1}, we take $q=3$ and $\hslash_{q, t}=1$ to have the following upper bound:
\begin{align}
    &\frac{2B_t^4 \rho_1\rho_2 s}{N_1^3} \cdot \EE\biggl[ \sum_{j=1}^{n-1} \Bigl(\sum_{l=1}^{N_1} \ind(e_l^\top y + \barb_t > 0) \cdot \ind(E_{l,j}\neq 0)\Bigr)^3 \cdot  \ind(\cap_{\iota=0}^3\cE_{\iota})\Biggiven y \biggr] \\
    &\hspace{1cm} \le C B_t^4 \rho_1^4 \rho_2 s \cdot \biggl( n |\barb_t| \Phi\biggl(\frac{-\barb_t}{\sqrt{\frac{2}{3} \hslash_{3, \star}^2 + \frac{1}{3}}}\biggr) + t\log(n) + n(\rho_2 s t\log(n))^3 \biggr). 
    \label{eq:E-2nd-term-1.1}
\end{align}
For the second term on the right-hand side of \eqref{eq:E-2nd-term-1}, we take $q=4$ and obtain
\begin{align}
    &\frac{2B_t^4 }{N_1^4} \cdot \EE\biggl[ \sum_{j=1}^{n-1} \Bigl(\sum_{l=1}^{N_1} \ind(e_l^\top y + \barb_t > 0) \cdot \ind(E_{l,j}\neq 0)\Bigr)^4 \cdot  \ind(\cap_{\iota=0}^3\cE_{\iota})\Biggiven y \biggr] \\
    &\quad \le C B_t^4 \rho_1^4 \cdot \Biggl(
        n\,|\barb_t|\,\Phi\!\Biggl(\frac{-\barb_t}{\sqrt{\frac{3}{4}\,\hslash_{4,\star}^2+\frac{1}{4}}}\Biggr)
         + \Phi\Bigl(-\frac{\barb_t + \hslash_{4, t} \zeta_t}{\sqrt{1-\hslash_{4, t}^2}}\Bigr)^4 t\log (n)
        + n\bigl(\rho_2\,s\,t\log(n)\bigr)^4
        \Biggr)\,.
        \label{eq:E-2nd-term-1.2}
\end{align}
We conclude by combining \eqref{eq:E-2nd-term-1.1} and \eqref{eq:E-2nd-term-1.2} that
\begin{align}
    \braRNum{1} \le & C B_t^4 \rho_1^4 \cdot \Biggl(
        n\,|\barb_t|\,\Phi\!\biggl(\frac{-\barb_t}{\sqrt{\frac{3}{4}\,\hslash_{4,\star}^2+\frac{1}{4}}}\biggr) + \rho_2 s n |\barb_t| \Phi\biggl(\frac{-\barb_t}{\sqrt{\frac{2}{3} \hslash_{3, \star}^2 + \frac{1}{3}}}\biggr)\\
        & \hspace{3cm} + \biggl(\Phi\Bigl(-\frac{\barb_t + \hslash_{4, t} \zeta_t}{\sqrt{1-\hslash_{4, t}^2}}\Bigr)^4 + \rho_2 s \biggr) t\log (n)
        + n\bigl(\rho_2\,s\,t\log(n)\bigr)^4  
        \Biggr) \eqqcolon V_0. 
\end{align}
Similarly, $\braRNum{2}$ can be bounded by $V_0$.
We are now ready to invoke \Cref{lem:efron-stein-bounded-variance}. 
Since $V \leq 2V_0$ with probability 1, the final bound for $|Z- \EE[Z]|$ is then given by 
\begin{align}
    |Z- \EE[Z]| &\le C \sqrt{V_0 \log(\delta^{-1})}, 
\end{align}
where the inequality holds with probability at least $1-\delta$ over the randomness of standard Gaussian vectors $z_{-1:T}$.
Plugging in the formula for $V_0$, we obtain the following upper bound 
\begin{align}
    |Z- \EE[Z]| &\le  C B_t^2 \rho_1^2 \cdot \Biggl(
        n\,|\barb_t|\,\Phi\!\Biggl(\frac{-\barb_t}{\sqrt{\frac{3}{4}\,\hslash_{4,\star}^2+\frac{1}{4}}}\Biggr) + \rho_2 s n |\barb_t| \Phi\biggl(\frac{-\barb_t}{\sqrt{\frac{2}{3} \hslash_{3, \star}^2 + \frac{1}{3}}}\biggr)
         \nonumber\\ 
        &\hspace{3cm}
        + \biggl(\Phi\Bigl(-\frac{\barb_t + \hslash_{4, t} \zeta_t}{\sqrt{1-\hslash_{4, t}^2}}\Bigr)^4 + \rho_2 s \biggr) t\log(n)
        + n\bigl(\rho_2\,s\,t\log(n)\bigr)^4 
        \Biggr)^{1/2} \cdot \log \delta^{-1}
\end{align}
with probability $1- \delta$.
For notational convenience, we define $\cK_t$ as the $1/4$ power of each term inside the bracket in the above equation (see \eqref{eq:K_t-def} for the definition).
The fluctuation of $Z$ is controlled by 
\begin{align}
    |Z- \EE[Z]| &\le C L^2 \rho_1^2 ts \log n \cdot \cK_t^2 \cdot \log \delta^{-1}, 
\end{align}
where we plug in the definition $B_t = L \sqrt{ts \log n}$ and $L = \gamma_2 + |b_t| \gamma_1$ is the Lipschitz constant for the activation function $\varphi$.

\paragraph{Expectation $\EE[Z]$}
For $\EE[\|E^\top \varphi(E y_t^\star; b_t)\|_2^2]$, we have
\begin{align}
    \frac{1}{N_1^2}\cdot \EE\bigl[\|E^\top \varphi(E y_t^\star; b_t)\|_2^2\bigr] 
    &\le \frac{1}{N_1^2}\sum_{l, l'=1}^N \EE\bigl[\bigl|\varphi(\tilde h_l^\top y; b_t) \cdot \varphi(\tilde h_{l'}^\top y; b_t)\bigr|\bigr] \cdot \langle \tilde h_l, \tilde h_{l'}\rangle\\
    &\le C_1^2 \cdot \hat\EE_{l, l'} \Bigl[ \EE\bigl[\bigl|\varphi(\tilde h_l^\top y; b_t) \cdot \varphi(\tilde h_{l'}^\top y; b_t)\bigr|\bigr] \cdot \langle \tilde h_l, \tilde h_{l'}\rangle \Bigr],
\end{align}
where in the first inequality, we obtain the upper bound by also adding the rows of $F$ that are not contained in the submatrix $E$ to the sum. 
Here, we use the notation $$\tilde h_l = (H_{l, 1}, \ldots, H_{l, i-1}, H_{l, i+1}, \ldots, H_{l, n-1})^\top$$ to denote the $l$-th row of $H$ with the $i$-th entry removed. 
This structure comes from the definition \eqref{eq:H_decomposition} where we decompose the matrix $H$ into submatrices $E$, $F$ and the column vector $\theta$ as the non-zero entries in $H_{:, i}$ if the feature of interest is the $i$-th feature.
In the second inequality, we use the fact that $N/N_1 \le C_1$, and define $\hat\EE_{l, l'}$ as the empirical expectation over $l, l'\in [N]^2$.
Invoking \Cref{prop:2nd-moment-expectation} with $L = \gamma_2 + |b_t| \gamma_1$, $\barb=\barb_t = b_t + \kappa_0$, we conclude that 
\begin{align}
    &\hat\EE_{l, l'} \Bigl[ \EE\bigl[\bigl|\varphi(\tilde h_l^\top y; b_t) \cdot \varphi(\tilde h_{l'}^\top y; b_t)\bigr|\bigr] \cdot \langle \tilde h_l, \tilde h_{l'}\rangle \Bigr] \label{eq:E-2nd-expectation}\\
       &\quad \le C L \cdot (n\lor d)^{-c_0} + C L^2 \cdot \Phi(|\barb_t|) \cdot \hat\EE_{l, l'}\left[\Phi\Bigl( |\barb_t|\sqrt{\frac{1-\langle \tilde h_l, \tilde h_{l'}\rangle}{1+\langle \tilde h_l, \tilde h_{l'}\rangle}}\Bigr) \langle \tilde h_l, \tilde h_{l'}\rangle\right] \\
       &\quad \le C L \cdot (n\lor d)^{-c_0} + C L^2 \cdot \Phi(|\barb_t|) \cdot \hat\EE_{l, l'}\left[\Phi\Bigl( |\barb_t|\sqrt{\frac{1-\langle h_l, h_{l'}\rangle}{1+\langle h_l, h_{l'}\rangle}}\Bigr) \langle h_l, h_{l'}\rangle\right],
\end{align}
where in the first inequality, we directly apply \Cref{prop:2nd-moment-expectation} to the expectation term, and in the second inequality, we use the fact that $\langle \tilde h_l, \tilde h_{l'}\rangle \le \langle h_l, h_{l'}\rangle$ for $l, l' \in [N_1]$ and the fact that the term inside the expectation is non-decreasing when increasing the value of $\langle \tilde h_l, \tilde h_{l'}\rangle$.
Just as before, since \highlight{$c_0>4$ is large enough}, the first term is negligible, and we can absorb it into the constant $C$ and focus on the second term: 
\begin{align}
    \frac{1}{N_1^2}\cdot \EE\bigl[\|E^\top \varphi(E y_t^\star; b_t)\|_2^2\bigr] \le C L^2 \cdot \Phi(|\barb_t|) \cdot \hat\EE_{l, l'}\left[\Phi\Bigl( |\barb_t|\sqrt{\frac{1-\langle h_l, h_{l'}\rangle}{1+\langle h_l, h_{l'}\rangle}}\Bigr) \langle h_l, h_{l'}\rangle\right].
\end{align}
Since $\|E^\top \varphi(E y_t^\star; b_t)\|_2^2$ is non-negative, the same upper bound applies to $\EE[Z]$, where $Z$ includes the indicator condition $\ind(\cap_{\iota=0}^3 \cE_{\iota})$.

Finally, we plug in $\delta = n^{-c}$ to conclude that with probability at least $1-n^{-c}$ it holds that
\begin{align}
    \frac{1}{N_1^2}\norm{E^\top \varphi(E y_t^\star; b_t)}_2^2\cdot  \ind(\cE_0) \cdot \ind(\cap_{\iota=0}^3\cE_{\iota})
    & \le C L^2 \cdot \rho_1^2 s t^2  (\log n)^2 \cdot \cK_t^2\\
    &\qquad \quad + C L^2 \cdot \Phi(|\barb_t|) \cdot \hat\EE_{l, l'}\left[\Phi\Bigl( |\barb_t|\sqrt{\frac{1-\langle h_l, h_{l'}\rangle}{1+\langle h_l, h_{l'}\rangle}}\Bigr) \langle h_l, h_{l'}\rangle\right].
\end{align}

Note that the joint event $\ind(\cap_{\iota=1}^3\cE_{\iota})$ holds with probability at least $1- n^{-c}$ as we discussed earlier. 
Therefore, we can safely drop the indicator $\ind(\cap_{\iota=1}^3\cE_{\iota})$ in the above inequality. This completes the proof of \Cref{lem:E-2nd}.

\subsubsection{Concentration for $\norm{F^\top \varphi(Fy_t + \theta \cdot v^\top \barw_{t-1}; b_t)}_2^2$: Proof of \Cref{lem:F-2nd}}
\label{app:proof-F-2nd}
In the following proof, we will use $C$ to denote universal constants that change from line to line.
    Let us fix $\{\alpha_{\tau, t-1}\}_{\tau=-1}^{t-1}$ and $b_{t}$. Then $y_t^\star \sim \cN(0, I_{n-1})$. For simplicity, we will denote $y_t^\star$ by $y$ in the following.
    Let us define the good event 
\begin{align}
    \cE=\bigl\{ \max_{\tau=-1, 0, \ldots, t-1} \norm{z_\tau}_\infty \le (1+\sqrt c)\sqrt{2\log (nt)} \bigr\}.
\end{align}
It then follows from \Cref{lem:max gaussian_tail} that $\PP(\overline\cE) \le (nt)^{-c} \le n^{-c}$, and also $\norm{y}_\infty \le (1+\sqrt c)\sqrt{2t\log (nt)}$ on $\cE$. In particular, 
\begin{align}
    |\varphi(f_l^\top y + \theta_l v^\top \barw_{t-1}; b_t)| \ind(\cE) \le (\gamma_2+ |b_t| \gamma_1) ((1+\sqrt c)\sqrt{2t\log (nt)} + \theta_l \norm{v}_2 \alpha_{-1, t-1}) + (n\lor d)^{-c_0}\defeq B_t, 
\end{align}
where the inequality holds by the Lipschitz continuity of $\varphi$ in \Cref{assump:activation} and also the fact that \highlight{$b_t + \kappa_0 \le 0$} for the bias.
         Define 
            \begin{align}
                Z = \frac{1}{N_2^2} \sum_{l, l'=1}^{N_2} \langle f_l, f_{l'}\rangle \cdot \varphi\bigl(f_l^\top y + \theta_l v^\top \barw_{t-1}; b_{t}\bigr) \cdot \varphi\bigl(f_{l'}^\top y + \theta_{l'} v^\top \barw_{t-1}; b_{t}\bigr) \ind(\cE).    
                \end{align}
            Using the Cauchy-Schwarz inequality, we have
            \begin{align}
                Z 
                &\le  \frac{1}{N_2^2} \sum_{l, l'=1}^{N_2} \bigl( \varphi\bigl(f_l^\top y + \theta_l v^\top \barw_{t-1}; b_{t}\bigr)^2 + \varphi\bigl(f_{l'}^\top y + \theta_{l'} v^\top \barw_{t-1}; b_{t}\bigr)^2 \bigr) \cdot \ind(\langle f_l, f_{l'}\rangle\neq 0) \cdot \ind(\cE) \\
                &\le   \frac{2\rho_2}{N_2 } \sum_{l=1}^{N_2} \varphi\bigl(f_{l}^\top y + \theta_{l} v^\top \barw_{t-1}; b_{t}\bigr)^2 \ind(\cE), 
                \label{eq:signal-2nd-moment-ub-1}
            \end{align}
            where the first inequality follows from $ab\le a^2 + b^2$, and the second inequality follows from the fact that $ \langle f_l, f_{l'}\rangle^2 $ is nonzero for at most $N_2\rho_2 $ terms when going over $l'$  by definition \eqref{eq:rho-def}.
            Next, we concentrate the right-hand side of \eqref{eq:signal-2nd-moment-ub-1}. 
            Note that by the Lipschitz continuity of $\varphi$, we have
            \begin{align}
                |\varphi(f_l^\top y + \theta_l v^\top \barw_{t-1}; b_{t})| \le (\gamma_2+ |b_{t}| \gamma_1) (|f_l^\top y| + \theta_l \norm{v}_2 \alpha_{-1, t-1}) + (n\lor d)^{-c_0}. 
            \end{align}
            By the Cauchy-Schwarz inequality, we further obtain
            \begin{align} \label{eq:signal-2nd-moment-ub-2}
                \varphi(f_l^\top y + \theta_l v^\top \barw_{t-1}; b_{t})^2 \le C (\gamma_2+ |b_{t}| \gamma_1)^2 \bigl( (f_l^\top y)^2 + (\theta_l \norm{v}_2 \alpha_{-1, t-1})^2\bigr) + C(n\lor d)^{-2c_0}.
            \end{align}
            To this end, we apply the Cauchy-Schwarz inequality again to obtain that
            \begin{equation} 
                \frac{1}{N_2} \sum_{l=1}^{N_2} (f_l^\top y)^2 \le \frac{1}{N_2}\sum_{l=1}^{N_2}  \Bigl(\sum_{j=1}^{n-1} y(j)^2 \ind(f_l(j)\neq 0)\Bigr) \cdot \norm{f_l}_2^2 \le \rho_2 \cdot \norm{y}_2^2.
            \end{equation}
            Under the good event $\cE$, we have $\norm{y}_2 \le (1+\sqrt c)\sqrt{2t\log (nt)}$. 
            In fact, $\norm{y}_2^2 \sim \chi^2_{n-1}$, and we can apply the concentration inequality for the chi-squared distribution in \Cref{lem:chi-squared} to obtain that with probability at least $1 - \delta$, it holds over the randomness of $y$ that
            \begin{align}
                \frac{1}{N_2} \sum_{l=1}^{N_2} (f_l^\top y)^2 \ind(\cE) \le \frac{1}{N_2} \sum_{l=1}^{N_2} (f_l^\top y)^2\le  C\rho_2 \cdot \bigl( n + \log\delta^{-1}\bigr).
            \end{align}
            Applying a union bound over $\{\alpha_{\tau, t-1}\}_{\tau=-1}^{t-1}$ and $b_t$ similar to \Cref{lem:sparse-activation}, and since $Z$ is uniformly bounded, we conclude that with probability at least $1 - n^{-c}$, it holds for all $t\le n^c$ that
            \begin{align}
                \frac{1}{N_2} \sum_{l=1}^{N_2} (f_l^\top y)^2 \ind(\cE) \le C\rho_2 \cdot \bigl( n + t\log(n)\bigr).
                \label{eq:signal-2nd-moment-ub-3}
            \end{align}
            Combining \eqref{eq:signal-2nd-moment-ub-1}, \eqref{eq:signal-2nd-moment-ub-2}, and \eqref{eq:signal-2nd-moment-ub-3}, we conclude that with probability at least $1 - n^{-c}$, it holds for all $t\le n^c$ that
            \begin{align}
                Z \le C (\gamma_2+ |b_{t}| \gamma_1)^2  \rho_2 \cdot \Bigl( N_2^{-1} \norm{\theta}_2^2 \norm{v}_2^2 \alpha_{-1, t-1}^2 + \rho_2 n +\rho_2 t\log n \Bigr).
            \end{align}
            As the good event $\cE$ holds with sufficiently high probability if we choose $c$ large enough in the definition of $\cE$, 
            A similar bound holds for the original quantity $\norm{F^\top \varphi(Fy_t + \theta \cdot v^\top \barw_{t-1}; b_t)}_2^2$.
            This completes the proof of \Cref{lem:F-2nd}.

\subsubsection{Concentration for $\langle z_\tau, E^\top \varphi(E y_t^\star; b_t)\rangle$: Proof of \Cref{lem:E-1st}}
\label{app:proof-E-1st}
In the following proof, we will use $C$ to denote universal constants that change from line to line.
When treating $\{\alpha_{\tau, t-1}\}_{\tau=-1}^{t-1}$ and $b_t$ to be deterministic, we have $z_\tau \overset{d}{=} \alpha_{\tau, t-1} y_t^\star + \sqrt{1-\alpha_{\tau, t-1}^2} \cdot z$, where $z \sim \cN(0, I_{n-1})$ and is independent of $y_t^\star$.
In the following, we use $y$ to replace $y_t^\star$, and $\alpha$ to replace $\alpha_{\tau, t-1}$ for notational simplicity.
Therefore, the concentration we consider can be reduced to  the concentration of
\begin{align}
    \alpha \cdot \frac{1}{N_1} \langle y, E^\top \varphi(E y; b_t) \rangle + \sqrt{1-\alpha^2} \cdot \frac{1}{N_1} \langle z, E^\top \varphi(E y; b_t)\rangle, 
\end{align}
Firstly, note that when conditioned on $y$, $\langle z, E^\top \varphi(E y; b_t)\rangle$ is a gaussian random variable with mean zero and variance $\norm{E^\top \varphi(E y; b_t)}_2^2$, it holds with probability at least $1-\delta$ over the randomness of $y$ that 
\begin{align}
    \frac{1}{N_1} |\langle z, E^\top \varphi(E y; b_t)\rangle| \le \frac{1}{N_1} \sqrt{ 2 \norm{E^\top \varphi(E y; b_t)}_2^2  \log \delta^{-1}}, 
\end{align}
where the second order term has already been handled in \Cref{lem:E-2nd}. 
Similar to the proof of \Cref{lem:sparse-activation}, we can use a covering
argument over $\{\alpha_{\tau, t-1}\}_{\tau=-1}^{t-1}\in\SSS^{t+1}$, $b_t\in\RR$,  $\tau =-1, 0, \ldots, t-1$ and $t\le n^c$ to obtain that with probability at least $1 - n^{-c}$, it holds for all $(\tau,t)$ that 
\begin{align}
    \frac{1}{N_1} |\langle z, E^\top \varphi(E y; b_t)\rangle| \le \frac{C}{N_1} \sqrt{\norm{E^\top \varphi(E y; b_t)}_2^2 \cdot t \log(n)}.
\end{align}
Now it remains to control the first term.
Define good event 
$$\cE = \bigl\{ \norm{y}_\infty \le (1+\sqrt c)\sqrt{2t\log(nt)} \bigr\}.$$
In fact, the above good event can be directly implied by the following good event:
\begin{align}
    \cE = \Bigl\{ \max_{\tau = -1, 0, \ldots, t-1} \norm{z_\tau}_\infty \le (1+\sqrt c)\sqrt{2\log(nt)} \Bigr\}.
\end{align}
For notational simplicity, we will just focus on the latter definition of the good event.
It follows from \Cref{lem:max gaussian_tail} that $\PP(\cE) \ge 1 - (tn)^{-c}\ge 1 - n^{-c}$. 
Let us define 
\begin{align}
Z = \frac{1}{N_1}\langle y, E^\top \varphi(Ey;b_t)\rangle \cdot \ind(\cE), \quad\text{and}\quad  V \defeq \EE \biggl[ \sum_{i=1}^{n-1} ( Z -  Z^{(i)})^2 \Biggiven y \biggr],
\end{align}
where $ Z^{(i)}  = \langle y^{(i)}, E^\top \varphi(Ey^{(i)};b_t)\rangle \cdot \ind(\cE^{(i)})$ and $y^{(i)}$ is given by replacing the $i$-th coordinate $y_i$ with an independent copy $y_i' \sim \cN(0, 1)$. 
Note that this is equivalent to replacing the $i$-th coordinate of each $z_\tau$ with an independent copy $z_\tau^{(i)}$.
Thus, the good event $\cE$ can be also changed to $\cE^{(i)}$ accordingly.
Next, we show how to control the variance $V$. 
Let us define 
$$Z_{l} = e_l^\top y \cdot \varphi(e_l^\top y; b_t) \cdot \ind(\cE) \quad\text{and}\quad Z_{l}^{(i)} = e_l^\top y^{(i)} \cdot \varphi(e_l^\top y^{(i)}; b_t) \cdot \ind(\cE^{(i)})$$
for any $l\in [N_1]$. 
On the joint event $\cE\cup \cE^{(1)} \cup \ldots  \cup \cE^{(n-1)}$, we have by the Lipschitzness of $\varphi$ in \Cref{assump:activation} that 
\begin{align}
    |Z_l| \le C (\gamma_2 + |b_t| \gamma_1) t \log(nt) \eqdef B_t, \quad \forall l\in [N_1]. \label{eq:E-1st-Z-ub}
\end{align}
This bounds also holds for all $Z_l^{(i)}$ for $i\in[n-1]$.
By a reformulation, we obtain for the joint event $\cE\cup \cE_1 \cup \ldots \cup \cE_{n-1}$ that
\begin{align}
   (Z - Z^{(i)})^2 &= \frac{1}{N_1^2}\cdot\biggl( \sum_{l=1}^{N_1} (Z_l  - Z_l^{(i)}) \biggr)^2  = \frac{1}{N_1^2}\cdot \biggl( \sum_{l=1}^{N_1} (Z_l - Z_l^{(i)}) \cdot \ind(E_{l, i}\neq 0) \biggr)^2  \\ 
   &\le \frac{\rho_1}{N_1} \cdot \sum_{l=1}^{N_1} (Z_l - Z_l^{(i)})^2 \cdot \ind(E_{l, i}\neq 0) \le \frac{2\rho_1}{N_1} \cdot \sum_{l=1}^{N_1} \bigl(Z_l^2 + (Z_l^{(i)})^2 \bigr)\cdot \ind(E_{l, i}\neq 0) \\
   &\le \frac{2\rho_1}{N_1} B_t^2 \cdot \sum_{l=1}^{N_1} \bigl(\ind(e_l^\top y + \barb_t >0) + \ind(e_l^\top y^{(i)} + \barb_t > 0) + 2 B_t^{-1} (n\lor d)^{-c_0}\bigr)\cdot \ind(E_{l, i}\neq 0), 
\end{align}
where the first inequality holds by the Cauchy-Schwarz inequality, the second one holds by $(a-b)^2 \le 2(a^2 + b^2)$, and the last line holds by \Cref{assump:activation} and the upper bound in \eqref{eq:E-1st-Z-ub}.
Since \highlight{$c_0$ is some sufficiently large constant}, we can safely ignore the term involving $(n\lor d)^{-c_0}$ in the sequel (when invoking a constant factor $C$).
Taking a summation over $i=1, \ldots, n-1$ on both sides and taking the conditional expectation, we obtain that 
\begin{align}
    V \le \frac{C\rho_1}{N_1} B_t^2 \cdot \sum_{i=1}^{n-1}\sum_{l=1}^{N_1} \bigl(\ind(e_l^\top y + \barb_t >0) + \EE\bigl[\ind(e_l^\top y^{(i)} + \barb_t > 0) \given y\bigr] \bigr)\cdot \ind(E_{l, i}\neq 0). 
\end{align}
Let us define 
\begin{align} 
    g(y) &= \frac{2\rho_1 B_t^2}{N_1} \cdot  \sum_{i=1}^{n-1}\sum_{l=1}^{N_1} \ind(e_l^\top y + \barb_t >0) \cdot \ind(E_{l, i}\neq 0) .
\end{align}
Therefore, the moment generating function of $V$ is controlled by 
\begin{align}    
    \EE[\exp(\lambda V)] &\le \EE\left[\exp(\lambda g(y)) \cdot \exp \bigl(\lambda \EE[g(y^{(i)})\given y]\bigr) \right] \le \EE\left[\exp(\lambda g(y)) \cdot \exp \bigl(\lambda g(y^{(i)})\bigr)  \right]
\end{align} 
for $\lambda > 0$. Here, the last inequality follows from the Jensen's inequality. 
To this end, we notice that $g$ is a non-decreasing functions of $y$. Then by \Cref{lem:order-inequality},  we have that $\EE[\exp(\lambda g(y)) \cdot \exp (\lambda g(y^{(i)})) ] \le \EE[\exp(2\lambda g(y)) ]$. Therefore, we just need to focus on the moment generating function of $g(y)$.
Note that since $e_l$ is $s$-sparse, with probability at least $1- \delta$ over the randomness of $y$, we have
\begin{align}
    g(y) \le \frac{2s\rho_1 B_t^2}{N_1} \cdot \sum_{l=1}^{N_1} \ind(e_l^\top y + \barb_t >0) \le  C s\rho_1 B_t^2  \cdot \bigl( \Phi(|\barb_t|) + \rho_1 s \log \delta^{-1} \bigr).
\end{align}
where in the last inequality, we invoke \Cref{lem:sparse-activation}. This can be transformed into the following tail bound 
\begin{align}
    \EE\bigl[\exp(\lambda V)\bigr] \le \EE\bigl[\exp(2\lambda g(y))\bigr], \where \PP\bigl(g(y) > C s\rho_1 B_t^2 \Phi(|\barb_t|) + v \bigr) \le \exp\Bigl(-\frac{v}{C \rho_1^2 s^2 B_t^2}\Bigr), 
\end{align}
and any $v>0$. 
In particular, for $V_+$ and $V_-$ defined in \eqref{eq:variance-split}, we always have $0\le V_+\le V$ and $0\le V_- \le V$.
With the sub-exponential tail bound, we now invoke 
\labelcref{lem:efron-stein-subexp-variance-1} to conclude that with probability at least $1-\delta$ over the randomness of $y$, 
\begin{align}
    |Z - \EE[Z]| &\le C B_t \bigl(\sqrt{s\rho_1 \Phi(|\barb_t|) \log \delta^{-1}} + \rho_1 s \log \delta^{-1}\bigr). 
    \label{eq:E-1st-Z-concentration}
\end{align}
Since $Z$ is Lipschitz over $\{\alpha_{\tau, t-1}\}_{\tau=-1}^{t-1}$ and $\{z_\tau\}_{\tau=-1}^{t-1}$, we follow a similar covering argument over the balls $\{\SSS^{t-1}\}_{t=1}^{T}$ with $T\le n^c$. 
Note that the failure probability of the joint event $\cE\cup \cE^{(1)}\cup \ldots \cup \cE^{(n-1)}$ is at most $n^{1-c}$. 
In addition, we can set $\delta = n^{-c} (n^{-c} \varepsilon^{n^c})$ in \eqref{eq:E-1st-Z-concentration}, where $\epsilon$ is the approximation error in the covering argument in the infinity norm.
By a union bound of the covering net of size $n^c \varepsilon^{-n^c}$, we will obtain a failure probability at most $n^{-c}$ as well. 
By decreasing the constant $c$ slightly (up to $2$), we can combine the two failure probabilities to obtain that for all $t\le n^c$, it holds with probability at least $1- n^{-c}$ that 
\begin{align}
    |Z - \EE[Z]| &\le C B_t \bigl(\sqrt{s\rho_1 \Phi(|\barb_t|) \cdot t \log(n)} + s\rho_1 \cdot t \log(n)\bigr). 
\end{align}
Next, let us evaluate the expectation $\EE[Z]$. 
By definition, 
\begin{align}
    \Bigl|\EE[Z] - \frac{1}{N_1}\EE\bigl[\langle y, E^\top \varphi(Ey;b_t)\rangle\bigr]\Bigr| & = \frac{1}{N_1}\EE\bigl[\langle y, E^\top \varphi(Ey;b_t)\rangle \cdot \ind(\overline\cE)\bigr] \\
     &\le \frac{1}{N_1} \sqrt{\EE\bigl[\langle y, E^\top \varphi(Ey;b_t)\rangle^2\bigr] \cdot \PP(\overline\cE) }. 
\end{align}
Since $\PP(\overline \cE) \le n^{-c}$, while $\EE\bigl[\langle y, E^\top \varphi(Ey;b_t)\rangle^2\bigr]$ is at most $C (\barb_t^2 + (\gamma_1 + |\barb_t| \gamma_2)^2)$ for some universal constant $C$ by the Lipschitzness of $\varphi$ given by \Cref{assump:activation}. 
We can pick $c$ in the definition of $\cE$ to be sufficiently large, Thereby, the approximation error in the expectation is negligible. We thus just need to evaluate 
\begin{align}
    \frac{1}{N_1}\EE\bigl[\langle y, E^\top \varphi(Ey;b_t)\rangle\bigr] 
    &= \frac{1}{N_1} \sum_{l=1}^{N_1} \EE\bigl[e_l^\top y \cdot \varphi(e_l^\top y; b_t)\bigr] = \EE_{x\sim\cN(0, 1)}[x\varphi(x;b_t)] \eqdef \hat\varphi_1(b_t).
\end{align}
Hence, we conclude that for all $\tau\le t-1$ and $t\le n^c$, it holds with probability at least $1-n^{-c}$ that
\begin{align}
    \Bigl|\frac{1}{N_1} \langle z_\tau, E^\top \varphi(Ey_t; b_t)\rangle - \alpha_{\tau, t-1} \cdot \hat\varphi_1(b_t)\Bigr| 
    &\le \alpha_{\tau, t-1} \cdot C B_t \bigl(\sqrt{s\rho_1 \Phi(|\barb_t|) \cdot t \log(n)} + s\rho_1 \cdot t \log(n)\bigr) \\
    &\qquad + \sqrt{1-\alpha_{\tau, t-1}^2} \cdot \frac{C}{N_1} \sqrt{ 2 \norm{E^\top \varphi(E y; b_t)}_2^2 \cdot t \log(n)}.
\end{align}
Plugging in the definition of $B_t=C(\gamma_2 + |b_t| \gamma_1) t \log(nt)$, we complete the proof of \Cref{lem:E-1st}.

\subsubsection{Concentration for $\langle z_\tau, F^\top \varphi(Fy_t + \theta \cdot v^\top \barw_{t-1}; b_t)\rangle$: Proof of \Cref{lem:F-1st}}\label{app:proof-F-1st}
        In this proof, we will show the concentration for the term $N_2^{-1} \langle z_\tau, F^\top \varphi(Fy_t + \theta \cdot v^\top \barw_{t-1}; b_t)\rangle$.
        Similar to the proof of \Cref{lem:E-1st}, when fixing $\{\alpha_{\tau, t-1}\}_{\tau=-1}^{t-1}$ and $\{b_{t,l}\}_{l=1}$, we have $y_t^\star \sim \cN(0, I_{n-1})$. For simplicity, we will denote $y_t^\star$ by $y$ in the following.
        Note that $z_\tau \overset{d}{=} \alpha_{\tau, t-1} y + \sqrt{1-\alpha_{\tau, t-1}^2} \cdot z$ where $z \sim \cN(0, I_{n-1})$ is independent of $y$. In the sequel, we also simplify $\alpha_{\tau, t-1}$ to $\alpha$.
        Therefore, the concentration we consider can be reduced to 
        \begin{align}
            \alpha \cdot \frac{1}{N_2} \langle y, F^\top \varphi(F y + \theta \cdot v^\top \barw_{t-1}; b_t)\rangle + \sqrt{1-\alpha^2} \cdot \frac{1}{N_2} \langle z, F^\top \varphi(F y + \theta \cdot v^\top \barw_{t-1}; b_t)\rangle.
        \end{align}
        The concentration for the second part follows directly from the Gaussian tail bound. That said, with probability at least $1 - \delta$, it holds that 
        \begin{align}
            \frac{1}{N_2} \bigl|\langle z, F^\top \varphi(F y + \theta \cdot v^\top \barw_{t-1}; b_t)\rangle\bigr| \le \frac{1}{N_2} \cdot \sqrt{2 \norm{F^\top\varphi(F y + \theta \cdot v^\top \barw_{t-1}; b_t)}_2^2 \cdot \log \delta^{-1} }, 
        \end{align}
        where the right-hand side can be controlled by \Cref{lem:F-2nd}. Then by a covering argument over $\{\alpha_{\tau, t-1}\}_{\tau=-1}^{t-1}$ and $b_t$ similar to \Cref{lem:sparse-activation} (with proper truncation of the random variables that yields a sufficiently small error probability), we conclude that with probability at least $1 - n^{-c}$, it holds for all $t=1, \ldots, T$ and $\tau=-1, 0, \ldots, t-1$ that
        \begin{align}
            \frac{1}{N_2} \bigl|\langle z, F^\top \varphi(F y + \theta \cdot v^\top \barw_{t-1}; b_t)\rangle\bigr| \le \frac{C}{N_2} \cdot \sqrt{\norm{F^\top\varphi(F y + \theta \cdot v^\top \barw_{t-1}; b_t)}_2^2 \cdot t\log(n)}. 
        \end{align}
        To control the first term, define good event 
        \begin{align}
            \cE=\bigl\{ \max_{\tau=-1, 0, \ldots, t-1} \norm{z_\tau}_\infty \le (1+\sqrt c)\sqrt{2\log (nt)} \bigr\}.
        \end{align}
        On this good event, $\norm{y}_\infty \le (1+\sqrt c) \sqrt{2t\log(nt)}$ and this good event holds with probability at least $1 - (tn)^{-c}\ge 1 - n^{-c}$. 
        We define 
        \begin{align}
            Z = \frac{1}{N_2} \langle y, F^\top \varphi(F y + \theta \cdot v^\top \barw_{t-1}; b_t)\rangle \ind(\cE), \quad \text{and}\quad V = \EE\biggl[\sum_{i=1}^{n-1} (Z - Z^{(i)})^2 \Biggiven y\biggr], 
        \end{align}
        where $Z^{(i)} = N_2^{-1} \langle y^{(i)}, F^\top \varphi(F y^{(i)} + \theta \cdot v^\top \barw_{t-1}; b_t)\rangle \ind(\cE^{(i)})$. 
        Here, we define $y^{(i)} = \sum_{j=-1}^{t-1}\alpha_{j, t-1} z_j^{(i)}$ with $z_\tau^{(i)}$ given by replacing the $i$-th coordinate of $z_j$ with an independent copy, and $\cE^{(i)}$ is the event defined with respect to $z_\tau^{(i)}$.
        Let us define 
        \begin{align}
            Z_l = f_l^\top y \cdot \varphi(f_l^\top y + \theta_l v^\top \barw_{t-1}; b_t) \ind(\cE),\quad Z_l^{(i)} = f_l^\top y^{(i)} \cdot \varphi(f_l^\top y^{(i)} + \theta_l v^\top \barw_{t-1}; b_t) \ind(\cE^{(i)}), 
        \end{align}
        where $f_l$ is the $l$-th row of $F$. 
        On the joint event $\cE \cup \cE^{(1)} \cup \cdots \cup \cE^{(n-1)}$, we have by the Lipschitz continuity of $\varphi$ in \Cref{assump:activation} that 
        \begin{align}
            |Z_l| \le C \bigl( (\gamma_2 +|b_t| \gamma_1) \cdot (\sqrt{t\log (n)} + \norm{v}_2 \alpha_{-1, t-1})  + (n\lor d)^{-c_0}\bigr) \cdot \sqrt{t\log (n)} \defeq B_t, 
        \end{align} 
        where we also use the fact that \highlight{$b_t + \kappa_0 \le 0$} for the bias. Note that the $(n\lor d)^{-c_0}$ term is negligible when \highlight{$c_0$ is sufficiently large}.
        For notation simplicity, we define $\tilde b_{t, l} = b_t + \kappa_0 + \theta_l v^\top \barw_{t-1}$.
        This bound also holds for $Z_l^{(i)}$.
        On the joint event $\cE \cup \cE^{(1)} \cup \cdots \cup \cE^{(n-1)}$, we have
        \begin{align}
            (Z - Z^{(i)})^2 
            &\le \frac{1}{N_2^2}\sum_{l=1}^{N_2} (Z_l - Z_l^{(i)})^2 \le \frac{1}{N_2^2} \cdot \biggl(\sum_{l=1}^{N_2} (Z_l - Z_l^{(i)})^2 \ind(F_{l, i}\neq 0) \biggr)^2 \\
            &\le \frac{\rho_2}{N_2} \sum_{l=1}^{N_2} (Z_l - Z_l^{(i)})^2 \ind(F_{l, i}\neq 0) \le \frac{2\rho_2}{N_2} \sum_{l=1}^{N_2} \bigl(Z_l^2 + (Z_l^{(i)})^2\bigr) \ind(F_{l, i}\neq 0) \\
            &\le \frac{2\rho_2 B_t^2}{N_2} \sum_{l=1}^{N_2} \bigl( \ind(f_l^\top y + \tilde b_{t,l} > 0) + \ind(f_l^\top y^{(i)} + \tilde b_{t,l} > 0) + 2 B_t^{-1} (n\lor d)^{-c_0}\bigr) \ind(F_{l, i}\neq 0), 
        \end{align}
        where the first inequality holds by the Cauchy-Schwarz inequality, the second one holds by $(a-b)^2 \le 2(a^2 + b^2)$, and the last line holds by \Cref{assump:activation} and the upper bound for $Z_l$ and $Z_l^{(i)}$. 
        We can also ignore the $2 B_t^{-1} (n\lor d)^{-c_0}$ term by multiplying some universal constant. 
        Taking a summation over $i=1, \ldots, n-1$ on both sides with the conditional expectation, we obtain
        \begin{align}
            V\le  \frac{C\rho_2 B_t^2}{N_2} \sum_{i=1}^{n-1}\sum_{l=1}^{N_2} \bigl( \ind(f_l^\top y + \tilde b_{t,l} > 0) + \EE\bigl[\ind(f_l^\top y^{(i)} + \tilde b_{t,l} > 0)\bigr] \bigr) \cdot \ind(F_{l, i}\neq 0).
        \end{align}
        Let us take 
        \begin{align}
            g(y) &\defeq \frac{C\rho_2 B_t^2}{N_2} \sum_{i=1}^{n-1}\sum_{l=1}^{N_2} \ind(f_l^\top y + \tilde b_{t,l} > 0) \ind (F_{l, i}\neq 0) \\
            &=\frac{C\rho_2 s B_t^2 }{N_2}  \sum_{l=1}^{N_2} \ind(f_l^\top y + \tilde b_{t,l} > 0) \le C\rho_2 s B_t^2.
        \end{align}
        Then we have by the monotonicity of $g$ and \Cref{lem:order-inequality} that $\EE[\exp(\lambda V)] \le \EE[\exp(2\lambda g(y))]$ for all $\lambda>0$.
        Invoking \Cref{lem:efron-stein-bounded-variance} for this bounded variance, we obtain that with probability at least $1 -\delta$ over the randomness of $y$, it holds that
        \begin{align}
            |Z - \EE[Z]| \le C B_t \sqrt{\rho_2 s} \cdot \log\delta^{-1}. 
        \end{align}
        By a covering argument over $\{\alpha_{\tau, t-1}\}_{\tau=-1}^{t-1}$ and $b_t$ similar to \Cref{lem:sparse-activation}, we conclude that $|Z - \EE[Z]| \le C B_t \sqrt{\rho_2 s} \cdot t\log(n)$ with probability at least $1 - n^{-c}$ for all $t=1, \ldots, T$ and $\tau=-1, 0, \ldots, t-1$.
        In addition, the approximation error 
        \begin{align}
            \Bigl| \EE[Z] - \frac{1}{N_2} \EE[\langle y, F^\top \varphi(F y + \theta \cdot v^\top \barw_{t-1}; b_t)\rangle] \Bigr| \propto \sqrt{\PP(\overline \cE)}
        \end{align}
        by the Cauchy-Schwarz inequality and the fact that $f_l^\top y \varphi(f_l^\top y + \theta_l v^\top \barw_{t-1}; b_t)$ has bounded second moment. 
        Therefore, by taking a sufficiently large $c$ in the definition of the good event $\cE$, we can make this approximation error negligible. Moreover, we also have 
        \begin{align} 
            \EE[f_l^\top y \varphi(f_l^\top y + \theta_l v^\top \barw_{t-1}; b_t)] = \sqrt{1-\theta_l^2} \cdot \EE_{x\sim \cN(0, 1)} \Bigl[x \varphi\bigl(\sqrt{1-\theta_l^2} x + \theta_l v^\top \barw_{t-1}; b_t\bigr)\Bigr].
        \end{align}
        Combining everything, we conclude that with probability at least $1 - n^{-c}$, it holds for all $t=1, \ldots, T$ and $\tau=-1, 0, \ldots, t-1$ that
        \begin{align}
            &\frac{1}{N_2} \Bigl| \langle z_\tau, F^\top \varphi(F y_t^\star + \theta \cdot v^\top \barw_{t-1}; b_t)\rangle - \sum_{l=1}^{N_2}\alpha_{\tau, t-1}\sqrt{1-\theta_l^2} \cdot \EE_{x\sim \cN(0, 1)} \Bigl[x \varphi\bigl(\sqrt{1-\theta_l^2} x + \theta_l v^\top \barw_{t-1}; b_t\bigr)\Bigr]\Bigr| \\
            &\quad \le \frac{C}{N_2} \sqrt{1-\alpha_{\tau, t-1}^2} \cdot \sqrt{\norm{F^\top\varphi(F y + \theta \cdot v^\top \barw_{t-1}; b_t)}_2^2 \cdot t\log(n)} \\
            &\hspace{2cm} + C \alpha_{\tau, t-1}\cdot (\gamma_2 +|b_t| \gamma_1) \cdot (\sqrt{t\log (n)} + \norm{v}_2 \alpha_{-1, t-1}) \cdot \sqrt{\rho_2 s} \cdot (t\log(n))^{3/2}
        \end{align} 
        This completes the proof of \Cref{lem:F-1st}.

\subsubsection{Concentration for $\theta^\top\varphi(Fy_t^\star + \theta v^\top \barw_{t-1}; b_t)$: Proof of \Cref{lem:F-signal}}
\label{app:proof-F-signal}
In the following, we will use $C$ to denote universal constants that change from line to line.
Let $f_l$ denote the $l$-th row of $F$.
Let us first fix $\{\alpha_{\tau, t-1}\}_{\tau=-1}^{t-1}$ and $b_t$. Then $ y_t^\star \sim \cN(0, I_{n-1})$. In the sequel, we will simplify $y \leftarrow y_t^\star$. 
Let us define the good event 
\begin{align}
    \cE=\bigl\{ \max_{\tau=-1, 0, \ldots, t-1} \norm{z_\tau}_\infty \le (1+\sqrt c)\sqrt{2\log (nt)} \bigr\}.
\end{align}
It then follows from \Cref{lem:max gaussian_tail} that $\PP(\overline\cE) \le (nt)^{-c} \le n^{-c}$, and also $\norm{y}_\infty \le (1+\sqrt c)\sqrt{2t\log (nt)}$ on $\cE$. In particular, 
\begin{align}
    |\varphi(f_l^\top y + \theta_l v^\top \barw_{t-1}; b_t)| \ind(\cE) \le (\gamma_2+ |b_t| \gamma_1) ((1+\sqrt c)\sqrt{2t\log (nt)} + \norm{v}_2 \alpha_{-1, t-1}) + (n\lor d)^{-c_0}\defeq B_t. 
\end{align}
where the last inequality holds by noting that $\varphi(\cdot ;b_t)$ is $\gamma_2 + |b_t| \gamma_1$-Lipschitz by \Cref{assump:activation}, and also the fact that $\barb_t = b_t + \kappa_0\le 0$.
The target function to study is
\begin{align}
    Z = \frac{1}{N_2}\sum_{l=1}^{N_2} \theta_l \varphi(  f_l^\top y ;  \tilde b_{t, l}) \ind(\cE), \where \tilde b_{t,l} = b_{t, l} + \theta_l \norm{v}_2 \alpha_{-1, t-1}.
\end{align}
Let $y^{(i)}$ be the vector obtained by replacing the $i$-th element of $y_t$ with an independent standard Gaussian random variable $y_t'(i)$. The good event $\cE^{(i)}$ is defined similarly.
Define $Z^{(i)}$ as the correspondence of $Z$ with $y^{(i)}$ and $\cE^{(i)}$. 
Let us define variance $V = \EE[\sum_{i=1}^{n-1} (Z - Z^{(i)})^2]$. Notice that this $V$ upper bounds both $V_+ = \EE[\sum_{i=1}^{n-1} (Z - Z^{(i)})^2 \ind(Z > Z^{(i)})]$ and $V_- = \EE[\sum_{i=1}^{n-1} (Z - Z^{(i)})^2 \ind(Z < Z^{(i)})]$.
Note that when changing one coordinate in $y$, the total number of terms affected in $Z$ is at most $N_2 \rho_2$ by definition \eqref{eq:rho-def}. 
It then holds by the Cauchy-Schwarz inequality that
\begin{align}
    V 
        &\le \frac{C\rho_2}{N_2}\sum_{i=1}^{n-1} \sum_{l=1}^{N_2} \theta_l^2 \cdot \EE\bigl[\bigl(\varphi(f_l^\top y; \tilde b_{t, l}) \ind(\cE) - \varphi(f_l^\top y^{(i)}; \tilde b_{t, l}) \ind(\overline\cE)\bigr)^2 \given y\bigr] \\ 
        &\le \frac{C B_t^2 \rho_2}{N_2}\sum_{i=1}^{n-1} \sum_{l=1}^{N_2} \theta_l^2  \ind(f_l(i)\neq 0),
\end{align}
where in the second inequality, the indicator is included since the term will be zero if $f_l(i) = 0$.
Additionally, we invoke the bound $B_t$ to upper bound the $\varphi(\cdot)$ term.
Let us define 
\begin{align}
    g(y) \defeq \frac{CB_t^2\rho_2}{N_2} \sum_{i=1}^{n-1} \sum_{l=1}^{N_2} \theta_l^2 \cdot \ind(f_l(i)\neq 0) \le \frac{CB_t^2\rho_2 s}{N_2}  \norm{\theta}_2^2.
\end{align} 
By \Cref{lem:order-inequality}, we know that the MGF of $V$ can be upper bounded by 
$
      \EE[\exp(\lambda V)] 
      \le \EE\bigl[\exp( 2\lambda g(y) )\bigr]. 
$
Thanks to the bounded variance, 
invoking \Cref{lem:efron-stein-bounded-variance}, we conclude that with probability at least $1 - \delta$ over the randomness of $y$, it holds that 
\begin{align}
    \bigl| Z - \EE[Z] \bigr| \le C B_t \norm{\theta}_2 \sqrt{\frac{\rho_2 s}{N_2}} \log(\delta^{-1}).
\end{align}
Next, we invoke a union covering argument over the ball $\SSS^{t+1}$ for $\alpha_{\tau, t-1}$ and also for $b_t$. Since $Z$ is Lipschitz and bounded, the approximation error can be made sufficiently small. Therefore, we conclude that with probability at least $1 - n^{-c}$, it holds for all $t\le n^c$ that
\begin{align}
    \bigl| Z - \EE[Z] \bigr| \le C B_t \norm{\theta}_2 \sqrt{\frac{\rho_2 s}{N_2}} \cdot t\log(n).
\end{align}
Similar to previous proof, the error in $\EE[Z]$ and $N_2^{-1}\EE[\theta^\top \varphi(Fy_t + \theta \cdot v^\top \barw_{t-1}; b_t)]$ can be made sufficiently small if we choose a large $c$ in the definition of the good event $\cE$.
Consequently we just need to plug in the expectatin 
\begin{align}
    \frac{1}{N_2} \EE[\theta^\top \varphi(Fy_t + \theta \cdot v^\top \barw_{t-1}; b_t)] = \frac{1}{N_2} \sum_{l=1}^{N_2} \EE_{x\sim\cN(0,1)}\bigl[\theta_l \cdot \varphi(\sqrt{1-\theta_l^2} \cdot x + \theta_l \cdot v^\top \barw_{t-1}; b_t) \bigr].
\end{align}
This completes the proof of \Cref{lem:F-signal}.

\subsection{Propogation of the Non-Gaussian Error}
In this subsection, we analyze how to  Non-Gaussian error $\Delta y_t$ propagates through the nonlinear activation. 

\subsubsection{Error Analysis for $\Delta E_t$: Proof of \Cref{lem:E-error}}\label{app:proof-E-error}
In the following proof, we will use $C$ to denote universal constants that change from line to line.

\paragraph{Bounding $\norm{\Delta E_t}_1$}
By definition of $\Delta E_t$, we have
\begin{align}
    \norm{\Delta E_t}_1 & = \norm{E^\top \varphi(E(y_t^\star + \Delta y_t); b_t) - E^\top \varphi(Ey_t^\star; b_t)}_1 \le \sqrt s \cdot \norm{\varphi(E(y_t^\star + \Delta y_t); b_t) - \varphi(Ey_t^\star; b_t)}_1 \\ 
    &\le  \sqrt s (\gamma_2 + |b_t| \gamma_1)\cdot \sum_{l=1}^{N_1} |e_l^\top \Delta y_t| \cdot \ind(e_l^\top y_t + \barb_t > 0 \:\lor\: e_l^\top y_t^\star + \barb_t >0) \\
    &\qquad + \sqrt s \cdot \sum_{l=1}^{N_1}  2(2+|b_t|) \cdot (n\lor d)^{-c_0} \cdot \ind(e_l^\top y_t + \barb_t \le 0 \:\land\: e_l^\top y_t^\star + \barb_t \le 0).
\end{align}
where $\barb_t = b_t + \kappa_0$ is the shifted bias.
The first inequality follows from the fact that $\norm{e_l}_1 \le \sqrt s$ as each row $e_l$ is $s$-sparse.
The second inequality holds by splitting the summation into two parts. 
For the first part $\{l: e_l^\top y_t + \barb_t > 0 \:\lor\: e_l^\top y_t^\star + \barb_t >0\}$ where the neuron is activated, we have the term bounded by the Lipschitz continuity of $\varphi$ times the pre-activation difference $|e_l^\top \Delta y_t|$.
Here, we recall from \Cref{assump:activation} that $\varphi$ is $(\gamma_2 + |b_t| \gamma_1)$-Lipschitz continuous.
For the second part $\{l: e_l^\top y_t + \barb_t \le 0 \:\land\: e_l^\top y_t^\star + \barb_t \le 0\}$ where the neuron is inactive, we simply apply the upper bound on $\varphi$ in \Cref{assump:activation} as $(2+|b_t|) \cdot (n\lor d)^{-c_0}$.
Note that $c_0$ can be chosen to be a sufficiently large constant. 
Thus, we just need to focus on the first part.
Using the Cauchy-Schwarz inequality twice, we have
\begin{align}
    \sum_{l=1}^{N_1} |e_l^\top \Delta y_t| \cdot \ind(e_l^\top y_t +\barb_t >0)
    &\le \sum_{l=1}^{N_1} \norm{e_l}_2 \cdot \norm{\Delta y_t \circ \ind(e_l\neq 0)}_2 \cdot \ind(e_l^\top y_t +\barb_t >0) \\ 
    &\le \sqrt{\sum_{l=1}^{N_1} \ind(e_l^\top y_t +\barb_t >0) \cdot \sum_{l=1}^{N_1} \norm{\Delta y_t \circ \ind(e_l\neq 0)}_2^2 }, 
    \label{eq:E-error-1}
\end{align}
where $x\circ y$ is the Hadamard product between two vectors $x$ and $y$. Note that the second term on the right hand side can be further bounded by
\begin{align}
    \sum_{l=1}^{N_1} \norm{\Delta y_t \circ \ind(e_l\neq 0)}_2^2  & = \sum_{l=1}^{N_1} \sum_{i=1}^{n-1} \Delta y_{t, i}^2 \cdot \ind(E_{l, i} \neq 0)  \le \rho_1 N_1 \cdot \norm{\Delta y_t}_2^2.
    \label{eq:E-error-2}
\end{align}
Plugging \eqref{eq:E-error-2} back into \eqref{eq:E-error-1}, and invoking \Cref{lem:E-activation-perturbed}, we conclude that with probability at least $1-n^{-c}$ for all $t\le n^c$,
\begin{align}
    \sum_{l=1}^{N_1} |e_l^\top \Delta y_t| \cdot \ind(e_l^\top y_t +\barb_t >0) 
    &\le C N_1 \cdot \sqrt{\bigl( \Phi(-\barb_t) + \rho_1 s t \log(n) + \rho_1 |\barb_t|^2 \norm{\Delta y_t}_2^2 \bigr) \rho_1 \norm{\Delta y_t}_2^2} \\ 
    &\le C N_1 \cdot \Bigl(\bigl( \sqrt{\rho_1 \Phi(-\barb_t)} + \rho_1 \sqrt{st \log n} \bigr) \cdot \norm{\Delta y_t}_2 + \rho_1 |\barb_t| \cdot \norm{\Delta y_t}_2^2 \Bigr).
\end{align}
Note that the ideal activation $\sum_{l=1}^{N_1} \ind(e_l^\top y_t^\star + \barb_t > 0)$ has an upper bound in \Cref{cor:E-activation-ideal} even tighter than the one we use above. Therefore, we just need to double the above error term. 
Thereby, we conclude that 
\begin{align}
    \norm{\Delta E_t}_1 &\le C N_1 (\gamma_2 + |b_t| \gamma_1)\cdot \Bigl(\bigl( \sqrt{s\rho_1 \Phi(-\barb_t)} + s\rho_1 \sqrt{t \log n} \bigr) \cdot \norm{\Delta y_t}_2 + \sqrt s \rho_1 |\barb_t| \cdot \norm{\Delta y_t}_2^2 \Bigr)  \\
    &\qquad + C N_1 \sqrt s (2+|b_t|) \cdot (n\lor d)^{-c_0}. 
\end{align}

\paragraph{Bounding $\norm{\Delta E_t}_2^2$}
The proof is similar to bounding $\norm{\Delta E_t}_1$. Again, we notice that for any test vector $x\in\RR^{N_1}$, 
\begin{align}
    \norm{E^\top x}_2^2 = \sum_{i=1}^{n-1} \Bigl( \sum_{l=1}^{N_1} E_{l, i} x_l \Bigr)^2 \le \sum_{i=1}^{n-1} \Bigl(\sum_{l=1}^{N_1} \ind(E_{l, i}\neq 0)\Bigr)  \cdot \Bigl(\sum_{l=1}^{N_1} E_{l, i}^2  x_l^2 \Bigr) \le \rho_1 N_1  \norm{x}_2^2.
\end{align}
Here, the first inequality holds by the Cauchy-Schwarz inequality while the second inequality holds by the sparsity assumption on the columns of $E$ and also the fact that $\sum_{i=1}^{n-1}  E_{l, i}^2  = \norm{e_l}_2^2 =1$. Thereby, it holds for $\norm{\Delta E_t}_2^2$ that 
\begin{align}
    \norm{\Delta E_t}_2^2 
    &\le \rho_1 N_1 \norm{\varphi(E(y_t^\star + \Delta y_t); b_t) - \varphi(Ey_t^\star; b_t)}_2^2 \le  \rho_1 N_1 (\gamma_2 + |b_t| \gamma_1)^2\cdot \sum_{l=1}^{N_1} |e_l^\top \Delta y_t|^2 \\ 
    &\le \rho_1 N_1 (\gamma_2 + |b_t| \gamma_1)^2\cdot \sum_{l=1}^{N_1} \norm{e_l}_2^2  \cdot \|\Delta y_{t} \circ \ind(e_l\neq 0)\|_2^2 \le (\gamma_2 + |b_t| \gamma_1)^2\cdot (\rho_1 N_1)^2 \norm{\Delta y_t}_2^2, 
\end{align}
where the second inequality holds by the Lipschitz continuity of $\varphi$ and the third inequality follows from the Cauchy-Schwarz inequality. The last inequality holds by invoking \eqref{eq:E-error-2}.
Hence, we complete the proof of \Cref{lem:E-error}.

\subsubsection{Error Analysis for $\Delta F_t$: Proof of \Cref{lem:F-error}}\label{app:proof-F-error}
In the following proof, we will use $C$ to denote universal constants that change from line to line.
Let $f_l$ be the $l$-th row of matrix $F$.
Note that 
\begin{align}
    \norm{\Delta F_t}_1 
    &\le \sqrt s \cdot \norm{\varphi(F(y_t^\star + \Delta y_t) + \theta \cdot v^\top \barw_{t-1}; b_t) - \varphi(Fy_t^\star + \theta \cdot v^\top \barw_{t-1}; b_t)}_1 \\
    &\le \sqrt s (\gamma_2 + |b_t| \gamma_1) \cdot \sum_{l=1}^{N_2} |f_l^\top \Delta y_t| 
    \le \sqrt s (\gamma_2 + |b_t| \gamma_1) \cdot \norm{\Delta y_t}_2  \cdot \sum_{l=1}^{N_2} \norm{f_l}_2 \\
    &\le \sqrt s N_2 (\gamma_2 + |b_t| \gamma_1) \cdot \norm{\Delta y_t}_2,  
\end{align}
where the first inequality follows from the fact that $\norm{f_l}_1 \le \sqrt s$ by the H\"older's inequality for $s$-sparse $f_l$ with $\norm{f_l}_2 \le 1$, the second inequality follows from the Lipschitzness of $\varphi$ and the third inequality follows from the Cauchy-Schwarz inequality.
In the  last inequality, we use the fact that $\norm{f_l}_2 \le 1$. 
Next, we turn to the bound for $\norm{\Delta F_t}_2$.
For any test vector $x\in \RR^{N_2}$, we have 
\begin{align}
    \norm{F^\top x}_2^2 &= \sum_{i=1}^{n-1} \Bigl(\sum_{l=1}^{N_2} F_{li} x_l\Bigr)^2 \le \sum_{i=1}^{n-1} \norm{F_{:,i}}_{2}^2 \cdot \norm{x}_2^2 \le \rho_2 N_2 \norm{x}_2^2, 
    \label{eq:F-error-1}
\end{align}
where we recall that $\rho_2 = \max_{i\in [n-1]} \norm{F_{:, i}}_0/N_2$. Since $\Delta F_t = F^\top \Delta\varphi_{F, t}$, we have $\norm{\Delta F_t}_2^2 \le \rho_2 N_2 \norm{\Delta \varphi_{F, t}}_2^2$.
Next, we use the same Lipschitzness of $\varphi$ to upper bound $\norm{\Delta \varphi_{F, t}}_2^2$ as 
\begin{align}
    \norm{\Delta \varphi_{F, t}}_2^2 
    &\le (\gamma_2 + |b_t| \gamma_1)^2 \cdot \sum_{l=1}^{N_2} |f_l^\top \Delta y_t|^2 \le (\gamma_2 + |b_t| \gamma_1)^2 \cdot \sum_{l=1}^{N_2} \norm{f_l}_2^2 \cdot \sum_{i=1}^{n-1} \Delta y_{t, i}^2 \ind(F_{l, i}\neq 0) \\ 
    &\le (\gamma_2 + |b_t| \gamma_1)^2 \cdot  \sum_{i=1}^{n-1} \sum_{l=1}^{N_2} \Delta y_{t, i}^2 \ind(F_{l, i}\neq 0) \le N_2 \rho_2  (\gamma_2 + |b_t| \gamma_1)^2 \cdot \norm{\Delta y_t}_2^2, 
    \label{eq:F-error-Delta varphi}
\end{align}
where we use the Cauchy-Schwarz inequality in the second inequality, the fact that $\norm{f_l}_2 \le 1$ in the third inequality, and the definition of $\rho_2$ in the last inequality.
Combining \eqref{eq:F-error-1} and \eqref{eq:F-error-Delta varphi}, we conclude that $\|\Delta F_t\|_2 \le \rho_2 N_2 (\gamma_2 + |b_t| \gamma_1) \cdot \norm{\Delta y_t}_2$.
This completes the proof of \Cref{lem:F-error}.

\subsection{Proofs for Technical Lemmas}
\label{sec:concentrate-add-proofs}
\subsubsection{Proof of \Cref{prop:2nd-moment-expectation}}\label{sec:proof-2nd-moment-expectation}
    We invoke the upper bound $|\varphi(x; b)| \le (n\lor d)^{-c_0} + L (x + \barb) \cdot \ind(x> -\barb)$ to obtain that 
    \begin{align}
        &\EE[\varphi(x;b) \varphi(\iota x + \sqrt{1-\iota^2} \cdot z; b) ]  \\
        &\quad \le (n\lor d)^{-2c_0} + 2 (n\lor d)^{-c_0} \cdot L \cdot \EE[(x+\barb) \cdot \ind(x > -\barb)] \\
        &\hspace{2cm} + L^2 \cdot \underbrace{\EE[(x+\barb) \cdot \ind(x > -\barb) \cdot (\iota x + \sqrt{1-\iota^2} z + \barb) \cdot \ind(\iota x + \sqrt{1-\iota^2} z > -\barb)]}_{\ds \braRNum{1}}.
    \end{align}
    Note that $\EE[x \ind(x > -\barb)] = p(|\barb|)$ for any $\barb$ by explicit calculation, where $p(x) = \exp(-x^2/2)/\sqrt{2\pi}$ is the standard Gaussian density function. Therefore, we have
    $$
    \EE[(x+\barb) \cdot \ind(x > -\barb)] = \EE[x \ind(x > -\barb)] + \barb \cdot \PP(x > -\barb) = p(|\barb|) - |\barb| \Phi(|\barb|) = F(|\barb|), 
    $$
    where we define $F(x) = p(x) - x \Phi(x)$. We note that the function $F(x)$ is monotonically decreasing for all $x\in \RR$.
    To see this, we take the derivative of $F(x)$ and using the fact that $p'(x) = -x p(x)$ and $\Phi'(x) = -p(x)$, which gives us 
    \begin{align}
        F'(x) = -\Phi(x) - x \Phi'(x) - xp(x) = -\Phi(x) + x p(x) - x p(x) = -\Phi(x) < 0.
        \label{eq:derivative-f}
    \end{align}
    In particular, function $F(x)$ is always positive for any $x\in \RR$ as $\lim_{x\to \infty} F(x) = 0$ by the Mills ratio $\lim_{x\to \infty} x\Phi(x)/p(x)=1 $. 
    Therefore, $F(|\barb|) \le F(0) = 1/2$ and the first two terms involving $(n\lor d)^{-c_0}$ are negligible.
    For the last term, by marginalizing $z$, we have 
    \begin{align}
        &\EE[(x+\barb) \cdot \ind(x > -\barb) \cdot (\iota x + \sqrt{1-\iota^2} z + \barb) \cdot \ind(\iota x + \sqrt{1-\iota^2} z > -\barb)] \\
        &\quad = \EE\left[(x+\barb) \cdot \ind(x > -\barb) \cdot \sqrt{1-\iota^2}\cdot  \left( \frac{\iota x + \barb}{\sqrt{1-\iota^2}} \cdot  \Phi\Bigl( -\frac{\barb+\iota x}{\sqrt{1-\iota^2}}\Bigr) +  p\Bigl(-\frac{\barb+\iota x}{\sqrt{1-\iota^2}}\Bigr) \right) \right]\\
        &\quad = \EE\left[(x+\barb) \cdot \ind(x > -\barb) \cdot \sqrt{1-\iota^2}\cdot  F\Bigl(-\frac{\barb+\iota x}{\sqrt{1-\iota^2}}\Bigr) \right].
    \end{align}
    Since $F(x)$ is monotonically decreasing, we can upper bound the expectation by just plugging in $x = -\barb$ to obtain that 
    \begin{align}
        \braRNum{1} &\le \EE \left[ (x+\barb)\cdot \ind(x>-\barb) \right] \cdot \sqrt{1-\iota^2} \cdot F\Bigl( - \barb\sqrt{\frac{1-\iota}{1+\iota}}\Bigr) = \sqrt{1-\iota^2} \cdot F(|\barb|) \cdot F\Bigl( |\barb|\sqrt{\frac{1-\iota}{1+\iota}}\Bigr).
    \end{align}
    Next, we prove that $F(x) \le 2\Phi(x)$ for all $x >0$. For any $x>0$, we have $F'(x) = -\Phi(x) $ by \eqref{eq:derivative-f}, and $\Phi'(x) = - p(x)$. Therefore, 
    \begin{align}
        \frac{F'(x)}{\Phi'(x)} = \frac{\Phi(x)}{p(x)} \le \frac{\Phi(0)}{p(0)} = \sqrt{\frac{\pi}{2}} \le 2, 
    \end{align}
    where we use the fact that $\Phi(x)/p(x)$ is monotonically decreasing. Noting that $\lim_{x\to \infty} F(x) = 0$ and $\lim_{x\to \infty} \Phi(x) = 0$, we thus conclude that $F(x) \le 2\Phi(x)$ for all $x>0$.
    Consequently, 
    \begin{align}
        \braRNum{1} \le 2\sqrt{1-\iota^2} \cdot F(|\barb|) \cdot F\Bigl( |\barb|\sqrt{\frac{1-\iota}{1+\iota}}\Bigr) \le 4\sqrt{1-\iota^2} \cdot \Phi(|\barb|) \cdot \Phi\Bigl( |\barb|\sqrt{\frac{1-\iota}{1+\iota}}\Bigr).
    \end{align}
    Therefore, we conclude the proof of this proposition. 

\subsubsection{Proof of \Cref{prop:hatvarphi_1-bound}}\label{sec:proof-hatvarphi_1-bound}
\begin{proof}[Proof of \Cref{prop:hatvarphi_1-bound}]
    Note that 
    \begin{align}
        \hat\varphi_1(b_t) &= \EE_{x\sim\cN(0, 1)}[\varphi(x; b_t) x] \\
        &\le L \cdot \EE[\ind(x+\barb_t>0) (x+\barb_t) x]  + \EE[|x| \ind(x+\barb_t \le  0)] \cdot (d\lor n)^{-c_0}\\
        &\le L \cdot \Bigl( \frac{|\barb_t|}{\sqrt{2\pi}} \exp(-\barb_t^2/2) + \Phi(|\barb_t|) +\frac{\barb_t}{\sqrt{2\pi}} \exp(-\barb_t^2/2)\Bigr) + C (d\lor n)^{-c_0}. 
    \end{align}
    Here, the last inequality holds by the following integral calculation:
    \begin{align}
        \int_{\barb}^{\infty} x p(x) \rd x = p(\barb), \quad \int_{\barb}^\infty x^2 p(x) \rd x = \barb p(\barb) + \Phi(\barb)
    \end{align}
    for the standard normal distribution $p(x) =\exp(-x^2/2)/\sqrt{2\pi}$.
    For $\barb_t < 0$, the first and the last term cancel in the bracket, and we conclude that $\hat\varphi_1(b_t) \le 2C_0 L \Phi(|\barb_t|)$ as $(d\lor n)^{-c_0}$ can be sufficienlty small. 
    On the other hand, using the condition $\varphi(x; b_t) \ge x \phi'(x+b) \ge C_0 x (x+b_t)$ for $x \ge -b_t$ by \Cref{assump:activation}, we have 
    \begin{align}
        \hat\varphi_1(b_t) \ge C_0 \EE[\ind(x + b_t > 0 ) (x+b_t) x] + \EE[\varphi(x;b_t) x \ind (-\barb_t \le x \le -b_t)] - (n\lor d)^{-c_0} \EE[|x|]. 
    \end{align}
    Here, we recall definition $\varphi(x;b) = \phi(x + b) + x \cdot \phi'(x + b) $. Therefore, $\varphi(x; b) \ge \phi(x+b)$ for $x>0$.
    By \Cref{assump:activation}, we know that $\phi'(x + b) \ge 0$ for all $x$. Since $-\barb_t > 0$, we have for $x\in [-\barb_t, -b_t]$ that 
    \begin{align}
        \varphi(x; b_t) \ge \phi(x+b_t) \ge -(n\lor d)^{-c_0}, 
    \end{align}
    where the last inequality holds by the monotonicity of $\phi$.
    Therefore, we conclude that 
    \begin{align}
        \hat\varphi_1(b_t) \ge C_0 \EE[\ind(x + b_t > 0 ) (x+b_t) x] - C \cdot (n\lor d)^{-c_0} \ge \frac{C_0}{2} \Phi(|b_t|). 
    \end{align}
    Since we can make $\kappa_0 = |b_t| - |\barb_t|$ log-polynomially small, e.g., $\kappa_0 = (\log (n\lor d))^{-C}$, for $|\barb_t| = \Theta( \log(n\lor d)^{C})$, we have $2 \Phi(|\barb_t|) \ge \Phi(|b_t|) \ge \frac{\Phi(|\barb_t|)}{2}$. This completes the proof. 
    \end{proof}

    \subsubsection{Proof of \Cref{lem:signal-bounds}}\label{sec:proof-signal-bounds}
    \begin{proof}[Proof of \Cref{lem:signal-bounds}]
        \textbf{Lower bounding the signal term.}\quad 
        Let us lower bound the signal term. 
        Note that by the monotonicity assumption in \Cref{assump:activation},
        \begin{align}
            \varphi(x; b_t) \given_{x>-b_t}= \phi(x+ b_t) + x \phi'(x + b_t) \given_{x>-b_t} \ge  C_0 x. 
        \end{align}
        For $x\in (-\barb_t, b_t)$, we have $\varphi(x; b_t) \ge \varphi(-\barb_t; b_t) \ge - (d\lor n)^{-c_0}$. 
        Together, we conclude that
        \begin{align}
            & \sum_{l=1}^{N_2}\EE_{x\sim\cN(0, 1)} \Bigl[\theta_l \cdot \varphi\bigl(\sqrt{1-\theta_l^2} x + \theta_l \sqrt d \alpha_{-1, t-1}; b_t\bigr) \Bigr] \\
            &\quad \ge  \sum_{l=1}^{N_2}\EE_{x_\sim\cN(0, 1)}\left[\theta_l \ind\Bigl( x + \frac{\theta_l \sqrt{d} \alpha_{-1, t-1} + b_t}{\sqrt{1-\theta_l^2}} > 0\Bigr) \cdot C_0\Bigl(  \sqrt{1-\theta_l^2} x + \theta_l \sqrt{d} \alpha_{-1, t-1}\Bigr)\right] \\
            &\hspace{2cm} -N_2 (d\lor n)^{-c_0} \\
            &\quad \ge  \sum_{l=1}^{N_2}\Phi\Bigl( \frac{-b_t - \theta_l \sqrt d\alpha_{-1, t-1}}{\sqrt{1-\theta_l^2}}\Bigr) \cdot C_0 \theta_l^2 \sqrt d \alpha_{-1, t-1} - N_2 (d\lor n)^{-c_0} \\ 
            &\quad \ge \frac{1-o(1)}{2}\sum_{l=1}^{N_2} \ind\Bigl(\theta_l > \frac{-b_t}{\sqrt{d} \alpha_{-1, t-1}}\Bigr) \cdot C_0 \theta_l^2 \sqrt d \alpha_{-1, t-1} , 
        \end{align}
        where in the second inequality, it follows from the direct calculation of the integral of the Gaussian that $\EE_{x\sim\cN(0, 1)}[\ind(x>a) x] = p(a) > 0$ with $p(a)$ being the density of $\cN(0, 1)$ at $a$.
        The $-(d\lor n)^{-c_0}$ on the right-hand side is negligible. 
        Note that the indicator is selecting the larger half of $\theta_l$, and we can thereby obtain the following lower bound
        \begin{align}
            C^{-1} N_2 \cdot  C_0 \sqrt d \alpha_{-1, t-1}  \cdot \overline{\theta^2} Q_t, \where Q_t =\frac{1}{N_2}\sum_{l=1}^{N_2} \ind\Bigl(\theta_l > \frac{-b_t}{\sqrt{d} \alpha_{-1, t-1}}\Bigr), \quad \overline{\theta^2} = \frac{\norm{\theta}_2^2}{N_2}.
            \label{eq:psi-lower-bound-0}
        \end{align}
        
        \noindent
        \textbf{Upper bounding the signal term.}\quad 
        To arrive at an upper bound, we use the fact that $\varphi(x; b_t) \le (d\lor n)^{-c_0} \ind(x< -\barb_t) + L x \ind(x\ge -\barb_t)$ to obtain that 
        \begin{align}
            & \sum_{l=1}^{N_2}\EE_{x\sim\cN(0, 1)} \Bigl[\theta_l \cdot \varphi\bigl(\sqrt{1-\theta_l^2} x + \theta_l \sqrt d \alpha_{-1, t-1}; b_t\bigr) \Bigr] \\
            &\quad \le L \sum_{l=1}^{N_2}\EE_{x_\sim\cN(0, 1)}\left[\theta_l \ind\Bigl( x + \frac{\theta_l \sqrt{d} \alpha_{-1, t-1} + b_t}{\sqrt{1-\theta_l^2}} > 0\Bigr) \cdot \Bigl(  \sqrt{1-\theta_l^2} x + \theta_l \sqrt{d} \alpha_{-1, t-1}\Bigr)\right] \\
            &\hspace{2cm} + N_2 (d\lor n)^{-c_0} \\
            &\quad \le C L\sum_{l=1}^{N_2} \bigl( \theta_l \sqrt{1-\theta_l^2} + \theta_l^2 \sqrt{d} \alpha_{-1, t-1} \bigr) \le C L N_2 \overline{\theta^2} \sqrt d \alpha_{-1, t-1}, 
        \end{align}
        where the last second inequality holds by noting that $\EE[\ind(x>a) x] = p(|a|)\le 1$, and the last one holds by noting that $\sqrt d \alpha_{-1, t-1} \gg 1$. 
    \end{proof}

\section{Proofs for SAE Dynamics Analysis}

In this section, we provide supplementary proofs for the results used in the proof of the main theorem in \Cref{sec:sae-dynamics}. 

\subsection{Proof of \Cref{prop:sparsity}}
\label{app:proof-s_i-s_star-s}
    Let us first prove that there must exists some $i\in[n]$ such that $\overline{\theta_i^2}\ge 1/s$. 
    Since the total sum $\sum_{j\in[n]}\sum_{l\in\cD_j} H_{l, j}^2 = \sum_{l=1}^N \norm{h_l}_2^2 = N$, and there are at most $N s$ non-zero entries in the weight matrix $H$, we have the average 
    \begin{align}
        \overline{H^2} \defeq \frac{\sum_{l=1}^N \sum_{j=1}^n H_{l,j}^2}{\sum_{l=1}^N \sum_{j=1}^n \ind(H_{l,j} > 0)} \ge \frac{N}{Ns} = \frac{1}{s}. 
    \end{align}
    On the other hand, we also have 
    \begin{align}
        \overline{H^2} = \frac{\sum_{j=1}^n |\cD_j| \cdot \overline{\theta_j^2}}{\sum_{j=1}^n |\cD_j|} \le \max_{j\in[n]} \overline{\theta_j^2}. 
    \end{align}
    It thus follows that there exists some $i\in[n]$ such that $\overline{\theta_i^2} \ge 1/s$.

    \paragraph{Proof of the first inequality}
    By definition of $h_\star$, we have $h_\star^2 \ge \hslash_{q, \star}^2$ for $q = 4$. To prove the upper bound on $h_\star$, we just need to show that 
    \(
        \hslash_{q, \star}^2 \ge \overline{\theta_j^2} 
    \)
    for any $j\in[n]$. Let us consider the kernel function in the definition of $\hslash_{q, \star}$:
    \begin{align}
        f(x) = \Phi\Bigl(\frac{-\barb}{\sqrt{\frac{q-1}{q} x + \frac{1}{q}}} \Bigr). 
    \end{align}
    In particular, we aim to show that $f(\cdot)$ is convex for $x\in [0, 1]$. 
    The second derivative of $f(x)$ is given by
    \begin{align}
        f''(x)
    = p \Bigl(\frac{-\barb}{\sqrt{\frac{q-1}{q}x + \frac{1}{q}}}\Bigr)
    \cdot \frac{\barb\bigl(\frac{q-1}{q}\bigr)^{2}}{4\bigl(\frac{q-1}{q}x + \frac{1}{q}\bigr)^{7/2}}
    \cdot \Bigl[\,3\Bigl(\frac{q-1}{q}x + \frac{1}{q}\Bigr) - \barb^{2}\Bigr].
    \end{align}
    Using the property that $\barb<-\sqrt 3$, we conclude that $f''(x) \ge 0$ for $x\in[0, 1]$, and $f$ is convex. Now, by definition of $\hslash_{q, \star}$, we have  
    \begin{align}
        f(\hslash_{q, \star}^2) \ge \max_{j\in[n]}\frac{1}{|\cD_j|} \sum_{l\in\cD_j} f(H_{l, j}^2)\ge \max_{j\in[n]} f\Bigl(\frac{1}{|\cD_j|} \sum_{l\in\cD_j} H_{l, j}^2\Bigr) = \max_{j\in[n]} f(\overline{\theta_j^2}), 
        \label{eq:hslash-Jensen}
    \end{align}
    where the second inequality follows from the convexity of $f(x)$ and Jensen's inequality. Moreover, the first derivative of $f(x)$ is given by
    \begin{align}
        f'(x)
        = -p \Bigl(\frac{-\barb}{\sqrt{\frac{q-1}{q}x + \frac{1}{q}}}\Bigr)
        \;\frac{\barb\bigl(\tfrac{q-1}{q}\bigr)}{2\bigl(\tfrac{q-1}{q}x + \tfrac{1}{q}\bigr)^{3/2}} > 0. 
        \label{eq:hslash-kernel-1st-derivative}
    \end{align}
    Therefore, we have by \eqref{eq:hslash-Jensen} that $\hslash_{q, \star}^2 \ge \overline{\theta_j^2}$ for any $j\in[n]$ and $q = 4$. Consequently, $h_\star^2 \ge \hslash_{4, \star}^2 \ge \max_{j\in[n]} \overline{\theta_j^2} \ge 1/s$. This proves the first inequality. 

    \paragraph{Proof of the second inequality}
    Since we have by definition of $\overline{\theta_i^2}$ that
    \begin{align}
        \overline{\theta_i^2} &= \frac{\norm{\theta_i}_2^2}{|\cD_i|} \le (1 - \hat\QQ_i(h_i)) \cdot h_i^2  + \hat\QQ(h_i) \cdot 1 \le \hat\QQ_i(h_i) + h_i^2, 
    \end{align}
    it follows from the condition $\overline{\theta_i^2} > \hat\QQ_i(h_i)$ that
    \begin{align}
        h_i \ge \sqrt{\overline{\theta_i^2} - \hat\QQ(s_i^{-1/2} )}.
    \end{align}
    This completes the proof of the third inequality.
    Hence, we have completed the proof of \Cref{prop:sparsity}.

\subsection{Proofs for Concentration Results Combined}

In the following, we present the proofs of the lemmas and propositions used in \Cref{sec:simplification-concentration}. 

\subsubsection{Proof of \Cref{lem:z_tau-ut}}
\label{sec:proof-z_tau-u}
    From $\Phi(|\barb_t|) \gg L s\rho_1 (t\log n)^3$, we deduce that $t\log n \ll n$, since 
    $L s\rho_1 n^3 \gg 1 \ge \Phi(|\barb_t|)$ (recalling that $\rho_1\ge n^{-1}$). Hence, we can directly apply \Cref{cor:norm-E-F} in what follows. 
    Using the bound in \Cref{lem:E-1st} together with \Cref{prop:hatvarphi_1-bound}, if we further assume 
    $\Phi(|\barb_t|) \gg L s\rho_1 (t\log n)^3,$
    then the desired concentration result is obtained as follows:
    \begin{align}
        \langle z_\tau, E^\top \varphi(E y_t^\star; b_t) \rangle = (1 \pm o(1)) \cdot N \alpha_{\tau, t-1} \hatvarphi_1(b_t)  \pm C \sqrt{1-\alpha_{\tau, t-1}^2} \cdot \sqrt{\norm{E^\top \varphi(E y_t^\star; b_t)}_2^2 \cdot t\log(n)}. 
        \label{eq:E-1st-concentration-sim}
    \end{align}
    Here, we use the fact that $|N_1/N - 1| \le \rho_1 \ll 1$, where $\rho_1 \ll 1$ can also be deduced from the condition $\Phi(|\barb_t|) \gg L s\rho_1 (t\log(n))^3$.
    For the concentration result for $\langle z_\tau, F^\top \varphi(F y_t + \theta \cdot v^\top \barw_{t-1}; b_t)\rangle  $ in \Cref{lem:F-1st}, we  use the Stein's lemma to derive that 
    \begin{align}
        &\frac{N_2}{N} \sum_{l=1}^{N_2} |\alpha_{\tau, t-1}| \sqrt{1-\theta_l^2} \cdot \EE_{x\sim \cN(0, 1)} \Bigl[x \varphi\bigl(\sqrt{1-\theta_l^2} x + \theta_l v^\top \barw_{t-1}; b_t\bigr)\Bigr] \\
        &\quad \le \frac{N_2|\alpha_{\tau ,t-1}|}{N} \sum_{l=1}^{N_2} (1-\theta_l^2) \cdot \EE_{x\sim\cN(0, 1)}\Bigl[\varphi'(\sqrt{1-\theta_l^2} x + \theta_l v^\top \barw_{t-1}; b_t)\Bigr] \\
        &\quad  \le  \rho_1 |\alpha_{\tau ,t-1}|  L = o(\Phi(|\barb_t|) \cdot |\alpha_{\tau, t-1}|)
        \label{eq:F-1st-expectation-ub}
    \end{align}
    where in the second inequality we use the Lipschitzness of $\varphi$ and in the last inequality we use $Ls\rho_1 (t\log n)^3 \ll \Phi(|\barb_t|)$.
    Moreover, we have 
    \begin{align}
        &L |\alpha_{\tau, t-1}| \cdot \frac{N_2}{N} \cdot (\sqrt{t\log (n)} + \norm{v}_2 |\alpha_{-1, t-1}| ) \cdot \sqrt{\rho_2 s} \cdot (t\log(n))^{3/2}  \\
        &\quad \le L |\alpha_{\tau, t-1}| \cdot \rho_1 \sqrt{\rho_2 s} (t\log n)^2 + \rho_1 \sqrt{\rho_2 s} (t\log n)^{3/2} \cdot d |\alpha_{-1, t-1} \alpha_{\tau, t-1}| \\ 
        &\quad \le o(\Phi(|\barb_t|) \cdot |\alpha_{\tau, t-1}|) + \rho_1 \sqrt{\rho_2 s} (t\log n)^{3/2} \cdot d \alpha_{-1, t-1} \alpha_{\tau, t-1}.
        \label{eq:F-1st-expectation-ub-2}  
    \end{align}
    where in the first inequality, we use $N_2/N \le \rho_1$ by definition and in the second inequality, we use the fact $\rho_1 \sqrt{\rho_2 s} (t\log n)^2 \le \rho_1 (t\log n)^2 \ll \Phi(|\barb_t|)$ under the condition $Ls\rho_1 (t\log n)^3 \ll \Phi(|\barb_t|)$. 
    Moreover, by \Cref{prop:hatvarphi_1-bound}, we know that $\hat\varphi_1(b_t) = \Omega(\Phi(|\barb_t|))$.
    Consequently, by combining \eqref{eq:F-1st-expectation-ub} and \eqref{eq:F-1st-expectation-ub-2} with the upper bound in \Cref{lem:F-1st}, we have
    \begin{align}
        &\bigl| \langle z_\tau, F^\top \varphi(F y_t + \theta \cdot v^\top \barw_{t-1}; b_t)\rangle \bigr|\\ 
        &\qquad \le o \left( N |\alpha_{\tau, t-1}| \hatvarphi_1(b_t)\right) + C\sqrt{1-\alpha_{\tau, t-1}^2} \cdot \sqrt{\norm{F^\top\varphi(F y + \theta \cdot v^\top \barw_{t-1}; b_t)}_2^2 \cdot t\log(n)} \\
        &\hspace{2cm} + C N L \rho_1 \sqrt{\rho_2 s} (t\log n)^{3/2} \cdot d |\alpha_{\tau, t-1} \alpha_{-1, t-1}|. 
        \label{eq:F-1st-concentration-sim}
    \end{align}
    Let us consider the good event with respect to some universal constant $C>0$:
    \begin{align}
        \cE: \Bigl\{ \norm{z_\tau}_\infty \le C \sqrt{\log(tn)}, \quad \forall \tau \le T\Bigr\}.
    \end{align}
    As we increase the constant $C$, the failure probability of the event $\cE$ can be made polynomially small, e.g., $1-n^{-c}$ for some other constant $c>0$ (See \Cref{lem:max gaussian_tail}).
    Conditioned on the success of this event, we have for the non-Gaussian components that 
    \begin{align}
        \left|\langle z_\tau, \Delta E_t\rangle + \langle z_\tau, \Delta F_t\rangle \right|
        & \le C\sqrt{\log(tn)} \cdot \bigl(\norm{\Delta E_t}_1 + \norm{\Delta F_t}_1\bigr) \\
        & \le C L N \sqrt{\log(n)} \cdot \bigl( \sqrt{s \rho_1} (\sqrt{\Phi(|\barb_t|)} + \sqrt{s\rho_1 t \log n}) \cdot \sqrt{d} \beta_{t-1} + \sqrt s \rho_1 |\barb_t| d \beta_{t-1}^2 \bigr) \\
        &\qquad + C L N \sqrt{\log(n)} \cdot \rho_1\sqrt{s d} \beta_{t-1} \\
        &\le C L N \sqrt{\log n} \cdot \bigl( \sqrt{s \rho_1 d \Phi(|\barb_t|)} \beta_{t-1} + \sqrt s \rho_1 |\barb_t| d \beta_{t-1}^2 \bigr), 
        \label{eq:drifting-error-1st-sim}
    \end{align}
    where in the second inequality, we invoke \Cref{lem:E-error} and \Cref{lem:F-error} to bound the $\ell_1$ norm of the error terms, and also the fact that $t$ is at most polynomial in $n$.
    In the last inequality, we use the fact that $\norm{\Delta y_t}_2 \le \sqrt d \beta_{t-1}$ by \Cref{lem:delta-y-l2}.
    Now, we combine the derived concentration results in \eqref{eq:E-1st-concentration-sim}, \eqref{eq:F-1st-concentration-sim} and \eqref{eq:drifting-error-1st-sim} with $1-\alpha_{\tau, t-1}^2 \le 1$ and the upper bound for $\norm{E^\top \varphi(E y_t^\star; b_t)}_2^2 + \norm{F^\top \varphi(F y_t + \theta \cdot v^\top \barw_{t-1}; b_t)}_2^2$ in \Cref{cor:norm-E-F} to obtain that 
    \begin{align}
        \langle z_\tau, u_t\rangle &= \langle z_\tau, E^\top \varphi(E y_t^\star; b_t)\rangle + \langle z_\tau, F^\top \varphi(F y_t + \theta \cdot v^\top \barw_{t-1}; b_t)\rangle + \langle z_\tau, \Delta E_t\rangle + \langle z_\tau, \Delta F_t\rangle\\
        &= N \alpha_{\tau, t-1} \hat\varphi_1(b_t) \cdot (1\pm o(1)) \pm C N L \rho_1 \sqrt{\rho_2 s} (t\log n)^{3/2} \cdot d |\alpha_{\tau, t-1} \alpha_{-1, t-1}| \\
        &\hspace{1cm}\pm C N \rho_1 L \sqrt{t \log n} \cdot \xi_t \pm  C L N \sqrt{\log n} \cdot \bigl( \sqrt{s \rho_1 d \Phi(|\barb_t|)} \beta_{t-1} + \sqrt s \rho_1 |\barb_t| d \beta_{t-1}^2 \bigr). 
    \end{align}
    Hence, we complete the proof of the \Cref{lem:z_tau-ut}.

\subsubsection{Proof of \Cref{lem:v-w}} \label{sec:proof-v-w}
    Recall by definition of $w_t$, $\langle v, w_t\rangle /\norm{v}_2$ can be decomposed into 
\begin{align}
    &\frac{\langle v, w_t\rangle}{\norm{v}_2} 
    = \langle z_{-1}, u_t\rangle + \norm{v}_2 \cdot \theta^\top \varphi(F y_t + \theta \cdot v^\top \barw_{t-1}; b_t) + \eta^{-1}\alpha_{-1, t-1}. 
    \label{eq:v-w-decomp}
\end{align}
Taking $\tau = -1$ in \Cref{lem:z_tau-ut}, we have
\begin{align}
    \langle z_{-1}, u_t\rangle &= N \alpha_{-1, t-1} \hat\varphi_1(b_t) \cdot (1\pm o(1)) \pm C N L \rho_1 \sqrt{\rho_2 s} (t\log n)^{3/2} \cdot d |\alpha_{-1, t-1}|^2 \\
    &\qquad\pm C N \rho_1 L \sqrt{t \log n} \cdot \xi_t \pm  C L N \sqrt{\log n} \cdot \bigl( \sqrt{s \rho_1 d \Phi(|\barb_t|)} \beta_{t-1} + \sqrt s \rho_1 |\barb_t| d \beta_{t-1}^2 \bigr). 
    \label{eq:z-u-tau=1}
\end{align}
Moreover, by a direct decomposition of the second term, we have
\begin{align}
    \norm{v}_2 \theta^\top \varphi(F y_t + \theta \cdot v^\top \barw_{t-1}; b_t) & = \norm{v}_2 \theta^\top \varphi(F y_t^\star + \theta \cdot v^\top \barw_{t-1}; b_t) + \norm{v}_2 \theta^\top \Delta \varphi_{F, t} \\
    & = \norm{v}_2 \theta^\top \varphi(F y_t^\star + \theta \cdot v^\top \barw_{t-1}; b_t) \pm \norm{v}_2 \norm{\theta}_2 \cdot \norm{\Delta \varphi_{F, t}}_2.
\end{align}
Notice that $\norm{v}_2 = \sqrt{d} \cdot (1\pm C\sqrt{\log (n) /d})$ with probability at least $1-n^{-c}$ by concentration of $\chi^2$ random variables (see \Cref{lem:chi-squared}).
By \Cref{lem:F-error}, we have $\norm{\Delta \varphi_{F, t}}_2 \le \sqrt{\rho_2 N_2} L \cdot \norm{\Delta y_t}_2 \le \sqrt{\rho_2 N_2 d} L \beta_{t-1}$. Therefore, 
\begin{align}
    \norm{v}_2 \norm{\theta}_2 \cdot \norm{\Delta \varphi_{F, t}}_2 \le C \sqrt d \cdot \sqrt{N_2 \overline{\theta^2}} \cdot L \sqrt{\rho_2 N_2 d} L \beta_{t-1}\le C L N \rho_1 d \sqrt{\rho_2 } \beta_{t-1}. 
\end{align}
Now, combining the concentration results for $\theta^\top \varphi(F^\top y_t^\star + \theta \cdot v^\top \barw_{t-1}; b_t)$ in \Cref{lem:F-signal}, we obtain that
\begin{align}
    &\|v\|_2\,\theta^\top\varphi(Fy_t+\theta\cdot v^\top\barw_{t-1};b_t) \\
    &\quad = (1\pm o(1)) N \psi_t  \pm C\,N L\,\rho_1\sqrt{\rho_2\,s}\,(t\log n)^{3/2}\,d\,\alpha_{-1,t-1}
    \pm C\,N L\,\rho_1 d \sqrt{\rho_2} \beta_{t-1}.
    \label{eq:v-theta-F}
\end{align}
Furthermore, we have by \Cref{lem:signal-bounds} that $N \psi_t \gtrsim C_0 \overline{\theta^2} Q_t \cdot N_2 d \alpha_{-1, t-1}$. Under the conditions
\begin{gather}
    \frac{N_2}{N} C_0 \overline{\theta^2} Q_t \gg \max\Bigl\{L \rho_1\sqrt{\rho_2 s} (t\log n)^{3/2},\:  L d^{-1} \Phi(|\barb_t|), \: L\sqrt{t\log n } \rho_1 \frac{\xi_t}{d\alpha_{-1, t-1}} \Bigr\}, 
    \label{eq:condition-v-w-1}
\end{gather} 
we conclude by also noting that $\sqrt{d}\alpha_{-1, t-1} \gg 1$ that 
\begin{align}
    N \psi_t \gg \max\Bigl\{ C N L \rho_1 \sqrt{\rho_2 s} (t\log n)^{3/2} \cdot d \alpha_{-1, t-1}, \:  N \alpha_{-1, t-1} \hat\varphi_1(b_t), \: CN\rho_1 L \sqrt{t\log n} \cdot \xi_t \Bigr\}. 
\end{align}
Now we plug \eqref{eq:v-theta-F} and \eqref{eq:z-u-tau=1} into \eqref{eq:v-w-decomp} to obtain
\begin{align}
    \frac{\langle v, w_t\rangle}{\norm{v}_2} &= (1\pm o(1)) N \psi_t + \eta^{-1}\alpha_{-1, t-1} \\
    &\qquad \pm C L N \sqrt{d\rho_1 s \log n} \cdot \bigl( \sqrt{\Phi(|\barb_t|)} + \sqrt{\rho_1 d \rho_2 s^{-1}} +\sqrt{\rho_1 d } |\barb_t| \beta_{t-1} \bigr) \cdot \beta_{t-1}. 
    \label{eq:v-w-decomp-2}
\end{align}
Finally, under the conditions $\sqrt{t s \log n} |\barb_t| \beta_{t-1} \ll 1$, $st\log n \cdot \Phi(|\barb_t|) \ll \rho_1 d$, we have 
\begin{align}
    \sqrt{d\rho_1 s \log n} \cdot \bigl(\sqrt{\Phi(|\barb_t|)} + \sqrt{\rho_1 d \rho_2 s^{-1}} +\sqrt{\rho_1 d } |\barb_t| \beta_{t-1}\bigr) \le C \rho_1 d. 
\end{align}
Here, we use the fact that $\rho_2 \log n \ll 1$, which can be deduced from the following inequality under the condition $ \frac{N_2}{N} C_0 \overline{\theta^2} Q_t \gg L \rho_1\sqrt{\rho_2 s} (t\log n)^{3/2}$:
\begin{align}
    \rho_1 \gtrsim \frac{N_2}{N} C_0 \overline{\theta^2} Q_t \gg L \rho_1\sqrt{\rho_2 s} (t\log n)^{3/2} \ge \rho_1 \sqrt{\rho_2 \log n}. 
\end{align}
Moreover, under the condition
\begin{align}
    \frac{N_2}{N} C_0 \overline{\theta^2} Q_t \gg C L  \rho_1 \cdot \frac{\beta_{t-1}}{\alpha_{-1, t-1}}, 
\end{align}
we conclude that the second line of \eqref{eq:v-w-decomp-2} can be upper bounded by $o(N \psi_t)$. Hence, the proof of \Cref{lem:v-w} is completed.

\subsubsection{Proof of \Cref{lem:perp-wt}}\label{sec:proof-perp-wt}
    Recall from the definition of $w_t$ that 
    \begin{align}
        \norm{P_{w_{-1:0}}^\perp (w_t - \eta^{-1} \barw_{t-1})}_2^2 
        &= \sum_{\tau=1}^{t-1} \Bigl(\langle z_\tau, u_t\rangle - \langle P_{u_{1:\tau}} z_\tau, u_t\rangle + \frac{\langle u_{\tau}^\perp, u_t\rangle}{\norm{u_{\tau}^\perp}_2} \cdot \frac{\norm{w_{\tau}^\perp}_2}{\norm{u_{\tau}^\perp}_2}\Bigr)^2 \\
        &\qquad + \norm{ P_{w_{-1:t-1}}^\perp \tilde{z}_t}_{2}^2 \cdot \norm{u_t^\perp}_2^2.
        \label{eq:perp-wt-decompose}
    \end{align}
    \begin{lemma} \label{lem:psi-3-4}
        Assume that $T\le \sqrt{d}$ and $d \in (n^{1/c_1}, n^{c_1})$ for some universal constant $c_1\in(0, 1)$. Then there exist universal constants $c, C>0$ such that with probability at least $1 - n^{-c}$ over the randomness of i.i.d. standard Gaussian vectors $z_{-1:T}$, for all $t\in [T]$, 
        \[
            \sum_{\tau=1}^{t-1} \langle P_{u_{1:\tau}} z_\tau, u_t\rangle^2 + \sum_{\tau=1}^{t-1} \Bigl(\frac{\langle  u_{\tau}^\perp, u_t\rangle}{\norm{u_{\tau}^\perp}_2} \cdot \frac{\norm{w_{\tau}^\perp}_2}{\norm{u_{\tau}^\perp}_2}\Bigr)^2 + \norm{P_{w_{-1: t-1}}^\perp \tilde{z}_{t} }_2^2 \cdot \norm{u_t^\perp}_2^2 \le C d \cdot \|u_t\|_2^2.
        \]
        \end{lemma}
    \begin{proof}
        See \Cref{app:proof-varphi-3-4} for a detailed proof.
    \end{proof}
    \begin{lemma}[Upper Bound for $\|u_t\|_2^2$]\label{lem:ut-upper-bound}
        If $t\log n\ll n$, $-\barb_t = \Theta(\sqrt{\log n})$, $\rho_1\ll 1$, it holds with probability at least $1 - n^{-c}$ for all $t\le T < \sqrt d$ that
        \begin{align}
            \norm{u}_2 \le C N L \rho_1 (\xi_t + \sqrt d \beta_{t-1}). 
        \end{align}
    \end{lemma}
    \begin{proof}
        See \Cref{sec:proof-ut-upper-bound} for a detailed proof.
    \end{proof}
    Combining \Cref{lem:psi-3-4,lem:ut-upper-bound}, it holds with probability at least $1-n^{-c}$ for all $t\le \sqrt d$,
    \begin{align}
        &\sqrt{\sum_{\tau=1}^{t-1} \langle P_{u_{1:\tau}} z_\tau, u_t\rangle^2 + \sum_{\tau=1}^{t-1} \Bigl(\frac{\langle  u_{\tau}^\perp, u_t\rangle}{\norm{u_{\tau}^\perp}_2} \cdot \frac{\norm{w_{\tau}^\perp}_2}{\norm{u_{\tau}^\perp}_2}\Bigr)^2 + \norm{P_{w_{-1: t-1}}^\perp \tilde{z}_{t} }_2^2 \cdot \norm{u_t^\perp}_2^2}\\
            &\qquad  \le C\sqrt{d} \cdot \norm{u_t}_2 \le C N L \rho_1 \sqrt d \bigl(\xi_t + \sqrt d \beta_{t-1}\bigr).\label{eq:perp-norm-A2} 
    \end{align}
    It remains to upper bound $\sum_{\tau=1}^{t-1} \langle z_\tau, u_t\rangle^2$.
    Recall that $\beta_{t-1} = \sqrt{1 - \alpha_{-1, t-1}^2 - \alpha_{0, t-1}^2} = \sqrt{\sum_{\tau=1}^{t-1} \alpha_{\tau, t-1}^2}$.
    Using \Cref{lem:z_tau-ut}, we conclude that 
    \begin{align}
        \sqrt{\sum_{\tau=1}^{t-1} \langle z_\tau, u_t\rangle^2} & \le CN \beta_{t-1} \hat\varphi_1(b_t) \cdot (1\pm o(1))  + C N L \rho_1 \sqrt{\rho_2 s} (t\log n)^{3/2} \cdot d \, |\alpha_{-1, t-1}| \beta_{t-1} \nonumber\\[1mm]
        &\qquad \qquad  + C N \rho_1 L t \sqrt{\log n} \cdot \xi_t + C L N \sqrt{t\log n} \cdot \bigl( \sqrt{s \rho_1 d \Phi(|\barb_t|)} 
        + \sqrt s \rho_1 |\barb_t| d\, \beta_{t-1}\bigr)\cdot \beta_{t-1} \\ 
        & \le  C N \rho_1 L t \sqrt{\log n} \cdot \xi_t + CL N  \rho_1 d \beta_{t-1}, 
        \label{eq:perp-norm-B1}
    \end{align}
    where in the first inequality, the $\beta_{t-1}$ terms in the first line is obtained by the Pythagorean sum with respect to $\alpha_{\tau, t-1}$ for $\tau =1, \ldots, t-1$.
    In the second line, an additional $\sqrt{t-1}$ factor is added to the upper bound for $|\langle z_\tau, u_t\rangle|$ since $\sqrt{\sum_{\tau=1}^{t-1} x_\tau^2} \le \sqrt{t} \cdot \max_{\tau=1,\ldots,t-1} |x_\tau|$.
    In the last inequality, we use the conditions $\sqrt{\rho_2 s} (t\log n)^{3/2} \ll 1$, $\Phi(|\barb_t|) \ll \rho_1 d (st\log n)^{-1}$, and $\sqrt{st \log n}|\barb_t| \beta_{t-1} \ll 1$ to upper bound all the terms containing $\beta_{t-1}$ by $C L N \rho_1 d \beta_{t-1}$.
    Plugging \eqref{eq:perp-norm-A2} and \eqref{eq:perp-norm-B1} into \eqref{eq:perp-wt-decompose}, we obtain 
    \begin{align}
        \norm{P_{w_{-1:0}}^\perp w_t}_2 \le  C\sqrt{d} \cdot \norm{u_t}_2 + C\sqrt{\sum_{\tau=1}^{t-1} \langle z_\tau, u_t\rangle^2} + \eta^{-1} \beta_{t-1} \le C N L \rho_1 \sqrt d \bigl(\xi_t + \sqrt d \beta_{t-1}\bigr) + \eta^{-1} \beta_{t-1} . 
    \end{align}
    Here, we use the fact that $t\sqrt{\log n} \le \sqrt d$, which is implied by the condition $\rho_1 d (st\log n)^{-1} \gg \Phi(|\barb_t|) \gg L s\rho_1 (t\log(n))^3$. 
    Lastly, by condition $\eta^{-1}\ll N \Phi(|\barb_t|)$ and the fact that $L \rho_1 d \gg \Phi(|\barb_t|)$ by assumption, we can absorb the $\eta^{-1} \beta_{t-1}$ term into the $CN L \rho_1 d\beta_{t-1}$ term.
    Hence, we complete the proof of \Cref{lem:perp-wt}.

\subsubsection{Proof of \Cref{lem:w-0-1-norm}} \label{sec:proof-w-0-1-norm}
    Recall by definition of $w_t$ that 
    \begin{align}
        \norm{P_{w_{-1:0}}w_t }_2 = \sqrt{\frac{\langle v, w_t\rangle^2}{\norm{v}_2^2} + \bigl(\langle z_0, u_t \rangle +\eta^{-1} \alpha_{0, t-1}\bigr)^2}. 
    \end{align}
    By \Cref{lem:v-w}, we already have 
    \(
        {\langle v, w_t\rangle}/{\norm{v}_2} = (1\pm o(1)) N \psi_t.
        \)
    It remains to characterize $\langle z_0, u_t\rangle$. We have by \Cref{lem:z_tau-ut} that
    \begin{align}
        \langle z_0, u_t\rangle 
        &= N \alpha_{0, t-1} \hat\varphi_1(b_t) \cdot (1\pm o(1))  \pm C N L \rho_1 \sqrt{\rho_2 s} (t\log n)^{3/2} \cdot d \, |\alpha_{0, t-1} \alpha_{-1, t-1}| \nonumber\\[1mm]
        &\quad \pm C N \rho_1 L \sqrt{t \log n} \cdot \xi_t \pm C L N  \cdot \bigl( \sqrt{s \log (n) \rho_1 d \Phi(|\barb_t|)}
        + \sqrt{s\log (n)} \rho_1 |\barb_t| d \beta_{t-1} \bigr)\cdot \beta_{t-1} \\
        &= N \alpha_{0, t-1} \hat\varphi_1(b_t) \cdot (1\pm o(1))  \pm C N L \rho_1 \sqrt{\rho_2 s} (t\log n)^{3/2} \cdot d \, |\alpha_{0, t-1} \alpha_{-1, t-1}| \nonumber\\[1mm]
        &\quad \pm C N \rho_1 L \sqrt{t \log n} \cdot \xi_t \pm C L N \rho_1 d  \beta_{t-1}
    \end{align}
    Here, in the last term we use the condition $\sqrt{ts\log n} |\barb_t|\beta_{t-1}\ll 1$ to upper bound $\sqrt{s\log n} |\barb_t|\beta_{t-1} \ll 1$, and $\rho_1 d (st\log n)^{-1} \gg\Phi(|\barb_t|)$ to upper bound $\sqrt{s \log (n) \rho_1 d \Phi(|\barb_t|)} \le C \rho_1 d$.
    Note that the fluctuation terms are similar to the one for $\langle z_{-1}, u_t\rangle$ in the proof of \Cref{lem:v-w}. Specifically, under the same conditions
    \begin{align}
        \frac{N_2}{N} C_0 \overline{\theta^2} Q_t \gg \max\Bigl\{L \rho_1\sqrt{\rho_2 s} (t\log n)^{3/2},\:  L\sqrt{t\log n } \rho_1 \frac{\xi_t}{d\alpha_{-1, t-1}}, \: L \rho_1 \frac{\beta_{t-1}}{\alpha_{-1, t-1}} \Bigr\}
    \end{align}
    we have 
    \begin{align}
        N\psi_t \gg C N L \cdot \max\Bigl\{ \rho_1 \sqrt{\rho_2 s} (t\log n)^{3/2} \cdot d |\alpha_{0, t-1} \alpha_{-1, t-1}|, \:  \rho_1 \sqrt{t\log n} \cdot \xi_t ,\:  \rho_1 d \beta_{t-1} \Bigr\}.
    \end{align}
    Thus, we conclude that $\langle z_0, u_t\rangle = N\alpha_{0, t-1} \hat\varphi_1(b_t) \cdot (1\pm o(1)) \pm o(N \psi_t)$. 
    Thus, 
    \begin{align}
        \norm{P_{w_{-1:0}} w_t}_2 
        &=  \sqrt{\frac{\langle v, w_t\rangle^2}{\norm{v}_2^2} + (\langle z_0, u_t\rangle + \eta^{-1} \alpha_{0, t-1})^2} \\ 
        &= \sqrt{\bigl(N \psi_{t} \cdot (1 \pm o(1))\bigr)^2 + \bigl( N \alpha_{0, t-1} \hat\varphi_1(b_t) \cdot (1\pm o(1)) \pm o(N \psi_{t}) + \eta^{-1} \alpha_{0, t-1}\bigr)^2} \\
        &=(1\pm o(1)) \cdot \sqrt{(N\psi_t)^2 + (N \alpha_{0, t-1} \hatvarphi_1(b_t))^2 }.
    \end{align}
    Here, the last inequality holds by also noting that $\eta^{-1} \ll N \Phi(|\barb_t|) \le C N \hat\varphi_1(b_t)$.
    This completes the proof.

\subsection{Proofs for Recursion Analysis}
\subsubsection{Proof of \Cref{lem:recursion}}\label{sec:proof-recursion}
    What we need to prove here is that all the conditions in \Cref{lem:w_t-2-norm} hold for the current time step $t$ if the conditions in \Cref{lem:recursion} hold. 
    This is because the conditions in \Cref{lem:w_t-2-norm} are the union of the conditions in \Cref{cor:norm-E-F,lem:z_tau-ut,lem:v-w,lem:perp-wt,lem:w-0-1-norm}. In the following, we check all the listed conditions one by one. 

    \paragraph{Step \RNum{1}: Checking all conditions in \Cref{lem:w_t-2-norm}}
    For the first step, we divide the conditions in \Cref{lem:w_t-2-norm} into three groups. 

    \noindent\textbf{\emph{Group 1: Implication of \labelcref{cond:basic,cond:alpha-beta-t}.}}\quad
    We first notice that since $t\le T$, conditions 
    \begin{align}
        -\barb_t = \Theta(\sqrt{\log n})<\zeta_1, \quad \kappa_0 |\barb_t| = O(1), \quad \sqrt{\rho_2 s} (t\log n)^{3/2} \ll 1, \quad \eta^{-1} \ll N \Phi(|\barb_t|) \land N_2 dC_0\overline{\theta^2} \highlight{Q_t}
    \end{align}
    are guaranteed by \labelcref{cond:basic}.
    Here, we need to be more careful about condition $\eta^{-1} \ll N_2 dC_0\overline{\theta^2} \highlight{Q_t}$, as $Q_t$ is a function of $t$, and what we directly have in \labelcref{cond:basic} is for \highlight{$Q_1$} only.
    By definition 
    \(
        Q_t =\frac{1}{N_2}\sum_{l=1}^{N_2} \ind\bigl(\theta_l > \frac{-b}{\sqrt{d} \alpha_{-1, t-1}}\bigr), 
    \)
    we note that $Q_t$ is nondecreasing in $\alpha_{-1, t-1}$. Therefore, we have the following fact:
    \begin{fact}\label{fact:Q-monotone}
        If $\alpha_{-1, t-1} \ge \alpha_{-1, 0}$, then
        \(
            Q_t \ge  Q_1.
        \)
    \end{fact}
    \noindent
    In fact, the condition $\alpha_{-1, t-1} \ge \alpha_{-1, 0}$ is automatically guaranteed by \labelcref{cond:alpha-beta-t}.
    Therefore, the condition $\eta^{-1} \ll N_2 dC_0\overline{\theta^2} Q_t$ will hold for all successive $t$ as long as it holds for $t=1$ and $\alpha_{-1, t-1} \ge \alpha_{-1, 0}$.
    Meanwhile, we also have by the same reasoning that
    \begin{align} 
        \sqrt{d} \alpha_{-1, t-1} \ge \sqrt{d} \alpha_{-1, 0} \gg 1
    \end{align}
    where the last inequality is guaranteed by \textbf{InitCond-1}. The condition $\sqrt{t s \log n} |\barb_t| \beta_{t-1}\ll 1$ is guaranteed by \labelcref{cond:alpha-beta-t} as well. 
    
    \vspace{5pt}
    \noindent\textbf{\emph{Group 2: Implication of \labelcref{cond:Phi,cond:signal-t=1,cond:alpha-beta-ratio}.}} \quad 
    The direct implication of \labelcref{cond:Phi} is that
    \begin{align}
        \rho_1 d (st\log n)^{-1} \gg \Phi(|\barb_t|) \gg L s\rho_1 (t\log(n))^3. 
    \end{align} 
    Similarly, the direct implication of \labelcref{cond:signal-t=1,cond:alpha-beta-ratio} is that
    \begin{align}
        \frac{N_2}{N} C_0 \overline{\theta^2} Q_t \gg \max\Bigl\{L \rho_1\sqrt{\rho_2 s} (t\log n)^{3/2},\:  L d^{-1} \Phi(|\barb_t|), \: L \rho_1  \frac{\beta_{t-1}}{\alpha_{-1, t-1}} \Bigr\}.
    \end{align}
    Here, we use the fact that $t\le T$ and the monotonicity of $Q_t$ in \Cref{fact:Q-monotone}. 
    It remains to check whether $\frac{N_2}{N} C_0 \overline{\theta^2} Q_t \gg L\sqrt{t\log n } \rho_1 \frac{\xi_t}{d\alpha_{-1, t-1}}$ holds.
    
    \vspace{5pt}
    \noindent\textbf{\emph{Group 3: Implication of \labelcref{cond:Phi},\labelcref{cond:signal-t=1}, \labelcref{cond:alpha-beta-t,cond:alpha-beta-ratio}.}} \quad
    To verify this inequality $\frac{N_2}{N} C_0 \overline{\theta^2} Q_t \gg L\sqrt{t\log n } \rho_1 \frac{\xi_t}{d\alpha_{-1, t-1}}$, we just need to show that $\xi_t/\alpha_{-1, t-1} \le C \xi_1/\alpha_{-1, 0}$ for some universal constant $C>0$, as the corresponding inequality for the latter is already guaranteed by \labelcref{cond:signal-t=1}.
    Recall the definition of $\xi_t$ in \Cref{cor:norm-E-F}, the ratio $\xi_t/\alpha_{-1, t-1}$ is given by 
    \begin{align}
        \frac{\xi_t}{\alpha_{-1, t}} = \frac{\sqrt s\, t\log (n)\, \cK_t  + \rho_1^{-1} \sqrt{ \Phi(|\barb_t|) \cdot \hat\EE_{l,l'}\!\Bigl[
            \Phi\Bigl( |\barb_t|\sqrt{\frac{1-\langle h_l, h_{l'}\rangle}{1+\langle h_l, h_{l'}\rangle}}
            \Bigr)
            \langle h_l, h_{l'}\rangle
        \Bigr]} + \rho_2 \sqrt{n}}{\alpha_{-1, t-1}} + \sqrt{\rho_2 d}.
        \label{eq:xi_t-alpha-ratio-1}
    \end{align}
    We obtain the above formula by the nonnegativity of $\alpha_{-1, t-1}$ guaranteed by \labelcref{cond:alpha-beta-t}.

\begin{proposition}\label{prop:Kt}
    If $ -\barb_t\le  \sqrt{2\log n}$ for some universal constant $\kappa > 0$, then for $t\ge 2$, 
    \begin{align}
        \cK_t \le t \cdot \bigl( \cK_1 + C \sqrt{\log n} \cdot (\beta_{t-1} + |\alpha_{-1, t-1}| + |\alpha_{-1, 0}|) \bigr). 
    \end{align}
\end{proposition}
\begin{proof}
    See \Cref{sec:proof-Kt} for a detailed proof.
\end{proof}
\noindent
Combining \eqref{eq:xi_t-alpha-ratio-1}, \Cref{prop:Kt} and the fact that $\alpha_{-1, t-1} \ge t^2 \alpha_{-1, 0} \ge \alpha_{-1, 0}$ by \labelcref{cond:alpha-beta-t}, we have
\begin{align}
    \frac{\xi_t}{\alpha_{-1, t-1}} &\le \frac{\sqrt s t^2 \log(n)  \cK_1}{\alpha_{-1, t-1}}  + C \sqrt s t^2 \log(n)^{3/2} \Bigl(\frac{\beta_{t-1}}{\alpha_{-1, t-1}} + 2\Bigr) \\
    &\qquad + \frac{\rho_1^{-1} \sqrt{ \Phi(|\barb_t|) \cdot \hat\EE_{l,l'}\!\Bigl[
        \Phi\Bigl( |\barb_t|\sqrt{\frac{1-\langle h_l, h_{l'}\rangle}{1+\langle h_l, h_{l'}\rangle}}
        \Bigr)
        \langle h_l, h_{l'}\rangle
    \Bigr]} + \rho_2 \sqrt{n}}{\alpha_{-1, t-1}} + \sqrt{\rho_2 d} \\
    &\le \frac{\xi_1}{\alpha_{-1, 0}} + C \sqrt s t^2 \log(n)^{3/2} \Bigl(\frac{\beta_{t-1}}{\alpha_{-1, t-1}} + 2\Bigr),
    \label{eq:xi_t-alpha-ratio-2} 
\end{align}
where in the second inequality, we directly plug in the definition of $\xi_1$ with $t=1$ in \eqref{eq:xi_t-alpha-ratio-1} and use the fact that $\alpha_{-1, t-1} \ge t^2 \alpha_{-1, 0}$ to upper bound the first term in the right-hand side. 
Furthermore, for each term in  \labelcref{cond:alpha-beta-ratio}, we have the following relationship:
\begin{align}
    \frac{N_2}{N} \le \rho_1, \quad C_0\overline{\theta^2} Q_t = O(1), \quad L =\Omega(1), 
\end{align}
where the first inequality holds by direct definition of $\rho_1$ in \eqref{eq:rho-def}, the second equality holds by noting that $\overline{\theta^2} \le 1$, $Q_t\le 1$ and $C_0$ is a universal constant, and the last inequality holds by \Cref{assump:H}.
Together, we have the following implication: 
\begin{align}
    \frac{N_2}{N} C_0 \overline{\theta^2} Q_t \gg L \rho_1 \frac{\beta_{t-1}}{\alpha_{-1, t-1}} \: \Rightarrow\: \frac{\beta_{t-1}}{\alpha_{-1, t-1}} \ll 1
\end{align}
Therefore, we can further simplify the upper bound in \eqref{eq:xi_t-alpha-ratio-2} to
\begin{align}
    \frac{\xi_t}{\alpha_{-1, t-1}} &\le \frac{\xi_1}{\alpha_{-1, 0}} + C \sqrt s t^2 \log(n)^{3/2} \cdot \ind(t\ge 2). 
    \label{eq:xi_t-alpha-ratio-3}
\end{align}
Using \eqref{eq:xi_t-alpha-ratio-3}, in order for condition $\frac{N_2}{N}C_0 \overline{\theta^2} Q_t \gg L \sqrt{t\log n} \rho_1 \frac{\xi_t}{d\alpha_{-1, t-1}}$ to hold, we just need to ensure 
\begin{align}
    \frac{N_2}{N} C_0 \overline{\theta^2} Q_t \gg L \sqrt{t\log n} \rho_1 \frac{\xi_1}{d\alpha_{-1, 0}}, \quad \frac{N_2}{N} C_0 \overline{\theta^2} Q_t \gg  C L d^{-1}\rho_1 \sqrt{st^5} (\log n)^2 .
\end{align}
The first one is clearly given by \labelcref{cond:signal-t=1}, and the second one is satisfied because we have by using \labelcref{cond:Phi,cond:signal-t=1} that
\begin{align}
    \frac{N_2}{N} C_0 \overline{\theta^2} Q_t \gg L d^{-1} \Phi(|\barb|) \gg L^2 d^{-1} \rho_1 s (T\log(n))^3 \gtrsim C L d^{-1} \rho_1 \sqrt{st^5} (\log n)^2.
\end{align}
Here, the first inequality holds by the second condition in \labelcref{cond:signal-t=1}, the second inequality holds by \labelcref{cond:Phi}, and the last inequality holds by noting that we are considering any $t\le T$.
The last inequality shows that the last condition $\frac{N_2}{N} C_0 \overline{\theta^2} Q_t \gg L\sqrt{t\log n } \rho_1 \frac{\xi_t}{d\alpha_{-1, t-1}}$ also holds automatically under the conditions in \Cref{lem:recursion}. 
To this end, we have shown that all the conditions in \Cref{lem:w_t-2-norm} hold for $t$ if the conditions in \Cref{lem:recursion} are satisfied.

\paragraph{Step \RNum{2}: Deriving the recursion}
As we have shown in the previous step, all the conditions in \Cref{lem:w_t-2-norm} hold for $t$ if the conditions in \Cref{lem:recursion} hold.
Therefore, we can safely apply all the concentration results derived in \Cref{sec:simplification-concentration}.
We next show how to use the previous derived concentration result on $\langle v, w_t\rangle/\norm{v}_2$, $\norm{P_{w_{-1:0}} w_t}_2$, and $\norm{P_{w_{-1:0}}^\perp w_t}_2$ to control the recursion of $\beta_t/\alpha_{-1, t}$ and $1/\alpha_{-1, t}$.
Since $\beta_t$ is the projection of $w_t$ onto the $P_{w_{-1:0}}$ direction, and $\alpha_{-1, t}$ is the projection of $w_t$ onto the $v$ direction, we have
\begin{align}
    \frac{\beta_t}{\alpha_{-1, t}} 
    &= \frac{\norm{v}_2 \cdot \norm{P_{w_{-1:0}}^\perp w_t}_2}{\langle v, w_t\rangle} \le \frac{C L \rho_1 \sqrt d \bigl(\xi_t + \sqrt{d} \beta_{t-1}\bigr)}{C_0 \overline{\theta^2} Q_t \cdot N_2/N \cdot d \alpha_{-1, t-1}} \\
    &\le \frac{C L \rho_1}{C_0\overline{\theta^2} Q_t \cdot N_2/N } \cdot \biggl(\frac{1}{\sqrt d}\Bigl( \frac{\xi_1}{\alpha_{-1, 0}} + C \sqrt{s} t^2\log(n)^{3/2} \cdot \ind(t\ge 2)\Bigr) +  \frac{\beta_{t-1}}{\alpha_{-1, t-1}} \biggr). 
\end{align}
where in the first inequality, we use the upper bound for $\norm{P_{w_{-1:0}}^\perp w_t}_2$ in \Cref{lem:perp-wt} and the lower bound for $\langle v, w_t\rangle/\norm{v}_2$ in \Cref{lem:v-w} as $\langle v, w_t\rangle/\norm{v}_2 \ge (1- o(1)) N \psi_t \gtrsim N C_0 \overline{\theta^2} Q_t \cdot N_2/N \cdot d \alpha_{-1, t-1}$ by the lower bound of $\psi_t$ in \Cref{lem:signal-bounds}.
The second inequality holds by plugging in the upper bound for $\xi_t/\alpha_{-1, t-1}$ in \eqref{eq:xi_t-alpha-ratio-3}. 
Similarly, we have by definition of $\alpha_{-1, t}$ that
\begin{align}
    \frac{1}{\alpha_{-1, t}} &= \frac{\norm{v}_2 \cdot \norm{w_t}_2}{\langle v, w_t\rangle}  \le \frac{(1+ o(1)) \cdot \sqrt{\psi_t^2 + \hat\varphi_1(b)^2}+ CL\rho_1 \sqrt d \xi_t}{(1 - o(1)) \cdot \psi_t}\\
    &\le  \frac{(1+ o(1)) \cdot \sqrt{(C_0 \overline{\theta^2} Q_t \cdot N_2/N \cdot d \alpha_{-1, t-1})^2 + \bigl(C L \Phi(|\barb|)\bigr)^2}+ CL\rho_1 \sqrt d \xi_t}{C_0 \overline{\theta^2} Q_t \cdot N_2/N \cdot d \alpha_{-1, t-1}} \\
    &\le \frac{C L \rho_1}{C_0 \overline{\theta^2} Q_t \cdot N_2/N } \cdot \biggl( \frac{\Phi(|\barb|)}{\rho_1 d} \cdot \frac{1}{\alpha_{-1, t-1}} + \frac{1}{\sqrt d} \Bigl( \frac{\xi_1}{\alpha_{-1, 0}} + C \sqrt{s} t^2\log(n)^{3/2} \cdot \ind(t\ge 2)\Bigr) \biggr) \\
    &\qquad + (1+o(1)). 
\end{align}
where in the second inequality, we plug in the lower bound for $\psi_t$ and the upper bound for $\hat\varphi_1(b_t)$ in \Cref{prop:hatvarphi_1-bound}.
The last inequality holds by the triangle inequality and the upper bound for $\xi_t/\alpha_{-1, t-1}$ in \eqref{eq:xi_t-alpha-ratio-3}.
This completes the proof of \Cref{lem:recursion}.

\subsubsection{Proof of \Cref{lem:final-stage}} \label{sec:proof-final-stage}
    In the following proof, let us take $T_1 = \max\{(2\varsigma)^{-1}, 1\}$.
    As our goal is to establish that \eqref{eq:alpha-recursion} and \eqref{eq:beta-alpha-ratio-recursion} holds for all $t\le T_1$, we just need to show that \labelcref{cond:alpha-beta-t,cond:alpha-beta-ratio} hold for all $t\le T_1$, as they are the only conditions that might be violated over time, and the other conditions only depend on the initial conditions. 

    \paragraph{Initial step}
    For $t=1$, we have $\alpha_{-1, t-1} = \alpha_{-1, 0}$ and $\beta_{t-1} = \beta_0 = 0$. Hence, \labelcref{cond:alpha-beta-t,cond:alpha-beta-ratio} hold trivially. Before we start the proof, we first derive some useful inequalities.

    \paragraph{Useful inequalities}
    For $\lambda_1$, we have by \labelcref{cond:Phi-lambda_0,cond:xi_1} that 
    \begin{align}
        \lambda_1 = \frac{CL\rho_1}{C_0\overline{\theta^2} Q_1 \cdot N_2/N} = \frac{\rho_1 d^{1-\varsigma}}{\Phi(|\barb|)}, \quad 
        \lambda_1\xi_1 = \frac{\lambda_0 \xi_1}{Q_1} \ll \frac{d^{-\epsilon}}{\sqrt s \log n}\ll 1. 
        \label{eq:lambda_1}
    \end{align}
    Using the above two inequalities, we have by \eqref{eq:alpha-t=1} that 
    \begin{align}
        \frac{1}{\alpha_{-1, 1}} 
        &\le 1 + o(1) + \Bigl(\frac{\lambda_1 \Phi(|\barb|)}{\rho_1 d} + \frac{\lambda_1 \xi_1}{\sqrt d } \Bigr) \cdot \frac{1}{\alpha_{-1, 0}}  \le  1 + o(1) + (d^{1/2-\varsigma} + 1) \le 3 + d^{1/2-\varsigma}. 
        \label{eq:alpha-t=1-2}
    \end{align}
    In fact, we have the ratio $\alpha_{-1, 0} / \alpha_{-1, 1}$ as 
    \begin{align}
        \frac{\alpha_{-1, 0}}{\alpha_{-1, 1}} \le (3 + d^{1/2 - \varsigma}) \cdot \alpha_{-1, 0} \le (3 d^{-1/2} + d^{-\varsigma}) \cdot C\sqrt{\log M} \ll 1. 
        \label{eq:alpha-t=1-3}
    \end{align}
    Here, we use the fact that $\alpha_{-1, 0} =O(\sqrt{\log M})$ with sufficiently high probability $1 - n^{-c}$, and $M=\poly(n)$. 
    The above inequality demonstrates that $\alpha_{-1, 1}$ is guaranteed to grow in the first step.
    Thus, by definition of $Q_t$ in \eqref{eq:Q_t-bartheta-def}, we conclude that 
    \begin{align}
        Q_2 = \frac{1}{N_2} \sum_{l=1}^{N_2} \ind\Bigl(\theta_l \ge \frac{-b}{\sqrt d \alpha_{-1, 1}}\Bigr) \ge \frac{1}{N_2} \sum_{l=1}^{N_2} \ind\Bigl(\theta_l \ge  |b|(3 d^{-1/2} + d^{-\varsigma})\Bigr) \eqdef Q_2^\nu, 
    \end{align}
    where we take $\nu = |b| (3 d^{-1/2} + d^{-\varsigma}) = O(\sqrt{\log n} \cdot d^{-\varsigma \land 1/2})$ and denote the right-hand side of the above inequality as $Q_2^\nu$.
    Since $\theta_l\in[0, 1]$, we have 
    \begin{align}
        \overline{\theta^2} = \frac{1}{N_2}\sum_{l=1}^{N_2} \theta_l^2 \le Q_2^\nu \cdot 1^2 + (1 - Q_2^\nu) \cdot \nu^2 = Q_2^\nu (1 -\nu^2) + \nu^2 \:\Rightarrow\: Q_2^\nu \ge \frac{\overline{\theta^2} -\nu^2}{1-\nu^2} \ge \frac{\overline{\theta^2}}{2}, 
    \end{align}
    where the last inequality holds from \Cref{assump:H} that $\overline{\theta^2} =\Omega(\polylog(n)^{-1})\gg \nu^2$. 
    In the sequel, we will use $Q_2 \ge \overline{\theta^2}/2$ as the lower bound for $Q_2$.
    By definition of $T_1$, we have $T_1=(2\varsigma)^{-1}\lor 1 = \Theta(1)$. 
    In addition, for $\lambda_2$, we have 
    \begin{align}
        \lambda_2 = \frac{\lambda_0}{Q_2} \le \frac{2\lambda_0}{\overline{\theta^2}} = O(\polylog(n)), 
        \label{eq:lambda_2}
    \end{align}
    where in the inequality we use the lower bound for $Q_2$ and in the last equality we use $\overline{\theta^2} = \Omega(\polylog(n)^{-1})$ in \Cref{assump:H} and $\lambda_0 = O(\polylog(n))$ in \labelcref{cond:lambda_0}.
    
    We now have for the coefficient $\lambda_2 \Phi(|\barb|)/(\rho_1 d)$ that 
    \begin{align}
        \frac{\lambda_2 \Phi(|\barb|)}{\rho_1 d} = \frac{\lambda_0 \Phi(|\barb|)}{Q_2\rho_1 d} = \frac{Q_1 d^{-\varsigma}}{Q_2} \le d^{-\varsigma}, 
        \label{eq:alignment-growth-speed}
    \end{align}
    where the second identity holds from \labelcref{cond:Phi-lambda_0} and the last inequality holds by noting that $\alpha_{-1, 1} \ge \alpha_{-1, 0}$ by the first step's calculation in \eqref{eq:alpha-t=1-3} and using the monotonicity of $Q_t$ in \Cref{fact:Q-monotone}.
    Next, we upper bound the quantity $C_1$ in \eqref{eq:C1-def}:
    \begin{align}
        C_1 &= \Bigl(1 + o(1) + \frac{\lambda_2 \xi_1}{\sqrt d \alpha_{-1, 0}} + \frac{C\lambda_2 \sqrt s T_0^2 (\log n)^{3/2}}{\sqrt d} \Bigr) \cdot \frac{1}{1 - \lambda_2 \Phi(|\barb|) / \rho_1 d} \\
        &\le \Bigl(1 + o(1) + \frac{\lambda_1 \xi_1}{\sqrt d \alpha_{-1, 0}} + \frac{C \sqrt s \polylog(n)}{\sqrt d} \Bigr) \cdot \frac{1}{1 - d^{-\epsilon}}\\
        &\le \Bigl(1 + o(1) + \frac{d^{-\epsilon}}{\sqrt s \log n} \Bigr) \cdot (1 + o(1)) = 1 + o(1), 
    \end{align}
    where in the first inequality, we use the fact that $\lambda_2\le \lambda_1$ by the fact $Q_2\ge Q_1$, and we invoke the upper bound $T_0\le \log n$ and $\lambda_2 = O(\polylog(n))$ in \eqref{eq:lambda_2}. 
    In the last inequality, we use the previous bound for $\lambda_1 \xi_1$ in \eqref{eq:lambda_1} together with the fact that $\sqrt{d} \alpha_{-1, 0}\ge 1$ by \textbf{InitCond-1}. 

    \paragraph{Induction step}
    Suppose the induction hypothesis holds for $1, 2, \ldots, t$.   
    We will show that \labelcref{cond:alpha-beta-t,cond:alpha-beta-ratio} hold for $t+1\le T_1$ as well.
    To this end, it is evident that $\alpha_{-1, t}$ is always growing before reaching $C_1$, which is evident from \eqref{eq:alpha-recursion-1} by noting that $\lambda_2 \Phi(|\barb|) / \rho_1 d \le d^{-\varsigma} < 1$.
    
    We first look at the recursion of $\alpha_{-1, t}$. By \eqref{eq:alpha-recursion}, the ratio $\alpha_{-1, 0}/\alpha_{-1, t}$ is bounded by
    \begin{align}
        \frac{\alpha_{-1, 0}}{\alpha_{-1, t}} 
        &\le \Bigl(\frac{\lambda_2 \Phi(|\barb|)}{\rho_1 d}\Bigr)^{t-1} \cdot \Bigl(\frac{\lambda_1 \Phi(|\barb|)}{\rho_1 d} + \frac{\lambda_1 \xi_1}{\sqrt d } \Bigr) + C_1 \alpha_{-1, 0} \\
        &\le d^{-\varsigma(t-1)} \cdot \Bigl(d^{-\varsigma} + \frac{d^{-\epsilon}}{\sqrt{sd } \log n}\Bigr) + (1+o(1)) \cdot \frac{C \sqrt{\log M}}{\sqrt d} \le C d^{-\varsigma(t-1) - (\varsigma\land 1/2)} + d^{-1/2+\epsilon}. 
    \end{align}
    The first term on the right-hand side is decaying exponentially fast with respect to $t$. 
    The second term is much smaller than $1/T_0^2$ given that $T_0\le \log n$ by definition. Therefore, both terms are much smaller than $1/T_0^2$. This implies the first condition in \labelcref{cond:alpha-beta-t} holds for $t+1$.

    Next, we look at the conditions involving $\beta_{t}$. By previous analysis on $T_1$ and the upper bound in \eqref{eq:lambda_2}, we obtain
    \begin{align}
        \lambda_2^{T_1 -1} \le (\polylog(n))^{ (2\varsigma)^{-1} \lor 1} = O(\polylog(n)). 
    \end{align}
    By recursion of $\beta_{t}/\alpha_{-1, t}$ in \eqref{eq:beta-alpha-ratio-recursion}, we have
    \begin{align}
        \frac{\beta_{t}}{\alpha_{-1, t}} 
        &\le \frac{\lambda_2^{t-1}}{\sqrt d} \cdot \Bigl( (T_0 + \lambda_1) \cdot \frac{\xi_1}{\alpha_{-1, 0}} + C \sqrt{s} T_0^3 \log(n)^{3/2} \Bigr) \\
        &\le \frac{\polylog(n)}{\sqrt d} \cdot \Bigl( \frac{\lambda_1 \xi_1}{\alpha_{-1, 0}} + C \sqrt s \polylog(n)\Bigr) \\
        &\le \polylog(n) \cdot \Bigl( \frac{d^{-\epsilon}}{\sqrt s \log n} + \frac{C \sqrt s \polylog(n)}{\sqrt d} \Bigr) \le \frac{d^{-\epsilon} \polylog(n)}{\sqrt s}. 
    \end{align}
    Here, the second inequality holds by the upper bound for $\lambda_2^{T_1-1}$ and also the fact that $T_0+\lambda_1 \le 2\lambda_1 \log n $ since $\lambda_1\ge 1$ and $T_0\le \log n$. In the second inequality, we use the upper bound for $\lambda_1 \xi_1$ in \eqref{eq:lambda_1} and the fact that $\sqrt d \alpha_{-1, 0}\ge 1$ by \textbf{InitCond-1}.
    The last inequality holds because $\epsilon < 1/2$ by definition. 
    Using the above inequality with the fact that $\alpha_{-1, t} \le 1$, we obtain
    \begin{align}
        \beta_t \le \frac{d^{-\epsilon} \polylog(n)}{\sqrt s} \ll \frac{1}{\sqrt{T_0 s \log n}|\barb|}, 
    \end{align}
    where the last inequality holds by noting that both $T_0$ and $|\barb|$ are at most $O(\polylog(n))$. This implies that the second condition in \labelcref{cond:alpha-beta-t} holds for $t+1$.

    Eventually, for \labelcref{cond:alpha-beta-ratio}, we have 
    \begin{align}
        \frac{C_0\overline{\theta^2} Q_t \cdot N_2/N}{C L \rho_1} = \frac{Q_t}{\lambda_0}\ge \frac{Q_2}{\lambda_0} \ge \frac{\overline{\theta^2}}{2 \lambda_0} = \Omega(\polylog(n)^{-1}). 
    \end{align}
    Therefore, the left-hand side of the above inequality is also much larger than $\beta_t/\alpha_{-1, t}$. 
    To this end, we have finished the induction step and proved that \labelcref{cond:alpha-beta-t,cond:alpha-beta-ratio} hold for all $t\le T_1$. 

    \paragraph{Final step}
    According to the recursion in \eqref{eq:alpha-recursion}, let us consider the real value $t^\star$ that satisfies
    \begin{align}
        \Bigl(\frac{\lambda_2 \Phi(|\barb|)}{\rho_1 d}\Bigr)^{t^\star-1} \cdot \Bigl(\frac{\lambda_1 \Phi(|\barb|)}{\rho_1 d} + \frac{\lambda_1 \xi_1}{\sqrt d } \Bigr) \cdot \frac{1}{\alpha_{-1, 0}} = \log(d)^{-c_0}\label{eq:alpha-final-cond}
    \end{align}
    for some small constant $c_0>0$ to be determined later.
    We first note that we can obtain the $\varsigma\land 1/2$ factor by the inequality for $\lambda_1$ in \eqref{eq:lambda_1} that
    \begin{align} 
        \frac{\lambda_2 \Phi(|\barb|)}{\rho_1 d} \le \frac{\lambda_1 \Phi(|\barb|)}{\rho_1 d} \le \frac{\lambda_1 \Phi(|\barb|)}{\rho_1 d}  + \frac{\lambda_1 \xi_1}{\sqrt d }, \quad \text{and}\quad \frac{\lambda_1 \Phi(|\barb|)}{\rho_1 d}  + \frac{\lambda_1 \xi_1}{\sqrt d } \le   d^{-\varsigma} + \frac{d^{-1/2-\epsilon}}{\sqrt{s} \log n}
        \label{eq:lambda-2-phi-rho-d}
    \end{align}
    Using the above inequality \eqref{eq:lambda-2-phi-rho-d}, and taking a logarithm of both sides with base $d$ for \eqref{eq:alpha-recursion}, we have for $t^\star$ that
    \begin{align}
        t^\star \cdot \log_d \Bigl( d^{-\varsigma} + \frac{d^{-1/2-\epsilon}}{\sqrt{s} \log n}\Bigr) + \log_d \Bigl(\frac{1}{\alpha_{-1, 0}}\Bigr) &\ge -\frac{c_0 \log\log d}{\log d}, 
        \label{eq:alpha-final-cond-log}
    \end{align}
    which implies that
    \begin{align}
        t^\star \le \log_d\biggl(\frac{1}{d^{-\varsigma} + \frac{d^{-1/2-\epsilon}}{\sqrt{s} \log n}}\biggr)^{-1} \cdot \biggl(\log_d \Bigl(\frac{1}{\alpha_{-1, 0}}\Bigr) + \frac{c_0\log\log d}{\log d} \biggr) \le \frac{1/2}{\varsigma\land 1/2} = (2\varsigma)^{-1} \lor 1 = T_1.
        \label{eq:t-star-upper-bound}
    \end{align}
    In the second inequality, we use the fact that by \textbf{InitCond-1}, 
    \begin{align}
        \log_d \Bigl(\frac{1}{\alpha_{-1, 0}}\Bigr) = \log_d(\norm{v}_2) - \log_d \bigl((1-\varepsilon)\sqrt{2\log(M/n)} \bigr) \le \frac{1}{2} - \frac{\log\log(M/n)}{2\log d}.
    \end{align}
    Therefore, we can take $c_0$ to be small enough but still on a constant level such that 
    \begin{align}
        \frac{1}{2} - \frac{\log\log(M/n)}{2\log d} \le \frac{1}{2} - \frac{c_0\log\log d}{2\log d}. 
    \end{align}
    This justifies the second inequality in \eqref{eq:t-star-upper-bound}. 
    Thus, there must exists some time $t\le T_1$ such that 
    \eqref{eq:alpha-final-cond} holds. For this time $t$, we already have 
    \begin{align}
        \frac{1}{\alpha_{-1, t}} \le \Bigl(\frac{\lambda_2 \Phi(|\barb|)}{\rho_1 d}\Bigr)^{t-1} \cdot \Bigl(\frac{\lambda_1 \Phi(|\barb|)}{\rho_1 d} + \frac{\lambda_1 \xi_1}{\sqrt d } \Bigr) \cdot \frac{1}{\alpha_{-1, 0}} + C_1 \le d^{-\varsigma} + C_1 \le 1 + o(1). 
    \end{align}
    This implies that $\alpha_{-1, t} = 1 - o(1)$. 
    
\paragraph{Checking $\alpha_{-1, t-1} \ge \alpha_{-1, 1}$}
    An additional step is needed to show that $\alpha_{-1, t-1} \ge \alpha_{-1, 1}$ for all $t\ge 2$ and before $t^\star$ is reached.
    This is required because we want to ensure that before time $t^\star$, we always have $\alpha_{-1, t-1} \ge \alpha_{-1, 1}$, and the stopping time $T_0$ will not prohibit us from reaching $t^\star$.
    In fact, we have by \eqref{eq:alpha-recursion-1} that
    \begin{align}
        \frac{1}{\alpha_{-1, t}} \le \Bigl(\frac{\lambda_2 \Phi(|\barb|)}{\rho_1 d}\Bigr)^{t-2} \cdot \Bigl(\frac{1}{\alpha_{-1, 1}} - C_1\Bigr) + C_1. 
    \end{align}
    Therefore, the ratio $\alpha_{-1, t-1}/\alpha_{-1, 1}$ is bounded by
    \begin{align}
        \frac{\alpha_{-1, 1}}{\alpha_{-1, t-1}} \le \Bigl(\frac{\lambda_2 \Phi(|\barb|)}{\rho_1 d}\Bigr)^{t-2} \cdot \bigl(1 - C_1 \alpha_{-1, 1}\bigr) + C_1 \alpha_{-1, 1}. 
    \end{align}
    We consider two cases. If $C_1 \alpha_{-1, 1} \ge 1$, we can just stop the gradient at $t=1$ and obtain $\alpha_{-1, 1} = 1-o(1)$ since $C_1 = 1+o(1)$. 
    In this case, we reach strong alignment in just one step.
    In another case where $C_1 \alpha_{-1, 1} < 1$, since $\lambda_2 \Phi(|\barb|)/(\rho_1 d) \le d^{-\varsigma}$, we have the above ratio strictly upper bounded by $1$. Hence, the condition $\alpha_{-1, t-1} \ge \alpha_{-1, 1}$ holds for all $t\ge 2$ and before $t^\star$ is reached.
    
    In both cases, we have shown that $\alpha_{-1, t-1}\ge \alpha_{-1, 1}$ hold for $2\le t\le t^\star$. As we have shown that \labelcref{cond:alpha-beta-t,cond:alpha-beta-ratio} hold for all $t\le T_1$ from the induction step, $t^\star\le T_1$ from the final step, and $T_1\le \log(n)$ by definition, we conclude that \labelcref{cond:T0-1,cond:T0-2,cond:T0-3} in the definition of the stopping time $T_0$ hold for all $t\le t^\star$.
    In other words, we have shown that $T_0 \ge t^\star$.

Thus, we complete the proof of \Cref{lem:final-stage}.

\subsection{Proofs for Condition Simplification}

\subsubsection{Proof of \Cref{lem:final-stage-simplified}}
\label{sec:proof-final-stage-simplified}
Let us take $t^\star$ as the maximum number of iterations considered. 
In the following, we first provide a sufficient condition for \labelcref{cond:signal-t=1}, \labelcref{cond:Phi-lambda_0,cond:xi_1} to hold. 
Then, we give a reformulation of \labelcref{cond:basic,cond:Phi,cond:lambda_0}.

A sufficient condition for \labelcref{cond:signal-t=1} to hold is given by 
\begin{align}
    \frac{Q_1}{\lambda_0} \gg \max\Bigl\{ \sqrt{\rho_2 s} (\log n)^{3/2}, \: \frac{\Phi(|\barb|)}{\rho_1 d}, \: \sqrt{\log n} \cdot \xi_1\Bigr\}
    \label{eq:cond-sim-1}
\end{align}
under the condition $d\alpha_{-1, 0}\ge 1$. 
On the other hand, we note that \labelcref{cond:Phi-lambda_0,cond:xi_1} can be reformulated as 
\begin{align}
    \frac{Q_1}{\lambda_0} \cdot d^{-\varsigma} = \frac{\Phi(|\barb|)}{\rho_1 d} \gg d^{\epsilon - \varsigma} \sqrt{s}\log n \cdot \xi_1.
    \label{eq:cond-sim-3}
\end{align}
Since $d^{\epsilon} \sqrt s \log n\cdot \xi_1 \gg \sqrt{\log n} \cdot \xi_1$, we can safely delete the last term in \eqref{eq:cond-sim-1}. Also by noting that $d^{-\varsigma}\ll 1$, we can safely delete the second term in \eqref{eq:cond-sim-1}.
Furthermore, by definition of $\xi_1$, which we recall as follows:
\begin{align}
    \xi_1 = \sqrt s\, \log n\, \cK_1 + \rho_1^{-1} \sqrt{ \Phi(|\barb|) \cdot \hat\EE_{l,l'}\!\Biggl[
        \Phi\Bigl( |\barb|\sqrt{\frac{1-\langle h_l, h_{l'}\rangle}{1+\langle h_l, h_{l'}\rangle}}
        \Bigr)
        \langle h_l, h_{l'}\rangle
    \Biggr]} + \sqrt{\rho_2 d}\, |\alpha_{-1, 0}| + \rho_2 \sqrt{n}\,, 
\end{align}
we conclude that $\xi_1 \ge \sqrt{\rho_2 d} \alpha_{-1, 0} \ge \sqrt{\rho_2}$. Therefore, 
\begin{align}
    d^{\epsilon} \sqrt s \log n \cdot \xi_1 \ge d^{\epsilon} \sqrt{\rho_2 s} \log n \ge \sqrt{\rho_2 s} (\log n)^{3/2}, 
\end{align}
where in the last inequality, we use the definition  $\epsilon = C'\log \log n/(\varsigma \log d) \ge \log \log n/\log d$. 
Therefore, the first term in \eqref{eq:cond-sim-1} can also be deleted.
In summary, \labelcref{cond:signal-t=1} is automatically implied by \eqref{eq:cond-sim-3}.

A reformulation of \labelcref{cond:Phi} gives
\begin{align}
    \frac{1}{s \log n} \gg    \frac{\Phi(|\barb|)}{\rho_1 d} \gg \frac{L s \log(n)^3}{d}.
    \label{eq:cond-Phi-sim}
\end{align}
In the following, we will simplify the above condition.
Note that
$$\Phi(|\barb|)/(\rho_1 d) = Q_1/\lambda_0 d^{-\varsigma} \le d^{-\varsigma} \ll (s\log n)^{-1}$$ 
holds by using $\lambda_0=\Theta(\polylog(n))$ according to  \labelcref{cond:lambda_0} and $\lambda_0^{_1}Q_1 \cdot d^{-\varsigma}=\Phi(|\barb|)/(\rho_1 d)$ according to  \labelcref{cond:Phi-lambda_0}. Therefore, we can safely remove the first inequality in \eqref{eq:cond-Phi-sim}.

In the following, we aim to remove the condition $\sqrt{\rho_2 s} (t^\star \log n)^{3/2} \ll 1$ in \labelcref{cond:basic}. 
As $Q_1/\lambda_0 \gg d^\epsilon \sqrt s \log n\cdot \xi_1$ by \labelcref{cond:xi_1}, we conclude that $\xi_1\ll Q_1/\lambda_0 < 1$. By definition of $\xi_1$, this condition directly implies that $\rho_2\ll n^{-1/2}$. Therefore, we can safely delete the condition $\sqrt{\rho_2 s} (t^\star \log n)^{3/2} \ll 1$ in \labelcref{cond:basic}. 

To this end, we can summarize \labelcref{cond:Phi,cond:signal-t=1,cond:Phi-lambda_0,cond:xi_1} into one condition as follows:
\begin{equation}
    \frac{Q_1}{\lambda_0} \cdot d^{-\varsigma} = \frac{\Phi(|\barb|)}{\rho_1 d} \gg \max\Bigl\{ d^{\epsilon - \varsigma} \sqrt{s}\log n \cdot \xi_1, \: \frac{L s \log(n)^3}{d}\Bigr\},
    \label{eq:cond-sim-2}
\end{equation}
and \labelcref{cond:basic,cond:lambda_0} can be summarized into 
\begin{equation}
    \lambda_0 = O(\polylog(n)), \quad \kappa_0 = O((\log n)^{-1/2}),\quad \eta^{-1} \ll N \cdot \Bigl(\frac{\rho_1 d}{\lambda_0} \land \Phi(|\barb|)\Bigr). 
\end{equation}
Note that in the last condition, we have $\rho_1 d/\lambda_0 \gg \Phi(|\barb|)$ according to the first equality in \eqref{sec:proof-final-stage-simplified}. Hence, we only need to keep $\eta^{-1} \ll N\Phi(|\barb|)$. 
This completes the proof of \Cref{lem:final-stage-simplified}.

\subsubsection{Proof of \Cref{prop:2nd-moment-calc}}
\label{sec:proof-prop-2nd-moment-calc}
    To prove this lemma, we need to upper bound the expectation term on the left-hand side of \eqref{eq:2nd-moment-calc}.
    Recall that $\hat\EE_{l, l'}$ is given by uniformly samples $l, l'$ from $[N]$, and that $\langle h_l, h_{l'}\rangle \le 1$ always holds.
    We can upper bound the expectation term as follows:
        \begin{align}
            &\hat\EE_{l, l'}\biggl[\Phi\Bigl( |\barb_t|\sqrt{\frac{1-\langle h_l, h_{l'}\rangle}{1+\langle h_l, h_{l'}\rangle}}\Bigr) \langle h_l, h_{l'}\rangle\biggr]  \le \frac{1}{N^2} \sum_{j=1}^n \sum_{l, l'\in\cD_j} \Phi\Bigl(|\barb| \sqrt{\frac{1 - \langle h_l, h_{l'}\rangle}{1 + \langle h_l, h_{l'}\rangle}} \Bigr) \\
            &\quad = \frac{1}{N^2} \sum_{j=1}^n \sum_{l, l'\in\cD_j} \Phi\Bigl(|\barb| \sqrt{\frac{1 - \langle h_l, h_{l'}\rangle}{1 + \langle h_l, h_{l'}\rangle}} \Bigr) \cdot \ind(\norm{h_l\circ h_{l'}}_\infty = 1) \\
            &\hspace{2cm} + \frac{1}{N^2} \sum_{j=1}^n \sum_{l, l'\in\cD_j} \Phi\Bigl(|\barb| \sqrt{\frac{1 - \langle h_l, h_{l'}\rangle}{1 + \langle h_l, h_{l'}\rangle}} \Bigr) \cdot \ind(\norm{h_l\circ h_{l'}}_\infty \ge 2)\\
            &\quad \le \frac{1}{N^2} \sum_{j=1}^n \sum_{l, l'\in\cD_j} \Phi\Bigl(|\barb| \sqrt{\frac{1 - H_{l, j} H_{l', j}}{1 + H_{l, j} H_{l', j}}} \Bigr) + \frac{1}{N^2} \sum_{j=1}^n \sum_{l, l'\in\cD_j} \ind(\norm{h_l\circ h_{l'}}_\infty \ge 2), 
        \end{align}
        where in the identity, we split the summation according to whether how many non-zero entries are shared between two rows $h_l$ and $h_{l'}$ in the $H$ matrix. In the last inequality, we drop the indicator for the case $\norm{h_l\circ h_{l'}}_\infty = 1$ and use the fact that $\Phi(\cdot)\le 1$ for the case $\norm{h_l\circ h_{l'}}_\infty \ge 2$.
        For the first term, we use the fact that $|\cD_j|/N \le \rho_1$ for all $j\in[n]$ to obtain
        \begin{align}
            \frac{1}{N^2} \sum_{j=1}^n \sum_{l, l'\in\cD_j} \Phi\Bigl(|\barb| \sqrt{\frac{1 - H_{l, j} H_{l', j}}{1 + H_{l, j} H_{l', j}}} \Bigr) 
            &\le \rho_1^2 \cdot \sum_{j=1}^{n}\frac{1}{|\cD_j|^2}\sum_{l, l'\in\cD_j} \Phi\Bigl(|\barb| \sqrt{\frac{1 - H_{l, j}H_{l', j}}{1 + H_{l, j} H_{l', j}}}\Bigr) \\
            &\le n \rho_1^2 \cdot \Phi\Bigl(|\barb_t| \sqrt{\frac{1-h_\star^2}{1+h_\star^2}} \Bigr), 
        \end{align}
        where the last inequality holds by the definition of $h_\star$ in \eqref{eq:s_star-def-restate}.
        In addition, the second term is upper bounded by 
        \begin{align}
            \frac{1}{N^2} \sum_{j=1}^n \sum_{l, l'\in\cD_j} \ind(\norm{h_l\circ h_{l'}}_\infty \ge 2) 
            &\le \frac{1}{N^2}\sum_{l=1}^N \sum_{j\in [n]: \atop H_{l, j}\neq 0} \sum_{i\neq j: \atop H_{l, i}\neq 0} \sum_{l'=1}^N \ind(H_{l', i}\neq 0) \cdot \ind(H_{l', j} \neq 0) \\
            &\le \max_{l\in[N]} \frac{1}{N} \sum_{i, j\in[n]: i\neq j \atop H_{l, i}\neq 0, H_{l, j}\neq 0} \sum_{l'=1}^N \ind(H_{l', i}\neq 0) \cdot \ind(H_{l', j} \neq 0) \\
            &\le \max_{l\in[N]} \frac{1}{N} \sum_{i, j\in[n]: i\neq j \atop H_{l, i}\neq 0, H_{l, j}\neq 0} N\cdot \rho_1 \cdot \rho_2  
            \le s^2 \rho_1 \rho_2. 
        \end{align}
        In the first inequality, we notice that if $\norm{h_l\circ h_{l'}}_\infty \ge 2$, then there must exist two different feature indices $i\neq j$ such that both $h_l, h_{l'}$ are non-zero at these two indices. This is indeed reflected in the constraints $H_{l, j}\neq 0, H_{l, i}\neq 0$ and the two indicators $\ind(H_{l', i}\neq 0) \cdot \ind(H_{l', j} \neq 0)$. Therefore, summing over all posible $(i, j)$ pairs gives an upper bound for the second term. In the second inequality, we change the average over $l$ to be the maximum over $l$, and in the third inequality, we use the definition of $\rho_2$ and $\rho_1$ in \eqref{eq:rho-def} to upper bound sum of the double indicator term. 
        The last inequality holds by noting that each row $h_l$ is $s$-sparse.
        Combining the above two bounds, we obtain that
        \begin{align}
            \hat\EE_{l, l'}\biggl[\Phi\Bigl( |\barb_t|\sqrt{\frac{1-\langle h_l, h_{l'}\rangle}{1+\langle h_l, h_{l'}\rangle}}\Bigr) \langle h_l, h_{l'}\rangle\biggr] 
            &\le n\rho_1^2 \cdot \Phi\Bigl(|\barb_t| \sqrt{\frac{1-h_\star^2}{1+h_\star^2}} \Bigr) + \rho_1 \rho_2 s^2 \\
            &\le  C n\rho_1^2  \cdot \Phi(|\barb_t| )^{\frac{1-h_\star^2}{1+h_\star^2}} + \rho_1 \rho_2 s^2,  
        \end{align}
        where in the last inequality, we use the Mills ratio 
        \begin{align}
            \Phi\Bigl(|\barb_t| \sqrt{\frac{1-h_\star^2}{1+h_\star^2}} \Bigr) &\le \Bigl(|\barb_t| \sqrt{\frac{1-h_\star^2}{1+h_\star^2}} \Bigr)^{-1} \cdot \frac{1}{\sqrt{2\pi }} \exp\Bigl(-\frac{1-h_\star^2}{1+h_\star^2} \cdot \frac{\barb_t^2}{2}\Bigr)\\
            &\le C |\barb_t|^{-1} \cdot \frac{1}{\sqrt{2\pi }} \exp\Bigl(-\frac{\barb_t^2}{2}\Bigr)^{\frac{1-h_\star^2}{1+h_\star^2} } \le C \Phi(|\barb_t|)^{\frac{1-h_\star^2}{1+h_\star^2}}, 
        \end{align}
         and the above inequalities hold as long as $(1-h_\star^2)/(1+h_\star^2)$ is on a constant level.
         Therefore, we have proved \Cref{prop:2nd-moment-calc}.

\subsubsection{Proof of \Cref{prop:K1-bound}}
\label{sec:proof-K1-bound}
        Recall the definition $\cK_t$ in \eqref{eq:K_t-def} as 
        \begin{align}
            \cK_1 &\defeq
        \left(n\,|\barb|\,\Phi\!\biggl(\frac{-\barb}{\sqrt{\frac{3}{4}\,\hslash_{4,\star}^2+\frac{1}{4}}}\biggr) \right)^{1/4} 
        + \left(\rho_2 s n |\barb| \Phi\biggl(\frac{-\barb}{\sqrt{\frac{2}{3} \hslash_{3, \star}^2 + \frac{1}{3}}}\biggr) \right)^{1/4}  
        \\
        &\qquad + \biggl( \Phi\Bigl(-\frac{\barb + \hslash_{4, 1} \zeta_1}{\sqrt{1-\hslash_{4, 1}^2}}\Bigr) + \bigl(\rho_2 s\bigr)^{1/4} \biggr) \cdot \bigl(\log(n)\bigr)^{1/4} + n^{1/4} \rho_2\,s\,\log(n).
        \end{align}
        To upper bound the above terms, let us consider the following inequality for any $\tau\in(0, 1)$ and $|\barb|\ge 1$:
        \begin{align}
            \Phi(\tau |\barb|) 
            &\le \frac{1}{\sqrt{2\pi}} \cdot \exp\Bigl(-\frac{\tau^2|\barb|^2}{2}\Bigr) \cdot \frac{1}{\tau|\barb|} \le \frac{1}{\sqrt{2\pi} } \cdot \exp\Bigl(-\frac{\tau^2|\barb|^2}{2}\Bigr) \cdot \frac{|\barb|}{|\barb|^2 + 1} \cdot 2\tau^{-1} \\
            &\le \frac{1}{(\sqrt{2\pi})^{\tau^2}} \cdot \exp\Bigl(-\frac{\tau^2|\barb|^2}{2}\Bigr) \cdot \Bigl(\frac{|\barb|}{|\barb|^2 + 1}\Bigr)^{\tau^2} \cdot 2\tau^{-1} \le \frac{2}{\tau} \cdot \Phi(|\barb|)^{\tau^2}, 
            \label{eq:Phi-tau-b}
        \end{align}
        where in the first and the last inequalities, we use the Mills' ratio bound that $x/(x^2+1) \le \Phi(x)/p(x) \le x^{-1}$ for all $x>0$. 
        Now, we can apply \eqref{eq:Phi-tau-b} to upper bound the first term in $\cK_t$ as
        \begin{align}
            \left(n\,|\barb|\,\Phi\biggl(\frac{-\barb}{\sqrt{\frac{3}{4}\,\hslash_{4,\star}^2+\frac{1}{4}}}\biggr) \right)^{1/4} \le C \bigl( n\,|\barb| \bigr)^{1/4}\,\Phi (|\barb|)^{\frac{1}{3\hslash_{4, \star}^2 + 1}} \le C \bigl( n\,|\barb| \bigr)^{1/4}\,\Phi\!(|\barb|)^{\frac{1}{3h_\star^2 + 1}},
            \label{eq:K_1-1st-term}
        \end{align}
        where the last inequality holds because $h_\star \le \hslash_{4, \star}$ by definition.
        Similarly, we can upper bound the second term as 
        \begin{align}
            \left(\rho_2 s n |\barb| \Phi\biggl(\frac{-\barb}{\sqrt{\frac{2}{3} \hslash_{3, \star}^2 + \frac{1}{3}}}\biggr) \right)^{1/4} 
            \le C  (\rho_2 s n |\barb|)^{1/4} \cdot \Phi(|\barb|)^{\frac{3}{8\hslash_{3, \star}^2 + 4}}  \le C (\rho_2 s n |\barb|)^{1/4} \cdot \Phi(|\barb|)^{\frac{3}{8h_\star^2 + 4}}.
            \label{eq:K_1-2nd-term}
        \end{align}
        Here, the third term also follows from the above inequality as 
        \begin{align}
            \Phi\Bigl(-\frac{\barb + \hslash_{4, 1} \zeta_1}{\sqrt{1-\hslash_{4, 1}^2}}\Bigr) \le \Phi\Bigl(-\frac{\barb + h_\star \zeta_1}{\sqrt{1-h_\star ^2}}\Bigr)  \le C \cdot \Phi(|\barb|)^{\frac{(1 - h_\star \zeta_1/|\barb|)^2}{1 - h_\star^2}}, 
            \label{eq:K_1-3rd-term}
        \end{align}
        where in the first inequality, we use the derivative in \eqref{eq:derivative-hslash-t} and the fact that $\zeta_1/|\barb| = \Theta(1)>1$, which is given by the definition $\zeta_1 = (1+\varepsilon) 2 \sqrt{\log n}$ in \eqref{eq:def-zeta}, to conclude that increasing $\hslash_{4, 1}$ to $h_\star$ will only increase the value of the whole term. 
        In the second inequality, we apply \eqref{eq:Phi-tau-b} with 
        \begin{align}
            \tau = \frac{1 - h_\star \zeta_1/|\barb|}{\sqrt{1-h_\star^2}} \in (0, 1).
        \end{align}
        Here, we claim $\tau\in (0, 1)$ because by condition $\zeta_1 h_\star < 1 - \nu$ for some constant $\nu>0$, we have $1 - h_\star \zeta_1/|\barb| > 0$, and by noting that $\zeta_1 / |\barb| > 1$, we have 
        \begin{align}
            \tau < \frac{1 - h_\star}{\sqrt{1-h_\star}} = \sqrt{1-h_\star} \le 1.
        \end{align}
        In addition, since $|\barb| \le \sqrt{2\log n}$ by condition $\Phi(|\barb|) \ge \rho_1 \ge n^{-1}$, we also have $\zeta_1/|\barb| > 1$. Consequently, we obtain that 
        \begin{align}
            \frac{h_\star^{-1}- \zeta_1/|\barb|}{\sqrt{h_\star^{-2} -1}} \le \frac{h_\star^{-1} - 1}{\sqrt{h_\star^{-2} -1}} \le 1
        \end{align} 
        as $h_\star < 1$. Therefore, we can apply \eqref{eq:Phi-tau-b} with $\tau = (h_\star^{-1} - \zeta_1/|\barb|)/\sqrt{h_\star^{-2} - 1} \in (0, 1)$ in the last inequality in \eqref{eq:K_1-3rd-term}. 
        Now, we can combine \eqref{eq:K_1-1st-term}, \eqref{eq:K_1-2nd-term} and \eqref{eq:K_1-3rd-term} to obtain the desired result in \Cref{prop:K1-bound}.


\subsection{Proofs for Technical Lemmas}
\subsubsection{Proof of \Cref{lem:psi-3-4}}
\label{app:proof-varphi-3-4}
    By Cauchy-Schwartz, it holds that 
    $$
        \sum_{\tau=1}^{t-1} \langle P_{u_{1:\tau} }z_\tau, u_t\rangle^2 \le \sum_{\tau=1}^{t-1} \norm{P_{u_{1:\tau}} z_\tau}_2^2 \cdot \norm{u_t}_2^2. 
    $$
    One thing to be noted is that $z_\tau$ is independent of the filtration $\sigma(u_{1:\tau})$. Consiquently, when conditioned on $u_{1:\tau}$, $\norm{P_{u_{1:\tau}} z_\tau}_2^2 \sim \chi^2_{\tau}$. By the concentration of $\chi^2$ distribution in \Cref{lem:chi-squared} with a union bound over all $\tau\in [T]$, we obtain that with probability at least $1 - n^{-c}$ for some universal constant $c, C>0$ that
    \begin{align}
        \norm{P_{u_{1:\tau}} z_\tau}_2^2 \le \tau + C\sqrt{\tau \log (nT)} + C\log (nT) \le C (t + \log n), \quad \forall \tau \in [t-1], \quad t\in [T].
    \end{align}
    Therefore, we have that with probability at least $1 - n^{-c}$:
    \begin{align}
        \sum_{\tau=1}^{t-1} \langle P_{u_{1:\tau} }z_\tau, u_t\rangle^2 \le C (t^2 + t\log n) \cdot \norm{u_t}_2^2, \quad \forall t\in [T].    
        \label{eq:varphi_3}
    \end{align}
    For the second term, it follows from \Cref{lem:ratio-w-u} that with probability at least $1 - n^{-c}$:
    \begin{align}
        \sum_{\tau=1}^{t-1} \Bigl(\frac{\langle  u_{\tau}^\perp, u_t\rangle}{\norm{u_{\tau}^\perp}_2} \cdot \frac{\norm{w_{\tau}^\perp}_2}{\norm{u_{\tau}^\perp}_2}\Bigr)^2 \le C d \cdot \sum_{\tau=1}^{t-1} \frac{\langle  u_{\tau}^\perp, u_t\rangle^2}{\norm{u_{\tau}^\perp}_2^2} = C d \cdot \norm{P_{u_{1:t-1}}u_t}_2^2, \quad \forall t\in [T]. \label{eq:varphi_4-1}
    \end{align}
    For the last term $\norm{P_{w_{-1: t-1}}^\perp \tilde{z}_{t} }_2^2 \cdot \norm{u_t^\perp}_2^2$, we also note that $\tilde z_t$ is independent of the filtration $\sigma(w_{-1: t-1})$. Therefore, $\norm{P_{w_{-1: t-1}}^\perp \tilde{z}_{t} }_2^2 \sim \chi^2_{d-t+1}$.  We have by \Cref{lem:chi-squared}  with a union bound over all $t\in [T]$, and with probability at least $1 - n^{-c}$ that
    \begin{align}
        \norm{P_{w_{-1: t-1}}^\perp \tilde{z}_{t} }_2^2 \le d-t+1 + C\sqrt{(d-t+1)\log (nT)} + C\log (nT) \le C d, \quad \forall t\in [T].    \label{eq:varphi_4-2}
    \end{align}
    Combining \eqref{eq:varphi_4-1} and \eqref{eq:varphi_4-2}, we have with probability at least $1 - n^{-c}$ and for all $t\in [T]$:
    \begin{align}
        \sum_{\tau=1}^{t-1} \Bigl(\frac{\langle  u_{\tau}^\perp, u_t\rangle}{\norm{u_{\tau}^\perp}_2} \cdot \frac{\norm{w_{\tau}^\perp}_2}{\norm{u_{\tau}^\perp}_2}\Bigr)^2 + \norm{P_{w_{-1: t-1}}^\perp \tilde{z}_{t} }_2^2 \cdot \norm{u_t^\perp}_2^2 
         \le C d \cdot \bigl(\norm{P_{u_{1:t-1}}u_t}_2^2 + \norm{u_t^\perp}_2^2\bigr) = C d \cdot \norm{u_t}_2^2.
        \label{eq:varphi_4}
    \end{align}
    Now we combine \eqref{eq:varphi_3} and \eqref{eq:varphi_4} to obtain the desired result in \Cref{lem:psi-3-4}.

\subsubsection{Proof of \Cref{lem:ut-upper-bound}}
\label{sec:proof-ut-upper-bound}
        Recall that
    \[
    u_t = E^\top \varphi(Ey_t^\star; b_t) + F^\top\varphi(Fy_t^\star + \theta \cdot v^\top \barw_{t-1}; b_t)
    + \Delta E_t + \Delta F_t.
    \]
    By the triangular inequality, we have
    \[
        \norm{u_t}_2 \le 2\sqrt{\norm{E^\top \varphi(Ey_t^\star; b_t)}_2^2 
        + \norm{F^\top\varphi(Fy_t^\star + \theta \cdot v^\top \barw_{t-1}; b_t)}_2^2}
        + 2\sqrt{\norm{\Delta E_t}_2^2 + \norm{\Delta F_t}_2^2}
    \]
    By \Cref{cor:norm-E-F}, we have
    \begin{align}
        \sqrt{\norm{E^\top \varphi(E y_t^\star; b_t)}_2^2  + \norm{F^\top \varphi(F y_t + \theta \cdot v^\top \barw_{t-1}; b_t)}_2^2} \le C L N\rho_1 \xi_t.
    \end{align}
    By \Cref{lem:E-error,lem:F-error}, and the fact that $N_1\le N$ and $N_2 \le N \rho_1$, we derive that 
    \begin{align}
        \sqrt{\norm{\Delta E_t}_2^2 + \norm{\Delta F_t}_2^2} \le CL \rho_1 N \sqrt d \beta_{t-1} + C L \rho_1 \rho_2 \sqrt d \beta_{t-1} \le C L N \rho_1 \sqrt{d}\beta_{t-1}.
    \end{align}
    This completes the proof of \Cref{lem:ut-upper-bound}.

\subsubsection{Proof of \Cref{prop:Kt}}
\label{sec:proof-Kt}
        We recall from the definition of
        $\cK_t$ that 
        \begin{align}
            \cK_t &= 
            \left(n\,|\barb|\,\Phi\!\biggl(\frac{-\barb}{\sqrt{\frac{3}{4}\,\hslash_{4,\star}^2+\frac{1}{4}}}\biggr) \right)^{1/4} 
            + \left(\rho_2 s n |\barb| \Phi\biggl(\frac{-\barb}{\sqrt{\frac{2}{3} \hslash_{3, \star}^2 + \frac{1}{3}}}\biggr) \right)^{1/4}  
            \\
            &\qquad + \left( \Phi\Bigl(-\frac{\barb + \hslash_{4, t} \zeta_t}{\sqrt{1-\hslash_{4, t}^2}}\Bigr) + \bigl(\rho_2 s\bigr)^{1/4} \right) \cdot \bigl(t\log(n)\bigr)^{1/4} + n^{1/4} \rho_2\,s\,t\log(n),
        \end{align}
        The terms that implicitly change with $t$ is $\zeta_t$ and  $\hslash_{4, t}$. 
        Recall that 
        \begin{align} 
            \zeta_t = \zeta_1 + \ind(t \ge 2) \cdot C(\beta_{t-1} + |\alpha_{-1, t-1}| + |\alpha_{-1, 0}|)
        \end{align}
        with $\zeta_1 = \sqrt 2 (1 + \varepsilon) \sqrt{2\log n}$. 
        Moreover, by definition of $\hslash_{4, t}$, we can rewrite the term as 
        \begin{align}
            \Phi\Bigl(-\frac{\barb + \hslash_{4, t} \zeta_t}{\sqrt{1-\hslash_{4, t}^2}}\Bigr) = \max_{j\in[n]} \left(\frac{1}{|\cD_j|} \sum_{l\in\cD_j} \Phi\Bigl(-\frac{\barb + H_{l,j} \zeta_t}{\sqrt{1-H_{l,j}^2}}\Bigr)^4\right)^{1/4}. 
            \label{eq:Kt-exponential-term}
        \end{align}
        To understand how the change in $\zeta_t$ affects the term, we take the derivatives for positive power $q$: 
        \begin{align}
            \frac{\rd }{\rd \zeta} \Phi\Bigl(-\frac{\barb + x \zeta}{\sqrt{1-x^2}}\Bigr)^q \bigggiven_{\zeta = \zeta_t}
                &= q \Phi\Bigl(-\frac{\barb + x \zeta}{\sqrt{1-x^2}}\Bigr)^{q-1} 
                \cdot p\Bigl(-\frac{\barb + x \zeta}{\sqrt{1-x^2}}\Bigr) 
                \cdot \frac{\zeta x}{\sqrt{1-x^2}} > 0.
        \end{align}
        Here, we have the second derivative larger than $0$ since $\zeta_t \ge \zeta_1 = \sqrt 2 (1 + \varepsilon) \sqrt{2\log n} > |\barb|$ by assumption. 
        Now, our goal is to upper bound the derivative with respect to $\zeta$. We discuss in two cases:
        \begin{itemize}
            \item For $x \in \bigl[ (1 + |\barb|/\zeta)/2, 1 \bigr]$, we have $\barb+x\zeta \in [(\barb+\zeta)/2, \barb+\zeta]$, $x \ge (1+1/\sqrt 2)/2$
            and thus 
            \begin{align}
                \frac{\rd }{\rd \zeta} \Phi\Bigl(-\frac{\barb + x \zeta}{\sqrt{1-x^2}}\Bigr)^q 
                &\le q \cdot p\Bigl(-\frac{\barb + x \zeta}{\sqrt{1-x^2}}\Bigr) 
                \cdot \frac{\barb+x\zeta}{\sqrt{1-x^2}} \cdot \frac{\zeta x}{\barb + \zeta x} \\
                &\le q\cdot \sup_{z \ge \sqrt{1 - (1+1/\sqrt 2)^2/4}} p(z) \cdot z \cdot \frac{\zeta}{\barb+\zeta} \le q\cdot  \sup_{z \ge 0} p(z) \cdot z  = O(1). 
            \end{align}
            \item For $x \in [0, (1 + |\barb|/\zeta)/2]$, we have $\sqrt{1 - x^2} \ge  \sqrt{1 - (1+|\barb|/\zeta)^2/4} \ge \sqrt{1 - (1 + 1/\sqrt 2)^2/4} = \Omega(1)$. Thus 
            \begin{align}
                \frac{\rd }{\rd \zeta} \Phi\Bigl(-\frac{\barb + x \zeta}{\sqrt{1-x^2}}\Bigr)^q \le q \cdot \frac{\zeta}{\sqrt{1 - (1 + 1/\sqrt 2)^2/4}} = O(\sqrt{\log n}). 
            \end{align}
        \end{itemize}
        Therefore, we conclude that for $q = 1$, it holds for any $H_{l,j} \in[0, 1]$ that 
        \begin{align}
            \Phi\Bigl(-\frac{\barb + H_{l,j} \zeta_t}{\sqrt{1-H_{l,j}^2}}\Bigr)  - \Phi\Bigl(-\frac{\barb + H_{l,j} \zeta_1}{\sqrt{1-H_{l,j}^2}}\Bigr)  \le C \sqrt{\log n} \cdot (\beta_{t-1} + |\alpha_{-1, t-1}| + |\alpha_{-1, 0}|). 
            \label{eq:Kt-difference}
        \end{align}
        Since $ \norm{x}_4 - \norm{y}_4 \le \norm{x - y}_4 \le m^{1/4} \norm{x - y}_\infty$ for any $x, y \in \RR^m$ by the triangle inequality,
        we conclude that the same upper bound in \eqref{eq:Kt-difference} holds for each $j$ in \eqref{eq:Kt-exponential-term} as well. 
        Therefore, the same upper bound also holds after taking the maximum over $j\in[n]$ in \eqref{eq:Kt-exponential-term}.
        Therefore, we obtain that
        \begin{align}
            \cK_t \le t \cdot \bigl( \cK_1 + C \sqrt{\log n} \cdot (\beta_{t-1} + |\alpha_{-1, t-1}| + |\alpha_{-1, 0}|) \bigr). 
        \end{align}
        This completes the proof of \Cref{prop:Kt}.

%% file: paper/appendix/auxiliary.tex
\section{Auxiliary Lemmas}


\subsection{Concentration Inequalities}
\begin{lemma}[Chi-square concentration, Lemma 1 in \cite{laurent2000adaptive}]\label{lem:chi-squared}
    Let $X_1,\ldots,X_n$ be independent random variables such that $X_i \sim \cN(0, 1)$ for all $i$. Let $a\in \RR_+^n$ be a vector with nonnegative entries. Then the following holds for any $\delta \in (0, 1)$:
    \begin{align}
        \PP\left(\biggl|\sum_{i=1}^n a_i X_i^2 - \norm{a}_1\biggr| \ge 2\sqrt{\norm{a}_2^2 \log \delta^{-1}} + 2\norm{a}_\infty \log \delta^{-1}\right) \le \delta.
    \end{align}
\end{lemma}

\begin{lemma}[Tail probability for the maximum Gaussian random variables]\label{lem:max gaussian_tail}
    Let $X_1,\ldots,X_n$ be $\sigma^2$-subgaussian random variables with mean $0$. Then for any $t > 0$,
    \begin{align}
        \PP\left(\max_{i=1, \ldots, n} X_i \ge \sqrt{2\sigma^2\log n} + t\right) \le \exp\left(-\frac{t^2}{2\sigma^2}\right).
    \end{align}
    In particular, if $X_1, \ldots, X_n$ are independent standard normal random variables, then for any $c>1$,
    \begin{align}
        \PP\left(\max_{i=1, \ldots, n} X_i \ge c\sqrt{2 \log n} \right) \le n^{1 - c^2}.
    \end{align}
\end{lemma}

\begin{lemma}[Bernstein's inequality]
    \label{lem:bernstein}
    Let $X_1, \ldots, X_n$ be independent random variables with $|X_i - \EE[X_i]|\leq C$ for all $i\in[n]$. 
    Then for any $\delta\in(0,1)$,
    \begin{align}
        \PP\bigg(
            \bigg|\frac{1}{n}\sum_{i=1}^n X_i - \EE X_i\bigg| \leq  \sqrt{\frac{2  \cdot n^{-1} \sum_{i=1}^n \Var[X_i] \cdot \log \delta^{-1}}{n}} + \frac{C \log \delta^{-1}}{3n}
        \bigg) \ge 1- \delta.
    \end{align}
\end{lemma}

\begin{lemma}\label{lem:gaussian_max_lower_tail}
    Let \(X_1,\dots,X_n\) be independent standard normal random variables and define
    \(
    M_n = \max_{1\le i\le n} X_i.
    \)
    Then for any fixed \(\epsilon\in(0,1)\) and all sufficiently large \(n\) with \(2(1-\epsilon)^2\log n \ge 1\) that
    \[
    \mathbb{P}\Bigl(M_n\le (1-\epsilon)\sqrt{2\log n}\Bigr)
    \le \exp\Biggl(-\frac{n^{2\epsilon-\epsilon^2}}{3 \sqrt{\pi\log n}}\Biggr).
    \]
    \end{lemma}
    
    \begin{proof}
    Since \(X_1,\ldots,X_n\iidfrom N(0,1)\), it holds for any \(x\in \mathbb{R}\)
    \[
    \mathbb{P}(M_n\le x)=(1 - \Phi(x))^n,
    \]
    where \(\Phi(x)\) is the standard normal tail distribution function. In order to upper bound \((1-\Phi(x))^n\) when \(x=(1-\epsilon)\sqrt{2\log n}\), we use a well-known lower bound for the Gaussian tail. Specifically, for all \(x>0\) (see, e.g., \cite{ledoux2013probability} or \cite{BoucheronLugosiMassart}),
    \[
    \Phi(x) \ge \Delta(x) \coloneqq \frac{x}{1+x^2}\,\frac{1}{\sqrt{2\pi}}\,e^{-x^2/2}.
    \]
    Hence, further applying the fact that \(1-\Delta(x)\le \exp(-\Delta(x))\), we get
    \[
    \mathbb{P}(M_n\le x) \le \bigl(1-\Delta(x)\bigr)^n \le \exp\bigl(-n\Delta(x)\bigr).
    \]
    Now, for
    \(
    x=(1-\epsilon)\sqrt{2\log n},
    \)
    we have
    \[
    \frac{x}{1+x^2} = \frac{(1-\epsilon)\sqrt{2\log n}}{1+2(1-\epsilon)^2\log n} \ge \frac{\sqrt{2}}{3(1-\epsilon)\sqrt{\log n}},
    \]
    where the inequality holds for sufficiently large \(n\) such that \(2(1-\epsilon)^2\log n \ge 1\).
    Thus,
    \[
    \Delta(x) \ge \frac{1}{3(1-\epsilon)\sqrt{\pi\log n}}\,n^{-(1-\epsilon)^2}.
    \]
    Substituting this lower bound into our earlier inequality gives
    \begin{align}
    \mathbb{P}\Bigl(M_n\le (1-\epsilon)\sqrt{2\log n}\Bigr)
    &\le \exp\Biggl(-\frac{1}{3(1-\epsilon)\sqrt{\pi\log n}}\,n^{1-(1-\epsilon)^2}\Biggr)\\
    &= \exp\Biggl(-\frac{n^{2\epsilon-\epsilon^2}}{3 \sqrt{\pi\log n}}\Biggr).
    \end{align}
    This completes the proof.
    \end{proof}

\subsection{Efron-Stein Inequalities}
Let $Z$ be a function of independent random variables $X_1, \dots, X_n$ with domain $\cX$:
\begin{equation}
Z = f(X_1, \dots, X_n),
\label{eq:efron-stein-f}
\end{equation}
where $f: \mathcal{X}^n \to \mathbb{R}$ is a measurable function.
Let $X_1', \dots, X_n'$ be independent copies of $X_1, \dots, X_n$. Define the modified versions of $Z$ where one coordinate is replaced by its independent copy:
\begin{equation}
    Z^{(i)} = f(X_1, \dots, X_{i-1}, X_i', X_{i+1}, \dots, X_n).
    \label{eq:modified-z}
\end{equation}
Define the deviation terms:
\begin{align}
    V_+ &= \mathbb{E} \left[ \sum_{i=1}^{n} (Z - Z^{(i)})^2 \ind\{Z > Z^{(i)}\} \;\Bigg|\; X_1, \dots, X_n \right], \\
    V_- &= \mathbb{E} \left[ \sum_{i=1}^{n} (Z - Z^{(i)})^2 \ind\{Z < Z^{(i)}\} \;\Bigg|\; X_1, \dots, X_n \right].
    \label{eq:variance-split}
\end{align}
The following lemma is borrowed from Theorem 5 in  \citet{boucheron2003concentration} for the case where $V_+$ is dominated by some linear transformation of $Z$.
\begin{lemma}[Efron-Stein for dominated variance]\label{lem:efron-stein-dominated-variance}
    For $Z$ and $V_+$ defined in \eqref{eq:efron-stein-f} and \eqref{eq:variance-split}, respectively, suppose that there exist positive constants $a$ and $b$ such that $V_+ \leq aZ + b$.
    Then there is a universal constant $C>0$ such that for any $\delta \in (0, 1)$, with probability at least $1-\delta$,
    \begin{align}
        Z \le \EE[Z] + C\cdot\sqrt{(a \cdot \EE[Z] + b) \log(1/\delta)} + C\cdot a \log(1/\delta).
    \end{align}
\end{lemma}
The following lemma is borrowed from Theorem 2 in \citet{boucheron2003concentration}. 
\begin{lemma}[Efron-Stein for the moment generating function]\label{lem:efron-stein-mgf}
    For all $\theta > 0$ and $\lambda \in (0,1/\theta)$,
    \begin{align}
        \log \mathbb{E} \left[ \exp \left( \lambda (Z - \mathbb{E}[Z]) \right) \right] 
        \leq \frac{\lambda \theta}{1 - \lambda \theta} 
        \log \mathbb{E} \left[ \exp \left( \frac{\lambda V_+}{\theta} \right) \right].
    \end{align}
    On the other hand, for all $\theta > 0$ and $\lambda \in (0,1/\theta)$,
    \begin{align}
        \log \mathbb{E} \left[ \exp \left( -\lambda (Z - \mathbb{E}[Z]) \right) \right] 
        \leq \frac{\lambda \theta}{1 - \lambda \theta} 
        \log \mathbb{E} \left[ \exp \left( \frac{\lambda V_-}{\theta} \right) \right].
    \end{align}
\end{lemma}
The following lemma is borrowed from Lemma 11 in \citet{boucheron2003concentration} for transforming the upper bound on moment generating function (MGF) bound into an exponential tail bound.
\begin{lemma}
    \label{lem:mgf-tail}
    Suppose for any $\lambda \in (0, 1/a)$, there exists a constant $V>0$ such that:
    \begin{align}
        \log \EE[\exp(\lambda (Z - \EE[Z]))] \le \frac{\lambda^2 V}{1 - \lambda a}.
    \end{align}
    Then there exists some universal constant $C$ such that for any $\delta \in (0, 1)$, with probability at least $1-\delta$:
    \begin{align}
        Z - \EE[Z] \le C \cdot \sqrt{V \log(1/\delta)} + C \cdot a \log(1/\delta). 
    \end{align}
\end{lemma}

With the above lemmas, we can derive the following Efron-Stein inequality for sub-exponential variance.
\begin{lemma}[Efron-Stein inequality for sub-exponential variance]
    \label{lem:efron-stein-subexp-variance}
    Suppose either of the following two conditions is satisfied: 
    \begin{enumerate}[
        ref = Condition \arabic* of \Cref{lem:efron-stein-subexp-variance}
    ]
        \item The variance $V_+$ for $Z$ satisfies that $\EE[\exp(\lambda V_+)] \le \EE[\exp(\lambda V_+')] $ for any $\lambda>0$, where $V_+'$
    is $a$-subexponential with $a\in (0, 1)$:
    \begin{align}
        Q(v)\defeq \PP(V_+' > V + v) \le \exp(-v/a). 
    \end{align}
    when $V_+$ exceeds some threshold $V>0$. \label{lem:efron-stein-subexp-variance-1}
        \item The moment generating function of $V_+$ satisfies $$\log \EE[\exp(\lambda V_+)]\le \lambda V + \frac{\lambda a}{1-a\lambda} $$
        for some $V>0, 0<a<1$ and any $0<\lambda < a^{-1/2}$.
    \end{enumerate}
    Then, with probability at least $1 - \delta$, it holds that
    \begin{align}
        Z - \EE[Z] \le C \cdot \sqrt{V \log(\delta^{-1})} + C \cdot \sqrt a \log(\delta^{-1}).
    \end{align}
    Similarly, if $V_-$ satisfies either of the two conditions, then with probability at least $1 - \delta$, it holds that
    \begin{align}
        \EE[Z] - Z \le C \cdot \sqrt{V \log(\delta^{-1})} + C \cdot \sqrt a\log(\delta^{-1}).
    \end{align}
\end{lemma}
\begin{proof}
    We just prove the first condition and the second condition can be implied by the proof. 
    We explicit calculate the MGF for $V_+$. 
    The case for $V_-$ can be handled similarly.
    Take parameter $\lambda \in (0, a^{-1/2})$, we have for the moment generating function of $V_+$ that
    \begin{align}
        \EE[\exp(\lambda V_+)] 
            &= \exp(\lambda V)\cdot \biggl(\lim_{v\rightarrow 0_-} Q(v)  + \lambda \cdot \int_{0_+}^\infty \exp\bigl(\lambda\cdot v \bigr) \cdot Q(v) \rd v \biggr) \\
            &\le \exp(\lambda V)\cdot \biggl(1  +  \lambda \cdot \int_{0_+}^\infty \exp\bigl(-(a^{-1} - \lambda)\cdot v \bigr) \rd v \biggr) \\
            &= \exp(\lambda V)\cdot \biggl(1  + \frac{\lambda}{a^{-1} - \lambda} \biggr) 
            = \exp(\lambda V)\cdot \biggl(1  + \frac{\lambda \cdot a}{1 - a \lambda} \biggr). 
    \end{align}
    where we use $0_+$ and $0_-$ to denote the limit from the right and left side of $0$, respectively.
    Here, in the first line, we use integration by parts to obtain an integration term with respect to the tail probability $Q(v)$.
    In the final line, we have the denominator $1 - a \lambda > 0$ since $\lambda < a^{-1/2} < a^{-1}$ for $a \in (0, 1)$. 
    Taking the logarithm on both side, we obtain that 
    \begin{align}
        \log \EE[\exp(\lambda V_+)] \le \lambda V + \log(1 + \lambda a/(1 - a\lambda)) \le \lambda V + \frac{\lambda a}{1-a\lambda}. 
    \end{align}
    Now, we apply \Cref{lem:efron-stein-mgf} with $\lambda$ replaced by $\lambda /\theta$ for some $\theta \in (a\lambda, \lambda^{-1})$ to obtain that
    \begin{align}
        \log \EE[\exp(\lambda (Z - \EE[Z]))]& \le \frac{\lambda\theta}{1- \lambda\theta} \log \EE\Bigl[\exp\Bigl(\frac{\lambda V_+}{\theta}\Bigr)\Bigr] \le \frac{\lambda^2}{1 - \lambda \theta} \cdot \Bigl( V + \frac{a}{1 - a\lambda/\theta} \Bigr) \\
        &\le \cdot  \frac{\lambda^2(V + a)}{(1 - \lambda \theta)(1 - a\lambda/\theta)}. 
    \end{align}
    Note that such a $\theta$ exists since $\lambda < a^{-1/2}$. In particular, we have by the constraint on $\lambda$ that $a\lambda < \sqrt{a} < \lambda^{-1}$. 
    Let us just pick $\theta = \sqrt a$ and further restrict ourselves to $\lambda < a^{-1/2}/2$ to obtain that
    \begin{align}
        \log \EE[\exp(\lambda (Z - \EE[Z]))] \le \frac{\lambda^2(V+a)}{(1 - \lambda \sqrt a)^2} \le \frac{\lambda^2(V+a)}{(1 - 2\lambda \sqrt a)}. 
    \end{align}
    Now, we invoke \Cref{lem:mgf-tail} and conclude that there exists universal constant $C>0$ such that 
    \begin{align}
        Z - \EE[Z] \le C \cdot \sqrt{(V+a) \cdot \log(\delta^{-1})} + C \cdot \sqrt{a} \cdot \log(\delta^{-1}) \le 2C\cdot \bigl(\sqrt{V \log (\delta^{-1})} + \sqrt{a} \cdot \log(\delta^{-1})\bigr).
    \end{align}
    A similar bound holds for the lower tail with the condition on $V_-$. Hence, we complete the proof. 
\end{proof}

\begin{lemma}[Efron-Stein inequality for bounded variance]\label{lem:efron-stein-bounded-variance}
    Suppose that $\max\{V_+, V_-\} \le V_0$ with probability at least $1-\exp(-a)$ for some $a > n^{c_1}$ and $V_0>n^{-c_2}$ for some universal constant $c_1, c_2 > 0$. Also assume that $\max\{V_+, V_-\}$ is uniformly bounded by $V_1$ with $V_1\le n^{c_3}$ for some universal constant $c_3>0$.
    Then, with probability at least $1-\delta$, it holds that
    \begin{align}
        |Z - \EE[Z]| \le C \cdot \bigl(\sqrt{V_0 \log(\delta^{-1})} + \sqrt{a^{-1}V_1}\log(\delta^{-1})\bigr).
    \end{align}
\end{lemma}
\begin{proof}
    By \Cref{lem:efron-stein-mgf}, we have for the moment generating function (MGF) of $V_+$ that
    \begin{align}
        \log \EE[\exp(\lambda (Z - \EE[Z]))] 
            &\le \frac{\lambda \theta}{1 - \lambda \theta} \cdot \log \EE\Bigl[\exp\Bigl(\frac{\lambda V_+}{\theta}\Bigr)\Bigr] \\
            &\le \frac{\lambda \theta}{1 - \lambda \theta} \cdot \log\Bigl( \exp\Bigl(\frac{\lambda V_0}{\theta}\Bigr) + \exp\Bigl(\frac{\lambda V_1}{\theta}\Bigr) \cdot \PP(V_+ \ge V_0) \Bigr) \\
            &\le \frac{\lambda }{1-\lambda \theta} \cdot \biggl(\lambda V_0 + \theta\exp\Bigl(\frac{\lambda(V_1 - V_0)}{\theta} - a\Bigr)\biggr).
    \end{align}
    In the following, we take $\theta = 2\lambda(V_1-V_0)/a$, and the above upper bound can be simplified as 
    \begin{align}
        \log \EE[\exp(\lambda (Z - \EE[Z]))] 
        &\le \frac{\lambda^2}{1 - 2\lambda^2(V_1-V_0)/a} \cdot \biggl(V_0 + \frac{\exp(-a/2)}{2\lambda^2(V_1-V_0)/a} \biggr) \\ 
        & \le \frac{\lambda^2}{1 - \lambda \sqrt{2(V_1-V_0)/a}} \cdot \biggl(V_0 + \frac{\exp(-a/2)}{2\lambda^2(V_1-V_0)/a} \biggr).
    \end{align}
    Similarly for $V_-$, we also have
    \begin{align}
        \log \EE[\exp(-\lambda (Z - \EE[Z]))] 
        &\le \frac{\lambda^2}{1 - \lambda \sqrt{2(V_1-V_0)/a}} \cdot \biggl(V_0 + \frac{\exp(-a/2)}{2\lambda^2(V_1-V_0)/a} \biggr).
    \end{align}
    Therefore, in the following, we only need to consider the upper tail and the lower tail can be directly implied.
    As long as $1/(2\lambda^2 (V_1-V_0)/a)$ is polynomially in $n$, we will have $\exp(-a/2)/(2\lambda^2(V_1-V_0)/a)\le V_0$. 
    Take $t$ to be the deviation of $Z$ from its mean, i.e., $t = |Z - \EE[Z]|$, we have
    By Lemma 11 of \citet{boucheron2003concentration}, we conclude by using the Chernoff bound that 
    \begin{align}
        \log\PP(Z-\EE[Z]\ge t) 
        &\le \inf_{\lambda \in (0, \sqrt{a/2(V_1-V_0)})} \biggl\{\frac{\lambda^2 \cdot 2V_0}{1 - \lambda \sqrt{2(V_1-V_0)/a}} - t\lambda\biggr\} \\ 
        &\le - \frac{t^2}{2(4V_0 + t\sqrt{2a^{-1}(V_1-V_0)} / 3)}, 
    \end{align}
    where the last inequality holds as long as $t$ satisfies
    \begin{align}
        1 - \Bigl( 1 + \frac{t \sqrt{2a^{-1}(V_1-V_0)}}{2V_0}\Bigr)^{-1/2} \ge \frac{\exp(-a/4)}{V_0}.
        \label{eq:condition-t}
    \end{align}
    The lower bound holds similarly.
    A sufficient condition for \eqref{eq:condition-t} to hold is 
    \begin{align}
        t \ge \frac{8\exp(-a/4)}{\sqrt{2a^{-1}(V_1-V_0)}}.
    \end{align}
    This condition will be automatically satisfied if we pick $t = C \cdot (\sqrt{V_0 \log(\delta^{-1})} + \sqrt{a^{-1}(V_1-V_0)} \allowbreak\log(\delta^{-1} ) $.
    Therefore, we conclude that with probability at least $1-\delta$, it holds that
    \begin{align}
        |Z - \EE[Z]| \le C \cdot \bigl(\sqrt{V_0 \log(\delta^{-1})} + \sqrt{a^{-1}V_1}\log(\delta^{-1})\bigr).
    \end{align}
    This completes the proof.
\end{proof}

\begin{lemma}\label{lem:order-inequality}
    Let \( w = (w_1, w_2, \dots, w_d) \) be a random vector, and let \( w^{(i)} \) denote the vector where the \( i \)-th coordinate \( w_i \) is replaced by an independent copy \( w_i' \), while all other coordinates remain unchanged. 
    Suppose that \( f: \mathbb{R}^d \to \mathbb{R} \) and \( g: \mathbb{R}^d \to \mathbb{R} \) are both nondecreasing/nonincreasing functions with respect to the coordinate \( w_i \). Then, we have the inequality:
    \begin{align}
        \mathbb{E} \bigl[ f(w) g(w) \bigr] \geq \mathbb{E} \bigl[ f(w) g(w^{(i)}) \bigr].
    \end{align}
\end{lemma}
\begin{proof}[Proof of \Cref{lem:order-inequality}]
    By the monotonicity of \( f \) and \( g \), we have:
    \begin{equation}
        (f(w) - f(w^{(i)})) \cdot (g(w) - g(w^{(i)})) \ge 0.
    \end{equation}
    Expanding the product and taking expectations, we obtain:
    \begin{align}
        \mathbb{E} \bigl[ f(w) g(w) \bigr] - \mathbb{E} \bigl[ f(w) g(w^{(i)}) \bigr] - \mathbb{E} \bigl[ f(w^{(i)}) g(w) \bigr] + \mathbb{E} \bigl[ f(w^{(i)}) g(w^{(i)}) \bigr] \ge 0.
    \end{align}
    By the symmetry of expectations, the first and last terms are equal, and the second and third terms are also equal, so we obtain the desired inequality:
    \begin{equation}
        \mathbb{E} \bigl[ f(w) g(w) \bigr] \geq \mathbb{E} \bigl[ f(w) g(w^{(i)}) \bigr].
    \end{equation}
    This completes the proof.
\end{proof}